\documentclass[fleqn]{article}

\pdfoutput=1

\usepackage{arxiv}

\usepackage{amssymb}

\usepackage{hyperref}       

\RequirePackage{booktabs}

\def\thickhline{\noalign{\hrule height.9pt}}
\def\medhline{\noalign{\hrule height.7pt}}

\usepackage{array}
\newcolumntype{?}{!{\vrule width 1pt}}

\usepackage{graphicx}%
\usepackage{multirow}%
\usepackage{amsmath,amssymb,amsfonts}%
\usepackage{amsthm}%
\usepackage{mathrsfs}%
\usepackage[title]{appendix}%
\usepackage{xcolor}%
\usepackage{textcomp}%
\usepackage{manyfoot}%
\usepackage{booktabs}%
\usepackage{algorithm}%
\usepackage{algorithmicx}%
\usepackage{algpseudocode}%
\usepackage{listings}%

\usepackage{subfig}  

\theoremstyle{thmstyleone}%
\newtheorem{theorem}{Theorem}
\theoremstyle{thmstyletwo}%

\theoremstyle{thmstylethree}%

\begin{document}


\renewcommand{\headeright}{}
\renewcommand{\undertitle}{}


\title{\hspace*{-.11in}Tracking Changing Probabilities via Dynamic Learners}



\date{} 					

\author{Omid Madani\thanks{Much of this work was conducted while at Cisco Secure Workload.}\\
  omidmadani@yahoo.com
}


\newcommand{\ie}{{\em i.e.~}}





\newcommand{\sect}{Sect. }
\newcommand{\sects}{Sections }
\newcommand{\secs}{Sections }
\renewcommand{\sec}{Sect. } 
\newcommand{\fig}{Fig. }
\newcommand{\figs}{Figs. }

\renewcommand{\lg}{\mbox{log}} %
\renewcommand{\log}{\ln} 

\newcommand{\prior}{\mbox{prior}}

\newcommand{\tp}{{p^*}} 



\newcommand{\tpo}[1]{\bar{P}(#1)}

\newcommand{\ep}{\hat{p}}

\newcommand{\epo}[1]{\hat{P}(#1)}


\newcommand{\dev}{dev} 
\newcommand{\dvc}{deviates} 

\newcommand{\mdvc}{multidev} 

\newcommand{\dvt}{d}

\newcommand{\oat}[1]{#1^{(t)}}

\newcommand{\oaj}[1]{#1^{(j)}} 

\newcommand{\ott}[2]{#1^{(#2)}}

\newcommand{\fcapbf}{{\bf FC}}
\newcommand{\fcap}{FC}
\newcommand{\fc}{\fcap} 

\newcommand{\scdbf}{{\bf ScaleDrop}}
\newcommand{\scd}{\mbox{ScaleDrop}}

\newcommand{\qcap}{qcap }
\newcommand{\qcapn}{qcap}


\newcommand{\llsnbf}{\mbox{\bf LogLossNS}}
\newcommand{\llnsbf}{\mbox{\bf LogLossNS}}

\newcommand{\llns}{\mbox{LogLossNS}}
\newcommand{\llsn}{\mbox{LogLossNS}}

\newcommand{\llnssbf}{\mbox{{\bf LogLossNS }}}
\newcommand{\llnss}{\mbox{ LogLossNS }}

\newcommand{\llnsrbf}{\mbox{{\bf loglossRuleNS}}}
\newcommand{\llnsr}{\mbox{ loglossRuleNS}}

\newcommand{\seq}[3]{[#1]_{#2}^{#3}}

\newcommand{\todo}[1]{}
\newcommand{\done}[1]{}

\newcommand{\emamap}{EmaMap}

\newcommand{\emapr}{emaPR}
\newcommand{\qpr}{qPR}
\newcommand{\qprob}{qPR}

\newcommand{\lrmap}{rateMap}

\newcommand{\base}{\mbox{baseline}}

\newcommand{\score}{\mbox{score}}

\newcommand{\itm}[1]{{$#1$}}

\newcommand{\corep}{CORE} 
\newcommand{\core}{CORE } 
\newcommand{\coma}{CORE } 
\newcommand{\comap}{CORE} 

\newcommand{\sd}{SD }
\newcommand{\sdn}{SD}  
\newcommand{\sds}{SDs }  
\newcommand{\sdsn}{SDs}  

\newcommand{\sm}[1]{\mbox{a}(#1)}
\renewcommand{\u}[1]{\mbox{u}(#1)} 

\renewcommand{\sp}[1]{\mbox{sup}(#1)} 

\newcommand{\sma}{SMA } 
\newcommand{\sman}{SMA}
\newcommand{\smas}{SMAs }
\newcommand{\smasn}{SMAs}

\newcommand{\di}{DI }
\newcommand{\din}{DI}  
\newcommand{\dis}{DIs }  
\newcommand{\disn}{DIs}  

\newcommand{\dix}[1]{\mbox{DI}(#1)}

\newcommand{\pgs}{Prediction Games }
\newcommand{\pgsn}{Prediction Games}

\newcommand{\pr}{PR }
\newcommand{\prn}{PR}  
\newcommand{\prs}{PRs }  
\newcommand{\prsn}{PRs}  

\newcommand{\pscore}{\mbox{$\mathcal{L}$}}

\newcommand{\empscore}{\mbox{$\tilde{\mathcal{L}}$} }

\newcommand{\empllns}{\mbox{AvgLogLossNS}}
\newcommand{\empllnsp}{\mbox{AvgLogLossNS }}
\newcommand{\llnsemp}{\mbox{AvgLogLossNS}}

\newcommand{\iscore}{\mbox{ScoringRule}}

\renewcommand{\O}{O} 

\newcommand{\markns}{\mbox{{\bf isNS}}}


\newcommand\iverson[1]{\!\left[\left[#1\right]\right]}

\newcommand{\expd}[1]{\mathbb{E}(#1)}
\newcommand{\E}[1]{\expd{#1}}
\newcommand{\expdb}[2]{{\mathbb{E}_{#1}}(#2)} 
\newcommand{\Eb}[2]{\expdb{#1}{#2}}

\newcommand{\var}[1]{\mathbb{V}(#1)}

\newcommand{\exped}{Expedition }
\newcommand{\expedn}{Expedition}

\newcommand{\ra}{\rightarrow}
\newcommand{\la}{\leftarrow}

\newcommand{\rv}{r.v. }
\newcommand{\rvn}{r.v.}
\newcommand{\rvs}{r.v.'s}

\renewcommand{\c}{C}
\newcommand{\C}{C}
\newcommand{\Y}{Y}

\newcommand{\fontsmaller}{\fontsize{7pt}{7pt}\selectfont}
\newcommand{\fontsmall}{\fontsize{8pt}{9pt}\selectfont}
\newcommand{\fontnormal}{\fontsize{10pt}{12pt}\selectfont}

\newcommand{\ebl}{\mbox{BrierAvg} }
\newcommand{\ebln}{\mbox{BrierAvg}} 

\newcommand{\bitem}{\mbox{QRule}} 

\newcommand{\bl}{\mbox{$\mathcal{L}_2$}}

\newcommand{\brier}{\bl \ }
\newcommand{\briern}{\bl} 

\newcommand{\logloss}{\mbox{log-loss }}
\newcommand{\loglossn}{\mbox{log-loss}}

\newcommand{\llsp}{\mbox{LogLoss }} 
\newcommand{\llspn}{\mbox{LogLoss}} 
\newcommand{\lln}{\mbox{LogLoss}}

\newcommand{\bdist}{\mbox{QuadDist}} 

\renewcommand{\P}{\mathcal{P}}
\newcommand{\Q}{\mathcal{W}} 
\newcommand{\pmap}{W} 

\newcommand{\sdisro}{subdistribution } 
\newcommand{\sdisrop}{subdistribution}

\newcommand{\sdistro}{subdistribution} 

\newcommand{\distro}{\mathcal{D}}

\newcommand{\I}{\mathcal{I}}

\newtheorem{defn}{Definition} 


\newtheorem*{theorem*}{Theorem} 

\newtheorem{prop}{Proposition} 

\newtheorem{lemma}{Lemma}
\newtheorem*{lemma*}{Lemma} 

\newtheorem{corollary}{Corollary}
\newtheorem*{corollary*}{Corollary}

\newcommand{\maxprob}{P_{max}}

\newcommand{\minobs}{O_{min}}

\newcommand{\minlen}{L_{min}}

\newcommand{\NS}{NS }
\newcommand{\ns}{NS }
\newcommand{\nsn}{NS}

\newcommand{\nsm}{NS-marker }
\newcommand{\nsmn}{NS-marker}

\newcommand{\pmin}{p_{min}} 
\newcommand{\minprob}{p_{min}} 

\newcommand{\pns}{p_{NS}}
\newcommand{\oovprob}{p_{NS}}

\newcommand{\pt}{p_0}

\newcommand{\ps}{p_z}
\newcommand{\pg}{p_g} 

\newcommand{\oovthrsh}{c_{NS}}
\newcommand{\nstrsh}{c_{NS}}
\newcommand{\nsthrsh}{c_{NS}}

\newcommand{\up}{\mathcal{U}} 
\newcommand{\lo}{\mathcal{L}} 

\newcommand{\sema}{\mbox{SEMA }} 
\newcommand{\semap}{\mbox{SEMA}} 

\newcommand{\dial}{\mbox{DIAL }}
\newcommand{\dialp}{\mbox{DIAL}}

\newcommand{\size}{\mbox{size}}

\newcommand{\pma}{\alpha} 

\newcommand{\pred}{\mbox{pred}}


\newcommand{\eps}{\epsilon} 

\newcommand{\eq}{Eq. }  
\newcommand{\eqn}{Eq.}

\newcommand{\nodes}{|\nodeset|}  

\newcommand{\edgeset}{E}  

\newcommand{\V}{V}  

\newcommand{\G}{G}  

\newcommand{\n}{n}  
\newcommand{\m}{|\edgeset|}  






\newcommand{\grp}{g}  
\newcommand{\g}{\grp}  

\newcommand{\gsize}{|g|}  

\newcommand{\p}{p}  

\newcommand{\pe}{p_{edge}}  

\newcommand{\pn}{p_{node}}  

\newcommand{\qu}{Qs }  
\newcommand{\qs}{Qs }  
\newcommand{\qun}{Qs}  

\newcommand{\cz}{cell0 }
\newcommand{\czn}{cell0}

\newcommand{\ck}{cellk } 
\newcommand{\ckn}{cellk} 

\newcommand{\qmap}{qMap}
\newcommand{\q}{q} 

\newcommand{\qd}{DYAL } 
\newcommand{\dy}{DYAL } 
\newcommand{\dyal}{DYAL } 
\newcommand{\qdn}{DYAL} 

\newcommand{\bx}{Box } 
\newcommand{\bxn}{Box} 


\renewcommand{\i}{\mathcal{I}}  

\newcommand{\qi}{Modularity}  
\newcommand{\modularityr}{\qi}  

\newcommand{\Qn}{Q_{node}}  

\newcommand{\U}{U}  

\newcommand{\uscore}{u}  
\newcommand{\lscore}{l}  
\newcommand{\bscore}{\tilde{b}}  

\newcommand{\lrule}{LogRule} 

\newcommand{\lr}{\beta} 
\newcommand{\minlr}{\beta_{min}} 
\newcommand{\lrmin}{\beta_{min}} 
\newcommand{\lrmax}{\beta_{max}} 
\newcommand{\maxlr}{\beta_{max}} 

\newcommand{\bsn}{b_{node}}  
\newcommand{\bse}{b_{edge}}  

\newcommand{\bff}{}  

\newcommand{\ebf}{\bf \em}

\newcommand{\avgPR}{avgPR~} 
\newcommand{\topPR}{topPR}

\newenvironment{mysubfloat}[2][]{\subfloat[#1]{#2}}{}

\newcommand{\mysim}{\raise.17ex\hbox{$\scriptstyle\sim$}}
\newcommand{\mytilde}{\raise.17ex\hbox{$\scriptstyle\sim$}}

\newcommand{\KL}{\mbox{KL}}  
\newcommand{\kl}{\mbox{KL}}  
\newcommand{\klsp}{\mbox{KL }}  

\newcommand{\klb}{\mbox{KL}_{b}}  

\newcommand{\klns}{\mbox{KL}_{NS}}  

\newcommand{\ent}{\mbox{H}}  

\newcommand{\vs}{{\em vs.~}}  
\newcommand{\eg}{{\em e.g.~}} 
\newcommand{\etal}{{\em et~al.~}} 
\newcommand{\al}{{\em al~}} 

\newcommand{\co}[1]{}

\newcommand{\squeeze}{\vspace{-1em}}
\newcommand{\vsp}{\vspace}

\newcommand{\D}{\mathbf{D}}
\newcommand{\X}{\mathbf{X}}
\newcommand{\Z}{\mathbf{Z}}
\newcommand{\cgrp}{\hat{g}}  
\newcommand{\rgrp}{g^o}  

\newcommand{\refset}{S^o}  
\newcommand{\cset}{\hat{S}}  

\newcommand{\match}{\mbox{match}}  

\newcommand{\gscore}{s^o}  

\newcommand{\Prec}{precision}  
\newcommand{\R}{recall}  


\maketitle



\begin{abstract} 
Consider a predictor, a learner, whose input is a stream of discrete
items.  The predictor's task, at every time point, is {\em
  probabilistic multiclass prediction}, \ie to predict which item may
occur next by outputting zero or more candidate items, each with a
probability, after which the actual item is revealed and the predictor
learns from this observation.  To output probabilities, the predictor
keeps track of the proportions of the items it has seen.  The stream
is unbounded (lifelong), and the predictor has finite
limited space.
The task is {\em open-ended}: the set of items is unknown to the
predictor and their totality can also grow unbounded.  Moreover, there
is {\em non-stationarity}: the underlying frequencies of items may
change, substantially, from time to time. For instance, new items may
start appearing and a few recently frequent items may cease to occur
again. The predictor, being space-bounded, need only provide
probabilities for
those items which, at the time of prediction, have {\em sufficiently high}
frequency, \ie the {\em salient} items. This
problem is motivated in the setting of {\em \pgsn}, a
self-supervised learning regime where concepts serve as {\em both the
  predictors and the predictands}, and the set of concepts grows over
time, resulting in non-stationarities as new concepts are generated
and used. We
design and study a number of predictors, {\em sparse moving averages
  (\smasn)}, for the task.
 One \sma adapts the sparse exponentiated moving average (EMA) and
 another is based on queuing a few count snapshots,
 %
 keeping
 dynamic per-item histories. Evaluating the predicted probabilities,
 under noise and non-stationarity, presents challenges, and we discuss
 and develop evaluation methods, one based on
 bounding log-loss.  We show that a combination of ideas,
 %
 supporting {\em dynamic predictand-specific} learning
 rates, offers advantages in terms of faster adaption to change (plasticity),
 while also supporting low variance (stability). Code is made available.
\end{abstract}  
  

\vspace*{.25in}

\hspace*{0.5in} \mbox{\em "Occasionally, a new knot of significations
  is formed.. and our natural powers suddenly }\\ \hspace*{0.65in}
\mbox{merge with a richer signification.''}\footnote{When a new
concept, \ie a recurring pattern represented in the system, is well
predicted by the context predictors that it tends to co-occur with
(other concepts in an interpretation of an episode, such as a visual
scene), we take that to mean the concept has been incorporated or
integrated (into a learning and developing system), and we are taking
"skill", in the quote, to mean a pattern too (a sensorimotor
pattern). We found the quote first in \cite{Paolo2017SensorimotorLA},
and a more complete version is: {\em ``Occasionally, a new knot of
  significations is formed: our previous movements are integrated into
  a new motor entity, the first visual givens are integrated into a
  new sensorial entity, and our natural powers suddenly merge with a
  richer signification'' ("knot of significations" has also been
  translated as "cluster of meanings") \cite{MerleauPonty1964ThePO}.} } \ \ \ \ Maurice Merleau-Ponty
\cite{MerleauPonty1964ThePO}



%


\section{Introduction}

The external world is ever changing, and change is everywhere. This
non-stationarity takes different forms and occurs at different time
scales,
including periodic changes, with different periodicities, and abrupt
changes due to unforeseen events. In a human's life, daylight changes
to night, and back, and so is with changes in seasons taking place
over weeks and months \cite{ingold2000}.  Appearances of friends
change, daily and over years.  Language use and culture evolve over
years and decades: words take on new meanings, and new words, phrases,
and concepts are introduced.  Change can be attributed to physical
processes, or agents' actions (oneself, or others).  A learning system
that continually predicts its sensory streams, originating from the
external world, to build and maintain models of its external world,
needs to respond and adapt to such {\em \bf external non-stationarity}
in a timely fashion. As designers of such systems, we do not have
control over the external world, but hope that such changes are not
too rapid and drastic that render learning futile
\cite{nonstationarity2015,Schlimmer1986IncrementalLF,Kuh1990LearningTC,Helmbold1991TrackingDC,
  Widmer1996LearningIT,conceptDrift2023}.

We posit that complex adaptive learning systems need to cope with
another source of change too: that of their own {\bf \em internal}, in particular
{\bf \em developmental, non-stationarity}.
Complex learning systems have multiple adapting subsystems,
homogeneous and heterogeneous, and these parts evolve and develop
over time, and thus their behavior changes over time.  As designers of
such systems, we have some control over this internal change.  For
instance, we may define and set parameters that influence how fast a part can
change. To an extent, modeling the change could be possible too. Of
course, there are tradeoffs involved, for example when setting change
parameters that affect internal change,
favoring adaptability to the external world over stability of internal
workings.

A concrete example of internal \vs external non-stationarity occurs
within the {\em  \pgs} approach
\cite{pgs3,pgs1,expedition1}, where
we are investigating systems and algorithms
to help shed light on 1) how a learning system could develop
high-level (perceptual, structured) {\em concepts} from low-level sensory
information, over time and many learning episodes, primarily in an
unsupervised (self-supervised) continual learning fashion, and 2) how
such concept (structure) learning, and imposing one's structures onto
the input episodes, could be beneficial for subsequent efficient
learning, prediction, and reasoning applications.  In the course of
exploring such systems,
we have found that making a conceptual distinction
between internal (developmental or endogenous) and external
(exogenous) nonstationarity is fruitful. In \pgsn, the
system repeatedly inputs an episode, such as a line of text, and determines {\em which of its higher level
  concepts}, such as words (n-grams of characters), are applicable. This
process of mapping stretches of input characters into higher level
concepts is called {\em  interpretation}.  The system begins its
learning with a low level initial (hard-wired) concept vocabulary, for example the
characters, and over many thousands and millions of episodes of
interpretation (of practice), grows its
vocabulary of represented concepts, which in turn enable it to better
interpret and predict its world.  In \pgsn, each concept,
or an explicitly represented pattern (imagine ngrams of characters for
simplicity), is {\em  both a predictor and a predictand}, and a
concept can both be composed of parts as well as take part in the
creation of new (higher-level) concepts (a {\em  part hierarchy}).
After a new concept is generated, it is used in subsequent
interpretations, and existing concepts co-occurring with it need to
learn to predict it, as part of the process of integrating a new
concept within the system.  Prediction is probabilistic: each
predictor provides a probability with each predictand (item that it
predicts). Supporting probabilistic predictions allows for making
appropriate
decisions based on information utility \cite{pgs3}, such as answering the question of
which higher-level concepts should be
used  in the
interpretation of the current episode.  Thus, the generation and usage of new
concepts leads to internal non-stationarity from the
vantage point of the existing concepts, \ie the existing
predictors.\footnote{Internal \vs external are of course relative to a
point of reference.} From time to time, each concept, as a predictor,
will see a change in its input stream due to the creation and usage of
new concepts (see also \sec \ref{sec:internal}). This change is due to
the collective workings of the different system parts. We envision
such systems to contain different parts (of learning, inference, and
so on) working and developing together, a
growing collective intelligence. In this paper, we focus on advancing
solutions for the prediction task.




%


Previously, we addressed the non-stationarity faced by the predictors
(due to the discovery and use of new concepts) via, periodically,
cloning existing concepts and learning from scratch for all, new and
cloned, concepts \cite{expedition1}. However, this retraining from
scratch can be slow and inefficient, and leads to undesired
non-smoothness or abrupt behavior change in the output of the
system. It also requires extra complexity in implementation (cloning,
support for multiple levels), and ultimately is inflexible if the
external world also exhibits similar non-stationary patterns.  In this
paper, we develop {\em \bf sparse
  moving average (\sman)} techniques that handle such
transitions more gracefully.  We note that the structure of concepts
learned can be probabilistic too, and non-stationarity can be present
there as well, \ie {\em both intra-concept as well as inter-concept
  relations can be non-stationary}. The algorithms developed may
find applications in other non-stationary multiclass domains, such as
in recommenders and personal assistants (continual personalization),
and we conduct experiments on several data sources: Unix command
sequences entered by users, reflecting their changing daily tasks and
projects, and natural language text (in addition to experiments on
synthesized sequences).



%

Thus our
task is online multiclass probabilistic prediction under
non-stationarity, and we use the
term ``item'' in place of ``concept'' or ``class'' for much of the
paper.  We develop techniques for a {\em finite-space predictor},
one that
has limited space (memory), a constant independent of stream length, which
translates to a limit on how low a probability the predictor can
support and predict well.  In this paper, we focus on learning
probabilities above a minimum probability (support) threshold of a given $\pmin$
parameter (set to $0.01$ or $0.001$ in our experiments).
Supporting lower probabilities requires larger space in
general, and the utility of learning lower probabilities also depends
on the input (how stationary the world is).\footnote{We note that the
problem can also be formulated as each predictor having a fixed space
budget, and subject to that space, predict the top $k$ (recently)
highest proportions items. With this, the successfully predicted
probabilities, depending on the input stream, can be lower than a
fixed $\pmin$ threshold. }

The input stream to a predictor is unbounded (or indefinite), and the
set of items the predictor will see is unknown to it, and this set can
also grow without limit. In other words, the task is {\bf \em
  life-long} and {\bf \em open-ended}.  From time to time the
proportion of an item, a predictand, changes.  For instance, a new
item appears with some frequency, \ie an item that had hitherto 0
probability (non-existant), starts appearing with some significant
frequency above $\pmin$. Possibly, a few other items' probabilities
may need to be reduced at around the same time.  The predictand
probabilities that a single predictor needs to support simultaneously
can span several orders of magnitude (\eg supporting both 0.5 and
0.02).  As a predictor adapts to changes, we want it, to the extent
feasible, to change only the probabilities of items that are affected.
As we explain, this is a form of (statistical) stability \vs
plasticity dilemma
\cite{Abraham2005MemoryR,fpsyg2013,driftingDistros2024}.  While fast
convergence or adaptation is a major goal and evaluation criterion,
specially because we require probabilities, which take some time to
learn, we allow for a {\em grace period}, or an allowance for {\em
  delayed response} when evaluating different techniques: a predictor
need not provide a (positive) probability as soon as it has observed
an item. Only after
a few sufficiently recent observations of the item, do we
require
prediction.\footnote{For internal non-stationarity, we have some
control over the rate of change (\eg see \sec \ref{app:simpc}).} We
show how each predictor keeping its own {\em \bf predictand-specific
  learning-rates} and supporting {\em \bf dynamic rate changes}, \ie
both rate increases and decreases (rather than keeping them fixed, or
only allowing change in one direction, such as monotone decreasing),
improves convergence (after a change) and prediction accuracy,
compared to the simpler moving average techniques. These extensions
entail extra space and update-time overhead,
but we show that the flexibility that such extensions afford can be
worth the costs.


We list our assumptions and goals, {\em practical probabilistic
  prediction}, for non-stationary multiclass lifelong streams, as
follows:
\begin{itemize} 
\item (Assumption 1: Salience) Aiming at learning probabilities that
  are sufficiently high, \eg $p \ge 0.001$.  Such a range will be
  useful and adequate in real domains and applications (and 
  a consequence of the space constraints of a predictor).
  \item (Assumption 2: Approximation) Approximately learning the
    probabilities is adequate, for instance, to within a deviation
    ratio of 2 (the relative error can also depend on the magnitude of
    the target probability, \sec \ref{sec:logloss2}).
  \item (Performance Goal: Practical Convergence) Strive for efficient
    learners striking a good balance between speed of convergence and
    eventual accuracy, \eg a method taking 10 item observations to
    converge to within vicinity of a target probability is preferred
    over one requiring 100, even if the latter is eventually
    more stable or precise.
\end{itemize} 


Our focus is therefore neither convergence in the limit nor highly
precise estimation. We also
note
that in practical
applications, there may not be 'real' or true
probabilities,\footnote{In \pgsn, higher-level concepts and their
co-occurrence probabilities are the constructs of the system itself.}
but probabilities are useful foundational {\em means} to achieve
system functionalities beyond non-probabilistic prediction (\eg a
plain ranking) of the
future
possibilities. If a new target event of interest occurs every few
seconds, then a system that converges to sufficiently accurate
probabilities withing a few 10s of seconds is preferred to a system
requiring several minutes or hours. Similarly, if the event occurs
once a day, then convergence in days is strongly preferred over many
weeks or months.\footnote{Higher level concepts in \pgs are less
frequent than the lower-level ones, and the useful conditional
probabilities (co-occurrences) often can span below 0.1 down to 0.001,
and perhaps lower.} To the best of our knowledge, both the task and
problem formulation as well as much of the \sma and evaluation
techniques developed are novel (building on much previous work, see \sec
\ref{sec:related}).





This paper is organized as follows. We next present the formal problem
setting and introduce notation. This includes
the
 idealized non-stationary generation setting (\sec
\ref{sec:ideal}). We have this setting in mind when developing the
algorithms.
\sec \ref{sec:eval} explains the ways we evaluate
and compare the
prediction techniques. 
We review proper scoring,
motivate \logloss (logarithmic loss) over quadratic loss, and adapt
\logloss
to the challenges of non-stationarity, noise, and incomplete
distributions, where we also define drawing from (sampling) and taking
expectations with respect to incomplete distributions. To handle
infrequent (``noise'') and new items, we develop a bounded variant of
\logloss and show when it remains useful, \ie approximately proper,
for evaluation (\sec \ref{sec:llns} and Appendix \ref{sec:np}).
%
%
We next present and develop several \smasn:
The Sparse EMA \sma uses a learning rate (\sec
\ref{sec:ema}), while the \qs \sma is
based on counting and averaging (\sec \ref{sec:qus}), followed by a
hybrid (\sec \ref{sec:qdyal}) which combines the benefits that each
offers: the stability of EMA and the plasticity of \qun.
%
%
%
We begin with sparse EMA and present a convergence property under the
stationary scenario, quantifying the worst-case expected number of
time steps to convergence in terms of a (fixed) learning rate.
Here, we motivate why we seek to enhance EMA, to have a
variable rate, even though it could
handle some, small, drifts. This change alone is not sufficient.
%
We then present the \qu method, based on queuing of count-based
snapshots, which is more responsive to (large) change than
EMA, but has more variance (\sec \ref{sec:qus}). We present analyses
(unbiased estimation), and two variants of queuing.
This is followed by
the hybrid approach we call \qd (\sec \ref{sec:qdyal}), which keeps
predictand-specific learning rates for EMA and uses queuing to
%
adapt to changes in a timely manner.  We conduct both synthetic
experiments, in \sec \ref{sec:syn_exps}, where the input sequences are
generated by known and controlled changing distributions, and on
real-world data, in \sec \ref{sec:real}.
These sequences can exhibit a variety of phenomena (inter-dependencies),
including external non-stationarity.
We find that the hybrid \qd does well in a range of situations,
providing evidence that the flexibility it offers is worth the added
cost of keeping
per-perdictand learning rates and queues.
Appendices contain further material, in particular, the proofs of
technical results,
and additional properties and experiments. Python code for \smas and
experimentation is available on GitHub \cite{dyal_code}.


%









\section{Preliminaries: Problem Setting, Notation, and Evaluation}
\label{sec:prems}
\begin{table}[t]  \center
  \begin{tabular}{ |c| }     \hline 
    $\seq{o}{}{}, \seq{o}{1}{N}$, etc. : infinite \& finite item sequences,
    $\seq{o}{1}{N} := [\ott{o}{1},\cdots, \ott{o}{N}]$, 
   $\oat{o}$ is the item observed at time $t$.   \\ \hline  
  \prn s, \din s, \sdn s : \underline{\bf pr}obabilities (\prsn),  prob. \underline{\bf di}stributions (\disn), and \underline{\bf s}emi-\underline{\bf d}istributions (\sdsn).    \\ \hline 
    SMA : a Sparse Moving Average predictor, such as the sparse EMA and \qu predictors
    (\sect \ref{sec:ema}, \ref{sec:qus}, \ref{sec:qdyal}). \\ \hline
   $\P$, $\Q$, $\Q_1, \Q_2$, etc. : \co{Variables
  denoting \sdsn.} $\P$ is used for a true underlying \sdn, and $\Q$, $\Q_1$, $\Q_2$ are
  estimates (\eg output of an SMA). \\ \hline
    $\seq{\Q}{1}{N}$: \co{ The sequence of} first $N$ outputs of an SMA (predictor).
    Each prediction, $\oat{\Q}$, is a \sd (or converted to one).  \\ \hline
    $\sp{\Q}$: the support set \co{(non-zero \pr items)} of \sd $\Q$. Every \sd has
    finite support  (subset of all items $\I$). \\  \hline
    $\sm{\Q},\u{\Q}$: $\sm{\Q}$
    the \underline{a}llocated \pr mass,
    $\sum_{i \in\sp{\Q}} \Q(i)$,
    $\u{\Q} := 1-\sm{\Q}$ \co{the} (\underline{u}nallocated)\co{ mass)}  \\  \hline
     $\tp$ : For a binary event, observation $o \in \{0, 1\}$, the true probability of the positive, 1, outcome. \\ \hline
     positive outcome, 1 : whenever the target item of interest is observed (0 otherwise).  \\ \hline
    \NS item: a \underline{\bf N}oi\underline{\bf s}e item, \ie {\em \underline{\bf n}ot \underline{\bf s}een recently in the input stream sufficiently often} (currently). \\ \hline 
    $\minprob$ and $\pns$ : 
    used for filtering and scaling, and bounding \logloss on
    \ns items (\sec \ref{sec:evalns}). \\ 
    $\pmin$ can be interpretted as the minimum \pr supported by an \sman. \\ \hline
    $\nsthrsh$ : when
    evaluating, \co{the} count-threshold used by a
    practical \nsm algorithm 
    (\ref{sec:evalns}). \\ \hline
    $\minobs$ : in synthetic experiments,
    observation count threshold on salient items
    (\sec \ref{sec:syn_non_one},
    \ref{sec:syn_multi}). \\ \hline
    EMA : The Exponentiated Moving Average (EMA).
    We develop sparse multiclass variants of it (\sect
    \ref{sec:ema}). \\ \hline
    $\lr$, $\minlr$, $\maxlr$, ... : 
    learning rates (for EMA variants), $\lr \in [0,1]$, $\minlr$ is
     minimum allowed, etc.  \\ \hline
  \end{tabular}
  \vsp{.1in}
\caption{A summary of the main notations, abbreviations and
  terminology, parameters, etc. See \sect \ref{sec:prems}. }
\label{tab:notates}
\end{table}

\begin{figure}[tb]
\begin{center}
\hspace*{-0.41cm}  \subfloat[An example stream of observations, left to right (items:
    \itm{B}, \\ \itm{A}, \itm{J}, ...), and prediction outputs, the boxes,
    by a hypothetical predictor at a few of the time points
    shown.]{{\includegraphics[height=4cm,width=8cm]{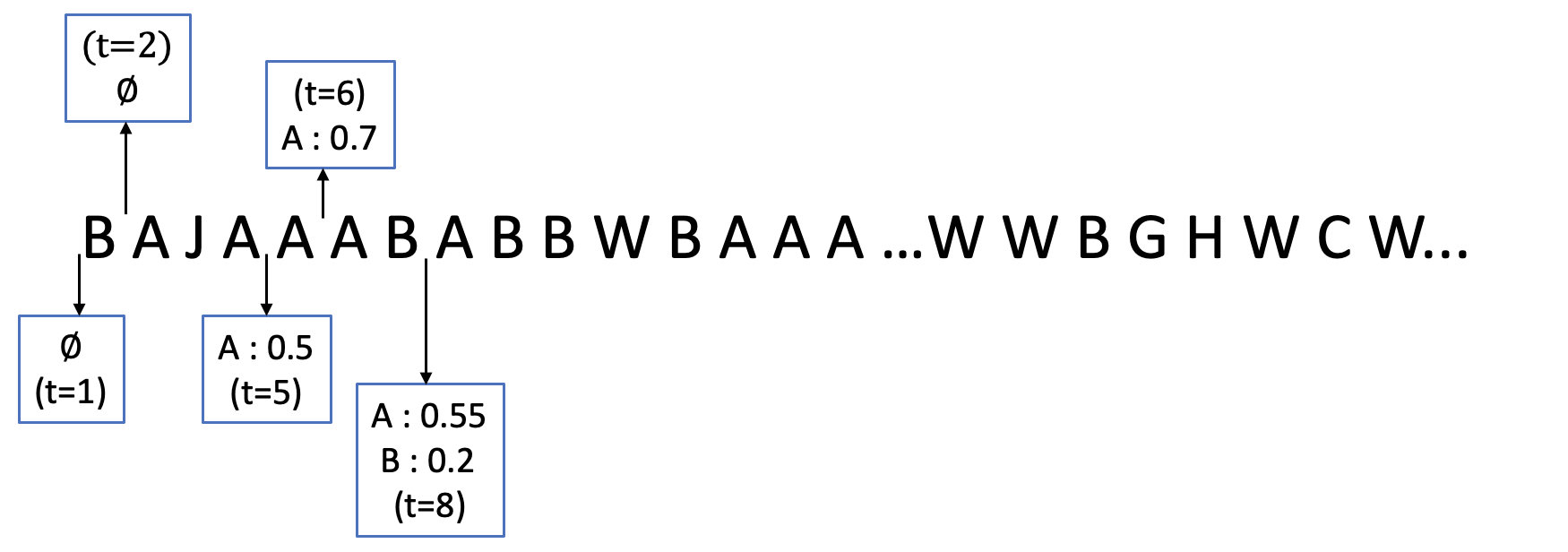}
  }} \hspace*{0.5cm}\subfloat[A categorical distribution generates the (synthetic)
    data stream, in an idealized setting (iid draws), but also itself
    changes every so often (here, at time $t=700$, $1999$,
    etc).]{{\includegraphics[height=4cm,width=5cm]{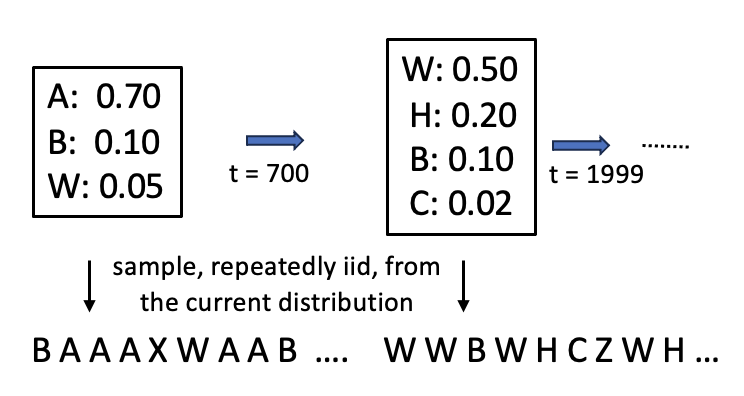}
  }}
\end{center}
\vspace{.2cm}
\caption{(a) An example sequence of items (the capital letters)
  together with several prediction outputs of a hypothetical
  predictor, in rectangular boxes, shown for a few time points (not
  all outputs are shown to avoid clutter). The sequence is observed
  from left to right, thus at time $t=1$, item \itm{B} is observed
  ($o^{(1)}=$ \itm{B}), and respectively at times $2$, $3$, and $4$,
  items \itm{A}, \itm{J}, and again \itm{A} are observed ($o^{(3)}=$
  \itm{J}, $o^{(4)}=$ \itm{A}, etc). A prediction output is a map of
  item to probability (\prn), and can be empty. Mathematically, it is
  a semi-distribution \sd (\sec \ref{sec:sds}). At each time point,
  before the observation, the predictor predicts, \ie provides a \sd
  (zero or more items, each with a \prn). In this example, at times $t
  \le 4$, nothing is predicted (empty maps, or
  $\ott{\Q}{1}=\ott{\Q}{4}=\{\}$, and only two empty outputs, at $t=1$
  and $t=2$, are shown). At $t=8$, $\ott{\Q}{8}$ is predicted, where
  $\ott{\Q}{8} = \{$\itm{A}:0.55, \itm{B}:0.2$\}$ (\ie \itm{A} is
  predicted with \pr 0.55, and \itm{B} with \pr 0.2).
  (b) The input sequence can be imagined as
  being generated by a \sd $\P$:
  at each time point, for the next entry of the sequence, an item is
  drawn, iid (\sec \ref{sec:ideal}). However, the \sd $\P$ changes
  from time to time, such as certain item(s) being removed and new
  item(s) inserted in $\P$. In the above example, in changing from the
  left (initial) $\ott{\P}{1}$ to the right distribution,
  $\ott{\P}{2}$ (at $t=700$), \itm{A} is dropped (becomes 0 \prn),
  while \itm{H} and \itm{C} are inserted, and \itm{W} increases in \pr
  while \itm{B} is unchanged.}
\label{fig:asequence} 
\end{figure}


Our setting is online discrete open-ended
probabilistic prediction under non-stationarity.
%
%
Time is discrete and is denoted by the variable $t$, $t=1,2,3,\cdots$.
A predictor {\em processes} a stream of items (observations): at each
time, it predicts then observes and updates.
Exactly one (discrete) item occurs and is observed
at each time $t$ in the stream.
We use the parenthesized superscript notation for the variable
corresponding to the item observed at time $t$, \ie $\oat{o}$, and
similarly
\co{
The item or observation at time 
$t$ is denoted by the
variable $\oat{o}$.
%
}
for other sequenced objects as well, for instance an estimated
probability $\ep$ at time $t$ is denoted $\oat{\ep}$, though we may
drop the superscript
whenever the context is sufficiently clear (that it is at a particular
time point is implicit).
Each item is represented and identified by an integer (or string) id
in the various data structures (sequences, maps, ...), and can be viewed
as belonging to
an infinite unordered (categorical and nominal\footnote{In this work,
we assume no ordering on the items. However, a future direction is
when the items have known associated costs or rewards. The term
``categorical'' may have finiteness connotations, but not here.})  set
(\eg the set of finite strings).  We use letters \itm{A}, \itm{B},
\itm{C}, $\cdots$ and also integers $0, 1, 2, 3, \cdots$ (but no
ordering implied) for referring to specific items in examples.  We use
the words ``sequence'' and ``stream'' interchangeably, but 'stream' is
used often to imply the infinite or indefinite aspect of the task,
while 'sequence' is used for the finite cases, \eg for evaluations and
comparisons of prediction techniques.
A sequence is denoted by the brackets, or array, notation
$\seq{\ }{}{}$, thus $\seq{o}{1}{N}$ denotes a sequence of $N$
observations: $[\ott{o}{1}$, $\ott{o}{2}$, ..., $\ott{o}{N}]$ (also
shown without commas $[\ott{o}{1}\ott{o}{2}\cdots\ott{o}{N}]$).  \fig
\ref{fig:asequence}(a) shows an example sequence. Online processing,
of an (infinite) stream or a (finite) sequence, means repeating the
{\em predict-observe-update} cycle, \fig\ref{fig:cycle}(a): at each
time point, predicting (a semi-distribution, see next
\sec\ref{sec:sds}) and then observing. The predictor is an online
learner that updates its data structures after each observation. This
mode of operation (and evaluation) has been named {\em prequential} in
probabilistic sequence forecasting \cite{prequential_dawid1984} (see
also \cite{bifet2015Efficient}).
Our problem setting is similar to another recent work on predicting
(infinite-support) changing distributions in computational learning
theory \cite{driftingDistros2024}. In that work, the authors do not
focus on space-bounded predictors (\eg the history can grow with the
stream) and the proposed algorithm, with a dynamic history size, is
not empirically assessed (see \sects \ref{sec:box} and
\ref{sec:related} for further details).  Table \ref{tab:notates}
provides a summary of the main notations and terminology used in this
paper.  \sec \ref{sec:code_description} describes the format we use
for pseudocode, when we have presented pseudocode for a few functions.

%

%
%



\subsection{Probabilities (\prn s), Distributions (\din s), and Semi-distributions (\sdn s)}
\label{sec:sds}


We use the abbreviation \pr to refer to a probability, \ie a real
number in $[0, 1]$. Prediction is probabilistic, and in particular in
terms of {\em semi-distributions}:\footnote{We are borrowing the
naming used in
\cite{Dupont2005LinksBP}.  Other
names include sub-distribution and partial distribution. } 
A (categorical) semi-distribution (\sdn) $\Q$, in this paper, is a
real-valued (or \prn-valued) function defined over a finite or
infinite set of items $\I=\{0, 1, 2,\cdots\}$, such that $\forall i,
\Q(i)\in [0,1]$, and $\sum_{i \in\I} \Q(i) \le 1.0$.  Whether or not
$\I$ is infinite, the {\em support}, denoted $\sup(\Q)$, \ie the set
of items $i$ with $\Q(i) > 0$, is finite in the \sdn s that we work
with in our experiments: In particular, $|\sup(\Q)|$, at any time $t$,
does not exceed 100s to 1000s depending on space limits on the
predictor (or algorithmic parameters, \sec \ref{sec:ns}). Let $\sm{\Q}
:= \sum_{i \in\I} \Q(i)$, \ie the allocated probability mass of $\Q$,
and let $\u{\Q}:=1-\sm{\Q}$, \ie the unallocated or free mass of
$\Q$.\footnote{To help clarity, we often use the symbol ':=' to stand
for "defined as" when defining and introducing notation. We will use
the plain equal sign, '=', for making claims or stating properties,
for instance $\u{\Q} + \sm{\Q}=1$. } When $\sm{\Q} > 0$ we say $\Q$ is
non-empty, and when $0 < \sm{\Q} < 1$, we call $\Q$ a {\em strict}
\sdn. Let $\min(\Q):=\min_{i \in \sup(\Q(i))} \Q(i)$, and defined as 0
when $\Q$ is empty. A useful view (when evaluating, see \sec
\ref{sec:evalns3}, and in proofs, see {\em augmenting} in \sec
\ref{sec:np}) is to imagine that the free mass $\u{\Q}$ is given to a
special item (not seen in the input stream), yielding a {\em
  distribution}: A probability distribution (\din) $\Q$ is a special
case of a \sdn, where $\sm{\Q}=1$. From a geometric viewpoint, \dis
are points on the surface of a unit-simplex, while \sds can reside in
the
interior of the simplex too: a typical predictor's predictions \sd
$\Q$ starts at 0, the empty \sdn, and, with many updates to it, moves
towards the surface of the simplex.  The prediction
output\footnote{Following literature on partially observed Markov
decision processes, we could also say that each predictor maintains a
{\em belief state}, a \sd here, which it reveals or outputs whenever
requested.  } at a certain time $t$ is denoted $\oat{\Q}$, and we can
imagine that a prediction method is any online technique for
converting a sequence of observations, $\seq{o}{1}{N}$, one
observation at a time, into a sequence of predictions $\seq{\Q}{1}{N}$
(of equal length $N$), as shown in \fig \ref{fig:cycle}(b). But note
that $\oat{\Q}$ is output before observing $\oat{o}$ (prediction
occurs before observing and updating).  We use $\P$ to refer to an
underlying or actual \sd generating the observations, in the ideal
setting (see next), and $\Q$ to refer to estimates and prediction
outputs (often strict \sdn s), and $\oat{\Q}$ form {\em moving}
(changing) \sdsn.


\co{
borrowing the naming used
in
\cite{Dupont2005LinksBP}, we refer to the prediction output of a
predictor, at any time point, as a {\em semi-distribution}
\cite{Dupont2005LinksBP},
for which we review definitions and introduce notation now.\footnote{Other
candidate names include sub-distribution and truncated distribution. }
}


\co{
one can imagine that feeding a
sequence of observations, $\seq{o}{1}{N}$, to a prediction method,
yields a sequence of predictions of the same length, $\seq{\Q}{1}{N}$,
as shown in \fig \ref{fig:cycle}(b).

$\seq{o}{1}{N}$

feeding a
sequence of observations, $\seq{o}{1}{N}$, to any prediction method,
yields a sequence of predictions of the same length, $\seq{\Q}{1}{N}$,
as shown in \fig \ref{fig:cycle}(b).
}

A \sd is implemented via a map data structure:\footnote{Although a
list or array of key-value pairs may suffice when the list is not
large or efficient random-access to an arbitrary item is not
required. In this paper, prediction provides the \prs for all items
represented, thus all items are visited, and updating takes similar
time. See also \sec \ref{sec:timestamps}.} item $\rightarrow$
\pr (or $\I \rightarrow [0,1]$), where the keys are item ids and the
values are \prn s (usually positive). The map can be empty, \ie no
predictions (for instance, initially at $t=1$). $\Q(i)$ is 0 if item
$i$ is not in the map.  Periodically, the maps are pruned, items with
smallest \prs dropped, for space and time efficiency.

We use braces and colons for presenting example \sdn s (as in the
Python programming language).  For instance, with $\P=\{0$:0.5,
1:$0.2\}$, $\P$ has support size of 2, corresponding to a binary event,
 and item $0$ has \pr 0.5, and $1$ has probability
$0.2$, $min(\P)=0.2$, $\sm{\P}=0.7$, and $\u{\P}=0.3$, so $\P$ is a
strict \sd in this example.




\co{

Prediction is probabilistic, and the predictions output of a
predictor, at each time point, is represented with a (finite) map data
structure: item $\rightarrow$ probability. The map can be empty, \ie
no predictions (for instance, initially).

Mathematically, the predictions output of a predictor, at any time
point, is a {\em semi-distribution} (or a sub-distribution), often denoted
$\Q$, for which we review and introduce notation.

A semi-distribution, in this paper, is a real-valued function defined
over a finite or infinite set of items $\I=\{1, 2, 3\cdots\}$, such
that $\forall i, \Q(i)\in [0,1]$ (or $0 \le \Q(i) \le 1$), and
$\sum_{i \in\I} \Q(i) \le 1.0$. We use $\sm{\Q}$ to denote $\sum_{i
  \in\I} \Q(i)$. When $\I$ is infinite (\eg the set of positive
integers), the {\em support} of \sd $\Q$, \ie the set of items $i$
with $\Q(i) > 0$, is finite in the semi-distributions that we deal
with (prediction outputs).  A probability distribution $\P$ is a
special (more strict) variant of a semi-distribution, such that
$sum(\P)=1$.

We use the abbreviation \di to refer to a probability
distribution $\P$, and \sd for a semi-distribution, thus, a \di $\P$
and a \sd $\Q$.

The prediction output at time $t$ is denoted
$\oat{\Q}$.

}

\subsubsection{Generating Sequences: an Idealized Stream}
\label{sec:ideal}

We develop prediction algorithms with
an idealized non-stationary setting in mind, which we describe next.
Under the {\bf \em stationary iid assumption} or generation setting,
an idealized setting, we imagine the sequence of observations being
generated via independently and identically drawing (iid) items, at
each time point, from a true (actual or underlying categorical) \sd
$\P$. Because $\sm{\P}$ can be less than 1, not exhaustively covering
the probability space, we explain what it means to generate from a \sd
below.  Under an {\em \bf idealized but non-stationary} setting, we
assume the sequence of observations is the concatenation of
subsequences, of possibly different lengths, where each subsequence
follows the stationarity setting, and is long enough for
(approximately) learning its \prsn. Thus one can imagine a sequence of
\sdn s $\seq{\P}{}{}=\ott{\P}{1},\ott{\P}{2},\cdots$ (moving \sdsn,
\fig \ref{fig:asequence}(b)), and the $j$th {\em stable subsequence}
of observations, $\seq{o}{t_j}{t_{j+1}}$, is generated by
$\ott{\P}{j}$, for the duration $[t_{j}, t_{j+1}]$ (and
$\ott{\P}{j+1}$ generates for $[t_{j+1}+1, \cdots]$). Each (stable)
subsequence corresponds to a {\em stable} period or phase, stationary
iid setting, where the underlying \sd $\P$ is not changing.  This
assumption of periods of stability are also similarly described in
\cite{multicp19}, who study keeping track of a changing multinomial
distribution with fixed known number of items $k$ (see \sec
\ref{sec:related} for further discussion).  We assume that each stable
subsequence
is long enough for learning the new \prsn: any item with a new \pr
should remain stable for some time, \ie {\em be observed sufficiently many
times}, before its \pr changes in a subsequent \sd $\P$. In synthetic
experiments (\sec \ref{sec:syn_non_one} and \ref{sec:syn_multi}), we
use an {observation-count threshold} $\minobs$ as a way of
controlling the degree of non-stationarity, and we report the
prediction performance of various predictors under different $\minobs$
settings.  See also \sec \ref{app:simpc} on how the concept generation
speed (the rate of generation of new items) can be controlled within
Prediction Games, which influences the extent of (internal)
non-stationarity.

\done{Generating from a strict \sd  will be explained shortly}



\subsubsection{Salient and Noise (\nsn) Items, and Generating with Noise}
\label{sec:ns}

\done{Define $\pmin$ .. Distinguish between definition of noise items
  when we assume there is a true $\P$ generating the data, and when we
  don't assume such. Also say that salient items can become noise
  items and vice versa, and this conversion could occur several times
  during a (long) sequence. }

In this paper, we are interested in predicting and evaluating \prn s
$p$ that are sufficiently large, $p \ge \pmin$, where the support
threshold $\pmin=0.01$ in most experiments. With respect to a \sd
$\Q$, an item $i$ is said to be salient iff $\Q(i) \ge \pmin$, and
otherwise it is \ns (\underline{n}on-\underline{s}alient, or
\underline{n}oi\underline{s}e, or not seen recently sufficiently
often). If $\Q$ is the output of a predictor, we say the predictor
(currently) regards $i$ as salient or \ns as appropriate.  In
synthetic experiments, when the underlying \sd $\P$ is known, an item
is either truly salient or \nsn, and as $\P$ is changed, the same item
$i$ may be \ns at one point ($\oat{\P}(i) < \pmin$), and become
salient at another ($\ott{\P}{t+1}(i) \ge \pmin$), and change status
several times.\footnote{The concept of out-of-vocabulary (OOV) in
language modeling is similar. Here, the idea of \ns items is more
dynamic.}  In practice, in evaluations, when we do not have access to
a true $\P$, we use a simple \ns marker based on whether an item has
occurred sufficiently often recently (see \ref{sec:evalns}). A
challenge for any predictor is to quickly learn to predict new salient
items while ignoring the noise.\footnote{Noise, as we have defined, is
relative and practical, and arises from the finiteness of the memory
of a predictor along with other computational constraints (applicable
to any practical predictor, and similar to
randomness defined relative to computational constraints). Given a finite-space predictor with $k$ bits of space, an
item can have positive but sufficiently low frequency, below $2^{-k}$,
therefore not be technically noise, but that would foil the
finite-space predictor. Identifying true noise would require access to
the infinite stream. }


For sequence generation, when drawing from a \sd $\P$, with
probability $\u{\P}=1-\sm{\P}$ we generate a noise (\nsn) item. A
simple option that we use in most experiments is to generate a unique
noise id: the item will appear only once in the sequence. \sec
\ref{sec:syn_multi} explains how we generate sequences with \ns items
in the synthetic experiments, and discusses alternative methods (see
also \ref{sec:llns} on drawing items from \sdsn). For evaluation too, \ns
items need to be handled carefully, specially when using
\loglossn. See \sec \ref{sec:evalns}.

\subsubsection{Binary Sequences, and the Stationary Binary Setting}
\label{sec:binary}

The {\em stationary binary setting}, and binary sequences generated in
such settings, will be useful in the empirical and theoretical
analyses of the prediction techniques we develop.  A stationary binary
sequence can be generated via iid drawing from a \di $\P=\{$ 1$:\tp$,
0$:1-\tp\}$, where $\tp > 0$ is the {\em true or target} \pr to
estimate well via a predictor that processes the binary sequence. A
binary sequence also arises whenever we focus on a single item, say item
$A$, in a given original sequence containing many (unique) items.  The
original sequence $\seq{o}{}{}$ is converted to a binary
$\seq{o_b}{}{}$, $\seq{o}{}{} \rightarrow \seq{o_b}{}{}$, in the
following manner: if $\oat{o}=A$ then $\oat{o_b}=1$, and $\oat{o_b}=0$
otherwise (or $\oat{o_b}=\iverson{\oat{o}=A}$, where $\iverson{x}$
denotes the Iverson bracket on the Boolean condition $x$).

\subsubsection{Examples and Possibilities}
\label{sec:example_seqs}

The ideal generation model described encompasses a variety of
challenging possibilities: a salient item's \pr may change suddenly
and by small (drifts) or large amounts, an extreme case being its \pr
changing to zero or below the support threshold $\pmin$
(disappearance). One salient item, with say probability $0.03$ may
remain stable for a long time, while other items, possibly with higher
\prs, may keep changing (shifting or oscillating or disappearing and
reappearing). \sec \ref{sec:syn_non_one} explores a setting involving
oscillations between two values. A good predictor makes (substantial) changes to its \pr
estimates only when appropriate (keeping \pr estimates of stable items
stable). The proportion of the noise items in the stream can change
too, \eg from 0.01 to 0.05 (proportion of the stream), to 0.1 and
back.  An online predictor needs to cope with and track such changes
effectively. A good evaluation technique should also appropriately
reward robust tracking, for detecting changes not only in the \pr
estimations of the salient items but also the mass allocated to the
\ns (noise) portion of the stream (\sec \ref{sec:evalns}).

Of course, real-world sequences can be even more complex and deviate
from the above idealization in multiple (non-iid sampling) ways,
motivating empirical comparisons of the \smas that we develop on
real-world datasets.

\subsection{Prediction Techniques: Sparse Moving Averages (SMAs)}
\label{sec:smas}

We refer to the predictors we develop as {\em sparse (multiclass)
  moving averages} (SMAs), such as the sparse EMA predictor (\sec
\ref{sec:ema}) and the \qu predictor (\sec \ref{sec:qus}). A moving
average in its most basic form tracks the changes of a scalar value by
keeping a running (recency-biased) average.  Here, instead of a
scalar, the observation at time $t$ can be viewed as a sparse
(possibly infinite) vector of 0s with a single 1 at the dimension
equal to the id of the observed item
\cite{updateskdd08,weib2012},\footnote{Variations of the task are
plausible and useful too, that is zero or more than a single 1, but
relatively few, in the input vector.  We focus on the multiclass case
(exactly one 1) in this paper.} and the techniques developed and
studied here are different ways of keeping {\em several} running
averages, \ie estimated proportions of the items deemed salient, to
predict the future probabilistically. The number of proportions that
are tracked is also kept under a limit, \ie kept sparse, for space and
time efficiency.


Note that there are two major motivations for having a limited
(sparse) memory: one stemming from resource boundedness or the need
for space and time efficiency, and another for adapting to
non-stationarities. The two goals or considerations can be consistent
or can trade off.

Note also that with an all-knowing adversary generating the input
sequence, in particular knowing the workings of the predictor, as soon
as positive \prs are output by the predictor for an item, the
adversary can change the subsequent proportion of that item in the
remainder of the stream, for instance to 0, rendering the \pr
predictions of the predictor, as soon as they become available,
useless.  Our basic assumption is that realistic streams have
structure and are not generated by such adversaries.

\co{
  
estimating the probability
of a single item $A$: Take any sequence, and convert it to a
corresponding sequence

A sequence of items can be viewed as a binary
sequence when we take an

when the item appears in the original
sequence, it is a positive event, and in the corresponding binary
sequence we observe a 1, and other wise.

Assume there is a single item that is our target of interest, and when
it occurs, it is considered a positive outcome (a 1, or a heads
outcome from tossing a 2-sided coin), and the occurrence of any other
item is regarded as a negative outcome (a 0, or a tails).

}

\co{
  under evaluation
  
we thus imagine $\ott{\P}{j}$ to be \sdn
s. When drawing from a \sd $\P$, with probability $1-sum(\P)$ we
generate (or observe) a noise item.

Certain items may appear once in the input stream or more generally
have a small occurrence probability $p$, $p < \pmin$, rendering them
{\em non-salient}. Some such items may increase in probability at some
point, from low probability (\eg 0) to $p \ge \pmin$, \ie and thus
become newly salient.  A challenge for any predictor is to quickly
learn to predict new salient items while ignoring the noise.  We use
\ns (\underline{n}oi\underline{s}e, or "\underline{n}ot
\underline{s}een sufficiently often recently") to refer to low
probability items. When thinking of sequence generation, we thus
imagine $\ott{\P}{j}$ to be \sdn s. When drawing from a \sd $\P$, with
probability $1-sum(\P)$ we generate (or observe) a noise item. \sec
\ref{sec:syn_multi} explains noise generation further. For evaluation
too, \ns items need to be handled appropriately, specially under
evaluation with \logloss (\sec \ref{sec:evalns}).
}








\begin{figure}[]
\begin{center}
  \centering
  \subfloat[Online learning (prequential): the predict-observe-update cycle.]{{\includegraphics[height=3cm,width=6cm]{
        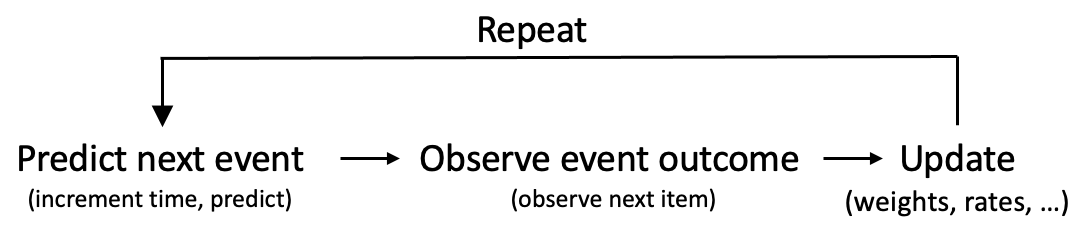} }}\hspace*{1cm}
  \subfloat[An SMA
    converts an item stream into a stream of (moving) \sdn s, $\seq{\Q}{}{}$, the
    predictions.]{{\includegraphics[height=3.4cm,width=6cm]{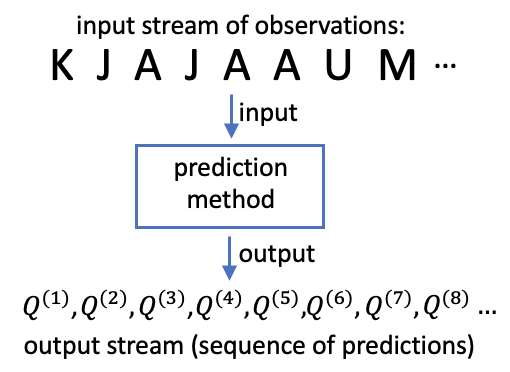}
  }}
\end{center}
\vspace{.12cm}
\caption{(a) Online processing a stream or sequence means repeating
  the prequential {\em predict-observe-update} cycle (b) An
  SMA converts a stream of item observations, $\seq{o}{}{}$, to a
  stream of predictions, $\seq{\Q}{}{}$, or $\seq{o}{1}{N} \rightarrow
  \seq{\Q}{1}{N}$. }
\label{fig:cycle} 
\end{figure}





\section{Evaluating Probabilistic Predictors}
\label{sec:eval}



In this section we develop and discuss techniques for evaluating
open-ended probabilistic multiclass sequential prediction in the face of
non-stationarity, in particular so that we can compare different
prediction techniques. As \fig \ref{fig:cycle}(b) shows, a prediction
technique is any method that converts a sequence of observations,
$\seq{o}{1}{N}$, into a same-length sequence, $\seq{\Q}{1}{N}$, of
predictions, each sequence member $\oat{\Q}$ is an \sdn, where
$\oat{\Q}$ is output before $\oat{o}$ is observed. We develop methods
for evaluating the predictions $\seq{\Q}{1}{N}$, given
$\seq{o}{1}{N}$, and possibly other extra information. A minimal
desirable would be that in the stationary iid setting, when there is a
true non-changing model (a distribution) that is generating the data,
that a predictor quickly converges in its predictions to a good
approximation of the true model. We develop this further next. We then
describe and motivate proper scoring (under unknown true
probabilities) and explain why we prefer \logloss (logarithmic loss)
over quadratic loss, and adapt \logloss
to \sdsn, and the challenges of noise and non-stationarity.

\subsection{Deviation Rates: when True \prs are Known}

For the purposes of evaluating the estimated \prsn, we first assume we
know the true \prsn, which is the case in our synthetic experiments
(\sec \ref{sec:syn_exps}). In particular, we begin with the binary iid
setting of \sec \ref{sec:binary} where we have a generated binary
sequence with $\tp$ being the \pr of 1.
We take a prediction method and feed it this sequence: it processes
this sequence and generates a sequence of probability estimates. Let
$\seq{\ep}{1}{N}$ denote its sequence of estimates for the positive,
1, outcome. Thus $\oat{\ep}$ is the estimated \pr at time $t$,
$\oat{\ep} \ge 0$.  Here, we motivate performance or error measures
(quality of output probabilities) that are based on relative error or
ratio, a function of $\frac{\oat{\ep}}{\tp}$, \vs based on the
difference in magnitude $\tp-\oat{\ep}$ (such as the quadratic
$(\tp-\oat{\ep})^2$).


In many applications, including the \pgs setting,
different items carry different rewards, and the system needs to
select a subset of items to optimize an expected utility measure
(expected reward). Sufficient accuracy of the \prs that the
different items receive is important.  While we do not specify the
rewards in this paper, nor the particular way the expected reward is
computed, in general, we seek to limit the {\em relative} error in
estimating probabilities. For instance, as we are computing
expectations, confusing a one-in-twenty event (true \pr or target
probability $\tp=0.05$) with a one-in-one-hundred event ($\tp=0.01$),
could be considered a worse error, compared to predicting an event
with $\tp=0.55$ with an estimated \pr of $0.5$ (even though the
absolute difference of the latter two is higher).

In particular, with the true probability $\tp$, $\tp > 0$, and the
estimates $\oat{\ep}$ forming a sequence $\seq{\ep}{1}{N}$, for a
choice of a deviation bound $d$, eg $d=2$, we define the {\em
  deviation-rate} $\dev(\seq{\ep}{1}{N}, \tp, d)$ as the following sequence
average:


\begin{align}
\hspace*{-.35in}  \label{eq:devr}
   \dev(\seq{\ep}{1}{N}, \tp, d) = \frac{1}{N}\sum_{t=1}^N \dvc(\oat{\ep}, \tp, d) \hspace*{.21in} \mbox{(the deviation rate, with  $\tp > 0$)}
\end{align}
\begin{align}
\hspace*{-.35in}  \label{eq:devc}
   \dvc(\ep, \tp, d) = \iverson{\ep = 0 \mbox{, or (otherwise), }  \max(
     \frac{\tp}{\ep}, \frac{\ep}{\tp} ) > d } \hspace*{.21in} \mbox{(it is 0 or 1)} 
\end{align}


where $\iverson{x}$ is the Iverson bracket, yielding value 1 when
condition $x$ is true, and 0 otherwise. Note that when $\ep$ is 0, the
condition is true and $\dvc(0, \tp, d)$ is 1 for any $d$.  We seek a
small {\em deviation-rate}, \ie a large deviation, when
$\max(\frac{\tp}{\oat{\ep}}, \frac{\oat{\ep}}{\tp} ) > d$, occurs on a
sufficiently small fraction of times $t$ only, \eg say 10\% of the
time or less. While initially, the first few times a predictor
observes an item, we may get large deviations, we seek methods that
reach lowest possible deviation-rates as fast as possible.  We do not
specify what $d$ should be, nor what an acceptable deviation rate
would be (which depends on the sequence, \ie what is feasible, in
addition to the prediction method), but, in our synthetic experiments
(\sec \ref{sec:syn_exps}), where we know the true \prsn, we report
deviation rates for a few choices of $d$. More generally, we track
multiple item \prn s, and we will report two variations: at any time
$t$, when the observed item's estimated \pr deviates, and when {\em
  any} (salient) item in $\P$ suffers too large of a deviation (\sec
\ref{sec:syn_multi}).

Deviation rates are easy to understand (interpretable), but often in
real-world settings, we do not have access to true \prn s.  Therefore,
we also report on {\em logarithmic loss (log-loss)}, and this approach
is discussed and developed next, in particular for the case of \sds
and \ns items.









\co{
These probabilities are used to determine how a system sees or
interprets its input, in terms of the concepts (patterns), that it has
built so far.

Because the concepts are built by the system, the question of the
probability of their occurrence may not be well defined.

but one challenge is that

we don't know how to measure accurate accuracy and how how accurate we
need to be.
}

%


\subsection{Unknown True \prsn: Proper Scoring}

Data streams from a real world application are often the
result of the confluence of multiple interacting and changing
(external) complex processes. But even if we assume there is a
(stationary) \di $\P$ that is sampled iid, and thus true \prs exist,
remaining valid for some period of time, see \sec \ref{sec:ideal}),
%
most often, we do not have access to $\P$.
This is a major challenge to evaluating probabilistic predictions.
There exist much work and considerable published research literature,
in diverse domains such as meteorology and finance,
addressing the issues of unknown underlying \prsn, when predicting
with (estimated) \prsn. In particular, the concept of {\em proper
  scoring rules} has been developed to encourage {\em reliable}
probability forecasts\footnote{Other terms, or related concepts, used in the literature
include calibrated, honest, trust-worthy, and incentive-compatible
probabilities. }
\cite{brier1950,good1952,toda63,winkler69,gneiting2007,ml_calibration_survey21,rev2022}. A
proper scoring rule would lead to the best possible score (lowest if
defined as costs, as we do here) if the technique's \pr outputs
matched the true \prsn.  Not all scoring rules are proper, for
instance, plain absolute loss (and linear scoring) is not proper
\cite{brier1950}. And some are {\em strictly} proper, \ie the true \di
is the unique minimizer of the score \cite{gneiting2007}. Two major
families of proper scoring rules are either based on (simple functions
of) the predicted \prs raised to a power, a quadratic variant is named
after Brier \cite{brier1950}, or based on the logarithm (log
probability ratios) \cite{good1952,toda63}: \logloss and
variants. Propriety may not be sufficient, and a scoring rule also
implicitly imposes measures of closeness or (statistical) distance on
the space of candidate \sdn s, and we may prefer one score over
another based on their differences in sensitivity.

We review proper scoring and the two major proper scorers below,
\logloss and quadratic loss, specializing their distance property for
our more general setting of \sds in particular, and we argue below,
that in our utility-based or reward-based application, we strongly
favor \logloss over quadratic scoring, for its better support for \prs
with different scales and decision theoretic use of the estimations.
This is despite several very desirable properties of the quadratic
(such as boundedness and symmetry) in contrast to \loglossn, and that
we have to do extra work to adapt \logloss for handling 0 \prs (\ns
items). See also \cite{local2020,univ2018} for other considerations in
favor of \loglossn. We note that our proposal for evaluating the noise
portion of the stream is useful under either loss (see \sec
\ref{sec:evalns3}), and our code reports either loss \cite{dyal_code}.

\co{ After describing scoring further, we go over a few of the
  properties and argue that the Brier score, despite its attractive
  properties, does not meet our evaluation goals. Log-Loss is better
  suited, but has to be adapted for our non-stationarity setting, in
  particular the issues of unseen or noise items, and we discuss and
  motivate our choices.  }

\subsection{Scoring Semi-Distributions (\sdn s)}
\label{sec:scoring}

For scoring (evaluating) a candidate distribution, we assume there
exists the actual or true, but unknown, distribution (\din) $\P$. We
are given a candidate \sd $\Q$ \ (\eg, the output of a predictor), as
well as $N$ data points, or a sequence of $N$ observations
$\seq{o}{1}{N}$, $t=1, \cdots, N$, to (empirically) score the
candidate.  Recall that we always assume that the support set of a \sd
is finite, and here we can assume the set $\I$ over which $\P$ and
$\Q$ are defined is finite too, \eg the union of the supports of
both. It is possible that the supports are different, \ie there is
some item $j$ that $\Q(j) > 0$, but $\P(j) = 0$, and vice versa.

A major point of scoring is to assess which of several \sd candidates
is best, \ie closest to the true $\P$ in some (acceptable) sense, even
though we do not have access to the \di $\P$.  It is assumed that for
the time period of interest, the true generating \di $\P$ is not
changing, and the data points are drawn iid from it. The score of a
\sd $\Q$ is an expectation, and different scorers define the
real-valued function $\iscore(\Q, o)$ (\ie the ``scoring rule'')
differently:
  \begin{align}
    \hspace*{-.2in}  \pscore(\Q|\P) = \expdb{o \sim \P}{\iscore(\Q, o)} \hspace*{0.3in}
  \mbox{(loss of a \sd $\Q$, given true \di $\P$)} 
\end{align}


Thus the \pscore() (score) of a \sd $\Q$ is the expected value of the
function \iscore(), where the expectation is taken with respect to
(wrt) the true \di $\P$.\footnote{Following \cite{selten98}, we have
used the conditional probability notation (the vertical bar), but with
the related meaning here that the loss (or score) is wrt an assumed
underlying \di $\P$ generating the observations (instead of being wrt an event).}  In practice, we
compute the average \empscore() \ as an empirical estimate of \pscore:

\begin{align}  
 \hspace*{-.25in}\empscore(\Q,\seq{o}{1}{N}) = \frac{1}{N} \sum_{t=1}^N  \iscore(\Q, \oat{o})
\hspace*{0.25in} \mbox{(an empirical estimate of loss)}
\end{align}
computed on the $N$ data points $\seq{o}{1}{N}$ (assumed drawn iid from $\P$).  Thus,
for scoring a candidate \sd $\Q$, \ie computing its (average) score,
we do not need to know the true \di $\P$, but only need a sufficiently
large sample of $N$ points drawn iid from it. However, for
understanding the properties of a scoring rule, the actual \sd \ $\P$
is important. In particular, as further touched on below, the score is
really a function of an (implicit) statistical distance
between the true and candidate distributions. The details of the
scoring rule definition, \iscore(), determines this
distance. Different scoring functions differ on the scoring rule
$\iscore(\Q, o)$, \ie how the \pr outputs of \sdn, given the observed
item was $o$, is scored. We also note that the minimum possible score,
when $\Q=\P$ (for strictly proper losses), can be nonzero. We next
develop a bounded version of \loglossn. Appendix \ref{sec:brier}
explains why the \brier (Brier) loss is not sufficiently sensitive and
does not suit our task, despite several attractive properties.


\subsection{On the Sensitivity of log-loss}
\label{sec:logloss2}



The \llsp and the log rule are based on simply taking the $\log()$
(the natural log) of the probability estimated for the observation $o$
\cite{good1952,toda63}, and thus should be sensitive to relative
changes in the \prn s (\pr ratios):
\begin{align}
\hspace*{-.35in}   \label{eq:logrule}
\lln(\Q|\P) := \expdb{o \sim \P}{\lrule(\Q, o)}, \mbox{ where \ }  \lrule(\Q, o) = -\log(\Q(o)).
\end{align}



The scoring rule can be viewed as the well-known KL (Kulback-Leibler)
divergence of $\Q$ from the Kronecker vector, \ie the distribution
vector with all 0s and a 1 on the dimension corresponding to the
observed item $o$. This boils down to $-\log(\Q(o))$.
%
%
Similar to the development for \bl, where \bitem \ is shown in terms
of the Euclidean distance (Appendix \ref{sec:brier}, Lemma
\ref{lem:quad_eq}), here \llsp can be shown to correspond to the KL
divergence \cite{kl1951,cover91} on \sdsn:




\begin{defn} The entropy of a non-empty \sd $\P$ is defined as:%
\begin{align}
  \hspace*{1cm} \ent(\P)=-\sum_{i \in \I} \P(i)\log(\P(i))  
\end{align}
The KL divergence of $\Q$ from $\P$ (asymmetric), also known as the
relative entropy, denoted $\kl(\P || \Q)$, is a functional, defined
here for non-empty \sd $\P$ and \sd $\Q$: 
\begin{align}\label{def:kl}
\hspace*{1cm} \kl(\P || \Q) := \sum_{i \in \I} 
\P(i)\log\frac{\P(i)}{\Q(i)}.
\end{align}
\end{defn}

Note that in both definitions, by convention (similar to the case for
\dis in \cite{cover91}), when $\P(i)=0$, we take the product
$\P(i)\log(x)$ to be 0 (or $i$ need only go over $\sup(\P)$, in the
above definitions). The divergence $\kl(\P || \Q)$ can also be
infinite (denoted $+\infty$) when $\P(i)>0$ and $\Q(i)=0$.

\begin{lemma} \label{lem:ll_kl}
  Given \di $\P$ and \sd $\Q$, defined over the same finite set $\I$, 
\begin{align*}
\hspace*{1cm} \lln(\Q|\P) = \ent(\P) + \kl(\P || \Q).
\end{align*}
\end{lemma}



This is established when both are \disn, for example by
\cite{selten98}, and the derivation does not change when $\Q$ is a
\sdn: $\lln(\Q|\P) = -\sum_{i \in \sup(\P)} \P(i)\log(\Q(i)) =
-\sum_{i \in \sup(\P)} \P(i)\log(\Q(i)) - \ent(\P) + \ent(\P)$ (\ie
add and subtract $\ent(\P)$), and noting that the first two terms is a
rewriting of the \kl(), establishes the lemma.

Note that the entropy term in \logloss is a fixed positive offset:
only \kl() changes when we try different candidate \sds with the same
underlying $\P$.  As in \sect \ref{sec:brier} (for \brier and
Euclidean), this connection to distance
helps us see several properties of \lln. From the properties of \klsp
\cite{cover91}, it follows that \llsp is strictly proper (and
asymmetric, unlike \briern). Furthermore, the expression for KL
divergence helps us see that \logloss is indeed sensitive to the
ratios (of true to candidate probabilities). In particular,
\logloss can grow without limit or simply {\bf explode} (\ie become
infinite) when $\P(i) > 0$ and $\Q(i) = 0$, for instance, in our
sequential prediction task, when an item has not been seen before.  We
discuss handling such cases in the next section.  On the other hand,
we note that the dependence on the ratio is dampened with a $\log()$
and weighted by the magnitude of the true \prsn, as we take an
expectation. In particular, we have the following property, a corollary
of the connection to \kl(), which quantifies and highlights the extent
of this ratio dependence:


\vspace*{.2in}
\begin{corollary} (\llsp is much more sensitive to relative changes in larger \prsn)
  Let $\P$ be a \di over two or more items (dimensions), with $\P(1) >
  \P(2) > 0$, and let $m=\frac{\P(1)}{\P(2)}$ (thus $m > 1$).
  Consider two \sdn s $\Q_1$ and $\Q_2$, with $\Q_1$ differing with
  $\P$ on item 1 only, in particular $\P(1) > \Q_1(1) > 0$, and let
  $m_1=\frac{\P(1)}{\Q_1(1)}$ (thus, $m_1 > 1$). Similarly, assume
  $\Q_2$ differs with $\P$ on item 2 only, and that $\P(2) > \Q_2(2) >
  0$, and let $m_2=\frac{\P(2)}{\Q_2(2)}$.  We have $\lln(\Q_2|\P) <
  \lln(\Q_1|\P)$ for any $m_2 < m_1^{m}$.
\end{corollary}



The proof is simple and goes by subtracting the losses, $\lln(\Q_1|\P)
- \lln(\Q_2|\P)$, and using the KL() expressions for the \logloss
terms and simplifying (Appendix \ref{sec:app_eval}).  Thus, with a
reasonably large multiple $m$, $m_2$ needs to be very large for
$\lln(\Q_2|\P)$ to surpass $\lln(\Q_1|\P)$.  As an example, let \di
$\P=\{$1:0.5, 2:0.05, 3:0.45$\}$, then $m=10$, and reduction by $m_1
\ge 1$, only in \pr of item 1, yields \llsp $l_1$ that is
$\frac{m_1^{10}}{m_2}$ higher than \llsp $l_2$ that would result from
reducing \pr of item 2 only, by $m_2$.  For example, with
$\Q_1=\{$1:0.25, 2:0.05, 3:0.45$\}$, and with $\Q_2=\{$1:0.5, 2:$p$,
3:0.45$\}$, as long as $p > \frac{0.05}{2^{10}}\approx 10^{-4}$,
$\lln(\Q_1|\P) > \lln(\Q_2|\P)$ (or we have to reduce 0.05 a thousand
times, $2^{10}$, to match the impact of just halving 0.5 to 0.25).




When we report deviation-rates in synthetic experiments (\eq
\ref{eq:devr}), we ignore the magnitude of the \prs involved (as long
as greater than $\pmin$), but log-loss incorporates and is
substantially sensitive to such.  Because items with larger \prs are
seen in the stream more often, such a dependence is warranted,
specially when we use the \prs to optimize expected rewards (different
predictands having different rewards). However, when a predictor has
to support predictands with highly different \prs (\eg at $0.5$ and
$0.05$), we should anticipate that \llsp scores will be substantially
more sensitive to better estimating the higher \prsn.

\co{
predictor predicts items with highly different \prs (\eg at $0.5$ and
$0.05$), we should anticipate that \llsp scores will be substantially
more sensitive to better estimating the higher \prsn.

We also note that \brier is not necessarily as sensitive to changes in
larger probabilities (as it does not involve an expectation), although
larger probabilities have the potential to lead to larger
shifts. \brier is only directly dependent on the magnitude of a
shift. Thus, with $\P=\{$1:0.8, 2:0.2$\}$, $\Q_1=\{$1:0.7, 2:0.2$\}$,
and $\Q_2=\{$1:0.8, 2:0.09$\}$, we have $\lln(\Q_1|\P) >
\lln(\Q_2|\P)$, while $\bl(\Q_1|\P) < \bl(\Q_2|\P)$ (larger shift of
0.15 in obtaining $\Q_2$).
}

\co{
points to make:
\begin{itemize}
  \item Sensitive to ratios , but also to how large $p$ is..
  \item In particular, to $r^{r2-r1}$
  \item \brier is not even necessarily sensitive to higher
    probabilities..
  \item the issue of unseen (not recently seen) items..
\end{itemize}
}

\subsection{Developing \logloss  for \ns (Noise) Items}
\label{sec:evalns}

\begin{figure}[t]
\hspace*{0.15in}\begin{minipage}[t]{0.5\linewidth}
  \small{
    \fcapbf($\Q$) // Filter \& cap $\Q$. \\
  \hspace*{0.2cm}  // $\Q$ is an item to \pr map.\\
  \hspace*{0.2cm} // Parameters: $\pns, \pmin$, both in $[0, 1)$\\
  \hspace*{0.2cm} $\Q' \leftarrow$ \scd($\Q$, 1.0, $\pmin$) // Filter \\ 
  \co{
  \hspace*{0.3cm} For item $i$, probability $prob$  in $\Q$: \\
  \hspace*{0.6cm}  If $prob < \pmin$: \\
  \hspace*{0.9cm}     continue // Filter low weight items.\\
  \hspace*{0.6cm}  $filtered[i] \leftarrow prob$ \\
  }
  \hspace*{0.2cm} If $\sm{\Q'} \le 1.0 - \oovprob$: // Already capped? \\
  \hspace*{0.6cm}   Return $\Q'$ // Nothing left to do. \\
  \hspace*{0.2cm} $\alpha \la \frac{1-\oovprob}{\sm{\Q'}}$ // Scale down by $\alpha$. \\
  \hspace*{0.2cm} Return \scd($\Q'$, $\alpha$, $\pmin$) \\ \\
  \co{
  \hspace*{0.3cm} $capped  \leftarrow \{\}$  // The scaled \& filtered map.\\
  \hspace*{0.3cm} For item $i$, probability $prob$ in $filtered$:  \\
  \hspace*{0.6cm}    $prob \leftarrow \alpha * prob$ / / // Scale by $\alpha$.  \\
  \hspace*{0.6cm}    If $prob \ge \minprob:$ // Keep only the salient.\\
  \hspace*{0.9cm}  $capped[i] \leftarrow prob$ \\
  \hspace*{0.3cm}  Return $capped$ \\ 
  }
  \scdbf($\Q$, $\alpha$, $\minprob$) // Filter and scale. \\ 
  \hspace*{0.3cm} $scaled \la \{\}$ \\
  \hspace*{0.3cm} For item $i$, \pr $prob$ in $\Q$:  \\
  \hspace*{0.6cm}    $prob \leftarrow \alpha * prob$  \\
  \hspace*{0.6cm}    If $prob \ge \minprob:$ \ \ // Keep the salient.\\
  \hspace*{0.9cm}  $scaled[i] \leftarrow prob$ \\
  \hspace*{0.3cm}  Return $scaled$ \\ \\
  }
  \hspace*{-.3cm} (a) Filtering \& capping a map $\Q$,
  yielding a \\ constrained \sdn: no small \prs and the sum is capped. 
\end{minipage}
\hspace*{0.cm}\begin{minipage}[t]{0.5\linewidth}
  \small{
  \llnsrbf($o, \Q, markedNS$)  \\
  \hspace*{0.2cm} // Parameters: $\oovprob, \pmin$. The current \\
  \hspace*{0.2cm} // observation is $o$, and $\Q$ is the \\
  \hspace*{0.2cm} // predictions (an item to \pr map). \\
  \hspace*{0.3cm}    $\Q' \leftarrow \fcap(\Q)$ // filter and cap $\Q$. \\
  \hspace*{0.3cm}    $prob \leftarrow \Q'$.get$(o, 0.0)$ \\
  \hspace*{0.3cm}  If $prob \ge \pmin$: // $o \in \sup{\Q'}$?\\
  \hspace*{0.6cm}    Return $-\log(prob)$ // plain log-loss.\\
  \hspace*{0.3cm}  // $o \not\in \sup{\Q'}$: the predictor thinks $o$ is\\
  \hspace*{0.3cm}  // a new/noise item. The \nsm \\
  \hspace*{0.3cm}  If not $markedNS$: // disagrees. \\ 
  \hspace*{0.6cm}     Return $-\log(\oovprob)$ // highest loss. \\
  \hspace*{0.3cm}  // They both agree $o$ is \ns\\
  \hspace*{0.3cm}  // Use the unallocated \pr mass, $\u{\Q'}$. \\ 
  \hspace*{0.3cm}  $p_{noise} \leftarrow \u{\Q'}$ // note: $p_{noise} \ge \oovprob$ \\
  \hspace*{0.3cm}  Return $-log(p_{noise})$ \\
 \hspace*{0.3cm}  \\ 
 \hspace*{-0.1cm} (b) When computing log-loss, handling new or \\ noise items (bounded loss
 when $\pns > 0$).
 }
\end{minipage}
\vspace*{0.2cm}
\caption{Functions used in evaluating the probabilities.  (a)
  CapAndFilter() is applied to the output of any predictor, at every
  time $t$, before evaluation, performing filtering (dropping small
  \prs below $\pmin$) and, if necessary, explicit capping, \ie
  normalizing or scaling down, meaning that the final output will be a
  \sd $\Q'$, where $\sm{\Q'}\le 1-\pns$ (or $\u{\Q'}\ge
  \pns$). $\pns=\pmin=0.01$ in experiments. (b) Scoring via
  \loglossn, handling \NS items (bounded \loglossn). }
\label{code:norming_etc}
\end{figure}


In our application, every so often a predictor observes new items that
it should learn to predict well. For such an item $o$, $\P(o)$ was 0
before some time $t$, and $\P(o) \ge \pmin$ after. Some proportion of
the input stream will also consist of infrequent items, \ie truly \ns
(noise) items, those whose frequency, within a reasonable recent time
window, falls below $\pmin$. A challenge for any predictor is to
quickly learn to predict new salient items while ignoring the noise.
Plain \logloss is infinite when the predictor assigns a zero \pr to an
observation $o$ ($\Q(o)=0$), and more generally, large losses on
infrequent items can dominate overall \logloss rendering comparing
predictors using \logloss uninformative and useless.  We need more
care in handling such cases, if \logloss is to be useful for
evaluation.

Our evaluation solution has 3 components:
\begin{enumerate}
\item Ensuring that the output map of any predictor, whenever used for
  evaluation, is a \sd $\Q$ such that $\sm{\Q}\le 1-\pns$, with $\pns
  > 0$, thus some mass,  $\u{\Q} \ge \pns$, is left for possibly
  observing an \ns item (the function \fcap() in \fig
  \ref{code:norming_etc}(a)).
\item Using a {\em simple \nsm} algorithm, a {\em (``third party'') 
  ``referee"}, for determining \ns status, as in \fig
  \ref{fig:lowest_loss}(a).
\item A scoring rule (a policy) specifying how to score in various cases,
  \eg based on whether or not an item is marked \ns by the referee
  (the function \llnsr() in \fig
  \ref{code:norming_etc}(b), a bounded version of \logloss).
\end{enumerate}

\co{
Our evaluation solution consists of first ensuring that the final
output of every predictor is a \sd and follows a few constraints,
achieved via what we call filtering and capping: the function \fcap() in \fig
\ref{code:norming_etc}(a).  Second we use a simple algorithm, what we
call an \nsmn, serving as a {\em "3rd-party referee"}. During
evaluation,
\nsm marks each item $\oat{o}$ as \ns or not, depending on whether its
seen-count surpasses a threshold $\nsthrsh$. \fig
\ref{fig:lowest_loss}(a) shows pseudocode for an implementation of
the \nsmn. Finally we need to specify a policy of how to score
depending on whether or not an item is marked \ns, the function
\llns() in \fig \ref{code:norming_etc}(b).
}

\fig \ref{code:norming_etc}(a) shows the pseudocode for filtering and
capping, \fcap(), applied to the output of any prediction method,
before we evaluate its output. The only requirement on the input map
$\Q$ passed to \fcap() is that the map values be non-negative.  In
this paper, all prediction methods output \pr maps, \ie with values in
$[0, 1]$, though their sum can exceed 1.0. The function removes small
values and ensures that the sum of the entries is no more than $1.0 -
\oovprob$, where $\oovprob=0.01$ in our experiments. Thus the output
of \fcap(), a map, corresponds to a \sd $\Q'$ such that $\sm{\Q'}\le
1-\pns$ and $\min(\Q') \ge \pmin$.



%

\begin{figure}[t]
  \begin{minipage}[t]{0.6\linewidth}
  \markns($o$) // whether observation $o$ is \ns or not. \\
  \hspace*{0.3cm} // Parameters: $\oovthrsh$ (\ns count threshold).   \\
  \hspace*{0.3cm} $flaggedNS \leftarrow recentFrqMap.get(o, 0) \le \oovthrsh$ \\
  \hspace*{0.3cm} // Increment count of $o$. \\ 
  \hspace*{0.3cm} Increment($o, recentFrqMap$) \\
  \hspace*{0.3cm} Return $flaggedNS$ // Return the 0 or 1 decision. \\ \\
  \hspace*{0.25cm} (a) Pseudo code for a simple \nsmn. \\  
  \subfloat{{ \includegraphics[height=2.5cm,width=7.5cm]{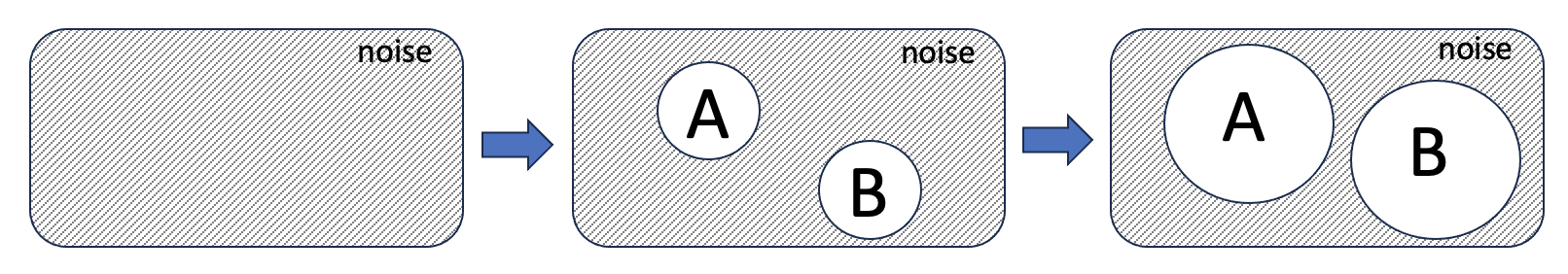}
  }} \\
  \hspace*{0.15cm} (b) Venn diagram examples, depicting underlying \sd
  $\P$ with a background of noise, two salient items \itm{A} and
  \itm{B} growing in \prs as we go from left to right (middle
  $\P^{(2)}$ could be $\{$\itm{A}:$0.1, \ $\itm{B}:$0.1\}$), when
  plotting the lowest achievable loss in \fig
  \ref{fig:lowest_loss}(c). \\ 
\end{minipage}
\hspace*{.5cm}
\begin{minipage}[t]{0.35\linewidth}
  \subfloat{{\includegraphics[height=5cm,width=5cm]{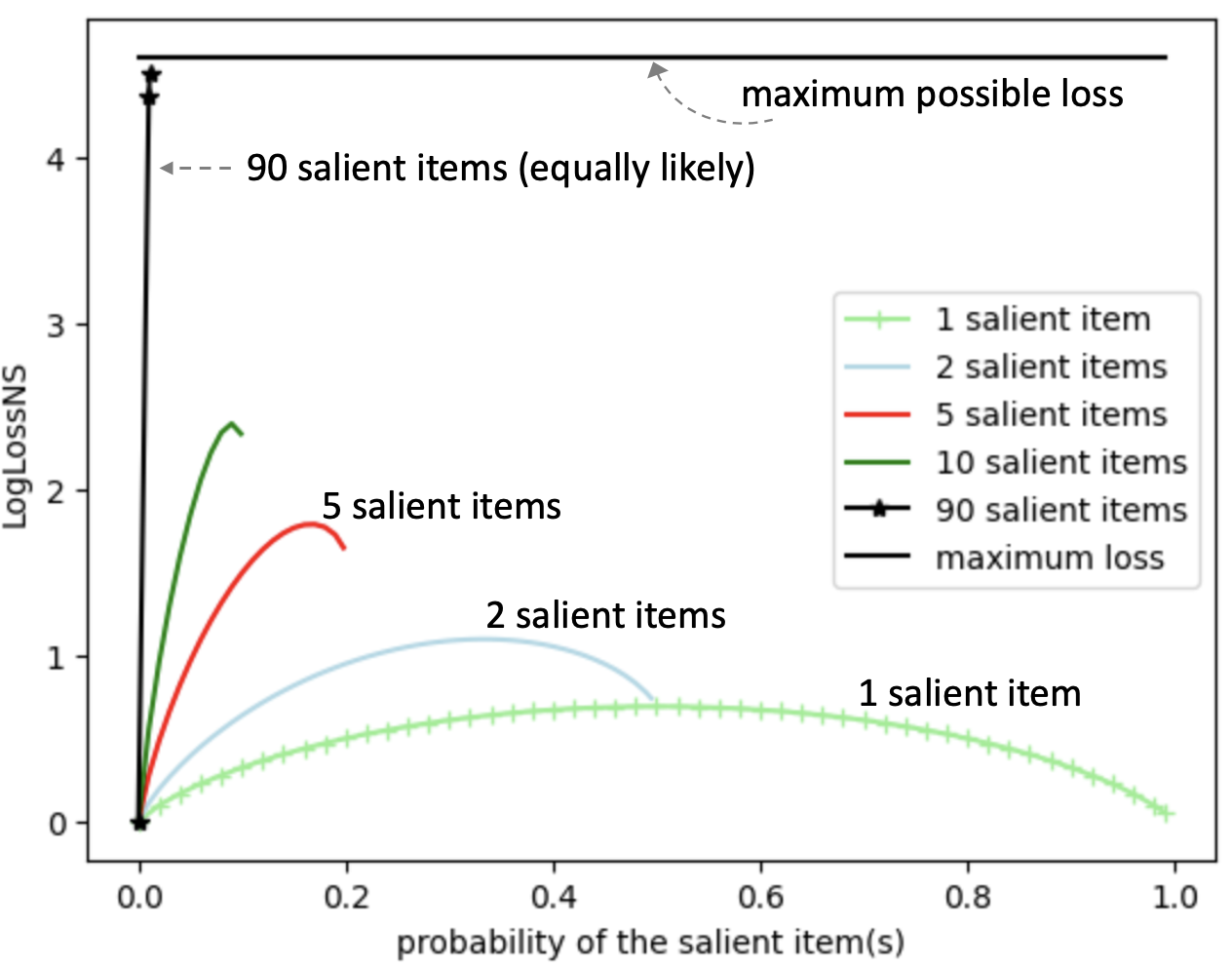} }}\\
    \small{
  \hspace*{0.1cm} (c) Lowest achievable \loglossn, via the 
  \llsn() function, under a  few generation
  regimes where salient items are equally likely.
  }
\end{minipage}
\caption{(a) A simple \nsmn, a referee to mark an item \ns or not,
  via a count map. In particular we used the \bx
  technique of \sec \ref{sec:box} (often with no limit on history size).
  (b) Venn diagrams, with background noise, are useful in picturing
  how sequences are generated, \eg in synthetic experiments. Here, as
  we go from left to right in generating a lowest achievable loss plot
  of \fig \ref{fig:lowest_loss}(c), three Venn digrams of the
  underlying \sds are shown for the case of two salient items (left
  $\P^{(1)}=\{\}$ while right $\P^{(3)}$ could be $\{$\itm{A}:$0.25,
  $\itm{B}:$0.25\}$).  (c) Optimal (lowest achievable) \logloss using
  the \llsn() function of \fig \ref{code:norming_etc}, as the \pr of
  $k$ salient items, all equally likely, is increased to maximum
  possible ($1/k$), from left to right. The lowest loss is (near) 0
  when any observed item is noise (on the left) or, on the right, when
  there is a single salient item with \pr 1.0. Maximum loss, of any
  predictor (not just the optimum), never exceeds $-\log(\pmin)$
  ($\approx 4.6$ in this paper, when $\pmin=0.01$),
  and for the optimum here, it is reached when there are
  $k\approx\frac{1}{\pmin}$ salient items, each with max \pr $\approx
  \pmin$.  }
\label{fig:lowest_loss}
\end{figure}

\subsection{Notes on  Pseudocode}
\label{sec:code_description}


We briefly describe how we present pseudocode, such as the examples of
\fig \ref{code:norming_etc} and \ref{fig:lowest_loss}(b). The format
is very similar to Python, \eg using indentation to delimit a block
such as function or loop body, and how we show iteration over the key
and value pairs of a map.  Function names are boldfaced. Importantly,
not all parameters, such as flags, or state variables of a function,
is given in its declaration to avoid clutter, but in the comments
below its declaration, we explain the various variables used. In some
cases, it is useful to think of a function as a method for an object,
in object-oriented programming, with its state variables (or fields),
some represented by data structures such as array or maps.  For a map
$\Q$, $\Q[i]$ is its value for key (item) $i$, but we also use
$\Q$.get($i, 0$) if we want to specify returning 0 when key $i$ does
not exist in the map. $\{\}$ denotes an empty map and $\sm{\Q}$ is the
sum of its values (0 if empty map). We use $\la$ for assignment, and
'//' to begin the comment lines (like C++). We favored simplicity and
brevity at the cost of some loss in efficiency (\eg some loops may be
partially redundant). Not all the functions or the details are given
(\eg for the \nsm of \fig \ref{fig:lowest_loss}), but hopefully enough
of such is presented to clarify how each function can work.  The code,
including several SMAs and a few sequences and evaluation utilities,
in Python, is available on GitHub \cite{dyal_code}.


\subsection{Empirical \logloss  for \ns (Noise) Items}
\label{sec:evalns2}

Once we have an \nsm and the \fcap() function, given any observation
$\oat{o}$, and the predictions $\oat{\Q}$ (where we have applied
\fcap() to get the \sd $\oat{\Q}$) and given the \ns status of $\oat{o}$
(via the \nsmn), we evaluate $\oat{\Q}$ using \llnsr(), and take the
average over the sequence:

\begin{align}
\label{eq:empllns}
\empllns(\seq{\Q}{1}{N}, \seq{o}{1}{N}) = \frac{1}{N} \sum_{t=1}^N
\llnsr(\oat{o}, \oat{\Q}, \markns(\oat{o}) ) 
\end{align}

The scoring rule \llnsr($o, \Q, \markns(o)$) checks if there is a \pr
$p, p > 0$ assigned to $o$, $p=\Q(o)$. If so, we must have $p \ge
\pmin$ (from the definition of \fcap()), and the loss is
$-\log(p)$. Otherwise, if $o$ is also marked \ns (there is agreement),
the loss is $\log(1-\sm{\Q})$. From the definition of \fcap(),
$1-\sm{\Q} \ge \oovprob$. Finally, if $o$ is not marked \nsn, the loss
is $-\log(\oovprob)$. In all cases, the maximum loss is
$-\log(\pns)$. Therefore, if we are interested in obtaining bounded
losses then $\pns$ should be set to a positive value. On the other
hand, if we are interested in learning a \pr in a range down to a
smallest value $p > 0$, then we should set $\pmin$ to a value not
higher than $p(1-\pns)$ so items with \pr as low as $p$ are not
possibly filtered in \fcap() (see also Appendix \ref{sec:pt} on the
threshold for distortion). $\pns$ should be set equal to $\pmin$ in
general: setting it lower unnecessarily punishes predictors for not
predicting below $\pmin > \pns$ (in \llnsr()), and if set above
$\pmin$, then a predictor need not bother learning to estimate around
$\pmin < \pns$ (the cost incurred for predicting noise or \nsn,
$-\log(\pns)$ would be better/lower).  \co{
\footnote{In
synthetic experiments, if the infrequent or \ns items occur only once
in the sequence, \ie have measure 0, which is the case in our
experiments, then $\pns$ need not be constrained by $\pmin$. If we
want to allow "borderline" \ns items, then the \pr of any such item
should not exceed $\pmin > 0$, \ie no non-salient item should have a
\pr higher than the salient, and we should have $\pns \le \pmin$. }
}
In our experiments, $\pmin=\pns=0.01$.

There are two parameters for the \nsmn, the size of the history (or
window size) kept and the count threshold $\nsthrsh$. For simplicity,
no bound on length of the history is imposed in most evaluations (with
an exception in \sec \ref{sec:unix_pairing}). By default, we set
$\nsthrsh=2$ for the \nsmn. This means an item is marked \ns iff it
has been seen $k \le 2$ so far in the test sequence. Thus, to have low
loss, a predictor should provide a (positive) \pr after the 2nd
observation of the item (for the 3rd and subsequent observations).  We
also report comparisons for other values in a few experiments to
assess sensitivity of our comparisons to $\nsthrsh$, and \sec
\ref{sec:unix_pairing} reports on comparisons where the history is
limited to the last 200.  When comparing different predictors,
computing and comparing \empllns() scores, we will often compare on
the same exact sequences (with an identical \nsmn).  We discuss a few
alternative options to bounding \logloss and using a referee in the
appendix, \sect \ref{sec:alts}.


\subsection{The Near Propriety of \llnsr}
\label{sec:llns}

We first formalize
drawing items using a \sd $\P$, which enables us to define taking
expectations wrt a \sd $\P$. For this, it is helpful to define the
operation of multiplying a scalar with a \sdn, or scaling a \sdn, which
is also extensively used in the analysis of approximate propriety,
Appendix \ref{sec:np}.



\setlength{\leftmargini}{15pt}  

\begin{defn} Definitions related to generating sequences using a \sdn,
  and perfect filtering and \nsm wrt to a \sdn:
  \begin{itemize}
\item (scaling a \sdn) Let \sd $\P$ be non-empty, \ie $\sm{\P}\in (0,
  1]$. Then, for $\alpha > 0, \alpha \le \frac{1}{\sm{\P}}$,
  $\P':=\alpha \P$ means the \sd where $\sup(\P'):=\sup(\P)$ and
  $\forall i \in \sup(\P)$, $\P'(i):=\alpha\P(i)$. When
  $\alpha=\frac{1}{\sm{\P}}$, $\alpha \P$ is a \di and one can
  repeatedly draw iid from it.
\item Drawing an item from a non-empty \sd $\P$, denoted $o\sim\P$,
  means
  $\sm{\P}$ of the time, drawing from $\alpha \P$, where
  $\alpha=\frac{1}{\sm{\P}}$ (\sd $\P$ scaled up to a \din),
  and $1-\sm{\P}$ of the time generating a unique noise item.
  Repeatedly drawing items in this manner $N$ times 
  generates a sequence
  $\seq{o}{1}{N}$ using $\P$ ("iid drawing" from $\P$),
  and is denoted $\seq{o}{1}{N}\sim\P$.
\item (perfect marker) Given a non-empty \sd $\P$, a perfect \nsm wrt
  to $\P$, denoted $isNS_{\P}()$, 
  marks an item $i$ as \ns iff $i \not \in \sup(\P)$.
  \end{itemize}
\end{defn}

\co{
\item (prefect filters and perfect markers) Given a non-empty \sd $\P$, a
  perfect \fcap() wrt to $\P$ uses $\pmin=\min(\P)$ and
  $\oovprob=1-\sm{\P}$. A perfect \nsm wrt to $\P$ marks an item $i$
  as \ns iff $i \not \in \sup(\P)$.

Note that with a perfect \fcap() wrt $\P$, $\fcap(\P)=\P$, but there
can be many other \sds $\Q$ with $\fcap(\Q)=\P$.  For a \di $\P$, a
perfect \fcap() is the identity function and a perfect \nsm generates
no \ns markings.
}

Therefore, for a \di $\P$, the perfect \nsm $isNS_{\P}()$ generates no
\ns markings on any sequence $\seq{o}{1}{N}\sim\P$. More generally,
for a \sd $\P$, the perfect marker generates \ns markings at about
$\u{\P}$ fraction of the time on the stream $\seq{o}{}{} \sim
\P$. Appendix \ref{sec:practical} also defines a closer to practical
threshold-marker.


Given a non-empty \sd $\P$, \fcap() and a perfect \nsm $isNS_{\P}()$,
$\llns(\Q|\P)$ is defined as:
\begin{align}
\hspace*{-0.25in}  \llns(\Q|\P) := \expdb{o\sim \P}{\llnsr(o, \Q, isNS_{\P}(o))} \hspace{.1in}
  \mbox{(\logloss for NS)}
\end{align}

Thus \empllns() is the empirical average version of \llns(), where we
use a practical \nsm instead of the perfect one.

\co{
  In almost all experiments, iff an item has been seen $k \le
  \oovthrsh$ so far in the test sequence (unbounded history), it is
  marked \nsn, where $\oovthrsh=2$ by default. Thus, to have low loss,
  a predictor should provide a (positive) \pr after the 2nd
  observation of the item (for the 3rd and subsequent
  observations). \sec \ref{sec:unix_pairing} reports on comparisons
  where the history is limited to the last 200.}

\co{
We show in Appendix \ref{sec:np} that in the iid generation setting
based on \sd $\P$, and when using \fcap() with parameters $\pns$ and
$\pmin$, $\P$ will score highest in several cases, and otherwise, we
provide evidence showing a minimizer of $\llns(\Q|\P)$ \sd $\Q^*$
}

\co{
scoring high is close to $\P$
when the total probability mass in $\P$ assigned to low \pr items
(below or near $\pns$) is small.
}

\co{
For instance, if this mass is $\le 0.1$ (the
rest assigned to items with \pr well above $\pns$ or otherwise
unassigned in $\P$, \ie to pure noise), then a $\Q$ scoring near or
better than $\P$ could have this $0.1$ mass transferred roughly
proportionately 
}



\co{
If we scale only, there is no distortion.
\vspace*{.2in}
\begin{lemma} Scoring with \empllns() is proper but not necessarily strictly
  proper for \sds when $\pmin = 0$, $\alpha \in (0, 1]$, and a perfect
  \nsmn.
\end{lemma}

}
  \co{
If we drop below $\pmin \in (0, 1)$, there is distortion, limited by
$\pmin$:

\begin{lemma}
  For a non-empty \sd $\P$, let $\Q^* = \argmax_{\Q} \expdb{o \sim
    \P}{\llns(\Q, o$, isNS($o$)$ )}$. Then $\forall i \in \I,$ if
  $\P(i) \ge \pmin$, then $\Q^*(i) \ge \P(i)$, and if $\Q^*(i) \ge \pmin$,
  then $\frac{\Q(i)}{\P(i)} \le 1+\pmin$.
\end{lemma}

}


\co{
\vspace*{.2in}
\begin{lemma}
  (Scoring with \empllns() is proper for \sds when using perfect
  filtering and marking) Assume we generate sequences with a non-empty
  \sd $\P$ and use a perfect \fcap() and \ns marker to score, then
  $\P$ is a minimizer of \llns(), and more generally any \sd $\Q$ with
  $\fcap(\Q)=\P$ minimizes \llns().  
\end{lemma}
}

\subsection{Understanding the Behavior of \llsn}

\fig \ref{fig:lowest_loss}(c) shows
the lowest achievable loss, the optimal loss, when using \llsn() in a
few synthetic scenarios: as the \pr of $k \ge 1$ equally likely
salient items is raised from 0 to near $k^{-1}$ each, and where we
assume an item is marked seen only when its true \pr is above $\pmin$
(and $\pmin=\oovprob=0.01$). On the left extreme, there are no salient
items (all \prs at 0 or below $\pmin$), and a predictor achieves 0
loss by predicting an empty map, \ie no salient items
($\sm{\Q}=0$). Here, both the \nsm and the predictor agree that every
observed item is noise.  Next, consider the case of one (salient) item
as its \pr is raised from $\pmin$. At the right extreme, where its
probability nears 1.0, the minimum achievable loss approaches 0 as
well. Note that due to the capping, we never get 0 loss (with
$\pns=0.01$, we get $-\log(0.99)\approx 0$). Midway, when $\tp$ is
around 0.5, we get the maximum of the optimum curve, yielding \llnss
around $-\log(0.5)$ (the salient item is observed half the time, and
half the time both the \nsm and predictor agree that the item observed
is noise, with $p_{noise}=0.5$). Similarly, for the scenarios with $k
> 1$ items, the maximum optimal loss is reached when the items are
assigned their maximum allowed probability, \ie
$\frac{1}{k}$.


Appendix \ref{sec:np} gives other concrete examples of $\llns(\Q|\P)$
evaluation and develops some of its properties.
For example, the minimizer of $\llns(\Q|\P)$ may not be $\P$ (and may
not be unique).
We show that any minimizer $\Q^*$ cannot deviate from $\P$ much, and has a
certain form: certain low \pr items of $\P$ can be zeroed in $\Q^*$,
and their mass {\em proportionately} spread onto other (higher \prn)
items. This proportional spread implies low deviation when the total
probability mass $p$ of the low \pr items of $\P$ (near or below
$\pmin=\pns$) is small. For instance, if the total such mass $p$ is
$0.1$ in $\P$, then for any minimizer $\Q^*$, the \pr of any (salient)
item $i$ in $\Q^*$, including the unassigned $\u{\Q^*}$, cannot
deviate from its corresponding \pr in $\P$ by more than about $10\%$
in a ratio sense: $\frac{\Q^*(i)}{\P(i)}\le 1.1$ (in general
$\frac{1}{1-p}$, see \ref{sec:extent} and \ref{sec:practical}) . Note
that comparators using $\llns(\Q|\P)$ would prefer a minimizer $\Q^*$
over other $\Q$ with high likelihood (would score it better, given a
sufficiently large evaluation sample), motivating our focus on
characterizing $\Q^*$.

\co{
As another example, assume the true probability \di is $\P=\{A:0.78,
B:0.02\}$, thus $0.2$ of the time, a noise item is generated. Then
with $\pmin=\oovprob=0.01$ in \fcap(), $\P$ scores lowest at
$\llsn(\P|\P)=-0.78\log(0.78)-0.02\log(0.02)-0.2\log(0.20)\approx
0.594$, while with $\Q=\{A:0.8, B:0.02\}$,
$\llsn(\Q|\P)=-0.78\log(0.8)-0.02\log(0.02)-0.2\log(0.18)\approx
0.595$, or $\llsn(\Q|\P) > \llsn(\P|\P)$. Let $\Q_2=\{A:0.78\}$, then
$\llsn(\Q_2|\P)=-0.78\log(0.78)-0.02\log(0.01)-0.2\log(0.22)$.
}

\subsubsection{Visualizing Underlying \sds with Venn Diagrams}
\fig \ref{fig:lowest_loss}(b) shows a useful way to pictorially
imagine a \sd $\P$ that is generating the observation sequence as a
Venn diagram. Salient items (of $\P$) are against a background of
noise, and when we change $\P$, for instance increase the \pr of a
salient item \itm{A}, the area of its corresponding blob increases,
and this increase can be obtained from reducing the area assigned to
the background noise (as in \fig \ref{fig:lowest_loss}(b) and the
experiments of \fig \ref{fig:lowest_loss}(c)), or one or more other
salient item can shrink in area (while total area remaining constant
at 1.0). When drawing, each item, including a noise item, is picked
with probability proportional to its area.

\subsection{Evaluating a Space-Bounded Predictor under Nonstationarity}
\label{sec:evalns3}

We can use the same empirical evaluation \logloss formula of
\empllns() (Eq. \ref{eq:empllns}) in the idealized non-stationary
case, and as long as each subsequence is sufficiently long, a
convergent prediction method should do well. In synthetic experiments,
we try different minimum frequency requirements, $\minobs$: an item
needs to occur $\minobs$ times in a sequence before its probability
can be changed. The lower the underlying $\tp$ the more time points
required before $\tp$ is changed, as we expect one positive
observation every $\frac{1}{\tp}$ time points. Note that if $\minobs$
is set too low, \eg below 2 or 3, this would not allow sufficient time
for learning a \pr well, and the underlying $\tp$ loses its meaning.
We want a predictor that tracks changes in salient \prs well: a high
\pr item or items (\eg above 0.1) may change in \pr (\eg disappear),
while a low \pr item (\eg 0.03) may remain stable for a long time, or
vice versa. The proportion of the mass that should be allotted to the
noise portion of the stream
can change too (\eg from 0.1 to 0.5, then to 0.01, etc.) and good
predictors should track such too (and, for example, not allocate all
their mass onto the salient items).  A good evaluation technique
should score such honest and adaptive predictors well (see Appendix
\ref{sec:practical}). We note that evaluating proper \pr assignment to
the noise portion of the stream is applicable whether we use \logloss
or \brier loss.


With real-world sequences, a variety of phenomena such as periodicity
and other dependencies can violate the iid assumptions, and it is an
empirical question whether, on real sequences, the predictors compare
as anticipated based on their various, presumed or expected, strengths
and weaknesses stemming from consideration of the ideal setting of \sec
\ref{sec:ideal}. This underscores the importance of empirical
experiments on different real-world sequences.




\section{The (Sparse) EMA Predictor}
\label{sec:ema}

Our first \sma is the sparse exponentiated moving average (EMA), which
gets its name from exponentially decreasing multipliers over the past
values (as a function of time, see part (a) of the comparison picture
\fig \ref{fig:histories}).  In past work we developed a variant
specially designed for multiclass ranking and classification,
where we observed its advantages
for non-stationary tasks \cite{updateskdd08,ijcai09}. Here we develop
and analyze sparse EMA, focusing on probability prediction. \fig
\ref{fig:ema} presents pseudo code.  The \sma keeps a map, of item to
\prn, and it can be shown that the map is always a \sd
\cite{updateskdd08}. Initially, at $t=1$, the map is empty. Each
map-entry can be viewed as a connection, or a prediction relation, a
weighted directed edge, from the predictor to a predictand, where the
weight is the current \pr estimate for the predictand. Prediction is
straightforward: output the map's key-value (or key-\prn) pairs.  An
EMA update is a convex combination of the present observation with the
past (running) average.\footnote{EMA (also, EWMA) and other moving
averages can also be viewed as a type of filter in signal theory and a
(time) convolution \cite{wikip1}, and find a variety of applications,
such as in time-series analysis and forecasting (\eg
\cite{intro_time_series2018,coherent_forecasting2015}), and in
reinforcement learning for value function updates (in temporal
difference learning, \cite{rl1,wikip1}). } When learning proportions
as in this paper, the observation is either 0 or 1, and the update can
be broken into two phases, first a weakening of the weights, of edges
to existing predictands, which can be seen as \prs flowing from
existing edges into an implicit (available) \pr reserve, and then a
strengthening of the edge to the observed item, \pr flowing from the
reserve to the target edge, as pictured in \fig
\ref{fig:ema_phases}. This picture is useful when we extend EMA (\sec
\ref{sec:qdyal}). We next describe \pr and \sd estimation and their
challenges, in particular speed of convergence \vs variance, and
handling non-stationarity (versions of plasticity \vs stability
trade-offs \cite{Abraham2005MemoryR,fpsyg2013})\footnote{A similar
trade-off has been termed drift erorr \vs statistical (variance) error
in computational learning theory (\eg
\cite{driftingDistros2024,drifting_distros_2012}).}, motivating
extensions of plain static (fixed-rate) EMA (\sec \ref{sec:harm}),
further developing techniques used later in \sec \ref{sec:qdyal}.


\co{

plain EMA ("Emma"), pictures of its two phases, semi distro, what we
know...

NOTE: the two phases are important/useful for understanding \qd later
...

A few detailed examples of change:

\begin{itemize}
\item new concepts need to change their estimates fast: harmonic decay
(start high, lower the rate) can do that.. but insufficient..

\item .. how about existing concepts to new concepts:

\item plasticity vs stability: we need a low rate for good probability
estimates (if we want to lower distortion), but we may have too slow
of a convergence ..

\item Imagine the binary case, a true probability being sampled iid ...

\item --> case for low rate: the rate determines the fraction of times (time
points) that the estimate substantially deviates from the the true
probability.. once converged we prefer low rate

\item --> high rate: allows to quickly converge, specially for high
probabilities.. (agile)

\end{itemize}

}

\begin{figure}[t]
  \hspace*{.1in}\begin{minipage}[t]{0.55\linewidth}
  {\bf EmaUpdate}($o$) // latest observation $o$. \\
  \hspace*{0.3cm} // $\emamap$ is item to \prn, learning rate $\lr \in (0, 1]$. \\ 
  \hspace*{0.3cm} // Other param: $doHarmonicDecay$ flag. \\
  \hspace*{0.3cm} For each item $i$  in $\emamap$: \\
        \hspace*{0.6cm}    $\emamap[i] \leftarrow (1 - \lr) * \emamap[i]$ // Weaken.\\
  \hspace*{0.3cm} // Strengthen edge to $o$ (insert edge if not there).\\
  \hspace*{0.3cm}$\emamap[o] \leftarrow \emamap[o] + \lr$ // Boost. \\
  \hspace*{0.3cm}If $doHarmonicDecay$: // reduce rate?\\
  \hspace*{0.6cm} $\lr \leftarrow DecayRate(\lr)$ // see (b) and \sec \ref{sec:harm} \\ \\
 \hspace*{1cm} (a) Plain EMA updating.
\end{minipage}
\begin{minipage}[t]{0.55\linewidth}
        {\bf DecayRate}($\lr$)  \\
        \hspace*{0.3cm} // Other parameters: $\lrmin \in (0,1)$, is the \\ 
        \hspace*{0.3cm} // minimum allowed learning rate. \\
        \hspace*{0.3cm} // $\lrmax \in (0,1]$, is the maximum and \\
    \hspace*{0.3cm} // the initial learning rate for EMA with \\
    \hspace*{0.3cm} // harmonic decay. \\
 \hspace*{0.3cm}     $\lr \leftarrow \frac{1}{\frac{1}{\lr}+1.0}$ // harmonic decay.\\
 \hspace*{0.3cm}     Return $\max(\lr, \minlr)$ \\ \\
 \hspace*{1cm} (b) Harmonic decay of rate.
\end{minipage}
\caption{Pseudo code of (a) sparse EMA with a single learning rate $\lr$,
  either fixed ("static" EMA), or (b) decayed with a harmonic schedule
  down to a minimum $\lrmin$ ("harmonic" EMA, \sec \ref{sec:harm}).
  The working of an EMA update can be split in two steps (\fig
  \ref{fig:ema_phases}): 1) weaken, \ie weaken all existing edge
  weights (entries of the map $\emamap$), and 2) strengthen, \ie boost (the weight of)
  the edge to the observed item (target).  The map entry is created if
  it doesn't already exist (edge insertion). Initially (at $t=1$),
  there are no edges.  EMA enjoys a number of desirable properties,
  such as the probabilities in $\emamap$ forming a \sdn, and approximate
  convergence (\sect \ref{sec:ema}).  }
\label{fig:ema}
\end{figure}



\begin{figure}[t]
\begin{center}
  \centering
  \subfloat[First phase of EMA: weaken all edges.]{{\includegraphics[height=4cm,width=6cm]{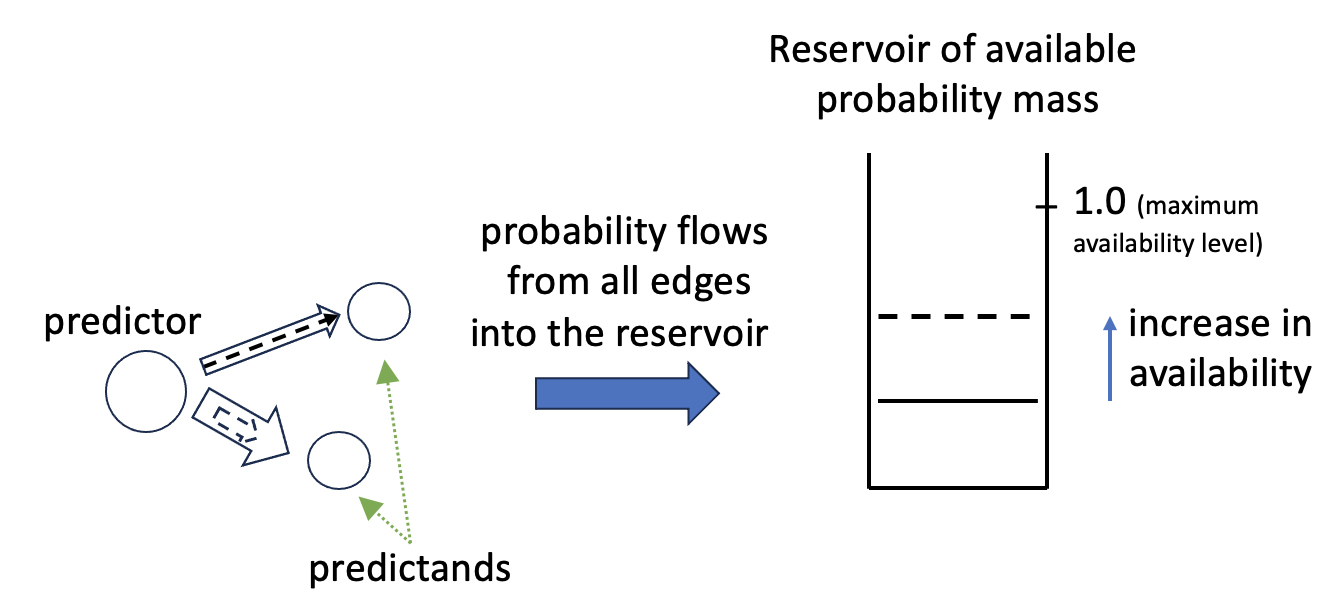}
  }}
\hspace*{.51in}  \subfloat[Second phase: boost weight of target.]{{\includegraphics[height=4cm,width=6cm]{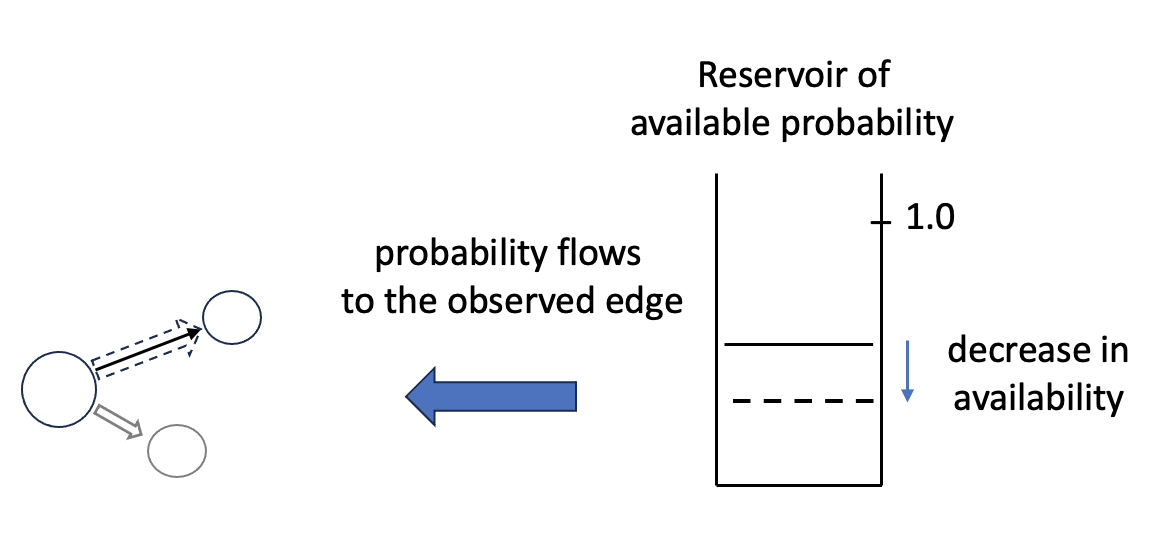}
  }}
\end{center}
\vspace{.2cm}
\caption{An EMA update can be seen as having two phases: {\em
    weaken-and-boost}. In the first phase, all existing edges from the
  predictor to the predictands (entries in the map) are weakened,
  which can be viewed as probability flowing from the edges to the
  reservoir of (unused or available) \pr mass (the edges, after the
  weakening, reduced in weight, are shown smaller and dotted). In the
  2nd phase, the edge to the observed item is strengthened. The
  reservoir is implicit: with \sd $\Q$ corresponding to the map, the
  reservoir has \pr mass $1-\sm{\Q}$. Initially (at $t=1$), with an
  empty map (no edges), the reservoir is full at 1.0. }
\label{fig:ema_phases}
\end{figure}

\co{
motivate extensions to EMA that better handle certain types of
non-stationarity.

We next motivate extensions to EMA that better handle certain types of
non-stationarity.
}

\begin{figure}[t]
\hspace*{.2in}
  \subfloat[EMA convergence under different (fixed) rates, $\lr=0.01$
    and $\lr=0.001$ (binary iid setting).  Higher rates can yield faster convergence, but also 
    higher variance.]  {{\includegraphics[height=4.5cm,width=6.5cm]
      {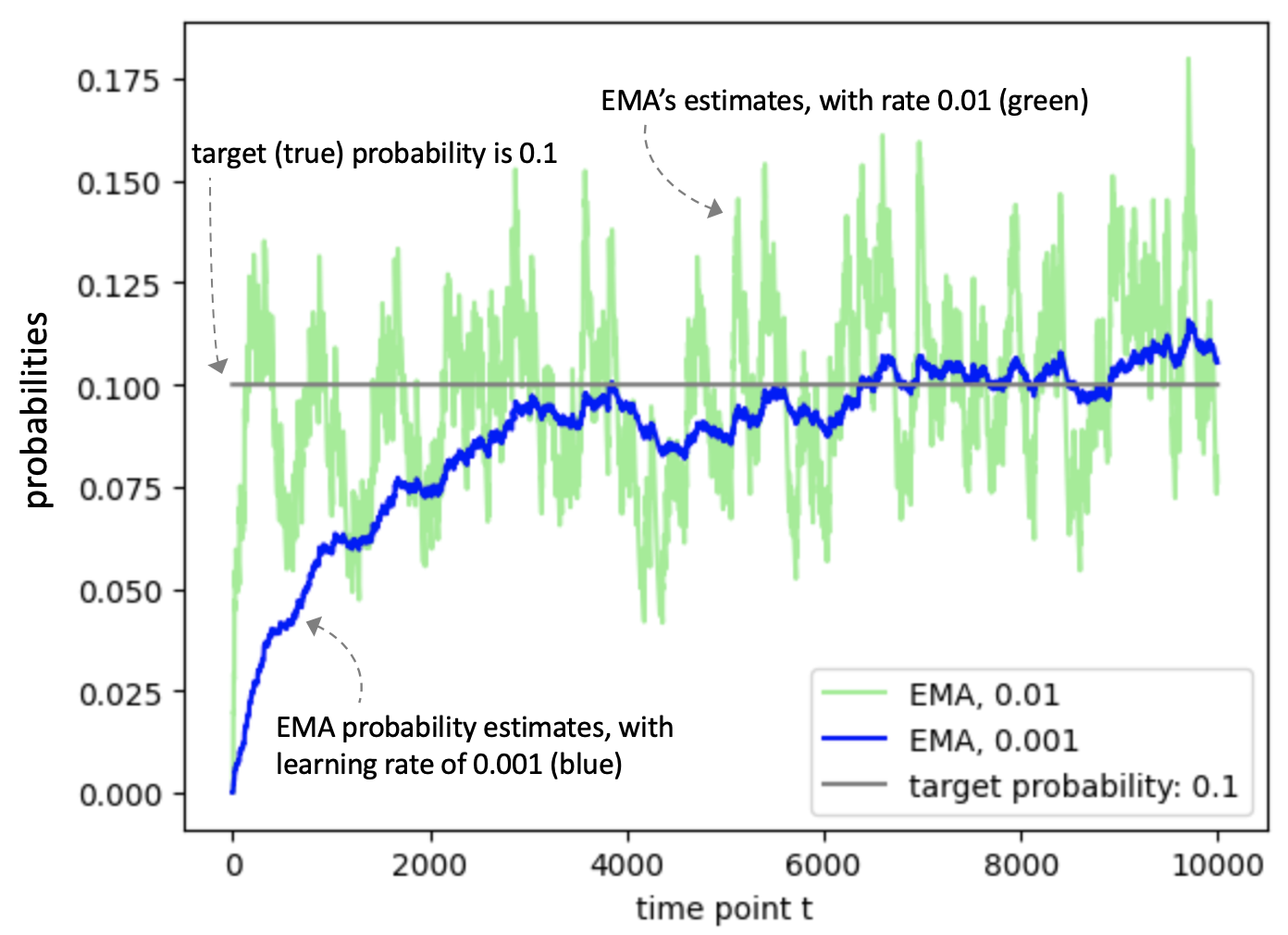} }}
  \hspace*{.21in}  \subfloat[Harmonic decay of rate speeds convergence
for higher target \prs such as $\tp=0.1$    (\vs static EMA).
  ]{{\includegraphics[height=4.5cm,width=6.5cm]{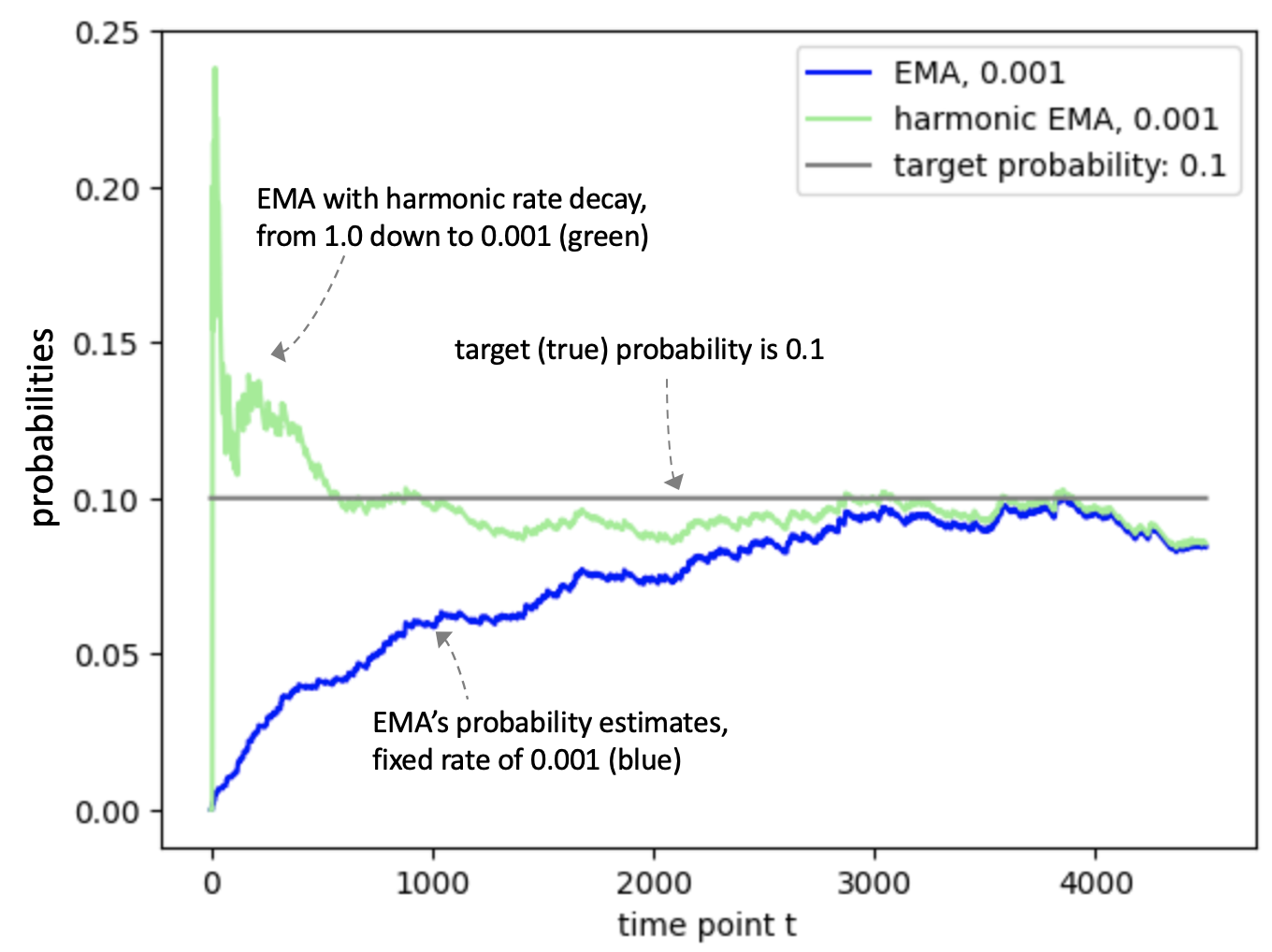}
  }}
\vspace{.2cm}
\caption{Given a single binary sequence, 10k long, with target
  $\tp=0.1$ (the true \pr of observing 1), convergence of the \pr
  estimates, $\oat{\ep}$, by EMA variants: (a) for the static or
  fixed-rate version of EMA, under two different learning rates,
  $\lr=0.01$ and $\lr=0.001$. A higher rate, $\lr=0.01$, can lead to
  faster convergence, but also causes high variance. (b) With harmonic
  decay, the estimates converge faster (possibly with a high-variance
  initial period). }
\label{fig:ema_convergence}
\end{figure}

\co{
\begin{table}[t]\center
  \begin{tabular}{ |c|c|c|c|c| }     \hline
    & \multicolumn{2}{c|}{ $\lr=\tp/5.0$ \co{$\frac{\tp}{5}$} } &
    \multicolumn{2}{c|}{ $\lr=\tp/20.0$ \co{$\lr=\frac{\tp}{20}$} } \\
     \cline{2-5}
     (deviation threshold) $\dvt \rightarrow$  &  1.5 & 2.0 & 1.5 & 2.0  \\ \hline
 $\tp=0.1$  & 47, 0.19 $\pm 0.02$  & 35, 0.04$\pm 0.01$ & 213, 0.01$\pm 0.01$ & 140, 0.001$\pm 0.002$ \\ \hline
 $\tp=0.01$ & 520, 0.20$\pm 0.08$ & 350, 0.05$\pm 0.04$ & 2200, 0.04$\pm 0.05$ & 1400, 0.02$\pm 0.03$ \\ \hline
  \end{tabular}
  }

\begin{table}[t] \center
  \begin{tabular}{ |c|c|c|c|c| }     \hline
     deviation threshold $\rightarrow$ &
    \multicolumn{2}{c|}{ $\dvt=1.5$ \co{$\lr=\frac{\tp}{20}$} } &
    \multicolumn{2}{c|}{ $\dvt=2.0$ \co{$\lr=\frac{\tp}{20}$} }  \\ \hline
      learning rate $\rightarrow$   &  $\lr=\tp/5$ & $\lr=\tp/20$
    & $\lr=\tp/5$ & $\lr=\tp/20$  \\ \hline
    $\tp=0.1$  & 47, 0.19 $\pm 0.02$  
    & 213, 0.01$\pm 0.01$
    & 35, 0.04$\pm 0.01$ 
    & 140, 0.001$\pm 0.002$ \\ \hline 
    $\tp=0.01$ & 520, 0.20$\pm 0.08$
    & 2200, 0.04$\pm 0.05$
    & 350, 0.05$\pm 0.04$
      & 1400, 0.02$\pm 0.03$ \\ \hline
  \end{tabular}
  \vspace*{.1in}
  \caption{10k-long sequences, 500 trials, for two target
    probabilities $\tp$, $\tp=0.1$ and $\tp=0.01$, and for two $\lr$
    settings for EMA, $\frac{\tp}{5}$ and $\frac{\tp}{20}$, the first
    time EMA's estimate $\oat{\ep}$ reaches to within deviation
    threshold $\dvt$ (averaged over the trials), and the
    deviation-rate thereafter (and the std over the 500 trials) are
    reported. The initial estimate is 0 ($\ott{\ep}{1}=0$). For
    instance, when $\tp=0.1$, EMA with $\lr=\frac{\tp}{5}=0.02$ takes
    47 time points on average to reach within $\dvt=1.5$ of $\tp=0.1$
    and thereafter it has deviation-rate of 19\% (while with
    $\lr=\frac{\tp}{20}$, it has only 1\% deviation-rate). A higher
    rate leads to faster convergence for EMA, but also higher
    variance. }
  \label{tab:ema_convergences}
\end{table}











\subsection{Convergence}
\label{sec:convergence}


It can be shown that the EMA update, with $\lr \in [0, 1]$, preserves
the \sd property: when the \pr map before the update corresponds to a
\sd $\oat{\Q}$, $\ott{\Q}{t+1}$ (the map after the update) is a \sd
too.  Furthermore, EMA enjoys several convergence properties under
appropriate conditions, \eg the sum of the edge weights (map entries)
increases, converging to 1.0 (\ie to a \din), and an EMA update can be
seen to be following the gradient of quadratic loss
\cite{updateskdd08}. Here, we are interested in convergence of the map
weights as individual \pr estimates.



In the stationary iid setting (\sec \ref{sec:ideal}), we show that the
\pr estimates of EMA converge, probabilistically, to the vicinity of
the true \prsn.  We focus on the EMA estimates for one item or
predictand, call it \itm{A}: at each time point $t$, a positive update
occurs when \itm{A} is observed, and otherwise it is a negative update
(a weakening), which leads to the stationary binary setting (\sec
\ref{sec:binary}), where the target \pr is denoted $\tp$. The
estimates of EMA for item \itm{A}, denoted $\oat{\ep}$, form a random
walk. While often, \ie at many time points, it can be more likely that
$\oat{\ep}$ moves away from $\tp$, such as when $\ep \le \tp < 0.5$,
when $\ep$ does move closer to $\tp$ (upon a positive observation),
the gap reduces by a relatively large amount. The following property,
in particular, helps us show that the expected gap, $|\tp - \ep|$,
shrinks with an EMA update, under an appropriate condition. It also
connects EMA's random walk, the evolution of the gaps for instance, to
discrete-time martingales, fundamental to the analysis of probability
and stochastic processes (\eg the gap is a supermartingale when
$\oat{\ep} < \tp$) \cite{martingales91,probgames1}.\\

\co{ the probability that the gap between $\ep$ and $\tp$ expanding
  may be more likely than shrinking, such as when $\ep \le \tp < 0.5$,
  when there is a progress, for instance, when there is a positive
  update and $\ep$ moves closer to $\tp$, the gap reduces
  substantially.  }

\begin{lemma} \label{lem:ema_shirnk}
  EMA's movements, \ie changes in the estimate $\oat{\ep}$,
  satisfy
  the following properties, where $\lr \in [0, 1]$:
\setlength{\leftmargini}{15pt} 
  \begin{enumerate}
  \item Maximum movement, or step size, no more than $\lr$: $\forall t,
    |\ott{\ep}{t+1}-\oat{\ep}| \le \lr$.
  \item Expected movement is toward $\tp$: Let $\oat{\Delta} := \tp - \oat{\ep}$. Then,
     $\expd{\Delta^{(t+1)}|\oat{\ep}=p} =
    (1-\lr)(\tp-p) = (1-\lr) \oat{\Delta}$.
  \item Minimum expected progress size: With $\ott{\delta}{t} :=
    |\oat{\Delta}|-|\ott{\Delta}{t+1}|$, $\expd{\oat{\delta}}\ge \lr^2$
    whenever $|\oat{\Delta}|\ge\lr$ (\ie whenever $\ep$ is sufficiently
    far from $\tp$).
  \end{enumerate}
\end{lemma}

\co{
\begin{lemma} EMA's movements, \ie changes in the estimate $\oat{\ep}$,
  enjoy the following two properties, where $\lr \in [0, 1]$:
  \label{lem:ema_shirnk}
\setlength{\leftmargini}{15pt} 
  \begin{enumerate}
  \item Maximum movement, or step size, no more than $\lr$ change: $\forall t,
    |\ott{\ep}{t+1}-\oat{\ep}| \le \lr$.
  \item Expected movement: Let $\oat{\Delta} = \tp - \oat{\ep}$. Then,
     $\expd{\Delta^{(t+1)}|\oat{\ep}=p} =
    (1-\lr)(\tp-p) = (1-\lr) \oat{\Delta}$.
  \end{enumerate}
\end{lemma}
}
\co{

  \begin{proof}
  (proof of the first part) On a negative update,
  $\oat{\ep} - \ott{\ep}{t+1} = \oat{\ep} - (1-\lr)\oat{\ep} =\lr\oat{\ep}\le \lr$, and on a positive update,
  $\ott{\ep}{t+1}-\oat{\ep} = (1-\lr)\oat{\ep}+\lr-\oat{\ep}=\lr-\oat{\ep}\lr\le \lr$ (as $\oat{\ep}\in[0,1]$).

  (proof of the second part) We write the expression for the expectation and simplify:
 $\tp$ of the time, we have a positive update, \ie both weaken and
  boost ($(1-\lr)p+\lr$), and the rest, $1-\tp$ of the time, we have
  weaken only ($(1-\lr)p$). In both cases, the term $(1-\lr)p$, is
  common and is factored:
  \begin{align*}
    \hspace*{-.4in}\expd{\Delta^{(t+1)}|\oat{\ep}=p} & = \tp\left(\tp - ((1-\lr)p+\lr) \right)+(1-\tp)\left(\tp-(1-\lr)p)\right) \\
    &= (\tp+(1-\tp))(\tp - (1-\lr)p) - \tp\lr \hspace*{.3in} \mbox{
      ( $\tp - (1-\lr)p$ is common and is factored )}\\
    &= \tp - (1-\lr)p - \tp\lr = \tp(1-\lr) - (1-\lr)p  \\
    &= (1-\lr)(\tp-p)
  \end{align*}
\end{proof}

}

The lemma is established by writing down the expressions, for an EMA
update and the gap expectation, and simplifying (see Appendix
\ref{app:ema}). It follows that
$\expd{\ott{\ep}{t+1}|\oat{\ep}=\tp}=\tp$ (or
$\expd{\ott{\Delta}{t+1}|\oat{\Delta}=0}=0$, and the {\em expected
  direction} of movement is always toward $\tp$. The above property
would imply that a positive gap should always shrink in expectation,
\ie $\expd{|\ott{\Delta}{t+1}|} < |\oat{\Delta}|$ when
$|\oat{\Delta}|> 0$, but there is an exception when $\ep$ is close to
$\tp$, \eg when $\ep=\tp$. This also occurs with a high $\lr=1.0$
rate: the expectation of distance $|\ott{\Delta}{t+1}|$ often expands
when $\lr=1.0$.  However, it follows from the maximum movement
property, that when $\ep$ is sufficiently far, at least $\lr$ away
from $\tp$, the sign of $\oat{\Delta}$ does not change, and property 2
implies that the gap is indeed reduced in expectation (property 3),
which we can use to show probabilistic convergence:

\co{
, within a bounded number of time
points, to within $\lr$ of $\tp$:


or more
}

\begin{theorem}
  \label{thm:converge}
  EMA, with a fixed rate of $\lr \in (0, 1]$, has an expected
    first-visit time bounded by $O(\lr^{-2})$ to within the band
    $\tp\pm\lr$.  The required number of
    updates, for first-visit time, is lower bounded below by
    $\Omega(\lr^{-1})$.
\end{theorem}

\co{
  
\begin{proof}
We are interested in maximum of first-time $k$ when the expected
$\expd{\ott{\ep}{k}} \in [\tp-\lr, \tp+\lr]$.  Using the maximum
movement constraint, as long as $|\tp - \oat{\ep}| > \lr$, an EMA
update does not change the sign of $\tp - \oat{\ep}$ ($\ep$ does not
switch sides wrt $\tp$, \eg if greater than $\tp$, it remains greater
after the update).  Before an estimate $\oat{\ep} < \tp$ changes
sides, and exceed $\tp$, it has to come within the band
$\tp\pm\lr$. Therefore, start with an arbitrary location
$\ott{\ep}{1}$ outside the band, say $\ott{\ep}{1} < \tp - \lr$
(similar arguments apply when $\ott{\ep}{1} > \tp + \lr$), and
consider the sequence, $\ott{\ep}{1}, \ott{\ep}{2}, \ott{\ep}{3},
\cdots, \ott{\ep}{k}$, where $\forall t, 1\le t \le k, \ott{\ep}{t} <
\tp-\lr$.  We now show each time step is a positive progress, in
expectation, towards the band $\tp\pm\lr$ (a step up), in particular
$\expd{\ott{\ep}{t+1}-\ott{\ep}{t}}\ge \lr^2$:
\begin{align*}
\hspace*{-.2in}\expd{\ott{\ep}{t+1}-\ott{\ep}{t}}&=
\expd{\tp-\ott{\ep}{t}-(\tp-\ott{\ep}{t+1})} =
\expd{\tp-\ott{\ep}{t}} - \expd{\tp-\ott{\ep}{t+1}}
\mbox{\ \ (linearity of expectation)} \\
&= \expd{\tp-\ott{\ep}{t}} - (1-\lr)\expd{\tp-\ott{\ep}{t}}
= \lr \expd{\tp-\oat{\ep}} \mbox{\ \  (use of Lemma \ref{lem:ema_shirnk}, property 2)} \\
& \ge \lr^2 \mbox{\ \  (from our assumption on the sequence:
  $\forall t\le k, \ott{\ep}{t} \le \tp-\lr$)}
\end{align*}

We can now lower bound the expected position of $\ott{\ep}{k}, k\ge 2$
wrt $\ott{\ep}{1}$, to be at least $(k-1)\lr^2$ above $\ott{\ep}{1}$:

\begin{align*}
\expd{\ott{\ep}{k} - \ott{\ep}{1}| \ott{\ep}{1}=p} &=
\expd{\ott{\ep}{k} - \ott{\ep}{k-1} + \ott{\ep}{k-1} - \ott{\ep}{1} | \ott{\ep}{1}=p} \\
&= \expd{\sum_{2\le t\le k}\ott{\ep}{t} - \ott{\ep}{t-1} | \ott{\ep}{1}=p}
 \mbox{\ \  (insert all intermediate sequence members)} \\
&= \sum_{2\le t\le k} \expd{\ott{\ep}{t} - \ott{\ep}{t-1} | \ott{\ep}{1}=p} \ge (k-1)\lr^2,
\end{align*}
where we used the linearity of expectation, and the $\lr^2$ lower
bound for the last line.  With $\tp\le 1$, an upper bound of
$\frac{1}{\lr^2}$ on maximum first-visit time follows: $k$ cannot be
larger than $\frac{1}{\lr^2}$ if we want to satisfy $\forall 1\le t\le
k, \expd{\ott{\ep}{t}}< \tp-\lr$, or one of $t \in \{1, \cdots,
\frac{1}{\lr^2}+1\}$ has to be within the band $\tp\pm\lr$.

The lower bound $\Omega(\lr^{-1})$ on $k$ follows from the upperbound
of $\lr$ on any advancement towards $\tp$.
\end{proof}

}

The proof, presented in Appendix \ref{app:ema}, works by using
property 3 that expected progress toward $\tp$ is at least $\lr^2$
while our random walker $\oat{\ep}$ is outside the band
($|\tp-\oat{\ep}|>\beta$).  It would be good to tighten the gap
between the lower and upper bounds. Table \ref{tab:ema_convergences}
suggests that the upper bound may be subquadratic, perhaps linear
$\Theta(\lr^{-1})$. We expect that one can use techniques such as
Chernoff bounds to show that the number of steps to a first time visit
of the bound is also bounded by $\lr^{-2}$ with very high probability,
and that there exists a stationary distribution for the walk
concentrated around $\tp$. We leave further characterizing the random
walk to future work.

\co{
Therefore, we can now use the second part of Lemma
\ref{lem:ema_shirnk}, on expected movement, to lower-bound the
expected number of steps (advancements) towards the band $\tp\pm\lr$.

Assume $\ott{\ep}{1} \le \tp - \lr$, \eg $\ott{\ep}{1}=0$ (similar
arguments apply when $\ott{\ep}{1} > \tp + \lr$).
}



\co{
Convergence to the target is probabilistic and approximate (to within
a $\dvt$, given appropriate $\lr$), and in this paper we empirically
show the convergence for some values of $\tp$ and $\lr$.  \fig
\ref{fig:ema_convergence} shows the convergence of EMA under two
different rates, $\lr=0.01$ and $\lr=0.001$, on a single binary
sequence with $\tp=0.1$.  Table \ref{tab:ema_convergences} shows first
time convergences and deviation rates for a few choices of $\lr$ and
for two different $\tp$ values. One notes that convergence is roughly
linear in $\lr^{-1}$.
}



\subsection{Convergence Speed \vs Accuracy Tradeoff}

\co{
Consider the case when a new item starts appearing in the stream with
probability $\tp$ (\eg, $\tp=0.1$), \ie a change in probability from 0
(so far) to $\tp$ (from now on). This is also the situation that any
predictor faces initially, starting with an empty map, as it learns to
predict items in its stream.  We now derive a rough lower estimate on
the number of positives and total observations needed for the EMA
estimate, to get close to $\tp$ (\eg within a factor of 2), starting
from $\ep=0$.  We are interested in $\tp \ge 0.01$, and assume we use
$\lr \ll \tp$ (\eg $\lr < \tp/k$, with $k\ge 2$).
Each positive observation leads to an advancement in \pr estimate
$\ep$ close to $\lr$ towards the target $\tp$. Note that as $\ep$ gets
larger, the advancement delta, $\lr(1- \ep)$, is reduced, going to 0
when the estimate is 1. A weakening corresponds to a proportional
reduction (by a fixed factor $1-\lr$). We are ignoring the weakening
step, $1-\lr$, in deriving this lower estimate. Thus the number of
positive outcomes (observations) needed to get close to the target
probability, starting from 0, is at least $\frac{\tp}{\lr}$, and the
expected {\em total} number of observations (time points) required
(irrespective of negative or positive), since a positive is observed
every $\frac{1}{\tp}$ times on average, is thereby at least
$\frac{\tp}{\lr}\frac{1}{\tp}=\frac{1}{\lr}$ (or
$\Omega(\frac{1}{\lr}$)). Interestingly, this lower bound is
independent of $\tp$: the number of positives required is dependent,
but the total time points required is not, and only depends on $\lr$,
which is consistent with the realization that while a higher $\tp$ is
more distant from 0 compared to a smaller $\tp$, the number of
positive updates will be higher for higher $\tp$ as well.  In any
case, as may be expected, the smaller the rate $\lr$, the longer it
takes for $\ep$ to converge (get close).  A higher rate $\lr$, \eg
$0.01$, require 10 times fewer observations than $0.001$, or leads to
faster convergence, for higher target \prs (with $\tp \gg \lr$). As
seen Table \ref{tab:ema_convergences} suggests, this lower bound may
be exact. However, a higher rate leads to a higher variance as well,
as discussed next.
}






It follows from the lower bound, $\Omega(\frac{1}{\lr})$, that the
smaller the rate $\lr$ the longer it takes for $\ep$ to get close to
$\tp$.  A higher rate $\lr$, \eg $0.01$, require 10 times fewer
observations than $0.001$, or leads to faster convergence, for higher
target \prs (with $\tp \gg \lr$). On the other hand, once sufficiently
near the target \pr $\tp$, we desire a low rate for EMA, to keep
estimating the probability well (\fig
\ref{fig:ema_convergence}). Alternatively, a high rate would cause
unwanted jitter or variance. Consider when $\lr \gtrapprox \tp$, \ie
the $\lr$ is itself near or exceeds the target $\tp$ in
magnitude. Then when the estimate is also near, $\ep \approx \tp$,
after an EMA positive update, we get $\ep \gtrapprox 2\tp$, or the
relative error shoots from near 0 to near 100\% (an error with the
same magnitude as the target being estimated).\footnote{When the
estimate is at target, $\oat{\ep}=\tp$, the only situation when there
is no possible movement away from $\tp$ is at the two extremes when
$\tp=1$ or $\tp=0$ ($\ott{\ep}{t+1}=\tp$ for any $\lr$, if $\tp\in
\{0,1\}$ and $\ott{\ep}{t}=\tp$).}

Once $\ep$ is sufficiently near, we can say we desire stability,
better achieved with lower rates. At an extreme, if we knew that $\tp$
would not change, and we were happy with our estimate $\ep$, one could
even set $\lr$ to $0$. In situations when we expect some
non-stationarity, \eg drifts in target \prsn, this is not wise.  A
rule of thumb that reduces high deviation rates while being sensitive
to target \pr drifts, is to set $\lr$ to, say, $\frac{\tp}{10}$, so
that the deviation rate at $\dvt=1.5$ is no more 10\% (some acceptable
percentage of the time). As we don't know $\tp$, if we are interested
in learning \prs in a range, \eg $[0.01, 1.0]$, and we are using plain
EMA, then we should consider setting $\lr$ to $\frac{0.01}{10}$ (a
function of the minimum of the target \pr range).




\subsection{From $\Omega(\lr^{-1})$ to $O(\frac{1}{\tp})$, via Rate Decay }
\label{sec:harm}








The above analysis tells us that when we want to learn target \prs in
a diverse range, such as $[0.01, 1.0]$, and when using plain
fixed-rate (static) EMA, we need to use a low rate to make sure
smaller \prs are learned well, which sacrifices speed of convergence
for larger $\tp$ for accuracy on smaller ones: a relatively large
target, say $\tp > 0.1$, requires 100s or 1000s of time points to
converge to an acceptable deviation-rate with a low $\lr \approx
0.001$,
instead of 10s, with
$\lr \lessapprox 0.01$.  This motivates considering alternatives, such as EMA with a changing
rate. A variant of EMA, which we will refer to as {\em harmonic EMA},
has a rate decaying over time with each update, in a {\em harmonic}
manner.  the rate $\oat{\lr}$ starts at a high value
$\ott{\lr}{1}=\lrmax$ (\eg $\lrmax=1.0$), and is reduced gradually
with each (positive or negative) update, via: 
\begin{align*}
  \hspace*{-.25in}
  \ott{\lr}{t+1}\leftarrow \left(\frac{1}{\oat{\lr}}+1\right)^{-1}
  \mbox{\ \ (the harmonic decay of the learning rate: a double reciprocal) }
\end{align*}
For instance, with $\ott{\lr}{1}=1$, then $\ott{\lr}{2} =
(1+1)^{-1}=\frac{1}{2}$,
$\ott{\lr}{3}=((\frac{1}{2})^{-1}+1)^{-1}=\frac{1}{3}$,
$\ott{\lr}{4}=\frac{1}{4}$, and so on, yielding the fractions in the
harmonic series. We let the $\lr$ go down to no lower than the floor
$\lrmin\in(0,1)$ as shown in \fig \ref{fig:ema}(b), though making the
minimum a fraction of the \pr estimates may work better. \fig
\ref{fig:histories}(a) shows how past observations are down-weighted
by this variant (and compared to plain static EMA and other \smasn).
We have observed in prior work that such a decay regime is beneficial
for faster learning of the higher \prs (\eg $\tp \gtrsim 0.1$), as it
is equivalent to simple counting and averaging to compute
proportions.\footnote{This manner of reducing the rate is equivalent
to reduction based on the update count of EMA, and we referred to it
as a count-based (or frequency-based) decay and it was shown
equivalent to averaging \cite{expedition1}. To see the equivalence,
one can expand and note the telescoping product, or use induction: The
estimate after the update at $t, t\ge2$ is
$\oat{\ep}=\oat{\lr}\oat{o}+(1-\oat{\lr})\ott{\ep}{t-1}$, where
$\ott{\lr}{1}=1,\cdots,\ott{\lr}{t}=\frac{1}{t}$, or
$1-\oat{\lr}=\frac{t-1}{t}$, and we have
$\ott{\ep}{t-1}=\frac{\sum_{i=1}^{t-1}\ott{o}{i}}{t-1}$ (induction
hypothesis), and combining, we obtain the simple average:
$\oat{\ep}=\frac{\sum_{i=1}^{t}\ott{o}{i}}{t}$.} This does not
negatively impact convergence or the error-rate on the lower \prs
(such as $\tp \lesssim 0.05$).  In particular, with harmonic-decay,
one requires $O(\frac{1}{\tp})$ time points instead of
$\Omega(\frac{1}{\lr})$ for convergence to within a positive
(constant) multiplicative deviation $d$ (Appendix A in
\cite{expedition1}). Note that the changing rate $\lr$ also indicates
the predictor's confidence in its current estimate, assuming the
target \pr does not change: the lower the $\lr$ compared to an
estimate $\ep$, the less likely $\ep$ will change substantially in
subsequent updates.



This manner of changing the rate also establishes a connection between
the learning-rate based EMA, and the counting-based methods of \sec
\ref{sec:qus}.






\co{

In particular, it could be useful to start with a high learning
rate and to lower it over time.

Previously (see Appendix A,
\cite{expedition1}), we have observed that starting with a high rate
$\lrmax$, \eg $\lrmax=1$, and then gradually reducing the learning
rate with each update, \eg via a harmonic decay\footnote{This manner
of reducing the rate is equivalent to reduction based on the update
count of EMA, and we initially referred to it as a count-based (or
frequency-based) decay in \cite{expedition1}.}  schedule (\fig
\ref{fig:ema}(b)), is beneficial for faster learning of the higher
\prs (\eg $\tp$ above 0.1), while not impacting convergence or the
error-rate on the lower \prs (such as $\tp$ below 0.05).  In
particular, with harmonic-decay, one requires $O(\frac{1}{\tp})$ time
points instead of $O(\frac{1}{\lr})$ for convergence to within $d$
(Appendix A, \cite{expedition1}).
}


\subsection{Non-Stationarity Motivates Per-Edge Learning Rates}


The harmonic decay technique is, however, beneficial only initially
for a predictor: once the learning rate is lowered, it is not raised
again in plain EMA of \fig \ref{fig:ema}.
%
If one could increase the rate, the predictor could learn to predict a
new item's \pr faster. However, that disrupts the \pr estimates for
existing currently stable items.
%
Different items should not, in general, interfere
with one another: imagine a predictor already predicting an item $A$
with a certain \pr sufficiently well. Ideally learning to predict a
new item $B$ should not impact the \pr of an existing item $A$, unless
$A$ and $B$ are related, or correlated, such as when $B$ is replacing
$A$. This consideration motivates keeping a separate learning rate for
each predictand, or per prediction edge
%
(supporting edge-specific rates). Such extensions would be valuable if
one could achieve them without adding substantial overhead, while
keeping the
desired convergence and \sd preservation properties of EMA (\sec
\ref{sec:convergence}). Section \ref{sec:qdyal} describes a way of
achieving this goal.  To support the extension, we also need to detect
changes, for when to increase a learning rate, and we next develop an
\sma
that can be used both for change detection, as well as for giving us
the initial estimate for a (new) predictand's \pr together with an
initial learning rate. We expect that this type of predictor would
have other uses as well (\eg see \sec \ref{sec:timestamps}).

\co{
But before that, we explain another moving
average technique that can be used both for detecting changes, as well
as giving us an initial estimate for a (new) predictand's probability
together with an initial learning rate. We expect that this type of
predictor would have other uses as well (\eg see \sec
\ref{sec:timestamps}).
}



\section{The Queues Predictor}
\label{sec:qus}

We begin with the stationary setting for a binary event, continually
estimating the \pr of outcome 1 from observing a sequence of 0s and
1s. We then
adapt the counting technique to the non-stationary case and present
the queuing technique for efficiently tracking the \prs of multiple
items. We conclude in \sec \ref{sec:box} with a counting variant we call the
\bx predictor, that could also naturally be implemented via a queue
for efficiency, and discuss a dynamic (adaptive windowing) variant of
it. Appendices \ref{sec:qanalysis} and \ref{app:qvars} contain further
analyses as well as a more efficient variant of queuing.


\subsection{The Stationary Binary Case}
\label{sec:simple_counting}


To keep track of the proportion of positive occurrences, two counters
can be kept, one for the count of total observations, simply time $t$,
and another for the count of positive observations, $\oat{N_p}$. The
probability of the target item, or the proportion of positives, at any
point, could then be the ratio $\oat{\ep}=\frac{\oat{N_p}}{t}$, which
we just write as $\ep=\frac{N_p}{t}$, when it is clear that an
estimate $\ep$ is at a time snapshot. This proportion estimator is
unbiased with minimum variance (MVUE), and is also the maximum
likelihood estimator (MLE) \cite{math_stats18,math_stats16} (see also
Appendix \ref{sec:qanalysis}).  Table \ref{tab:plain_counting}
presents deviation ratios defined here as fraction of the 20k
generated binary sequences, each generated via iid drawing from
$\P=\{$ 1$:\tp$, 0$:1-\tp\}$, for a few $\tp \in \{0.3, 0.1, 0.05,
0.01\}$.  The value of $\dvc(\tp, \ep, d)$ is either 0 or 1 (\eq
\ref{eq:devc}), and is snapshotted at the times $t$ when $N_p$ {\em
  first} hits 10, 50, or 200 along the sequence.  We get a deviation
fraction (ratio) once we divide the violation count, the number of
sequences with $\dvc$ value of 1, by the total number of sequences
(20k). As the number of positive observations increases (the higher
the $N_p$), from 10 to 50 to 200, the estimates $\oat{\ep}$ improve,
and there will be fewer violations, and the deviation ratios go down.
These ratios are also particularly helpful in understanding the
deviation rates of the queuing technique that we describe next.




The above counting approach is not sensitive to changes in the
proportion of the target.\footnote{In addition, with a fixed memory,
there is the potential of counting overflow problems. See \sec
\ref{sec:timestamps}.}  We need a way to keep track of only recent
history or limiting the window over which we do the averaging. Thus,
we are seeking a {\em moving average} of the proportions.  A challenge
is that we are interested in a fairly wide range of \prsn, such as
tracking both 0.1 or higher (a positive in every ten occurrences) as
well as lower proportions such as 0.01 (one in a hundred), and a fixed
history window of size $k$, a ``box'' predictor (\sec \ref{sec:box}
and \fig \ref{fig:histories}), of all observations for the last $k$ time points, is not
feasible in general unless $k$ is very small, \ie the space and
update-time requirements can be prohibitive
computationally.\footnote{Many predictors, thousands or millions and
beyond, could execute the same updating algorithm, \eg in \pgsn.}



\begin{table}[t]  \center
  \begin{tabular}{ |c|c|c|c| c|c|c| c|c| }     \hline
    (positives observed) $\rightarrow$  & \multicolumn{3}{c|}{$N_p$=10} & \multicolumn{3}{c|}{$N_p$=50} & \multicolumn{2}{c|}{$N_p$=200} \\ \hline
     (deviation thresh.) $\dvt \rightarrow$  & 1.1 & 1.5 & 2.0 & 1.1 & 1.5 & 2.0 & 1.1 & 1.5 \\ \hline
  $\tp=0.30$ & 0.735 $\pm$ 0.002 & 0.13 & 0.009 $\pm$ 0.001 & 0.42 & 0.001 & 0.00 & 0.11 & 0.00  \\ \hline
  $\tp=$ 0.10 & 0.742 $\pm$ 0.004 & 0.18 & 0.024 $\pm$ 0.001 & 0.48 & 0.002 & 0.00 & 0.15 & 0.00  \\ \hline
  $\tp = $ 0.05 & 0.75 $\pm$ 0.004 & 0.20 & 0.030 $\pm$ 0.001 & 0.49 & 0.003 & 0.00 & 0.17 & 0.00  \\ \hline
  $\tp = $ 0.01 & 0.76 $\pm$ 0.004 & 0.20 & 0.036 $\pm$ 0.001 & 0.50 & 0.004 & 0.00 & 0.18 & 0.00  \\ \hline
  \end{tabular}
  \vspace*{.1in}
\caption{Deviations in the plain stationary setting.  In this
  simplified setting, a binary sequence is generated iid using a fixed
  but unknown target probability $\tp$, or $\P=\{$1:$\tp$,
  0:$1-\tp\}$, and we use the plain count-based proportion
  $\ep=\frac{N_p}{N}$ to estimate $\tp$ (\sec \ref{sec:qus}).
  20k sequences are generated, each long enough so that there are 200
  positive, 1, outcomes. Fractions of sequences in which the deviation
  of the estimation $\ep$ from the target probability $\tp$ exceeded a
  desired threshold $\dvt$, \ie when $\max(\frac{\tp}{\ep},
  \frac{\ep}{\tp}) > \dvt$ (see \eq \ref{eq:devc}), is reported at a
  few time snapshots.  Thus, with $\tp=0.1$ (2nd row), in about $74\%$
  of the 20k sequences, immediately after $N_p=10$ positive outcomes
  has been observed in the sequence so far, we have either
  $\frac{\tp}{\ep} > 1.1$ or $\frac{\ep}{\tp} > 1.1$ (thus, either the
  estimate $\ep < \frac{0.1}{1.1}$ or $\ep > 0.11$). This deviation
  ratio goes down (improves) to $48\%$ after $N_p=50$ positive
  observations, and further down to $15\%$ when $N_p=200$.  In
  summary, we observe that as more time is allowed, the deviations go
  down, and for smaller $\tp$ the problem becomes somewhat harder
  (larger deviation rates). }
\label{tab:plain_counting}
\end{table}

\subsection{Queuing Counts}

Motivated by the goal of keeping track of only recent proportions, we
present a technique based on queuing a few simple count bins, which we
refer to as the {\bf \em \qu} method. Here, the predictor keeps, for
each item it tracks, a small number of count snapshots (instead of
just one counter), arranged as cells of a queue. Each positive
observation triggers allocation of a new counter, or queue cell.  Each
queue cell yields an estimate of the proportion of the target item,
and the counts over multiple cells can be combined to obtain a single
\pr for that item.  Old queue cells are discarded as new cells are
allocated, keeping the queue size within a capacity limit, and to
adapt to non-stationarities.  We next explain the queuing in more
detail. Figure \ref{fig:queues} presents pseudocode, and \sec
\ref{sec:extraction} discusses techniques for extracting \prs and
their properties such as convergence.


\begin{figure}[t]
\hspace*{-.1in}\begin{minipage}[t]{0.5\linewidth}
    \fontsmall{
      {\bf PredictUpdateViaQs}($\seq{o}{}{}$) \\
      \hspace*{0.3cm} // Input is sequence $\seq{o}{}{}=[\ott{o}{1},\ott{o}{2},\cdots]$. \\  
\hspace*{0.3cm} $\qmap \leftarrow \{\}$ // An empty map, item$\rightarrow$queue. \\
\hspace*{0.3cm} $t \leftarrow 0$ // Discrete time. \\
\hspace*{0.3cm}  Repeat // Increment time, predict, observe.\\
\hspace*{0.5cm} $t \leftarrow t + 1$ // Increment time.\\
\hspace*{0.5cm} GetPredictions($\qmap$) // Get the predictions. \\
\hspace*{0.5cm} // Use observation at time $t$, $\oat{o}$, to \\
\hspace*{0.5cm} UpdateQueues($\qmap$, $\oat{o}$) update $\qmap$. \\ 
\hspace*{0.5cm} If t \% 1000 == 0: // Periodically prune $\qmap$. \\
\hspace*{0.8cm} PruneQs($\qmap$) // (a heart-beat method).\\ \\ 
{\bf GetPredictions}($\qmap$) \\
\hspace*{0.3cm} // Returns a map: item $\rightarrow$ \pr. \\
\hspace*{0.3cm} $\Q \leftarrow \{\}$ // allocate an empty map. \\
\hspace*{0.3cm} For each item $i$ and its queue $q$ in $\qmap$: \\
\hspace*{0.6cm} // One could remove 0 \prs here. \\
\hspace*{0.6cm} $\Q[i] \leftarrow$ GetPR$(q)$ \\
\hspace*{0.3cm} Return $\Q$ // The final predictions. \\ \\ \\
{\bf UpdateQueues}($\qmap, o$) // latest observation $o$. \\
\hspace*{0.3cm} If item $o \not \in \qmap$: // when $o \not \in \qmap$, insert.\\
\hspace*{0.6cm}   // Allocate \& insert q for $o$. \\
\hspace*{0.6cm}   $\qmap[o] \leftarrow Queue()$ \\ 
\hspace*{0.3cm} For each item $i$ and its queue $q$ in $\qmap$: \\
\hspace*{0.6cm} If $i \ne o$: \\ 
\hspace*{0.9cm} // All but one will be negative updates.\\
\hspace*{0.9cm} NegativeUpdate(q) // Increments a count. \\
\hspace*{0.6cm} Else: // Exactly one positive update. \\
\hspace*{0.9cm} // Add a new cell, with count 1. \\
\hspace*{0.9cm} PositiveUpdate(q) \\ 
    }
\end{minipage}
\hspace{0.3cm}
\begin{minipage}[t]{0.55\linewidth}
    \fontsmall{
        {\bf Queue}($cap=3$) // Allocates a queue object. \\
\hspace*{0.3cm}  // \mbox{Allocate $q$ with various fields (capacity, cells, etc.)} \\
\hspace*{0.3cm}  $q.qcap \leftarrow cap$ // max size $cap, cap > 1$. \\
\hspace*{0.3cm}  // Array (or linked list) of counts.\\
\hspace*{0.3cm}  $q.cells \leftarrow [0, \cdots, 0]$ \\
\hspace*{0.3cm}  // Current size or number of cells ($\le cap$).\\
\hspace*{0.3cm}  $q.nc \leftarrow 0$\\
\hspace*{0.3cm}  Return $q$  \\ \\
        {\bf GetPR}($q$) // Extract a probability, from the number \\
\hspace*{.3cm}         // of cells in $q$, $q.nc$, and their total count. \\
\hspace*{0.3cm}  If $q.nc \le 1$: \mbox{// Too few cells (grace period).}\\
\hspace*{0.7cm}  Return 0 \\
\hspace*{0.3cm}  Return $\frac{q.nc - 1}{GetCount(q) - 1}$ \\ \\ 
        {\bf GetCount}($q$) // \mbox{Get total count of all cells in $q$.} \\
\hspace*{0.3cm}  Return $\sum\limits_{0\le j < q.nc} q.cells[j]$ // sum over all the cells. \\ \\
        {\bf NegativeUpdate}($q$) \mbox{// Increments $\cz$ count.}\\
\hspace*{0.25cm} // The back (latest) cell of $q$ is incremented. \\
\hspace*{0.3cm}  If $q.nc \ge 1$: \mbox{// Do nothing, if no cells.}\\
\hspace*{0.45cm}   q.cell[0] $\leftarrow$ q.cells[0] + 1 \\ \\
        {\bf PositiveUpdate($q$) } // Adds a new (back) cell. \\ 
\hspace*{0.25cm} // Existing cells shift one position. Oldest cell is \\
\hspace*{0.25cm} // discarded, in effect, when $q$ is at capacity. \\
        \hspace*{0.3cm}If $q.nc < q.qcap$:\\
        \hspace*{0.7cm}         $q.nc \leftarrow q.nc + 1$ // Grow the queue $q$.\\
        \hspace*{0.3cm}For $i$ in [1, $\cdots, q.nc-1$]: // Inclusive.\\
        \hspace*{0.7cm} $q.cells[i] \leftarrow q.cells[i - 1]$ // shift (counts). \\
        \hspace*{0.3cm} $q.cells[0] \leftarrow 1$ // set newest, $\cz$, to count 1. \\ \\
    }
\end{minipage}

\caption{Pseudo code of the main functions of the \qu method
  (left), and individual queue operations (right). Each predictor
  keeps a one-to-one map of items to queues, $\qmap$, where each queue
  is a small list of counts, yielding a probability for a single item.
  For methods of extracting probabilities from the counts and their
  properties (variations of GetPR()), see \sec \ref{sec:extraction},
  and for pruning, see \sec \ref{sec:qspace} and \fig \ref{fig:qflow}. }
\label{fig:queues}
\end{figure}







\subsection{The \qu Method: Keep a Map of Item $\ra$ Queue}
\label{sec:qs}

The \qu method keeps a one-to-one map, $\qmap$, of items to (small)
queues.  At each time point $t$, $t\ge 1$, after outputting predictions
using the existing queues in the map, it updates all the queues in the
map.\footnote{And other functions can be supported, such as querying for
a single item and obtaining its probability and/or the counts. We are
describing the main functions of updating and prediction.} If the item
observed at $t$, $\ott{o}{t}$, does not have a queue, a queue is
allocated for it and inserted in the map first, before the updating of
all queues.  For any queue corresponding to an item $i \ne
\ott{o}{t}$, a {\bf \em negative update} is performed, while a {\bf
  \em positive update} is performed for the observed item
$\ott{o}{t}$. Thus at every time point, exactly one positive update
and zero or more negative updates occur. Every so often, the map is
pruned (\sec \ref{sec:qspace}). Operations on a single queue are
presented in \fig \ref{fig:queues}(b) and described next, and \fig
\ref{fig:q_pics} illustrates these queue operations with examples.




A new cell, {\bf \cz}, at the back of the queue, is allocated each
time a positive outcome is observed, and its counter is initialized to
1.  With every subsequent (negative) observation, \ie until the next
positive outcome, \cz increments its counter.  The other, {\bf \em
  completed}, cells are in effect frozen (their counts are not
changed). Before a new back cell \cz is allocated, the existing \czn,
if any, is now regarded as completed and all completed cells shift to
the ``right'' one position, in the queue (Figure \ref{fig:q_pics}(b)),
and the oldest cell, {\bf \ckn}, is discarded if the queue is at
capacity {\bf \qcapn}.  Each cell corresponds to one positive
observation and the remainder in its count $c$, $c-1$, corresponds to
consecutive negative outcomes (or the {\em time 'gap'} between one
positive outcome and the next). Thus, the number of positives
corresponding to a cell is 1 (a positive observation was made that led
to the cell's creation).  Each cell yields one estimate of the
proportion, for instance $\frac{1}{c}$, and by combining these
estimates, or their counts, from several queue cells, we can attain a
more reliable \pr estimate, discussed next. \fig
\ref{fig:histories}(d) shows that, in effect, queues of different
items with different occurrence probabilities, correspond to different
lengths of history, the lower the \prn, the longer the history (but
the same maximum number, \qcapn, of positive observations).




\subsection{Extracting \prs from the Queue Cells}
\label{sec:extraction}

The process that yields the final count of a completed queue cell is
equivalent to repeatedly tossing a two-sided coin with an unknown but
positive heads probability $\tp > 0$, and counting the tosses until
and including the toss that yields the first heads (positive)
outcome. Here, we first assume $\tp$ does not change (the stationary
case). The expected number of tosses until the first heads is observed
is $1/\tp$. We are interested in the reverse estimation problem:
Assuming the observed count of tosses, until and including the first
head, is $c$, the reciprocal estimator $\frac{1}{c}$ is a {\em
  biased}, upper estimator of $\tp$, for $\tp < 1$, also known to be
the MLE. The positive bias can be shown by looking at the expectation
expression (Appendix \ref{sec:qanalysis}). The (bias) ratio,
$\frac{c^{-1}}{\tp}$, gets larger for smaller $\tp$ (as $\tp
\rightarrow 0$). See Appendix \ref{sec:qanalysis} which contains
derivations and further analyses. 

More generally, with $k$ completed cells, $k\ge 2$, cell $j$ having
count $\c_j$, we can pool all their counts and define the statistic
(random variable) $G_k=\frac{k-1}{\sum_{1\le j\le k} \c_j - 1}$. $G_k$
is shown to be {\em the minimum variance unbiased estimator} of $\tp$
\cite{marengo2021}, meaning in particular that $\expd{G_k}=\tp$. The
powerful technique of Rao-Blackwellization is a well-known
tool\footnote{Many thanks to J. Bowman for pointing us to the paper
\cite{marengo2021}, and describing the proof based on
Rao-Blackwellization,
\href{https://stats.stackexchange.com/questions/624192/tossing-until-first-heads-outcome-and-repeating-as-a-method-for-estimating-pro}{in
  the Statistics StackExchange.}}  in mathematical statistics,
used to derive an improved estimator (and possibly optimal, in several
senses) starting from a crude estimator, and to establish the
minimum-variance property
\cite{rbk_encp_math,rao45,blackwell47,lehman50,wikip1}. Note
that we need at least two completed cells for appropriately using this
estimator.  Appendix \ref{app:rb} contains further description of how
Rao-Blackwellization is applied here.


The back cell, \czn, of the queue, with count $\c_0$, is incomplete,
and the reciprocal estimator $\frac{1}{\c_0}$ can generally be even
higher (worse) than the simple biased MLE estimator derived from a
completed cell.  However, as we explain, due to non-stationarity, in
our implementation of \pr estimation via the \qu method, we use \cz as
well (function {\bf GetPR}($q$) in \fig \ref{fig:queues}), thus we use
$k/(\sum_{0\le j\le k} \c_j - 1)$ (where the queue has $k+1$ total
cells, and the count in the denominator includes {\em all} queue
cells). We have found that with small capacity \qcapn, using an extra
cell (even if often incomplete) can noticeably lower the variance of
the estimate. More importantly, in the presence of non-stationarity,
the estimate $\frac{1}{\c_0}$ is crucial for providing an upper bound
estimate on \prn: imagine $\tp$ has a sudden (discrete) drop from
$0.1$ to $0$. It is only \cz that would reflect this reduction over
time: the other completed cells are unaffected no matter how many
subsequent negative observations take place.\footnote{An alternative
for incorporating \cz is to separate the MLE \pr $p_0$ from \cz from
the MVUE \pr $p_1$ derived from the completed cells, and use $p_0$
only when it is lower than $p_1$, for example, when it is
significantly lower according to the binomial-tail test, \sec
\ref{sec:binom}. This is an efficient constant-time test. } See also
\sec \ref{sec:qspace} on pruning.

With our GetPR() function, the \qu technique needs to see two
observations of an item, in sufficiently close time proximity, to start
outputting positive \prs for the item.


%
\co{
However, this upper bound is useful both for space management
(removing infrequent/noise items, see Section \ref{sec:qspace}) and,
when faced with potential non-stationarity, for obtaining a
potentially better estimate of probability, when the probability goes
down substantially, \eg from $0.5$ to $0.05$.
}

\co{
}

\co{
todo and structure:

1) picture for a sequence such as JAABAEBAAABBEA... (perhaps this
goes into the intro?) (should it be all capital letters as item
examples? lower case and upper case? ids, eg integers in parentheses?
I am thinking capital letters is fine...

2) Present the problem and the queue approach.  Pseudo code for the
queue algorithm(s), for the multiclass case.

3) Possible discussion: variations, eg that it stems from the
'stationary boolean' case that could be solved with two
counters.. it's a generalization (note some discussion points will be
made as you present/explain the algorithm too.. and maybe we'll decide
this is unnecessary)..

And alternative is like below: start with the boolean case 1st ..

}
  




\begin{figure}[t]
\begin{center}
  \centering
  \hspace*{-.8cm }\subfloat[An example sequence and its binary version for item \itm{A}.]
          {{\includegraphics[height=4cm,width=5cm]{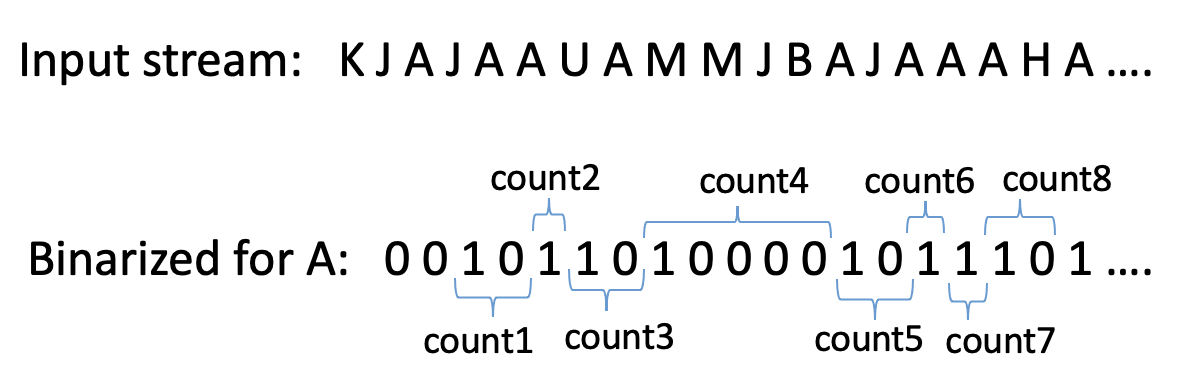} }}
  \hspace*{.3cm}\subfloat[Contents of q(\itm{A}) at a few times.]
           {{\includegraphics[height=4cm,width=4cm]{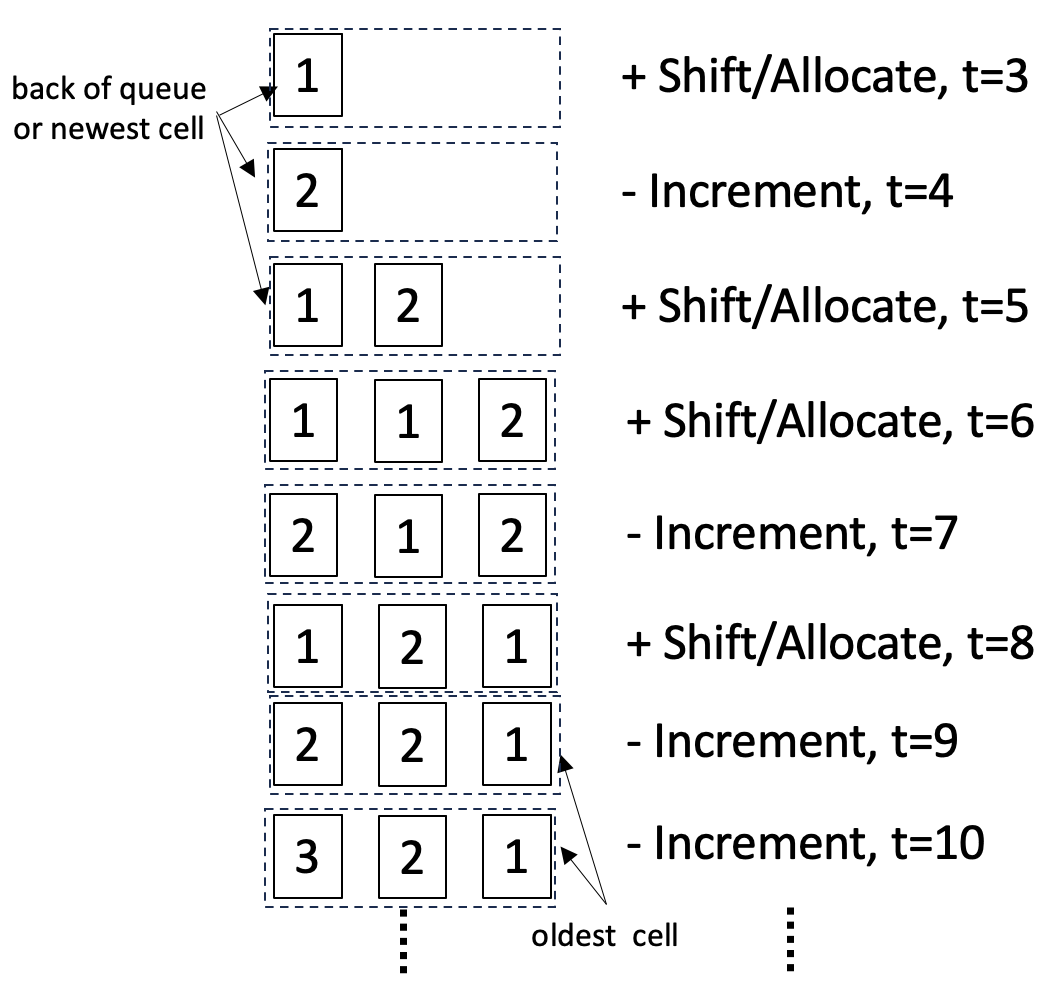} }}
  \hspace*{.5cm}\subfloat[Cells (counts) yield probabilities.]
           {{\includegraphics[height=4cm,width=4cm]{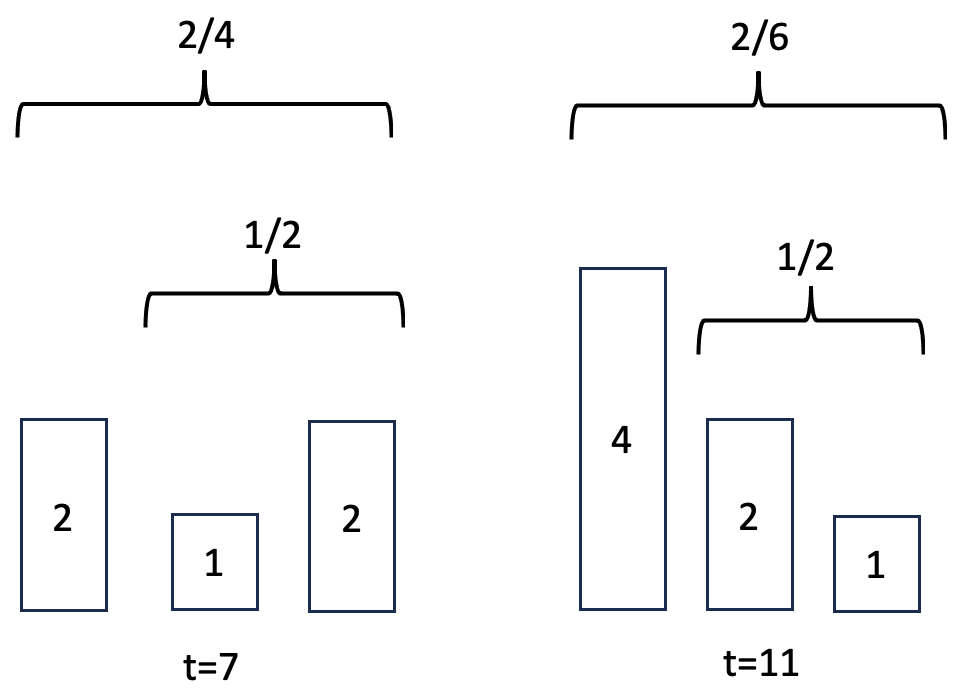} }}
\end{center}
\caption{Illustrating the workings of the \qs predictor for
  tracking the probability of a single item: An example sequence (top,
  (a)), made binary for item \itm{A} (bottom), \ie whenever A occurs,
  the entry in the binarized sequence is a 1, leading to a positive
  update of the queue for \itm{A}, otherwise it is a 0, leading to a negative
  update (once the queue is allocated).
  Part (b) shows another view of the queue for \itm{A} upon queue
  creation and the first few updates (horizontal at each time point,
  left (newest cell, \czn) to right-most (oldest, \ckn).  At $t=3$
  when the first \itm{A} is observed (the first positive update), the
  queue for \itm{A} is created with \cz (newest cell of the queue),
  initialized with count 1. At time $t=4$ a negative update is
  performed and \cz (its counter) is incremented. At $t=5$, another
  positive update, existing \cz (its count content) is shifted one
  position to right, and a new \cz is allocated with value 1, becoming
  the new back of the queue. More generally, on positive updates,
  existing cells shift to right and a new \cz with an initial count of
  1 is allocated, and on negative updates, the existing \cz is
  incremented.  At
  $t=8$, the value in \ck is discarded upon another positive update
  (assuming capacity 3), and so on.  (c) Cell counts, shown at two
  snapshots, can be used to estimate a probability. There are several
  options for pooling the cells to get a probability. For instance,
  the back (left-most) partial cell can be ignored (see Section
  \ref{sec:extraction}). }
\label{fig:q_pics} %
\end{figure}

\co{

(b) Another view: cell contents of the queue for \itm{A} upon creation
  and on the first few updates (positive and negative observations,
  \itm{A} and not \itm{A}).

The following develops this idea: the predictor keeps a queue of
cells, with some maximum capacity, for each item (or item-type).

several counter
cells, in a queue, one queue for each item (class or type), as
follows.

Each time a positive observation of an item is made, a new cell is
generated, initialized with a count of 1, and is added to the back of
the queue.  If the predictor is seeing the item for the first time, a
queue object is created for it, and inserted into a map of
item\_to\_queue that the predictor keeps. Thus a predictor keeps an
item

}

\subsection{Predicting with \sdn s}
\label{sec:qu_sds}

Each queue in $\qmap$ provides a \pr for its item, but the \prs over
all the items in a map may sum to more than 1 and thus violate the \sd
property. For instance, take the sequence
\itm{A}\itm{A}\itm{A}\itm{A}\itm{B}\itm{B}\itm{B}\itm{B} (several
consecutive \itm{A}s followed by several \itm{B}s), and assume
\qcap$=3$: after processing this sequence, the \pr for \itm{B} is 1.0,
while for \itm{A} is $\frac{2}{6}$.  See also Appendix \ref{app:qsum_etc}. The \sd
property is important for down-stream uses of the \prs provided by a
method, such as computing expected utilities (and for a fair
evaluation under \kl() divergence). For evaluations, we apply the
\fcap() function in \fig \ref{code:norming_etc}(a) which works to
normalize (scale down) and convert any input map into an \sd as long
as the map has non-negative values only.

\begin{figure}[t]
\hspace*{0.0cm}  \begin{minipage}[t]{3in}
 \subfloat{{ \includegraphics[height=4cm,width=7cm]{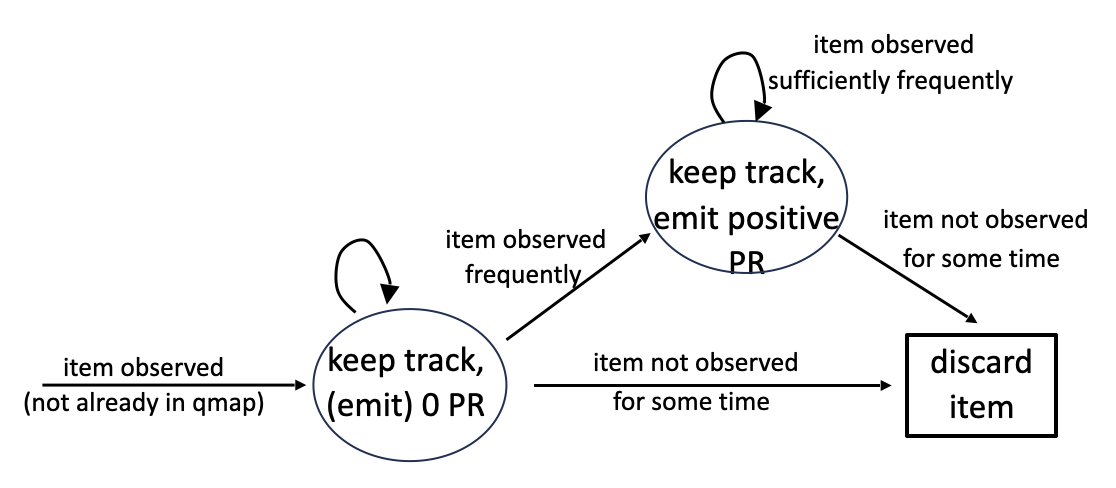} }}
\end{minipage}
\hspace*{1.2cm}
\begin{minipage}[t]{3.5in}
\vspace*{0.5cm}
{\bf PruneQs}($\qmap$) // Pruning: A heart-beat\\ 
\hspace*{0.3cm} // Parameters: size \& count limits $s_1, s_2,$ \\
\hspace*{0.3cm} // such as $s_1 = 100$ and $s_2 = 10^5$. \\
\hspace*{0.3cm} // For any item, $\c_0$ is its count in \cz. \\
\hspace*{0.3cm} Drop any item with $\c_0 > s_2$ from $\qmap$ \\
\hspace*{0.3cm} If $|\qmap| < 2s_1$: \\ %
\hspace*{0.6cm}   Return // Nothing left to do. \\
\hspace*{0.3cm} Drop highest $\c_0$ counts until $|\qmap|$ is $s_1$.
\end{minipage}
\caption{Left: The state changes of an item as it enters the
  $\qmap$. The \qu technique keeps track of any item it sees in its
  stream (inserts it in $\qmap$), for some (minimum, grace-period)
  time. The item may become salient (positive \prn) and stay in that
  state indefinitely or, eventually or after the grace-period,
  discarded. Right: The pruning logic: items with large $\c_0$ (count
  in \czn) are discarded. If the map is too large, it is pruned too
  (again, rank and remove by descending $\c_0$, see \sec
  \ref{sec:qspace}).  }
\label{fig:qflow} %
\end{figure}

The following properties hold, the proof of which (Lemma
\ref{lem:qu_prs}) along with other properties of the \prs, are in
Appendix \ref{app:qsum_etc}. Below, $\Q()$ is the \pr map output of
the \qu technique, or $\Q(i)$ is the \pr obtained from GetPR() when
the queue for item $i$ is passed to it.

\begin{lemma}
  For the \qu technique with \qcap $ \ge 2$, for any item $i$ with
  a queue $\q(i)$, $|\q(i)| \le$ \qcapn, using the \pr estimate
  $\Q(i)=(|\q(i)|-1)/(\sum_{0\le j < |\q(i)|}\c_j-1)$, for any time point $t\ge 1$: 
  \begin{enumerate}
  \item $\oat{\Q}(i)$, when nonzero, has the form $\frac{a}{b}$, where $a$
    and $b$ are integers, with $b\ge a\ge 1$.
    \co{
  \item If $i$ is not observed at $t$, then $\ott{Y_i}{t+1}=
    \ott{Y_i}{t}+1$ or $i$ is removed from the map $\qmap$. When $i$
    is observed at $t$ (exactly one such), then
    $\ott{Y_i}{t+1}\le\ott{Y_i}{t}$ when $|\oat{\q}(i)| =$ \qcapn, and
    $\ott{Y_i}{t+1}=\ott{Y_i}{t}+1$ when $|\oat{\q}(i)| < $ \qcapn.
    }
  \item If $i$ is observed at $t$, then $\ott{\Q}{t+1}(i) \ge
    \ott{\Q}{t}(i)$. If $i$ is not observed at $t$, then
    $\ott{\Q}{t+1}(i) < \ott{\Q}{t}(i)$.
  \end{enumerate}
\end{lemma}

\subsection{Managing Space Consumption} 
\label{sec:qspace}

There can be many infrequent items and the queues map of a predictor
can grow unbounded, wasting space and slowing the update times, if a
hard limit on the map size is not imposed.  The size can be kept in
check via removing the least frequent items from the map every so
often, for instance via a method that works like a "heart-beat"
pattern: map expansion continues until the size reaches or exceeds a
maximum allowed capacity $2s_1$, \eg $s_1=100$, remove the items in
order of least frequency, \ie rank by descending count $\c_0$ in \cz
in each queue,\footnote{One could also use estimates from other
completed cells as well, such as $\min(\frac{1}{\c_0},$ GetPR()), but
only when GetPR() $> 0$, \ie this value should be used only when other
completed cells exist (otherwise a new item could get dropped
prematurely). With nonstationarity, it is possible the estimates from
older cells would be outdated, and an item that used to be infrequent
may have become sufficiently frequent now.} until the size is
precipitously shrunk back to $s_1$.  When we are interested in
modeling (tracking) probabilities down to a smallest \pr $\pmin$, in
general we must have $s_1 > \frac{1}{\pmin}$ (\eg with $p_{min}=0.01$,
$s_1$ needs to be above 100). The higher $s_1$, the less likely that
we drop salient items, \ie items with (true) \pr above $\pmin$ (when a
true \pr exists, under the stationary setting). Note that those queues
that are dropped, their \pr when normalized, must be below $\pmin$
whenever $s_1 > \frac{1}{\pmin}$, and the normalized \pr is what we
use when we want to use predictions that form a \sd. Appendix
\ref{app:qsum_etc} explores the maximum possible number of \prs above
a threshold before normalizing.

The above periodic pruning logic can also limit the count values
within each queue, in particular in \czn, but not in the worst-case:
suppose \itm{A} is seen once, and from then on \itm{B} is observed for
all time. Thus $|\qmap|=2$, and the above pruning logic is never
triggered, while the count for \itm{A} in its \cz can grow without
limit (log of stream size). Thus, we can impose another constraint
that if \cz $\c_0$ count-value of a queue in the $\qmap$ \ is too
large (\eg $\frac{1}{\c_0} \le 10^{-4}$), such a queue (item) is
removed as well.  The trigger for the above checks can be periodic, as
a function of the update count of the predictor (\fig
\ref{fig:queues}). The pruning logic is given in \fig \ref{fig:qflow}
(right).

\fig \ref{fig:qflow}, left, summarizes the states and state
transitions of \qu upon observing an item not already in $\qmap$, as
it pertains to the item. Any item seen will be kept track of in
$\qmap$ for some time (assuming $s_1 > 1$). Being tracked does not
imply positive \pr (salience), however, if the item is seen more in
sufficiently close proximity, it becomes salient (we require a minimum
of two observations). It may also be discarded without ever becoming
salient, or discarded after becoming salient. It may also enter the
map and never exit.


\subsection{Complexity of the \qu Method}

Let $k=\qcapn$.  At each time point, prediction using the $\qmap$  and
update take the same $O(k|\qmap|)$, where we assume $O(1)$ time for
summing numbers and taking ratios: queue update involves
updating the queue of every item in the map, and the updates, whether
positive or negative, take $O(k)$ time. In our experiments, \qcap $k$ is
small, such as $3$ or $5$. With a $\pmin$ of $0.01$ and periodic
pruning of the map, the size of the queue for each predictor can grow
to up to a few 100 entries at most.

Number of time steps required for estimating a \pr output for an item
with underlying \pr $\tp$ (at the start of a new stable period),
similar to EMA with harmonic decay (\sec \ref{sec:harm}), is
$\theta(\frac{1}{\tp})$ (time complexity).




\subsection{The Box Predictor (a Fixed-Window Baseline), and its Variants}
\label{sec:box} 

\begin{figure}[t]
\hspace*{0.5cm}  \begin{minipage}[t]{0.45\linewidth}
  \begin{minipage}[t]{0.95\linewidth}
  \subfloat{{\includegraphics[height=2.5cm,width=7.5cm]{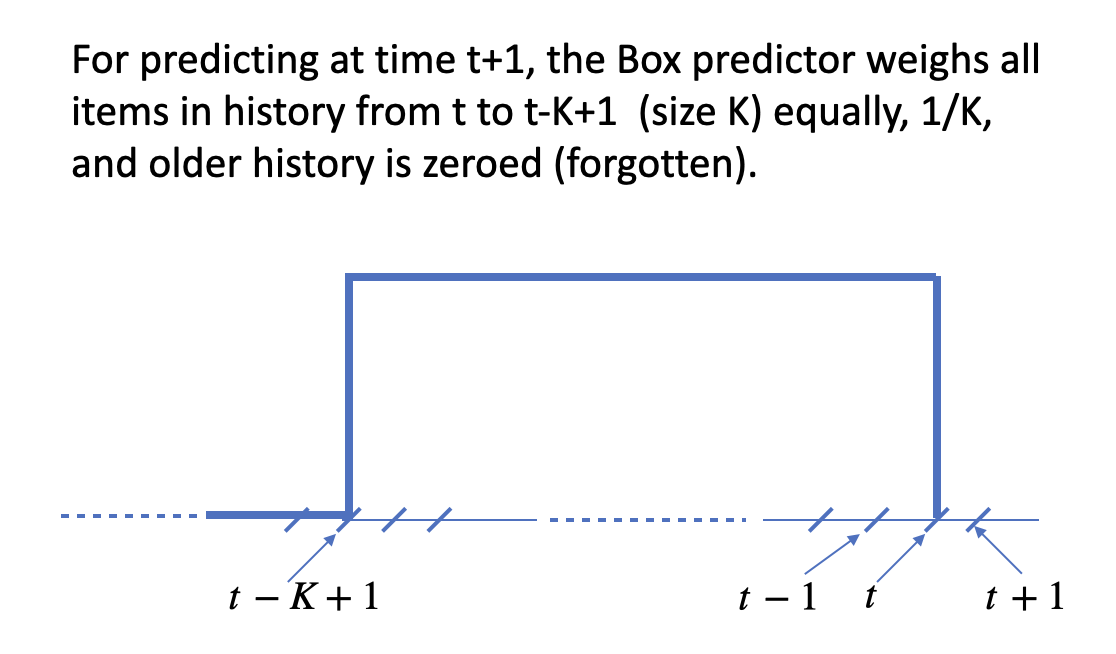}}}\\
  \hspace*{0.20cm} (a) How history is weighed by the \bx predictor. \\ \\
  \end{minipage}
  \begin{minipage}[b]{0.95\linewidth}
    \hspace*{0.5cm}  {\bf GetBoxPR}($o, q$, $counts$) // Get \prn. \\
\hspace*{0.95cm}  If $o \not \in counts$:  Return 0 \\
\hspace*{0.95cm}  Return $counts[o] \ / \ q.nc$ \\  \\
  \hspace*{2.cm} (b) Obtaining \pr of an item.
  \end{minipage}
\end{minipage}
\hspace*{1.5cm}
\begin{minipage}[t]{0.53\linewidth}
\vspace*{.15cm }
{\bf UpdateBoxQ}($q, counts, o$)  \\
\hspace*{0.3cm} // add item $o$ to queue $q$ and update $counts$. \\
\hspace*{0.3cm} $counts[o] \leftarrow counts[o] + 1$ \\
\hspace*{0.3cm} q.add(o) // add $o$ to back of queue $q$. \\
\hspace*{0.3cm} // q.nc is the number of cells in queue.\\
\hspace*{0.3cm} If $q.nc > q.K$: // reached max $K$? \\
\hspace*{0.6cm} // Drop the item from front of queue. \\
\hspace*{0.6cm} $i \leftarrow q.pop()$ // Front (oldest) item. \\
\hspace*{0.6cm} $counts[i] \leftarrow counts[i] - 1$ \\
\hspace*{0.6cm} If $counts[i] == 0:$ \\
\hspace*{0.9cm} // remove from the counts map as well.\\ 
\hspace*{0.9cm} counts.remove(i)\\   \\
  \hspace*{1.cm} (c) Updating a Box predictor. \\ \\
\end{minipage}
\caption{The \bx predictor.  (a) The history is weighed by a box function. (b) Obtaining the
  probability of an item $o$. (c) Updating is $O(1)$ when a linked-list
  queue and a counts map is used.  }
\label{fig:box}
\end{figure}

The plain (discrete) \bx predictor, also known as the sliding window
predictor, keeps a history window of fixed size $K$, of the last $K$
observations (\fig \ref{fig:box}). The sliding windows idea
(of fixed or variable size) is a common tool used in various non-stationary
problems
\cite{adaptive_windows_bifet2007,survey_on_concept_drift_2014gama,nonstationarity2015}. For
our setting, the \bx \sma can be implemented relatively efficiently
via a single queue, and thus it has similarities to the \qu technique
(one long queue \vs several small queues). With a hash map of item to
observation counts and a linked list queue, updating can be more
efficient at $O(1)$ (instead of $O(K)$): This involves adding a cell
and possibly dropping one (oldest) cell from the queue, and updating
up to two item counts. See \fig \ref{fig:box}.  However, the space
consumption is a rigid $\Theta(K)$, and unlike the \qu predictor, $K$
could be relatively large such as 100 or 1000, depending on how small
one wants to go in tracking \prsn. While the worst-case space
consumption of \qu predictor is similar, when the input stream has a
few salient items with relatively high \pr the \qu predictor would be
significantly more advantageous. Importantly, the \bx predictor may not be sufficiently
responsive to non-stationarity: for instance, with $K=100$, when a new
item gets a high \prn, it will take 10s of positive observations for
the \bx predictor to approach the target \pr (for the \qu predictor,
we would neeed a handful of positive observations, with a \qcap at say
3 or 5). The response time would be slower with increasing $K$ (akin
to convergence \vs stability tradeoff when setting the rate of
fixed-rate EMA).  Thus the box-predictor resembles the fixed-rate EMA,
but requires more space. The more dynamic \qu variants react better
under non-stationarity.

There also exist dynamic versions of the \bx predictor, or adaptive
windowing, where the size of the window is dynamically adjusted, \eg
\cite{adaptive_windows_bifet2007,driftingDistros2024,histo_change2008}.
The recent probability prediction work of \cite{driftingDistros2024},
published in the same year as ours, is closest to our work to the best
of our knowledge, as the observations are discrete and the task is
open-ended. In that work, the algorithm, similar to prior work on
adaptive windowing, adjusts the history size $k$ (up to $\lg_2(N)$ where
$N$ is the stream length). Prior work has focused on other, numeric,
data types (\eg \cite{adaptive_windows_bifet2007,histo_change2008}).
A challenge in adaptive windowing is how to efficiently pick the
window size (which may need to be executed at every time point, or
triggered by an efficient change detection method). In our task, we
have argued that different items, with different probability ranges
(\eg from 0.1 down to 0.001), in effect require different history
sizes. We note that the algorithm of \cite{driftingDistros2024} is
presented and evaluated in a theoretical context and while the
$\lg_2()$ can grow with sequence length, it is actually too small, for
instance only 10 for $1000$ and 20 for $10^5$ long sequences, while
for \prs around 0.01 say, one requires histories of order
100s.\footnote{One could imagine using multiples of the ceiling
$\lg_2(N)$ (for small $N$) but it is not clear how to adjust as stream
size $N$ grows, and it appears that the same history span won't work
for different \pr values (\fig \ref{fig:histories}(d) shows that \qs
uses different spans).  If one were interested in high probabilities,
\eg above 0.1, such \bx variants may suffice. The desired theoretical
property used to motivate the algorithm is based on total variational
loss \cite{driftingDistros2024}, which is an absolute loss
measure. Our arguments on the insensitivity of quadratic loss to
smaller probabilities, \sec \ref{sec:brier}, applies to absolute loss
too. A remaining challenge is a time-efficient and effective selection
of the dynamic box size.  }

%
%

Figure \ref{fig:histories} gives a visualization of how the history is
weighted by several \smasn.
In plain EMA updating, the value of an observation at time $t-i$ is
multiplied by (weakened) by $(1-\lr)^i$ with fixed-rate $\lr$ (powers
of 0.9, for $\lr=0.1,$ and $0.99$, for $\lr=0.01$, in the
picture). For the harmonic-decay variant, the multiplier (the
weakening factor) decreases more rapidly initially, but eventually the
exponential decrease of plain EMA yields a lower rate (the crossing is
seen for $\lr=0.1$).  The reduction of weight for the EMA variants is
smooth, while for \bx variants, there is an abrupt change, \eg from
$1/K$ to $0$ as shown in the Figure.


\begin{figure}[t]
  \subfloat{{\includegraphics[height=5.5cm,width=15.5cm]{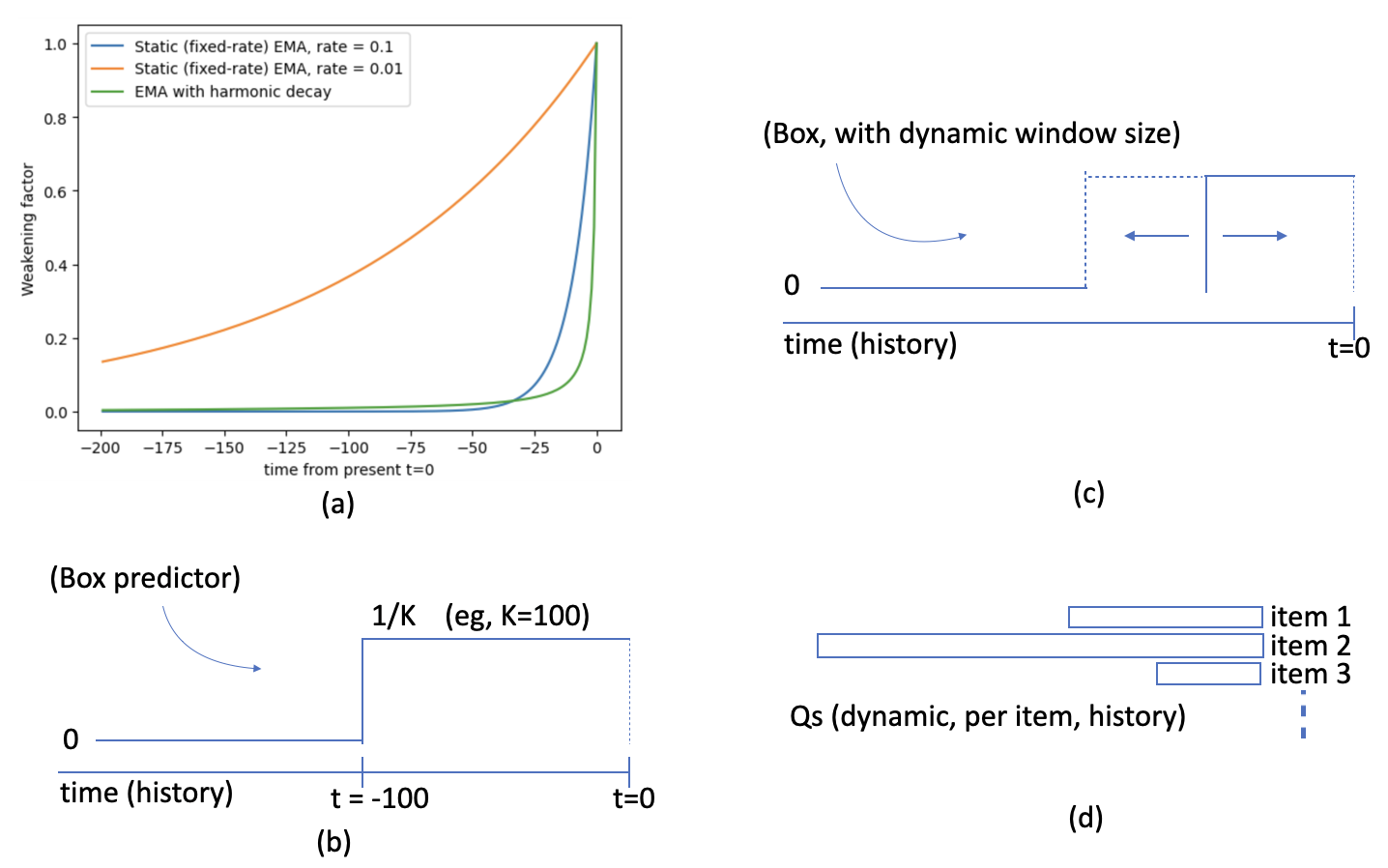}}}

\caption{A visualization of how history is weighted by different
  techniques (in contributing to the present value of the moving
  average). (a) The EMA variants down-weigh the history smoothly (for
  static EMA, it is $(1-\lr)^i$ for time $t-i$, \ie $i$ time points
  before current time $t$). (b) For the plain \bx predictor history is
  zeroed (forgotten) beyond the recent window, and otherwise,
  observations within the window get equal weighting. (c) Dynamic
  variants of the Box (adaptive windowing) can adjust the window size
  by trying to detect change (implicitly or explicitly), so that when
  there is change, more of the irrelevant past is appropriately
  forgotten (conversely, expand the window when no change, for better
  estimation). (d) For the \qs \sman,
  history has a shape similar to that for \bxn, but a
  different history (span) is kept for each item: lower probability
  items require and use longer histories.  }
\label{fig:histories}
\end{figure}

\section{\qdn: Combining Sparse EMA and \qs}
\label{sec:qdyal}
\label{sec:dyal}


The \qu technique can give us a good roughly unbiased starting point
for an estimate of the \pr of an item, but it has a high variance (\eg
see \fig \ref{fig:binary_nonstationary}(c)) unless we use considerable
queue space (many queue cells, and trading off plasticity for
stability). With sparse EMA (\sec \ref{sec:ema}), and with our
assumption that there will be a period of stability when a salient
item is observed (stability for that item),
one could begin with the \qs estimate and fine tune it to achieve
lower variance in the estimation.
We next show that these complementary benefits of each approach can
indeed be put together, and we call the combination {\bf \qdn}, for
{\em \underline{dy}namic \underline{a}djustment of the
  \underline{l}earning} (rate), with the mnemonic of {\em dialing} up
and down the learning (rate), shown in \fig \ref{fig:ema_qdyal}.


%

An update in the \dyal \sma has the simple logical structure of the
plain EMA, \ie weaken-and-boost (\fig \ref{fig:ema_phases}): weaken
all existing edges, then boost (strengthen) the edge to the observed
item (\sec \ref{sec:ema}).  However each edge in \qdn, in addition to
its weight (a \prn, a floating point), also has its own learning rate,
which is used for the weakening and boosting of that edge. This is
unlike plain EMA, wherein a single rate is used for all edges of a
predictor. \fig \ref{fig:maps_edges} shows the
the parameters kept with each edge, for the three main \smas of this
paper.
Associated with each \qd edge, there is also a queue, and during
weakening or boosting an edge, it is possible that the queue estimates
are used (a \pr as well as a learning rate, both derived from the
queue). We refer to this possibility as {\em listening} to the queue
(\sec \ref{sec:binom}). Each learning rate is decayed, via harmonic
decay (\sec \ref{sec:harm}), after an edge update. For a visualization
of the change patterns in learning rates, see \fig
\ref{fig:binary_nonstationary}(e) and (f) and \fig
\ref{fig:evolution_2_seqs}.  The queues are also used for deciding
when an item should be discarded.
Thus the queues are used as ``gates'', the interface to the external
observed stream, playing a major role in determining what is kept
track of and what is discarded, and providing rough initial estimates,
of the \pr and $\lr$, whenever the \pr of an item needs to
substantially change. Note that in the pseudo code of \fig
\ref{fig:ema_qdyal} we use three maps ($\emamap, \lrmap, \qmap$),
while in the picture of \fig \ref{fig:maps_edges}(c) they are combined
into one logical graph.

Prediction in \qd is as in plain EMA, and the $\emamap$ is
output. Like EMA, \qd ensures that $\emamap$ remains a \sd after
updating.
When updating, all the three maps are updated in general, as given in
the UpdateDyal() function. Upon each observation $o$, first the queues
information on $o$ is obtained. This is the current queue probability
$\qprob$ (possibly 0) on $o$, and the $qcount$ (invoke GetCount()).
Then $\qmap$ is updated (a \qun-type update). Finally, all the edges
in the EMA map, except for $o$, are weakened. If there was no queue
for $o$, \ie when $\qprob$ is 0, $o$ must be new, or \ns in general,
and nothing more is done, \ie no (EMA) edge strengthening is done
(note that a queue is allocated for $o$ upon the $\qmap$ update).
Weakening, for each edge, is either a plain EMA weakening, or the
queue estimates are used, which we cover next. Similarly, for
strengthening, the condition for listening to the queue is examined,
as described next, and the appropriate strengthening action, plain EMA
strengthening or listening to the queue, is taken.

%


\begin{figure}[t]
\begin{center}
  \centering
          \includegraphics[height=4cm,width=15cm]{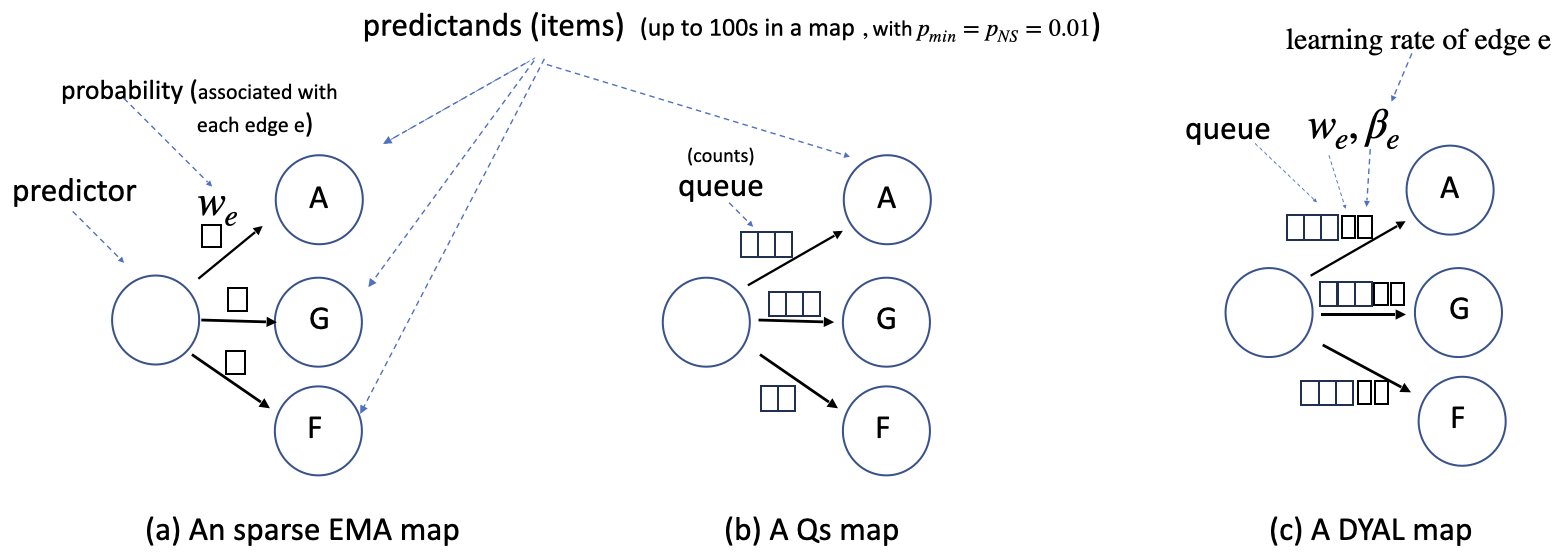} 
\end{center}
\caption{A comparison of the parameters kept with each edge
  (prediction relation) by the 3 main \smas of this paper.  Each
  \sma keeps and updates one (logical) map, with one map entry per
  predictand. Each map entry is viewed as an edge (\ie
  prediction relation). In this picture 3 predictands are shown in
  each map (items \itm{A}, \itm{G}, and \itm{F}).  (a) For EMA,
  associated with each edge $e$ is a single weight or \pr $w_e$, that
  is kept updated. There is one learning-rate $\lr$, associated with
  the entire predictor. (b) For the \qs \sma (center), associated with
  each edge $e$ is a small queue of counts, integers, up to 3 cells in
  this picture. There is no (explicit) learning-rate in \qun,
  and the \prs are derived from the queue contents (\sec
  \ref{sec:qus}).  (c) In \qdn, \sec \ref{sec:qdyal}, which combines the two
  techniques, each edge has its own learning rate $\lr_e$, as well as
  a queue (like \qun) and a \pr $w_e$ (like EMA). The learning-rate is
  dynamic: it
  is increased from time to time, but the default
  pattern of change is gradual decay (such as harmonic decay, \sec
  \ref{sec:harm}), down to a minimum. }
\label{fig:maps_edges} %
\end{figure}




\co{
We refer to this hybrid technique as {\bf \qdn}, for {\em
  \underline{dy}namic \underline{a}djustment of the
  \underline{l}earning} (rate), with the mnemonic of {\em dialing} up
and down the learning. In order to achieve this, \qdn, for {\em every}
predictand that it keeps track of, maintains a small queue, and a
learning rate, in addition to the PR (a weight).  \fig
\ref{fig:ema_qdyal} shows the pseudo code for the major
functions. Like EMA, \qd has a map for the \pr weights, $\emamap$, but
it also has a (parallel) map of learning rates, $\lrmap:$ item $\ra
\lr$, and similarly a map of (small) queues, $\qmap$ (like \qun). The
space overhead can therefore be thought of as a few extra bytes per
predictor-predictand relation (per prediction edge, in addition to the
bytes needed for the \pr weight, in plain EMA).
}


\co{
Queues are still used, but the queues act as ``gates'', and so we have
a two stage approach.  Use EMA (another map of edges) with a per-edge
learning rate. The learning rates are (eg via harmonic decay) to hone
in/converge on to the target probability of existing explicitly
modeled items. But also use the queues approach and listen to the
queues: do a check and if the probability estimate from the queue is
statistically and substantially higher or lower (with statistical
significance), switch to the queue estimate (and adjust learning rates
accordingly, etc).

Points to touch on:
\begin{itemize}
  \item Assumption: periods of stability ...
  \item but there is a price to pay, for nonstationarity (always being
    vigilant): there is a limit to how close we can get.. the level of
    precision...
  \item ema works in two stages:
  \item stage 1, weakening: probability ``flows'', from the edges, to
    a freq source (reservoire)
  \item stage 2, boosting: probability ``flows'', from the reservoire,
    to the target edge..
  \item We support different (learning) rates for the edges here: each
    edge (predictand) has its own learning rate (as predictands are
    inserted at different times with different initial probs and
    rates).
  \item We also support lowering each learning rate with time, \eg
    harmonic decay. This helps/allows convergence (to a more accurate
    region around the target probability).
\end{itemize}

\begin{itemize}
  \item Use EMA with a lowering learning rate (eg harmonic decay) to
    hone in/converge onto the target probability of existing
    explicitly modeled items.
  \item But also use the queues approach and ``listen'' to the queues:
    do a check and if the probability estimate from the queue is
    statistically significantly and substantially higher, switch to
    the queue estimate.
\end{itemize}

}

\co{
This way, \ie with some additional logic, we can use less space (than
the \qu approach would need) and still hone in on the target
probabilities much more accurately (but still, within effective
precision limits: we can not get arbitrarily close).
}

\begin{figure}[t]
\hspace{0.2cm} \begin{minipage}[t]{0.55\linewidth}
\fontsmall{
\hspace{-0.2cm} {\bf UpdateDyal}($o$) // latest observation $o$. \\
        \hspace*{0.3cm} // Data structures: $qMap$, $\emamap$, $\lrmap$. \\ 
        \hspace*{0.3cm} $\qprob, qcount \leftarrow$ GetQinfo($qMap, o$) \\
        \hspace*{0.3cm} UpdateQueues($qMap$, $o$) \\
        \hspace*{0.3cm} $free \leftarrow$ WeakenEdges($o$) // Weaken, except for $o$.\\
        \hspace*{0.3cm} If $\qprob$ == $0$: // item is currently \nsn? \\
        \hspace*{0.6cm} Return \\
        \hspace*{0.3cm} $\emapr \leftarrow$ $\emamap.get(o, 0.0)$ // 0 if $o\not\in$ Map. \\
        \hspace*{0.3cm} // listen to queue? \\ 
        \hspace*{0.3cm} If Q\_SignificantlyHigh($\emapr, \qprob, qcount$): \\
        \hspace*{0.6cm} // set initial rate and \pr using queue for $o$.\\ 
        \hspace*{0.6cm} $\lrmap[o] \leftarrow qcount^{-1}$ \\ 
        \hspace*{0.6cm} $\delta \leftarrow \min(\qprob - \emapr, free)$ \\
        \hspace*{0.3cm} Else: // EMA update with harmonic decay. \\
        \hspace*{0.6cm} $\lr \leftarrow \lrmap[o]$ \\
        \hspace*{0.6cm}  $\delta \leftarrow \min((1 - \emapr) * \lr, free)$ \\
        \hspace*{0.6cm} $\lrmap[o] \leftarrow DecayRate(\lr)$ // \sec \ref{sec:harm} \\
        \hspace*{0.3cm} \mbox{$\emamap[o] \leftarrow \delta + \emapr$} \\ \\
\hspace{-0.2cm} {\bf Q\_SignificantlyHigh}($\emapr, \qprob, qcount$) \\
        \hspace*{0.3cm} // Parameter: $sig\_thrsh$ (significance threshold). \\
 \hspace*{0.3cm}           If $\emapr == 0$: // when 0, listen to \qu. \\ 
 \hspace*{0.6cm}              Return True \\
 \hspace*{0.3cm}           If $\qprob \le \emapr$: \  Return False \\
 \hspace*{0.3cm}           Return $qcount * KL(\qprob, \emapr) \ge sig\_thrsh$ \\ \\
\hspace{-0.2cm} {\bf Q\_SignificantlyLow}($\emapr, \qprob, qcount$) \\
        \hspace*{0.3cm} // Parameter: $sig\_thrsh$ (significance threshold). \\
        \hspace*{0.3cm} If $\emapr \le \qprob$: \ \ Return False \\
        \hspace*{0.3cm} Return $qcount * KL(\qprob, \emapr) \ge sig\_thrsh$\\
        }
\end{minipage}
\hspace{0.1cm}
\begin{minipage}[t]{0.55\linewidth}
  \fontsmall{
    {\bf WeakenEdges}($o$) \\
 \hspace*{0.3cm} // Weaken weights and return available \pr mass. \\
  \hspace*{0.3cm} // Data structures: $qMap$, $\emamap$, $\lrmap$. \\ 
  \hspace*{0.3cm} // Parameter: $\minprob$. \\ 
  \hspace*{0.3cm} // \mbox{Weaken all except for $o$. Returns the free mass.} \\
  \hspace*{0.3cm}  $used  \leftarrow 0$ \\
  \hspace*{0.3cm} For each item and learning-rate, $i$, $\lr$ in $\lrmap$: \\
        \hspace*{0.6cm} If $i == o$: // Don't weaken $o$\\
        \hspace*{0.9cm}    $used \leftarrow used + \emamap[o]$ \\
        \hspace*{0.9cm}    Continue \\
        \hspace*{0.6cm}    // Possibly listen (switch) to the queue. \\
        \hspace*{0.6cm} $\qprob, qcount \leftarrow$ GetQinfo($qMap, i$) \\
        \hspace*{0.6cm} // if too low, drop $i$\\
        \hspace*{0.6cm} If $\max(\emamap[i], \qprob) < \minprob$:\\
        \hspace*{0.9cm}    $\emamap$.delete(i) \& $\lrmap$.delete(i) \\
        \hspace*{0.9cm}    Continue // remove item and go to next. \\
        \hspace*{0.6cm} If Q\_SignificantlyLow($\emamap[i], \qprob, qcount$): \\
        \hspace*{0.9cm}    $\emamap[i] \leftarrow \qprob$ // Set to q info. \\ 
        \hspace*{0.9cm} $\lrmap[i] \leftarrow qcount^{-1}$ \\
        \hspace*{0.6cm} Else: // weaken as usual \\
        \hspace*{0.9cm}    $\emamap[i] \leftarrow (1-\lr) * \emamap[i]$\\
        \hspace*{0.9cm} $\lrmap[i] \leftarrow DecayRate(\lr)$ \\
        \hspace*{0.6cm}    $used \leftarrow used + \emamap[i]$ \\
  \hspace*{0.3cm}  Return $1.0 - used$ // The free (available) mass \\ \\
        {\bf GetQInfo}($qMap$, $o$) \\
        \hspace*{0.3cm} If $o \notin qMap$: // no queue for $o$? \\
        \hspace*{0.6cm}    Return 0, 0 \\
        \hspace*{0.3cm} $q \leftarrow  qMap$.get($o$) \\
        \hspace*{0.3cm} Return GetPR($q$), GetCount($q$) \\
        }
\end{minipage}

\caption{Pseudo code for \qdn, an extension of EMA, where the overall
  weaken-and-boost update structure remains the same. Here each
  predictand entry, or edge, has a small queue and its own learning
  rate, in addition to the \pr estimate (a weight) (see \fig
  \ref{fig:maps_edges}), and during weakening and boosting, the queue
  estimates could be used to reset the the edge weight and its
  learning rate. }

\label{fig:ema_qdyal}
\end{figure}

\co{
For below: $p_1, p_2, \lr \in [0, 1]$.

To go from 0 to $p_2$ with learning rate $\lr$, assuming $\lr << p_2$,
requires observing $O(\frac{1}{\lr}p_2 )$ heads (positive outcomes),
and each positive outcome requires $\frac{1}{p_2}$ tosses.  Thus we
require $O(\frac{1}{p_2}\frac{p_2}{\lr})=O(\frac{1}{\lr})$ tosses to
converge.


Similarly, to go from $p_1$ to $p_2$ with $\lr$, assuming $p_1 < p_2$
and $\lr << p_2$, requires observing $O(\frac{1}{\lr}(p_2 - p_1))$
heads (positive outcomes), and since positive outcome requires
$\frac{1}{p_2}$ tosses, we get $O((\frac{1}{p_2})\frac{p_2 - p_1}{\lr})
= O(\frac{1}{\lr}(1-\frac{p_1}{p_2}))$, to go/converge from $p_1$ to
$p2, p_2 > p_1$.  (verify!)

}

\subsection{When and How to Listen (Switch) to the Queue Estimates: Statistical (Binomial-Tail) Tests}
\label{sec:binom}

The queue for each item provides two numbers, the number of cells of
the queue is the implicit number of positives observed recently, and
the total count across the queue cells is the number of trials
(roughly, the GetCount() function).
A \pr estimate $\qprob$ is derived from these numbers (GetProb()).
%
%
We also have an estimate of \prn, \emapr, from the EMA weights
$\emamap$, which we expect to be generally more accurate with lower
variance than $\qpr$, but specially in the face of non-stationarity,
from time to time this estimate could be out-dated and should be
discarded. Based on the counts and the queue estimate $\qprob$, we can
perform a binomial-tail test that asks whether, when assuming $\emapr$
is the true \prn, one can observe the alternative \pr $\qprob$ in
$qcount$ trials, with some reasonable probability. This binomial tail
can be approximated (lower and upper-bounded) efficiently in $O(1)$
time when one has the number of trials and the observed \pr $\qprob$
\cite{arratia89,ash90,binomText2021}, and it tells us how likely it is
that a binary event with assumed true \pr $\emapr$ could lead to the
counts and the \qpr \ estimate (of the queue). As seen in \fig
\ref{fig:ema_qdyal}, the approximation is based on the (binary) \kl
\ divergence ($\kl(\qpr||\emapr)$). When this event is sufficiently
unlikely, or, equivalently, when the binomial-tail score is
sufficiently high, \qd switches to the queue estimate and sets the new
rate $\lr$ accordingly too (see \sec \ref{sec:harm} explains the
connection between rate $\lr$ and counts). By default we use a score
threshold of 5, corresponding to 99\% confidence (meaning 0.01
probability of a false-positive, that is no change and the EMA
estimate should have been kept).  In a few experiments, we report on
the sensitivity of \qd to the choice of the binomial threshold.

As mentioned in \sec \ref{sec:harm}, specially now that each edge has
its own rate, the rate $\lrmap[i]$ can be used as a measure of the
predictor's uncertainty around the \pr estimate
$\emamap[i]$. Initially, when set to the queue estimate, the rate can
be relatively high, and is lowered over time via harmonic decay.


\subsection{\sd Maintenance and Convergence}
\label{sec:sd_maintain}

We can verify from the logic of \qdn, \ie the weakenings and
strengthening of the weights and the bounding of $\delta$ (the \pr to
add) by the $free$ variable (the available, or unallocated, \pr mass),
that the $\emamap$ of \qd always corresponds to an \sdn: the \pr
values kept are positive and never sum to more then 1.0.  In the
stationary setting, the properties of plain EMA apply and the \prs
should coverge to the true \prsn, except there is some low probability
that once in a while switching (listening) to the queue may occur with
a more variant estimate (a false-positive switch). We leave a more
formal analysis to future work.

\subsection{Pruning (Space Management) and Asymptotic Complexity} 
\label{sec:pruning}

The pruning logic for the three maps of \qd is identical to the logic
for pruning for the \qu method of \sec \ref{sec:qspace} except that
when an entry is deleted from the queue, its corresponding entries
(key-value pairs) in $\lrmap$ \ and $\emamap$ \ are also deleted when
they exist. Thus the set of keys in the $\qmap$ \ of the predictor
will always be a superset of the key sets in $\lrmap$ \ and $\emamap$
(note: the keysets in $\lrmap$ \ and $\emamap$ are always kept
identical).  The size of $\emamap$ does not exceed
$O(\frac{1}{\pmin})$ as it is a \sdn, and $\qmap$ is pruned
periodically as well. Updating time cost for \qd is similar to \qun:
each edge (predictand) is examined and the corresponding queue and
possibly EMA weights and learning rate are updated, each test and
update (weakening and boosting) take $O(1)$ (per edge) where we
assume \qcap is constant, and update and prediction take $O(|\qmap|)$
time.





\section{Experiments}
\label{sec:syn_exps}

We begin with the synthetic experiments, wherein we generate sequences
knowing the true \sdn s $\oat{\P}$. At any time point, for evaluation,
the \fc() function of \fig \ref{code:norming_etc} is applied to the
output of all predictors, with $\pns=\pmin=0.01$.  The default
parameters for \qd are \qcap of 3 and binomial-tail threshold of $5$,
and we report the performance of \qd for different $\lrmin$, often set
at 0.001. We are interested in $\lrmin \le 0.001$ because our target
range is learning \prn s in $[0.01, 1.0]$ well.  For static EMA, we
report the (fixed) $\lr$ used, for harmonic EMA, the $\lrmin$, and for
\qun, the \qcapn. The \nsm used is described in \ref{sec:evalns2}
(default $\nsthrsh=2$). All code is implemented in Python, and we report timing
for several of the experiments (those taking longer than a few
minutes).


\subsection{Tracking a Single Item, Stationary}
\label{sec:single1}





All the prediction techniques are based on estimating the probability
(\prn) for each item separately (treating all other observations as
negative outcomes), so we begin with assessing the quality of the
predictions for a single item in the binary stationary setting of \sec
\ref{sec:binary}.
Thus, when $\tp=0.1$, about $10\%$ of the sequence is 1, the rest 0.
Table \ref{tab:stat_devs} presents the deviation rates of \qun, EMA,
and \qdn, under a few parameter variations, and for $\tp \in \{0.1,
0.05, 0.01\}$.  Sequences are each 10k time points long, and deviation
rates are averaged over 200 such sequences.\footnote{The methods keep
track of the probabilities of all the items they deem salient, in this
case, both 0 and 1, but we focus on the \pr estimates
$\oat{\ep}$ for item 1.}

\co{ , for which the true probability is $\tp$, as a prediction method
  processes a binary sequence $\seq{o}{}{}$, in the stationary
  setting.  Here, a sequence can be viewed as binary, $\oat{o} \in
  \{0, 1\}$, and we are assessing the quality of probability estimates
  for item 1.  }

Under this stationarity setting, higher \qcap helps the \qu technique:
\qu with \qcap 10 does better than \qcap of 5, but, specially for
$d=1.5$, both tend to substantially lag behind the best of the EMA
variants. EMA with harmonic decay, with an appropriately low $\lrmin$,
does best across all $\tp$. If we anticipate that the useful items to
predict will have \prs in the $0.01$ to $1.0$ range, in a stationary
world, then setting $\lrmin$ for harmonic EMA to a low value,
$\frac{0.01}{k}$, where $k \ge 10$, is adequate.\footnote{In the
stationary setting, one can set $\lrmin=0$. } In this stationary and
binary setting, the complexity of \qd is not needed, and harmonic EMA
is sufficient. Still \qd is the second best.  Static EMA is not
flexible enough, and one has to anticipate what $\tp$ is and set the
rate appropriately. For instance, when $\tp=0.01$, EMA with the same
$\lr=0.01$ is not appropriate, resulting in too much
variance. Finally, we observe that the deviation rates, as well as the
variances, for any method, degrade (increase) somewhat as $\tp$ is
lowered from $0.1$ to $0.01$. In ten thousand draws (during sequence
generation), there are fewer positive observations with lower $\tp$,
and estimates will have higher variance (see also Appendix
\ref{sec:one_qcell}).

\co{
  
\begin{itemize}
\item Explain how each sequence is generated. 

\item How each predictor type observes and updates.

\item How you measure violation and average.
  
\item In these experiments, being stationary, EMA variants with a low
  rate (0.001) do best.  EMA with harmonic decay does best over all.

\item As the target probability $p$ goes down, fewer positive outcomes
  are observed (in a span of fixed 10000 draws), and the performances
  degrade (more violations and higher variance) for all methods.
\end{itemize}

}

\begin{table}[t]\center
  \begin{tabular}{ |c?c|c?c|c| }     \hline
    deviation threshold $\rightarrow$  &  1.5 & 2 & 1.5 & 2   \\ \thickhline  
  &  \multicolumn{2}{c?}{\qun, 5 (\qcap of 5) } &  \multicolumn{2}{c|}{\qun, 10 (\qcap of 10)}  \\ \hline
 $\tp = 0.10$ & 0.385 $\pm$ 0.026  & 0.129 $\pm$ 0.020 &  0.191 $\pm$ 0.029  & 0.026 $\pm$ 0.010  \\ \hline
 $\tp = 0.05$ & 0.405 $\pm$ 0.034  & 0.142 $\pm$ 0.028  & 0.211 $\pm$ 0.044  & 0.035 $\pm$ 0.016  \\ \hline
    & \multicolumn{2}{c?}{static EMA, 0.01 ($\lr$ of 0.01) } &  \multicolumn{2}{c|}{static EMA, 0.001 ($\lr$ of 0.001) } \\ \hline
 $\tp=0.10$ & 0.075 $\pm$ 0.021  & 0.013 $\pm$ 0.006 & 0.113 $\pm$ 0.020  & 0.071 $\pm$ 0.012 \\ \hline
 $\tp=0.05$ &  0.211 $\pm$ 0.034  & 0.050 $\pm$ 0.019 & 0.118 $\pm$ 0.029  & 0.072 $\pm$ 0.016 \\ \hline
 & \multicolumn{2}{c?}{ harmonic EMA, 0.001 ($\lrmin$ of 0.001) } &  \multicolumn{2}{c|}{\qd, 0.001 ($\lrmin$ of 0.001)}  \\  \hline
$\tp=0.10$ & 0.006 $\pm$ 0.007  & 0.002 $\pm$ 0.003  &  0.018 $\pm$ 0.015  & 0.008 $\pm$ 0.006 \\ \hline
$\tp=0.05$ & 0.012 $\pm$ 0.013  & 0.005 $\pm$ 0.005 & 0.028 $\pm$ 0.023  & 0.014 $\pm$ 0.012  \\ \hline
  \end{tabular}
  \vspace*{.2cm}
  \caption{Synthetic single-item stationary: Deviation rates,
    for two deviation thresholds $d\in \{1.5, 2\}$, averaged over 200
    randomly generated sequences of 10000 binary events (0 or 1),
    for target probability $\tp \in \{ 0.05, 0.1\}$.
    As an example, for $\tp=0.1$, about $10\%$ of the items will be 1,
    the rest are 0s in the sequence, and the \sma predicts a
    probability $\hat{p}^{(t)}$ at every time point $t$ for
    $\oat{o}=1$ (then updates), and we observe from the table that
    about $38\%$ of time, $\max(\frac{\hat{p}^{(t)}}{p},
    \frac{p}{\hat{p}^{(t)}}) > 1.5$ (\ie $\hat{p}^{(t)} > 0.15$ or
    $\hat{p}^{(t)} < \frac{0.1}{1.5}$), for the \qu \sma with
    \qcap 5 (top left). The lower the deviation rates the better. In
    this stationary setting, we see improvements with larger queue
    capacities (as expected), and a lower $\minlr$ of 0.001 performs
    best for the EMA variants. For \qun, with \qcap$=10$,
    compare to Table \ref{tab:plain_counting}, $t_p=10$. }
\label{tab:stat_devs}
\end{table}


\co{
  
\begin{table}[t]\center
  \begin{tabular}{ |c|c|c|c|c|c|c| }     \hline
    deviation & \multicolumn{4}{c|}{queue 5 } &  \multicolumn{2}{c|}{10}  \\
    \cline{2-7}
 threshold $\rightarrow$ & 1.1 & 1.5 & 2 & 3 &  1.1 & 2  \\ \hline
p=0.1 &  0.84$\pm0.01$  &  0.38$\pm0.02$ & 0.12 $\pm 0.02$  & 0.01 $\pm 0.005$ & 0.76 $\pm 0.02$ & 0.02 $\pm 0.01$ \\ \hline
p=0.05 & 0.84$\pm0.02$ & 0.40$\pm0.04$  &  0.14$\pm0.03$ & 0.02 $\pm 0.01$ & 0.76 $\pm 0.03$ & 0.03 $\pm 0.02$ \\ \hline
p=0.01 & 0.85$\pm0.03$ & 0.40$\pm0.08$ & 0.15$\pm0.06$ & 0.02 $\pm 0.02$   & 0.77 $\pm 0.06$  & 0.04 $\pm 0.04$ \\ \hline
  \end{tabular}
  \vspace*{.2cm}
  \caption{Mean fractions of times there was a deviation violation,
    \ie whenever $\max(\frac{\hat{p}}{p}, \frac{p}{\hat{p}}) > d$,
    on randomly generated sequences of 10000 binary events (0 or 1),
    averaged over 200 trials (standard deviations are also reported),
    for a few choices of deviation thresholds $d\in \{1.1, 1.5, 2,
    3\}$, target probabilities $p \in \{ 0.01, 0.05, 0.1 \}$ (note: a
    fixed/stationary probability), and queue capacities of 5 and
    10. As an example, for $p=0.1$, about $10\%$ of the items will be
    1, the rest are 0s in the sequence, and the predictor predicts a
    probability $\hat{p}$ at every time point, and we observe from the
    table that about $38\%$ of time, $\max(\frac{\hat{p}}{p},
    \frac{p}{\hat{p}}) > 3$ (\ie $\hat{p} > 0.3$ or $\hat{p} <
    0.033$), for a predictor with a queue of capacity 5. The lower the
    deviation violations the better. We see improvements with larger
    queue capacities (as expected), while as $p$ goes down, the
    prediction task becomes a little harder (more violations and/or
    higher variance).  }
\label{tab:devs1}
\end{table}
}

\co{
\begin{table}[t]\center
  \begin{tabular}{ |c|c|c|c|c|c|c|c|c| }     \hline
    deviation & \multicolumn{2}{c|}{ fixed1 } &  \multicolumn{2}{c|}{fixed2}  & \multicolumn{2}{c|}{ harmonic } &  \multicolumn{2}{c|}{\qd}  \\
    \cline{2-7}
 threshold $\rightarrow$ &  1.5 & 2 & 1.5 & 2 &  1.5 & 2 & 1.5 & 2  \\ \hline
p=0.1 & 0.209 $\pm$ 0.034  & 0.050 $\pm$ 0.019 & 0.116 $\pm$ 0.029  & 0.071 $\pm$ 0.017 & 0.013 $\pm$ 0.015  & 0.005 $\pm$ 0.005 & 0.029 $\pm$ 0.030  & 0.016 $\pm$ 0.013  \\ \hline
p=0.05 &  0.211 $\pm$ 0.034  & 0.050 $\pm$ 0.019 & 0.118 $\pm$ 0.029  & 0.072 $\pm$ 0.016 & 0.012 $\pm$ 0.013  & 0.005 $\pm$ 0.005 & 0.026 $\pm$ 0.023  & 0.014 $\pm$ 0.011  \\ \hline
p=0.01 &  0.210 $\pm$ 0.037  & 0.049 $\pm$ 0.018 & 0.116 $\pm$ 0.029  & 0.073 $\pm$ 0.017 & 0.011 $\pm$ 0.014  & 0.005 $\pm$ 0.005 & 0.029 $\pm$ 0.028  & 0.014 $\pm$ 0.012  \\ \hline
  \end{tabular}
  \vspace*{.2cm}
  \caption{Same as above, but with static EMA (fixed and harmonic), and qdyal. }
\label{tab:devs2}
\end{table}
}

\co{
\begin{table}[t]\center
  \begin{tabular}{ |c|c|c|c|c| }     \hline
    deviation & \multicolumn{2}{c|}{\qun, 5 } &  \multicolumn{2}{c|}{\qun, 10}  \\
    \cline{2-5}    
 threshold $\rightarrow$ &  1.5 & 2 & 1.5 & 2   \\ \hline  
 0.025 $\leftrightarrow$ 0.10, 10 &  0.412 $\pm$ 0.035  & 0.161 $\pm$ 0.027  & 0.286 $\pm$ 0.038  & 0.109 $\pm$ 0.017   \\  \hline
 0.025 $\leftrightarrow$ 0.10, 50 &  0.407 $\pm$ 0.041  & 0.155 $\pm$ 0.032 & 0.268 $\pm$ 0.051  & 0.079 $\pm$ 0.019   \\  \hline
    & \multicolumn{2}{c|}{ fixed EMA, 0.01 } &  \multicolumn{2}{c|}{fixed EMA, 0.001 } \\ \hline
 0.025 $\leftrightarrow$ 0.10, 10 &   0.308 $\pm$ 0.032  & 0.141 $\pm$ 0.026 & 0.800 $\pm$ 0.042  & 0.470 $\pm$ 0.031 \\  \hline
 0.025 $\leftrightarrow$ 0.10, 50 &   0.334 $\pm$ 0.037  & 0.146 $\pm$ 0.027 & 0.693 $\pm$ 0.061  & 0.354 $\pm$ 0.033 \\  \hline
 & \multicolumn{2}{c|}{ harmonic EMA, 0.001 } &  \multicolumn{2}{c|}{qdyal, 0.001}  \\  \hline
 0.025 $\leftrightarrow$ 0.10, 10 &  0.712 $\pm$ 0.042  & 0.424 $\pm$ 0.039 & 0.551 $\pm$ 0.096  & 0.328 $\pm$ 0.059   \\  \hline
 0.025 $\leftrightarrow$ 0.10, 50 &   0.628 $\pm$ 0.067  & 0.317 $\pm$ 0.045  & 0.382 $\pm$ 0.093  & 0.217 $\pm$ 0.049  \\  \hline
  \end{tabular}
  \vspace*{.2cm}
  \caption{ Like above, tracking one event (binary sequence, 10k observations), but non-stationary.  }
\label{tab:devs_non0}
\end{table}
}

\co{
\begin{table}[t]\center
  \begin{tabular}{ |c|c|c|c|c| }     \hline
    deviation & \multicolumn{2}{c|}{\qun, 5 } &  \multicolumn{2}{c|}{\qun, 10}  \\
    \cline{2-5}    
 threshold $\rightarrow$ &  1.5 & 2 & 1.5 & 2   \\ \hline  
  [0.025, 0.1], 10 & 0.410 $\pm$ 0.035  & 0.158 $\pm$ 0.027  & 0.282 $\pm$ 0.037  & 0.110 $\pm$ 0.018  \\ \hline
  [0.025, 0.1], 50 & 0.412 $\pm$ 0.040  & 0.155 $\pm$ 0.033  & 0.261 $\pm$ 0.047  & 0.077 $\pm$ 0.020  \\ \hline
    & \multicolumn{2}{c|}{ static EMA, 0.01 } &  \multicolumn{2}{c|}{static EMA, 0.001 } \\ \hline
  [0.025, 0.1], 10 & 0.305 $\pm$ 0.032  & 0.140 $\pm$ 0.023  & 0.802 $\pm$ 0.042  & 0.472 $\pm$ 0.031  \\ \hline
  [0.025, 0.1], 50 & 0.331 $\pm$ 0.036  & 0.148 $\pm$ 0.028  & 0.692 $\pm$ 0.060  & 0.352 $\pm$ 0.037  \\ \hline
  & \multicolumn{2}{c|}{ harmonic EMA, 0.001 } &  \multicolumn{2}{c|}{qdyal, 0.001}  \\  \hline
   [0.025, 0.1], 10 & 0.705 $\pm$ 0.044  & 0.424 $\pm$ 0.036  & 0.661 $\pm$ 0.077  & 0.391 $\pm$ 0.050  \\ \hline
  [0.025, 0.1], 50 & 0.633 $\pm$ 0.065  & 0.319 $\pm$ 0.043  & 0.501 $\pm$ 0.090  & 0.269 $\pm$ 0.049  \\ \hline      
  \end{tabular}
  \vspace*{.2cm}
  \caption{ Like above, tracking one event but non-stationary (binary sequence, 10k
    observations in each sequence, 500 sequences).  }
\label{tab:devs_non0}
\end{table}
}

\co{

  
\begin{table}[t]\center
  \begin{tabular}{ |c|c|c|c|c| }     \hline
    deviation $\rightarrow$ &  1.5 & 2 & 1.5 & 2   \\   
    \cline{2-5}    
    threshold & \multicolumn{2}{c|}{\qun, 5 } &  \multicolumn{2}{c|}{\qun, 10}  \\ \hline
  [0.025, 0.25], 10 & 0.423 $\pm$ 0.029  & 0.189 $\pm$ 0.023  & 0.395 $\pm$ 0.021  & 0.234 $\pm$ 0.014  \\ \hline
  [0.025, 0.25], 50 & 0.382 $\pm$ 0.038  & 0.131 $\pm$ 0.028  & 0.222 $\pm$ 0.039  & 0.062 $\pm$ 0.017  \\ \hline 
  Uniform, 10 & 0.429 $\pm$ 0.030  & 0.207 $\pm$ 0.024  & 0.483 $\pm$ 0.028  & 0.296 $\pm$ 0.028  \\ \hline
  Uniform, 50 & 0.361 $\pm$ 0.038  & 0.128 $\pm$ 0.028  & 0.224 $\pm$ 0.040  & 0.074 $\pm$ 0.017  \\ \hline 
    & \multicolumn{2}{c|}{ static EMA, 0.01 } &  \multicolumn{2}{c|}{static EMA, 0.001 } \\ \hline
  [0.025, 0.25], 10 & 0.510 $\pm$ 0.024  & 0.357 $\pm$ 0.020  & 0.996 $\pm$ 0.006  & 0.760 $\pm$ 0.043  \\ \hline
  [0.025, 0.25], 50 & 0.255 $\pm$ 0.045  & 0.129 $\pm$ 0.028  & 0.705 $\pm$ 0.032  & 0.560 $\pm$ 0.023  \\ \hline 
  Uniform, 10 & 0.686 $\pm$ 0.030  & 0.477 $\pm$ 0.035  & 0.818 $\pm$ 0.038  & 0.693 $\pm$ 0.054  \\ \hline
  Uniform, 50 & 0.397 $\pm$ 0.056  & 0.209 $\pm$ 0.045  & 0.775 $\pm$ 0.085  & 0.602 $\pm$ 0.099  \\ \hline 
  & \multicolumn{2}{c|}{ Harmonic EMA, 0.01 } &  \multicolumn{2}{c|}{Harmonic EMA, 0.001 } \\ \hline
  [0.025, 0.25], 10 & 0.502 $\pm$ 0.025  & 0.351 $\pm$ 0.021  & 0.957 $\pm$ 0.021  & 0.668 $\pm$ 0.053  \\ \hline
  [0.025, 0.25], 50 & 0.247 $\pm$ 0.046  & 0.123 $\pm$ 0.028  & 0.606 $\pm$ 0.031  & 0.494 $\pm$ 0.018  \\ \hline 
  Uniform, 10 & 0.684 $\pm$ 0.033  & 0.476 $\pm$ 0.034  & 0.813 $\pm$ 0.036  & 0.683 $\pm$ 0.050  \\ \hline
  Uniform, 50 & 0.395 $\pm$ 0.056  & 0.204 $\pm$ 0.043  & 0.761 $\pm$ 0.079  & 0.592 $\pm$ 0.097  \\ \hline 
  & \multicolumn{2}{c|}{ \qd, 0.01 } &  \multicolumn{2}{c|}{\qd, 0.001 } \\ \hline
  [0.025, 0.25], 10 & 0.480 $\pm$ 0.029  & 0.332 $\pm$ 0.025  & 0.586 $\pm$ 0.066  & 0.408 $\pm$ 0.061  \\ \hline
  [0.025, 0.25], 50 & 0.251 $\pm$ 0.048  & 0.124 $\pm$ 0.031  & 0.148 $\pm$ 0.066  & 0.084 $\pm$ 0.037  \\ \hline 
  Uniform, 10 & 0.585 $\pm$ 0.035  & 0.362 $\pm$ 0.034  & 0.566 $\pm$ 0.052  & 0.350 $\pm$ 0.045  \\ \hline
  Uniform, 50 & 0.319 $\pm$ 0.065  & 0.140 $\pm$ 0.049  & 0.301 $\pm$ 0.081  & 0.153 $\pm$ 0.049  \\ \hline 
  \end{tabular}
  \vspace*{.2cm}
  \caption{First set of non-stationary experiments: ``oscillation''
    experiments..  Like above (actually the above will be removed?! ),
    0.025 $\leftrightarrow$ 0.25, as well as uniform, tracking one
    event but non-stationary: binary sequence, 10k observations in
    each sequence, averaged over 500 sequences. For the case of 0.025
    $\leftrightarrow$ 0.25, each subsequence (wherein $p$ is constant)
    is roughly same length: 400 long for min\_obs=10, or 2k for
    min\_obs=50. For drawing from Uniform, each subsequence is at
    least 1k (see ?? text for details). }
\label{tab:devs_non1}
\end{table}
}

\co{
    [0.025, 0.25], 10 & 0.387 $\pm$ 0.033  & 0.143 $\pm$ 0.024  & 0.260 $\pm$ 0.029  & 0.108 $\pm$ 0.012  \\ \hline
  [0.025, 0.25], 50 & 0.399 $\pm$ 0.038  & 0.145 $\pm$ 0.031  & 0.248 $\pm$ 0.045  & 0.082 $\pm$ 0.018  \\ \hline
  Uniform, 10 & 0.230 $\pm$ 0.048  & 0.057 $\pm$ 0.020  & 0.102 $\pm$ 0.033  & 0.021 $\pm$ 0.016  \\ \hline
  Uniform, 50 & 0.238 $\pm$ 0.054  & 0.064 $\pm$ 0.026  & 0.108 $\pm$ 0.040  & 0.021 $\pm$ 0.016  \\ \hline
    & \multicolumn{2}{c|}{ static EMA, 0.01 } &  \multicolumn{2}{c|}{static EMA, 0.001 } \\ \hline
  [0.025, 0.25], 10 & 0.312 $\pm$ 0.026  & 0.183 $\pm$ 0.020  & 0.876 $\pm$ 0.026  & 0.670 $\pm$ 0.028  \\ \hline
  [0.025, 0.25], 50 & 0.342 $\pm$ 0.030  & 0.177 $\pm$ 0.026  & 0.922 $\pm$ 0.043  & 0.708 $\pm$ 0.050  \\ \hline
  Uniform, 10 & 0.090 $\pm$ 0.039  & 0.047 $\pm$ 0.028  & 0.466 $\pm$ 0.131  & 0.275 $\pm$ 0.106  \\ \hline
  Uniform, 50 & 0.105 $\pm$ 0.062  & 0.055 $\pm$ 0.039  & 0.491 $\pm$ 0.146  & 0.290 $\pm$ 0.136  \\ \hline
  & \multicolumn{2}{c|}{ Harmonic EMA, 0.01 } &  \multicolumn{2}{c|}{Harmonic EMA, 0.001 } \\ \hline
  [0.025, 0.25], 10 & 0.306 $\pm$ 0.025  & 0.180 $\pm$ 0.021  & 0.778 $\pm$ 0.032  & 0.591 $\pm$ 0.016  \\ \hline
  [0.025, 0.25], 50 & 0.338 $\pm$ 0.033  & 0.172 $\pm$ 0.026  & 0.830 $\pm$ 0.043  & 0.658 $\pm$ 0.056  \\ \hline
  Uniform, 10 & 0.084 $\pm$ 0.045  & 0.043 $\pm$ 0.028  & 0.375 $\pm$ 0.136  & 0.209 $\pm$ 0.110  \\ \hline
  Uniform, 50 & 0.102 $\pm$ 0.069  & 0.049 $\pm$ 0.042  & 0.406 $\pm$ 0.146  & 0.239 $\pm$ 0.127  \\ \hline   
  & \multicolumn{2}{c|}{ \qd, 0.01 } &  \multicolumn{2}{c|}{\qd, 0.001 } \\ \hline
  [0.025, 0.25], 10 & 0.299 $\pm$ 0.028  & 0.173 $\pm$ 0.024  & 0.310 $\pm$ 0.089  & 0.179 $\pm$ 0.054  \\ \hline
  [0.025, 0.25], 50 & 0.330 $\pm$ 0.035  & 0.170 $\pm$ 0.028  & 0.192 $\pm$ 0.077  & 0.112 $\pm$ 0.043  \\ \hline
  Uniform, 10 & 0.060 $\pm$ 0.031  & 0.023 $\pm$ 0.019  & 0.180 $\pm$ 0.081  & 0.069 $\pm$ 0.041  \\ \hline
  Uniform, 50 & 0.077 $\pm$ 0.063  & 0.028 $\pm$ 0.031  & 0.174 $\pm$ 0.078  & 0.064 $\pm$ 0.040  \\ \hline 
}




\subsection{Tracking a Single Item, Non-Stationary}
\label{sec:syn_non_one}

\begin{table}[t]  \center
  \begin{tabular}{ |c?c|c?c|c| }     \hline
    deviation threshold$\rightarrow$ &  1.5 & 2 & 1.5 & 2   \\ \thickhline  
  change type and  $\minobs \downarrow$  & \multicolumn{2}{c?}{\qun, 5} &  \multicolumn{2}{c|}{\qun, 10}  \\ \hline
  0.025 $\leftrightarrow$ 0.25, 10 & 0.423 $\pm$ 0.029  & 0.189 $\pm$ 0.023  & 0.395 $\pm$ 0.021  & 0.234 $\pm$ 0.014  \\ \hline
  0.025 $\leftrightarrow$ 0.25, 50 & 0.382 $\pm$ 0.038  & 0.131 $\pm$ 0.028  & 0.222 $\pm$ 0.039  & 0.062 $\pm$ 0.017  \\ \hline 
  $\mathcal{U}(0.01,1.0)$, 10 & 0.443 $\pm$ 0.030  & 0.234  $\pm$ 0.042  & 0.468  $\pm$ 0.028  & 0.284 $\pm$ 0.035  \\ \hline
  $\mathcal{U}(0.01,1.0)$, 50 & 0.380 $\pm$ 0.050  & 0.173 $\pm$ 0.080  & 0.247 $\pm$ 0.060  & 0.102 $\pm$ 0.076  \\ \thickhline 
  & \multicolumn{2}{c?}{ static EMA, 0.01 } &  \multicolumn{2}{c|}{static EMA, 0.001 } \\ \hline
  0.025 $\leftrightarrow$ 0.25, 10 & 0.510 $\pm$ 0.024  & 0.357 $\pm$ 0.020  & 0.996 $\pm$ 0.006  & 0.760 $\pm$ 0.043  \\ \hline
  0.025 $\leftrightarrow$ 0.25, 50 & 0.255 $\pm$ 0.045  & 0.129 $\pm$ 0.028  & 0.705 $\pm$ 0.032  & 0.560 $\pm$ 0.023  \\ \hline 
  $\mathcal{U}(0.01,1.0)$, 10 & 0.686 $\pm$ 0.030  & 0.477 $\pm$ 0.035  & 0.818 $\pm$ 0.038  & 0.693 $\pm$ 0.054  \\ \hline
  $\mathcal{U}(0.01,1.0)$, 50 & 0.397 $\pm$ 0.056  & 0.209 $\pm$ 0.045  & 0.775 $\pm$ 0.085  & 0.602 $\pm$ 0.099  \\ \thickhline 
  & \multicolumn{2}{c?}{ harmonic EMA, 0.01 } &  \multicolumn{2}{c|}{harmonic EMA, 0.001 } \\ \hline
  0.025 $\leftrightarrow$ 0.25, 10 & 0.502 $\pm$ 0.025  & 0.351 $\pm$ 0.021  & 0.957 $\pm$ 0.021  & 0.668 $\pm$ 0.053  \\ \hline
  0.025 $\leftrightarrow$ 0.25, 50 & 0.247 $\pm$ 0.046  & 0.123 $\pm$ 0.028  & 0.606 $\pm$ 0.031  & 0.494 $\pm$ 0.018  \\ \hline 
  $\mathcal{U}(0.01,1.0)$, 10 & 0.684 $\pm$ 0.033  & 0.476 $\pm$ 0.034  & 0.813 $\pm$ 0.036  & 0.683 $\pm$ 0.050  \\ \hline
  $\mathcal{U}(0.01,1.0)$, 50 & 0.395 $\pm$ 0.056  & 0.204 $\pm$ 0.043  & 0.761 $\pm$ 0.079  & 0.592 $\pm$ 0.097  \\ \thickhline 
  & \multicolumn{2}{c|}{ \qdn, 0.01 } &  \multicolumn{2}{c|}{\qdn, 0.001 } \\ \hline
  0.025 $\leftrightarrow$ 0.25, 10 & 0.480 $\pm$ 0.029  & 0.332 $\pm$ 0.025  & 0.586 $\pm$ 0.066  & 0.408 $\pm$ 0.061  \\ \hline
  0.025 $\leftrightarrow$ 0.25, 50 & 0.251 $\pm$ 0.048  & 0.150 $\pm$ 0.031  & 0.099 $\pm$ 0.066  & 0.053 $\pm$ 0.037  \\ \hline 
  $\mathcal{U}(0.01,1.0)$, 10 & 0.560  $\pm$ 0.035  & 0.351 $\pm$ 0.054  &  0.529 $\pm$ 0.052  & 0.313  $\pm$ 0.045  \\ \hline
  $\mathcal{U}(0.01,1.0)$, 50 & 0.320 $\pm$ 0.11 & 0.189 $\pm$ 0.128  & 0.246 $\pm$ 0.081  &  0.128 $\pm$ 0.083  \\ \hline 
  \end{tabular}
  \vspace*{.2cm}
  \caption{Synthetic single-item non-stationary:
    Deviation-rates on sequences where $\tp$ oscillates back and forth
    between 0.025 and 0.25, or drawn uniformly at random from the
    interval $[0.01, 1.0]$ ($\tp\sim\mathcal{U}(0.01,1.0)$). Each
    sequence is 10k observations long, and the deviation-rate is
    averaged over 500 such sequences. For the case of 0.025
    $\leftrightarrow$ 0.25, each subsequence (wherein $\tp$ is
    constant) is roughly same length: 400 long for $\minobs=10$, and
    2k for $\minobs=50$. For the rows with $\mathcal{U}(0.01,1.0)$,
    each subsequence only has to meet the $\minobs$ constraint (see
    \ref{sec:syn_non_one} text for details). In this non-stationray
    setting, focused on one changing item, the \qs \smas do best
    perhaps, but \dyal variants do equally well or come close (in
    terms of both their best-case and worst-case performance). }
\label{tab:devs_non1}
\end{table}





We continue with tracking a single item as above, as we report
deviation rates when estimating a single $\tp$, but now the predictors
face non-stationarity. As in the above, sequences of 10000 items are
generated in each trial. We report on two main settings for
non-stationarity: In the first setting, $\tp$ oscillates between, 0.25
and 0.025, thus an abrupt or substantial change (10x) is guaranteed
to occur and frequently. This oscillation is shown as $0.25
\leftrightarrow 0.025$ in Table \ref{tab:devs_non1}. In the second
'uniform' setting, each time $\tp$ is to change, we draw a new $\tp$
uniformly at random from the interval $[0.01, 1.0]$, shown as
$\mathcal{U}(0.01,1.0)$, and in this setting, some changes are large,
others small and could be viewed as drifts. The stable period, during
which $\tp$ cannot change (to allow time for learning), is set as
follows. For both settings, within a stable period, the target item
(item 1) has to be observed at least $\minobs$ times
%
before $\tp$ is eligible to
change, where results for $\minobs \in \{10, 50\}$ are shown in Table
\ref{tab:devs_non1}. Additionally, we impose a general minimum-length
constraint (not just on positive observations) for the $0.25
\leftrightarrow 0.025$ setting: each stable period has to be
$\frac{\minobs}{\min(0.025, 0.25)}$, so that the different periods
would have similar length (expected 400 time points when $\minobs=10$,
and 2000 when $\minobs=50$). In this way, subsequences corresponding
to $\tp=0.25$ would not be too short (otherwise, deviation-rate
performance when $\tp=0.025$ dominates). Thus, with $0.25
\leftrightarrow 0.025$, we get an expected 25 stable subsequences (or
changes in $\tp$) in 10k long sequences with $\minobs=10$, and 5
switches in $\tp$ when $\minobs=50$. For the uniform setting, we did
not impose any extra constraint, and respectively with $\minobs$ of 10
and 50, we get around 200 and 50 stable subsequences in 10k time
points. In this setting, performances in periods when $\tp$ is low do
dominate the deviation rates (see also Table
\ref{tab:devs_non_1000ticks} where the deviation-rates improve when we
impose an overall minimum-length constraint for the uniform setting
too).

\begin{figure}[!ht]
\begin{center}
  \centering
  \subfloat[Min-Rate of 0.001]{{\includegraphics[height=4.5cm,width=6cm]{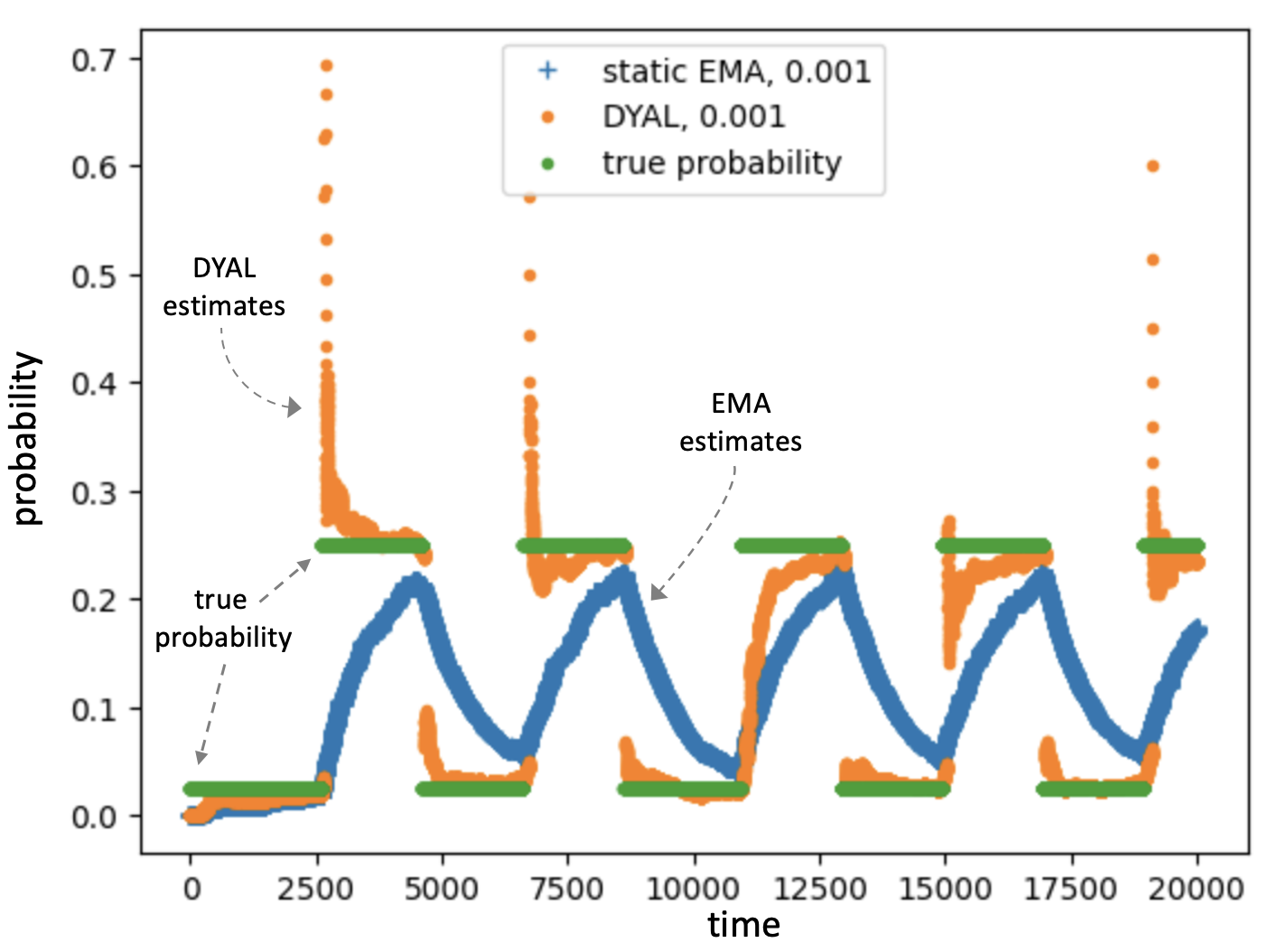}
  }}
\hspace*{.1in}  \subfloat[Min-Rate of 0.01]{{\includegraphics[height=4.5cm,width=6cm]{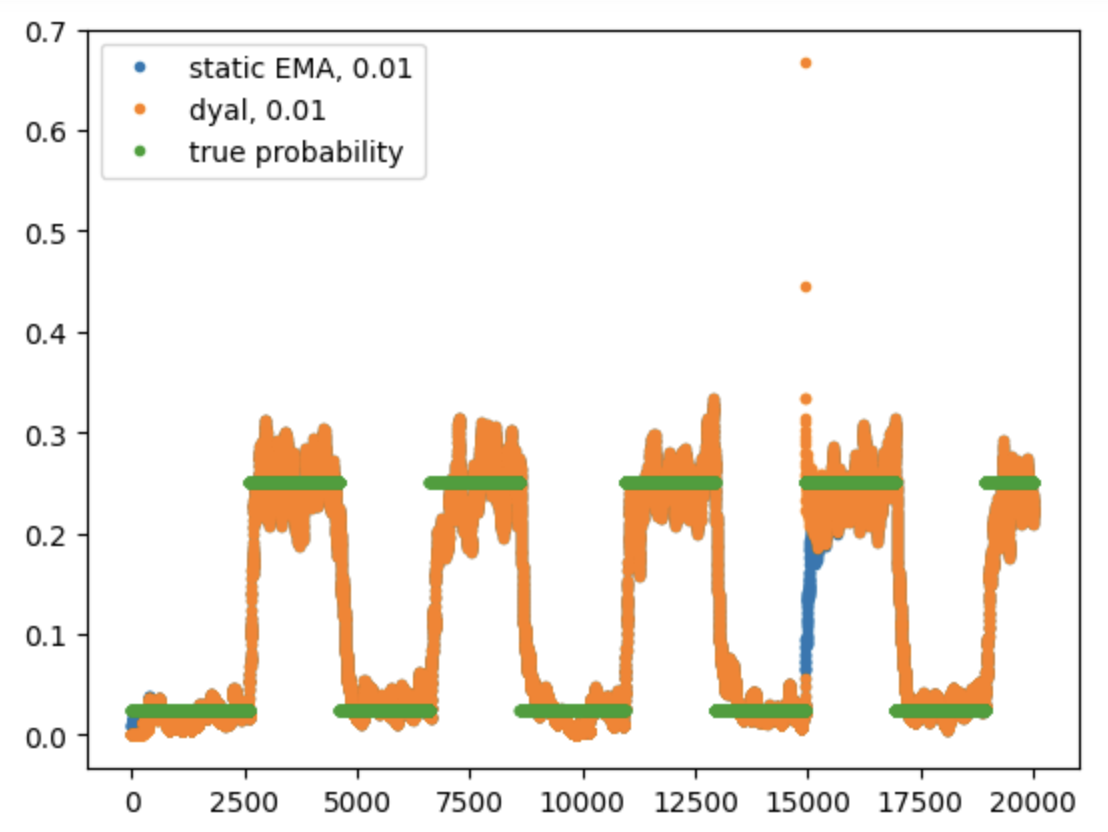}
}}\\
\subfloat[\qd Min-Rate of 0.001 vs 0.01]{{\includegraphics[height=4.5cm,width=6cm]{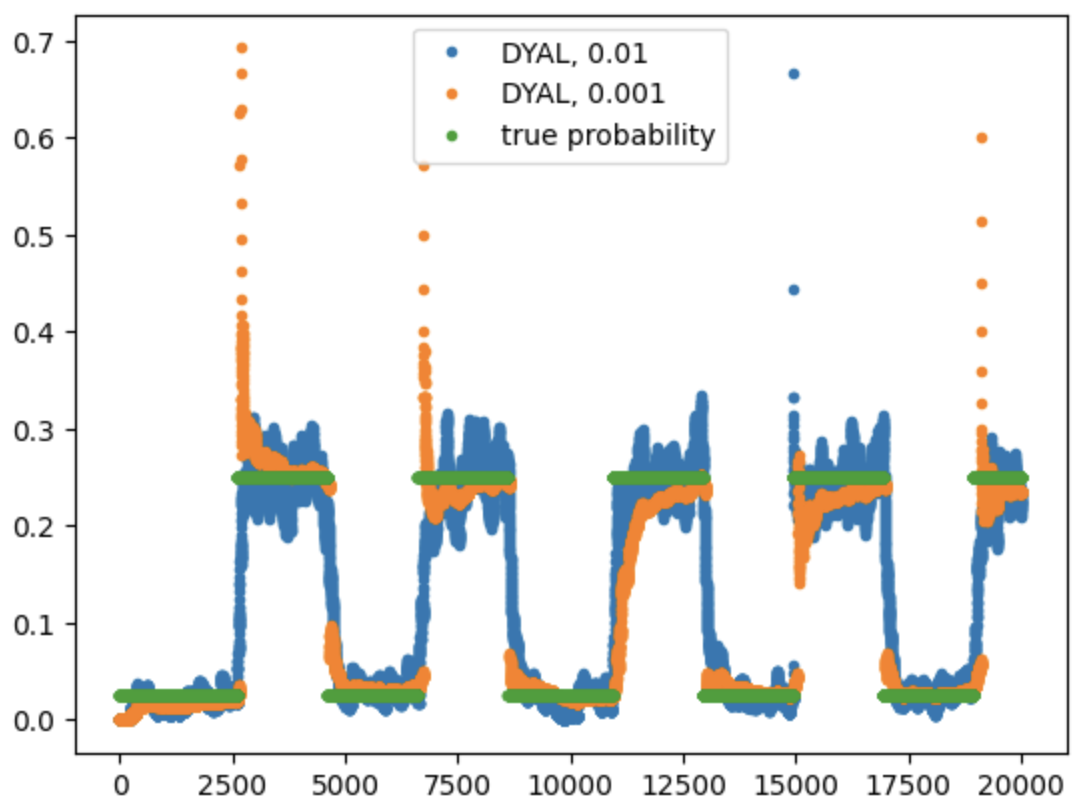}
  }} 
\hspace*{.1in}
  \subfloat[\qu with capacity 5 and 10. ]{{\includegraphics[height=4.5cm,width=6cm]{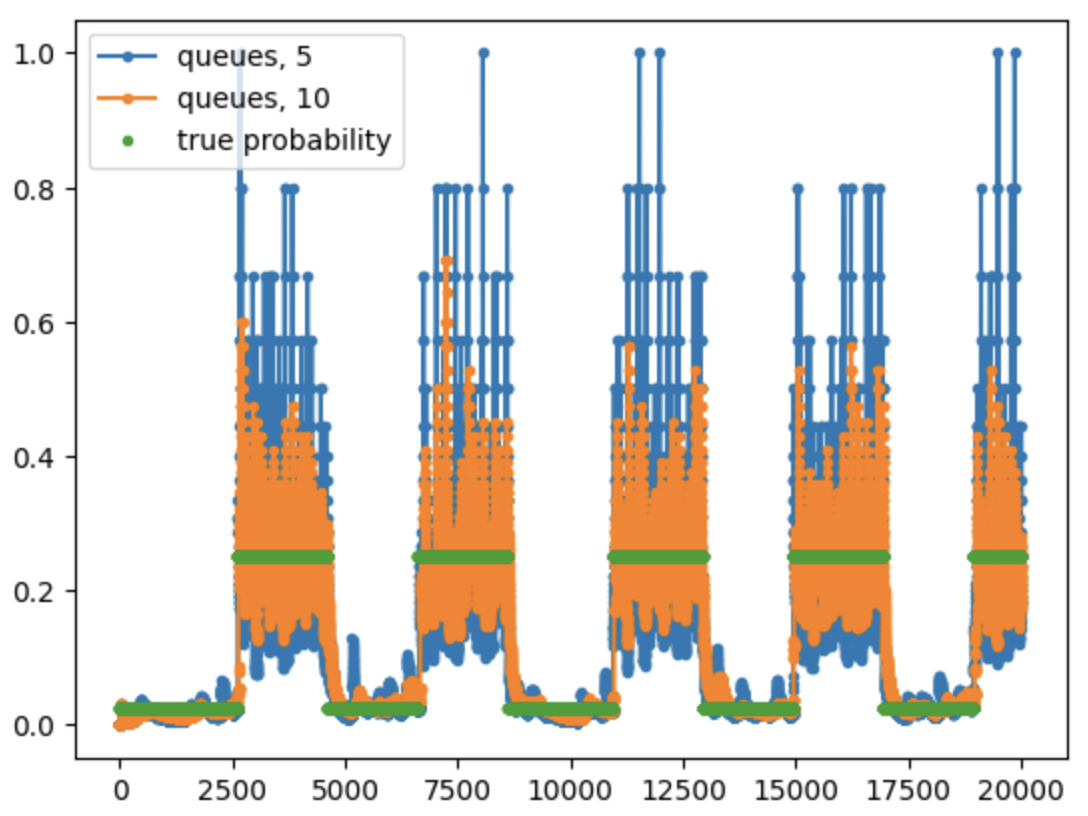}
  }}\\
\hspace{-0.5cm}  \subfloat[Changes in the learning rate $\lr$.]{{\includegraphics[height=4.5cm,width=6cm]{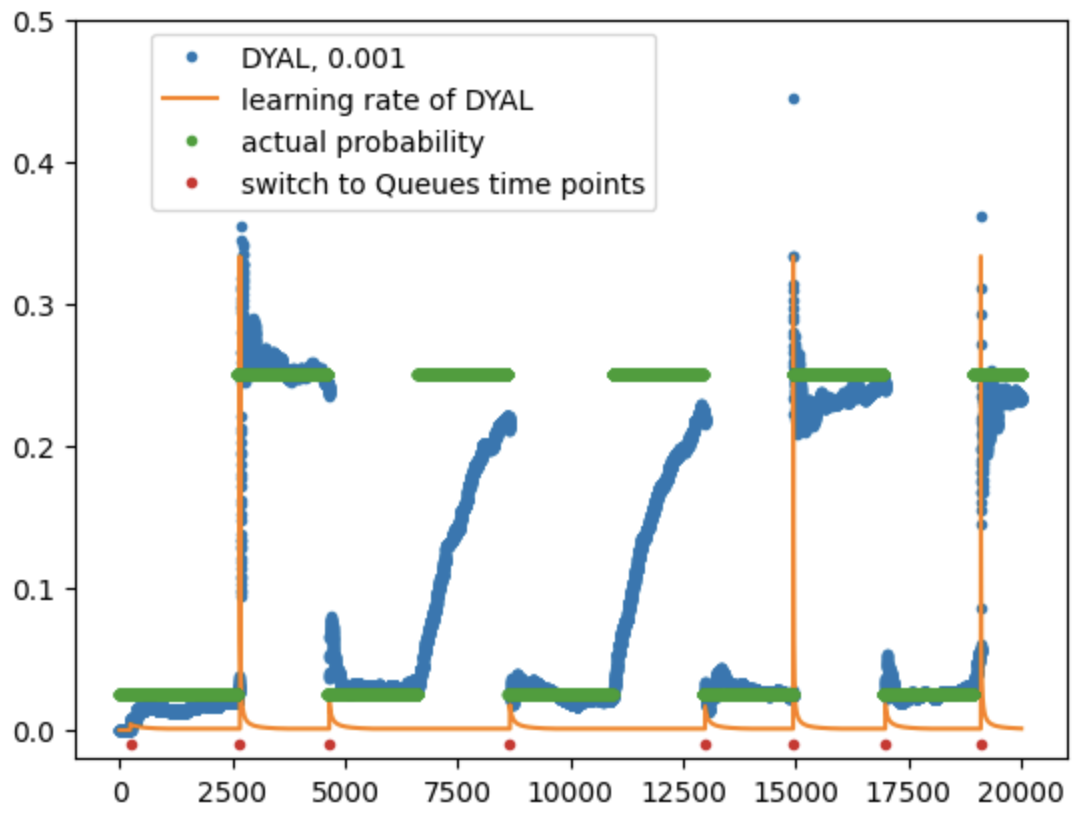}
  }}
\hspace*{.1in}  \subfloat[A close up of change in $\lr$.]{{\includegraphics[height=4.5cm,width=6cm]{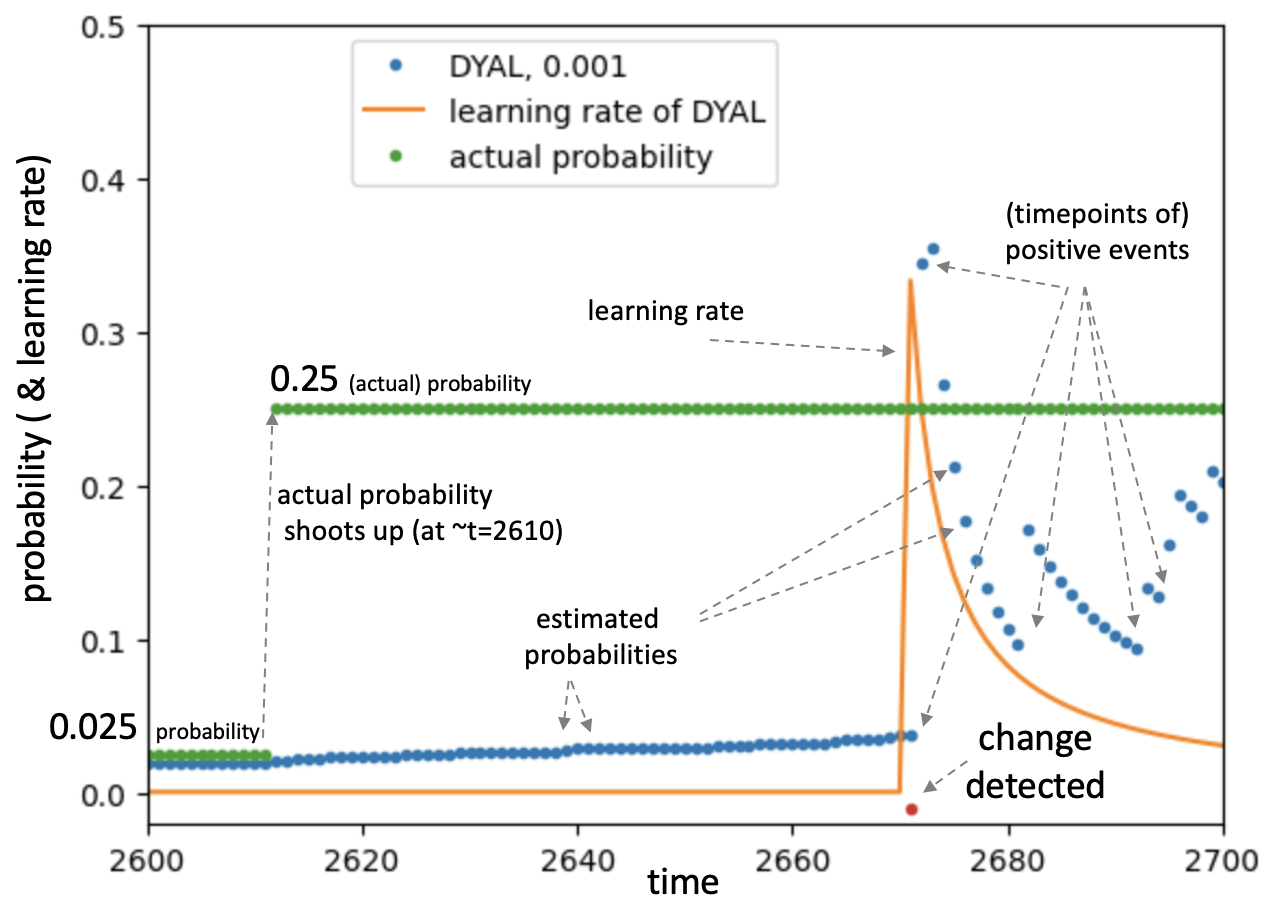}
}}
\end{center}
\vspace{.02cm}
\caption{Synthetic single-item oscillation experiments, 0.025
  $\leftrightarrow$ 0.25, on a single sequence. Parts (a) to (d) show
  the estimates of the three predictors, EMA, \qun, and \qdn, blue and
  orange colors, around the true probability (green, up-down
  steps). We can assess convergence speed and variance.  Parts (e) and
  (f) also show evolution of the learning rate and the rate
  spikes as the probability of a single event is tracked via \qdn.}
\label{fig:binary_nonstationary}
\end{figure}

\co{
\begin{figure} 
\begin{center}
  \centering
  \subfloat[Changes in the learning rate]{{\includegraphics[height=5cm,width=6cm]{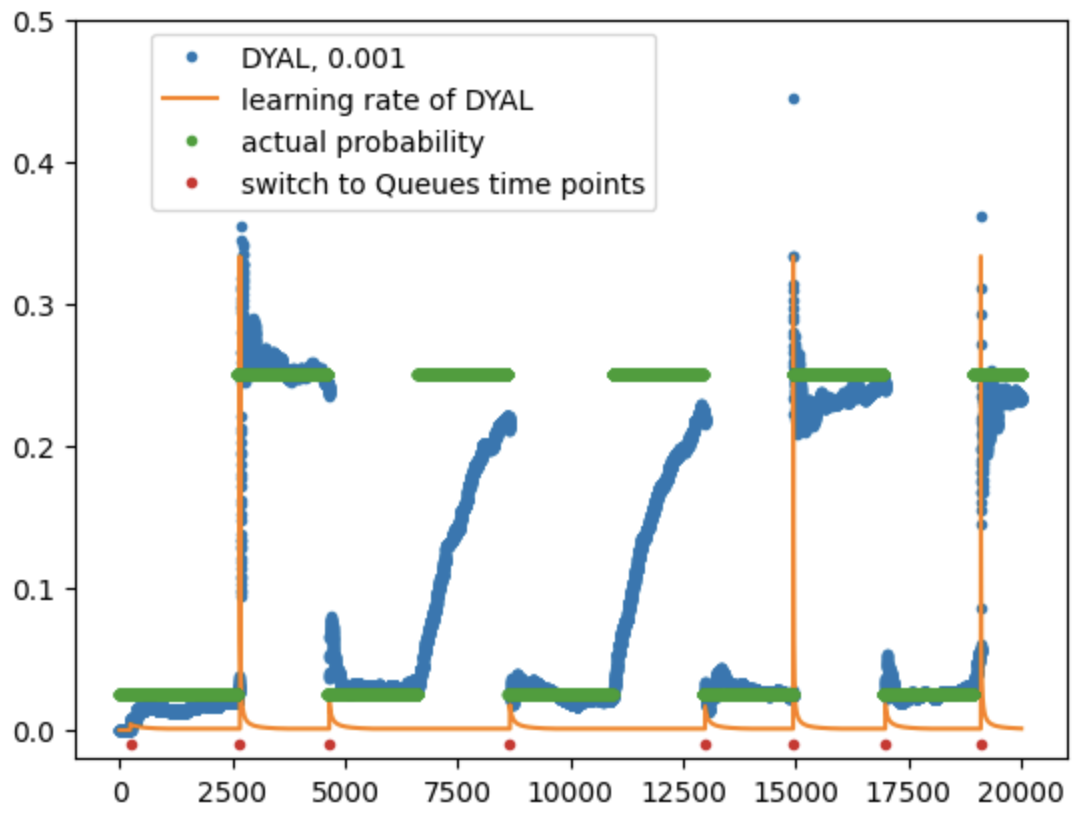}
  }}
\hspace*{.51in}  \subfloat[A close up.]{{\includegraphics[height=5cm,width=6cm]{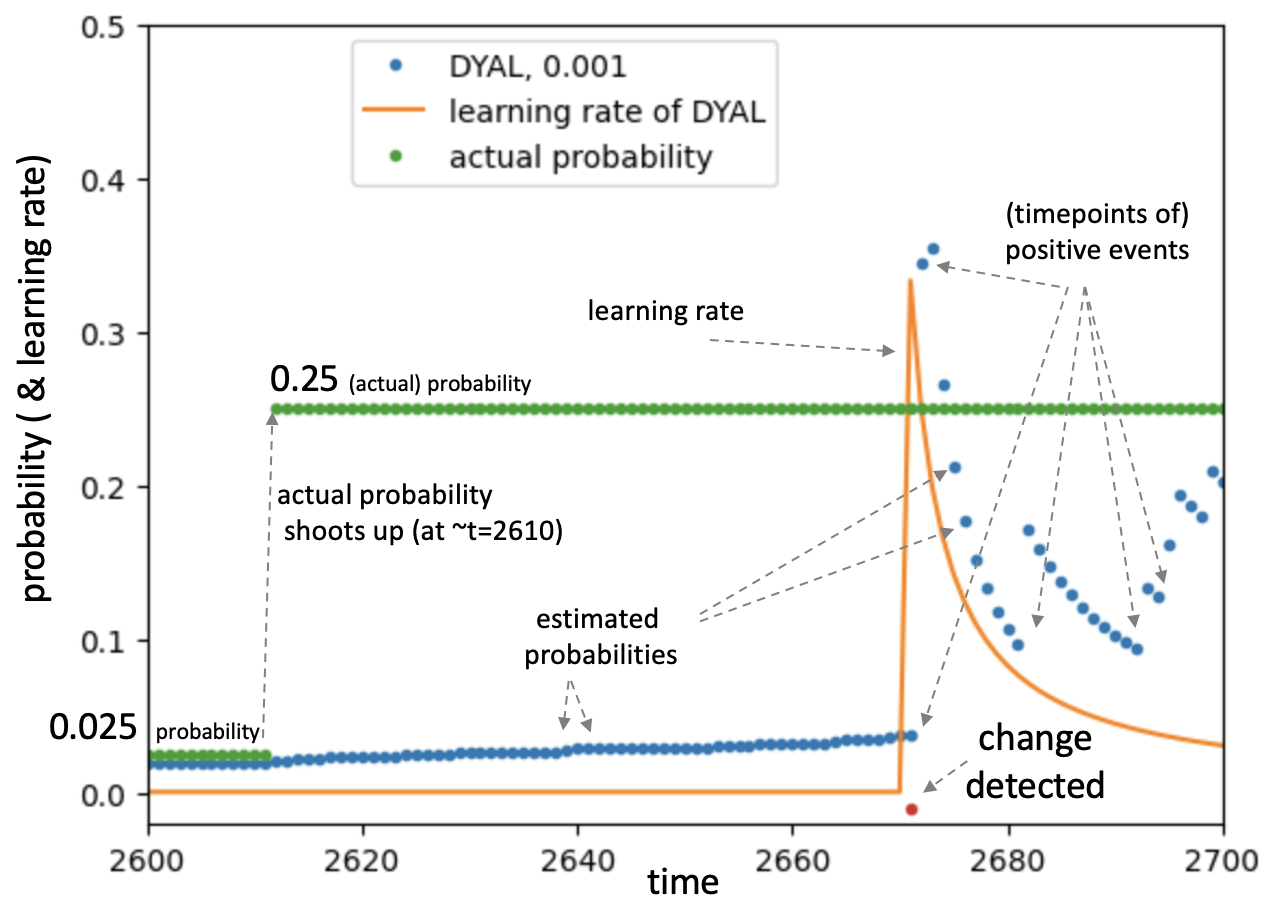}
}}
\end{center}
\vspace{.2cm}
\caption{Synthetic single-item (oscillating probability): The
  evolution of the learning rate and the rate hikes as the probability
  of a single event is tracked via \qdn.}
\label{fig:binary_nonstationary2}
\end{figure}
}


Compared to the results of Table \ref{tab:stat_devs}, here the \qu
predictor excels, specially when we consider deviation-rates when
$\minobs$=10. With such non-stationarity, \qcap=5 can even work better
than \qcap=10 (unlike the previous section).  In this non-stationary
setting, \qd is the 2nd best, and performs substantially better than
the other EMA variants. $\minobs$ of 10 may still be considered on the
low side (\ie relatively high non-stationarity). With $\minobs=50$,
first off, every technique's deviation-rate improves, compared to
$\minobs$=10 setting.  \qdn, with $\lrmin$=0.001, performs best for
the $0.25 \leftrightarrow 0.025$ and \qun, \qcap=10, works best with
uniform $\tp$, \qd remaining second best. When picking from uniform,
the change in $\p$ may not be high, \qu can pick up such change
faster, and more naturally, than \qdn, which performs explicit tests
and may or may not raise its learning rate depending on the extent of
change. However, \qu with 10, despite its simplicity, does require
significant extra memory (and we'll find that in the multiple items
settings, no setting for \qu works better than \qdn).

\figs \ref{fig:binary_nonstationary}
shows plots of the output estimates $\oat{\ep}$, along one of the oscillating
sequences ($0.25 \leftrightarrow 0.025$). As expected, we observe that
the \qu technique leads to high variance during stable periods, \qcap
of 5 substantially higher than \qcap of 10 (\fig
\ref{fig:binary_nonstationary}(d)).  Similarity for both (static) EMA
and \qdn, the rate of 0.01 exhibits higher variance than 0.001 (\figs
\ref{fig:binary_nonstationary}(b) and (c)). Static EMA and \qd with
0.01 are virtually identical in \fig
\ref{fig:binary_nonstationary}(b), except that \qd can have
discontinuities when $\tp$ changes (when it listen to its queue
estimates, as discussed next), and converges from both sides of the
target $\tp$, while static EMA converges from one side.

\fig \ref{fig:binary_nonstationary}(e) shows the evolution of the
learning rate of \qdn, and in particular when \qd detects a
(substantial or sufficiently significant) change (shown via red dots
close to the x-axis). At the scale shown, rate increases look like
pulses: they go up (whenever \qd listens to a queue) and then with
harmonic decay, they fairly quickly come back down (resembling spikes
at the scale of a few hundred time points, for \prs in $[0.01,
  1.0]$). We also note that this rate increase can happen when a high
probability (0.25) changes to a low probability (0.025) as well,
though the rate increase is not as large, as would be expected.  We
also note that \qd is not perfect in the sense that not all the
changes are captured (via switching to its queues' estimates), in
particular, there are a few ``false negatives'' (misses), when a
change in $\tp$ occurs in the sequence and \qd continues to use its
existing learning rates and adapt its existing estimates. This may
suffice when the $\lr$ at the time is sufficiently high or when the
estimate is close. In particular, we can observe in \figs
\ref{fig:binary_nonstationary}(b) and (c) that a higher rate,
$\minlr=0.01$, leads to fewer switches compared to 0.001, and in \figs
\ref{fig:binary_nonstationary}(b) we observe that \qd behaves almost
identically to static EMA when both have $0.01$ (there appears to be
one switch only at $t=15000$ for \qdn).

\fig \ref{fig:binary_nonstationary}(f) presents a close up of when a
(significant) change is detected and a switch does occur, from a low
true probability of $\tp=0.025$ to $\tp=0.25$ in the picture. Note
that before the switch, the estimates (the blue line) start rising,
but this rise may be too slow. Once sufficient evidence is collected,
both the estimate and the learning rate are set to the queues'
estimate, which is seen as a more abrupt or discontinuous change in the
picture. The estimates then more quickly converge to the new true \pr
of 0.25. While we have not shown the actual negative and positive
occurrences in the picture, it is easy to deduce where they are from
the behavior of $\oat{\ep}$ (specially after the switch, where the
estimates are high): the positives occur when there is an increase in
the estimate $\oat{\ep}$ (from $t$ to $t+1$). \fig
\ref{fig:evolution_2_seqs} and \fig \ref{fig:concats} show additional
patterns of change in the learning rate in multi-item real sequences.


Appendix \ref{sec:app_syn} reports on a few additional variations in
parameters and settings, \eg \qu with \qcap=3, and deviation threshold
$d=1.1$, as well as imposing a minimum-length constraint for the
uniform setting (Tables \ref{tab:devs_nonst_other_paparams} and
\ref{tab:devs_non_1000ticks}).

There are additional challenges (\eg of setting appropriate
parameters), when a predictor has to predict multiple items well, with
different items having different probabilities and exhibiting
different non-stationarity patterns. This is the topic of the next
section and the rest of the paper.

%
%

\co{
  
the positive event has to occur either min\_pos 10 or 50 before the
probability for the item is changed.. this determines the length of
each each sub-sequence.. also, for the uniform case, each-sub sequence
(with same probability) has to be also min\_pos / min\_prob events
(positive or negative) long at least (so if min\_pos is 10 and
min\_prob and 0.025, then each subsequence is at least 400 long, and
when min\_pos is 50, each subsequence, whether for 0.025 or for 0.25,
has to be at least 2000 long, but could be longer if the number of
positives is below 50)

The policy (combine these/above  2 paragraphs!) for generation and changing
the probability: A subsequence is the maximal sequence of observations
when generating probability is fixed (at 0.25 or 0.025).

Thus each 10k sequence is a concatenation of such subsequences. The
constraint on each subsequence has to contain at least min\_pos
positive outcomes and has to contain at least min\_pos/min(0.025,
0.25) observations total (so subsequence lengths, for different
probabilities, are roughly the same). For instance, with min\_pos of
10, and min\_prob of 0.025, each subsequence has to be 400 long (and
should contain 10 or more positives). As soon as the conditions are
met, we switch the probability.

}


\co{
Points to make re Table \ref{tab:devs_non1}:

\begin{itemize}
\item queues does better here with non-stationarity, compared to the
  stationary Table \ref{}
\item \qd does better than EMA variants, and 50 observations better
  than 10..
\item Drawing uniformly at random is easier on our methods than
  oscillating between 0.25 and 0.025
\item In the appendix, we show the methods for a few other params (\eg
  queue of capacity 3, and higher or lower learning rates for EMA
  variants), and also include 1.1 threshold
\end{itemize}
}


\co{
\subsubsection{A Variation: One Item Changing/Oscillating,  others constant }

We noted that EMA with 0.01 is close to \qd with 0.001 in terms of
... (actually, i wanted to show logloss, that static EMA is
competitive/close in terms of logloss, if we keep a high rate, but if
we add more static, it has to pay on those...) (Table \ref{}). Here,
we show deviations, specially on 'stable' items can tell a different
story.

Static EMA can not have it both ways.. \qdn, to a certain extent, can!!

Add one or more items with non-changing probs, eg one with 0.5 and
another with 0.1 or 0.02 (several such), etc, while one item
oscillates between 0.250 and 0.025 as in above.

Report violations, on any item, on 1.5 and 2.0.  Static EMA with 0.01
minlr doesn't look great any more?  rate wouldn't look .

\begin{table}[t]\center
  \begin{tabular}{ |c!{\vrule width 1pt}c|c|c|c|c!{\vrule width 1pt}c|c|c|c|c| }     \hline
    number of stationary items $\rightarrow$ &  0 & 1 & 2 & 3 & 4 & 0 & 1 & 2 & 3 & 4    \\   
    \cline{2-11}  
  (each at 0.1)   & \multicolumn{5}{c!{\vrule width 1pt}}{static EMA, 0.01} &  \multicolumn{5}{c|}{\qdn, 0.001}  \\ \hline
    0.025 $\leftrightarrow$ 0.25 (50, 2000) & 0.23 & 0.265 & 0.35 & 0.35 & 0.41
    & 0.07 & 0.09 & 0.11 & 0.11 & 0.14\\ \hline
  \end{tabular}
  \vspace*{.2cm}
  \caption{One salient item oscillates between 0.025 and 0.25 (
    $\minobs$ of 50, and ticks of 2000), while one or more others remain
    constant at 0.1 (any remaining probability is allocated to
    noise). Generate 50 such sequences. Report fraction of times the
    estimated probability for {\em any} salient item violates the 1.5
    threshold .. \qd remains substantially better than static EMA .. }
 \label{tab:adsf}
\end{table}

}



\subsection{Synthetic, Non-Stationary, Multi-Item }
\label{sec:syn_multi}

\begin{figure}[th]
\hspace*{0.2cm} \begin{minipage}[t]{0.55\linewidth}
    \fontsmall{
    {\bf GenSequence}($desiredLen, \minobs, \minlen$) \\
 \hspace*{0.3cm} // Create \& return a sequence of subsequences.\\
 \hspace*{0.3cm} $seq \leftarrow$ [] // A sequence of items.\\
 \hspace*{0.3cm} $prevSD \leftarrow \{\}$ \\
 \hspace*{0.3cm}  While len($seq$)$ < desiredLen$: \\
 \hspace*{0.6cm} // Extend existing sequence \\
 \hspace*{0.6cm} $\P \leftarrow$ GenSD($prevSD$) // get new SD. \\
 \hspace*{0.6cm} $seq$.extend(GenSubSeq($\P$), $\minobs,\minlen$) \\
 \hspace*{0.6cm} $prevSD \leftarrow \P$ \\
 \hspace*{0.3cm} Return $seq$ \\ \\
 {\bf GenSubSeq}($\P, \minobs, \minlen$)  \\
 \hspace*{0.3cm} // Generate a (stationary) subsequence, via repeated \\
 \hspace*{0.3cm} // sampling iid from \sd $\P$. It should be long enough \\
 \hspace*{0.3cm} // so that every item in $\P$ occurs $\ge \minobs$ times in it.\\
 \hspace*{0.3cm} $counts \leftarrow \{\}$ // An observation counter map. \\
 \hspace*{0.3cm}  $seq \leftarrow$ [] // A sequence of items.\\
 \hspace*{0.3cm} While $\min(counts) < \minobs$ or len($seq$) $ < \minlen$: \\
 \hspace*{0.6cm} $o \leftarrow$ DrawItem($\P$) // item $o$ drawn.\\
 \hspace*{0.6cm} seq.append(o) \\
 \hspace*{0.6cm} // Update counts only for salient items.\\
 \hspace*{0.6cm} If $o \in \P$: // increment o's observed count.\\
 \hspace*{0.9cm} $counts[o] \leftarrow counts.get(o, 0)+1$ \\
 \hspace*{0.3cm} Return $seq$ \\ \\
}
\end{minipage}
\hspace*{.5cm} \begin{minipage}[t]{0.55\linewidth}
  \fontsmall{
  {\bf GenSD}($prevSD$) // Generate a new SD. \\
  \hspace*{0.3cm} // Parameters: $recycle, \minprob, \maxprob, \pns$. \\
  \hspace*{0.3cm} $probs \leftarrow$ [] // \prs of the \sdn. \\
  \hspace*{0.3cm} // Repeat while unallocated mass is sufficient. \\
  \hspace*{0.3cm} While $\u{probs} > \pns + \minprob$: // sample \prsn. \\
  \hspace*{0.6cm}   $p_{max} \leftarrow \min(left - \pns, \maxprob)$\\
  \hspace*{0.6cm}   $p \sim \mathcal{U}([ \minprob, p_{max}])$ // uniformly at random.\\
  \hspace*{0.6cm} $probs$.append($p$) // add $p$ to \sdn. \\
  \hspace*{0.3cm} If $recycle$: // reuse items '1', '2', .. \\
  \hspace*{0.6cm} Random.shuffle(probs) // random permute \\
  \hspace*{0.6cm} // item '1' gets probs[0], '2' gets probs[1], etc. \\
  \hspace*{0.6cm} Return MakeMap(probs) \\
  \hspace*{0.3cm} Else: // Allocate new items (unused ids). \\
  \hspace*{0.6cm} Return MakeNewSDMap($probs, prevSD$) \\ \\
  {\bf DrawItem}($\P$) // Return a salient or noise item. \\
    \hspace*{0.3cm} $sump \leftarrow 0.0$ // via sampling. \\
    \hspace*{0.3cm} $p \sim \mathcal{U}([0, 1.0])$ // Uniformly at random. \\
    \hspace*{0.3cm} For $o, prob \in \P$: \\
    \hspace*{0.6cm}   $sump \leftarrow sump + prob$ \\
    \hspace*{0.6cm}    If $p \le sump$: \\
    \hspace*{0.9cm}      Return $o$ // Done.\\
    \hspace*{0.3cm} Return UniqueNoiseItem() // a unique noise id. \\ \\
}
\end{minipage}
\vspace*{.2cm}
\caption{Pseudocode for generating a sequence of subsequences. Each
  subsequence corresponds to a stable period, subsequence $j \ge 1$
  being the result of drawing iid from the $j$th \sdn, $\ott{\P}{j}$. }
\label{fig:gen_code}
\end{figure}

\fig \ref{fig:gen_code} provides pseudocode of the main functions for
generating item sequences under non-stationarity.  As in
the above, we think of a non-stationary sequence as a concatenation of
``stable'' (stationary) subsequences, subsequence $j$, $j \ge 1$, corresponding to
one \sd $\ott{\P}{j}$. The subsequence is long enough (drawn iid from
$\ott{\P}{j}$) so that $\min_{i\in \ott{\P}{j}} count(i) \ge \minobs$,
where $count(i)$ is the number of occurrences of item $i$ in the
subsequence.  We may have a minimum (overall) length requirement as
well, $\minlen \ge 0$.  Each $\ott{\P}{j}$ is created using the {\bf
  GenSD} function.


In GenSD, some probability, $\pns$, is reserved for noise, the \ns
items. Probabilities are drawn uniformly from what probability mass is
available, initially $1-\pns$, with the constraint that each drawn
probability $p$ should be sufficiently large, $p \ge
\minprob$. Optionally, we may impose a maximum probability $\maxprob $
constraint as well ($\maxprob < 1$). Assume we get the set $S=\{p_1,
p_2, \cdots, p_k\}$, once this loop in GenSD is finished, then we will
have $\sum_{p_i \in S} p_i \le 1-\pns$, and each \pr $p_k$ satisfies
the minimum and maximum constraints.

\begin{table}[th]  \center
  \begin{tabular}{ |c?c|c|c?c|c|c?c| }     \hline
     & 1.5any & 1.5obs & logloss  & 1.5any & 1.5obs & logloss & opt. loss     \\ \thickhline  
new items $\downarrow$   & \multicolumn{3}{c?}{\qun,  5 } &  \multicolumn{3}{c?}{\qun, 10} &   \\ \hline
$\minobs = 10$  & 0.94  & 0.24  & 1.17 & 0.88  & 0.17  & 1.19 & 1.040 $\pm$0.11  \\ \hline
$\minobs = 50$  & 0.92  & 0.22  & 1.13 & 0.78  & 0.10  & 1.09 & 1.028 $\pm$0.22  \\ \medhline
& \multicolumn{3}{c?}{static EMA, 0.01 } &  \multicolumn{3}{c?}{static EMA, 0.001} &   \\ \hline
10 & 0.86  & 0.20  & 1.26 & 1.00  & 0.95  & 2.33 & 1.040   \\ \hline
50 & 0.80  & 0.06  & 1.07 & 0.70  & 0.32  & 1.44 & 1.028   \\ \medhline
& \multicolumn{3}{c?}{harmonic EMA, 0.01 } &  \multicolumn{3}{c?}{harmonic EMA, 0.001} &   \\ \hline
10 & 0.86  & 0.19  & 1.25 & 0.98  & 0.88  & 2.25 & 1.040   \\ \hline
50 & 0.80  & 0.06  & 1.07 & 0.62  & 0.23  & 1.34 & 1.028   \\ \medhline
& \multicolumn{3}{c?}{\qdn, 0.01 } &  \multicolumn{3}{c?}{\qdn, 0.001} &   \\ \hline
10 & 0.94  & 0.09  & 1.11 & 0.89  & 0.09  & 1.14 & 1.040   \\ \hline
50 & 0.89  & 0.05  & 1.05 & 0.59  & 0.03  & 1.06 & 1.028   \\ \thickhline 
recycle items $\downarrow$  & \multicolumn{3}{c?}{\qun, 5 } &  \multicolumn{3}{c?}{\qun, 10} &   \\ \hline
$\minobs = 10$  & 0.91  & 0.23  & 1.16 & 0.83  & 0.13  & 1.14 & 1.040 $\pm$0.11  \\ \hline
$\minobs = 50$ & 0.91  & 0.22  & 1.12 & 0.76  & 0.10  & 1.08 & 1.028 $\pm$0.22  \\ \medhline
& \multicolumn{3}{c|}{static EMA, 0.01 } &  \multicolumn{3}{c|}{static EMA, 0.001} &   \\ \hline
10 & 0.87  & 0.16  & 1.16 & 1.00  & 0.73  & 1.65 & 1.040  \\ \hline
50 & 0.80  & 0.06  & 1.06 & 0.77  & 0.26  & 1.28 & 1.028  \\ \medhline
& \multicolumn{3}{c|}{harmonic EMA, 0.01 } &  \multicolumn{3}{c|}{harmonic EMA, 0.001} &   \\ \hline
10  & 0.87  & 0.15  & 1.15 & 0.98  & 0.66  & 1.54 & 1.040  \\ \hline
50 & 0.80  & 0.05  & 1.05 & 0.71  & 0.18  & 1.18 & 1.028  \\ \medhline
& \multicolumn{3}{c|}{\qdn, 0.01 } &  \multicolumn{3}{c|}{\qdn, 0.001} &   \\ \hline
10 & 0.91  & 0.10  & 1.11 & 0.83  & 0.16  & 1.18 & 1.040  \\ \hline
50 & 0.88  & 0.05  & 1.04 & 0.56  & 0.04  & 1.06 & 1.028  \\ \hline
   \end{tabular}
  \vspace*{.2cm}
  \caption{Synthetic multi-item experiments: Deviation rates and
    \loglossn, averaged over 50 sequences, $\sim$10k length each, $\minobs$
    of 10 and 50, uniform \sd generation vis GenSD(): $\maxprob=1.0$,
    $\minprob=\pns=0.01$ for GenSD(), and change the \sd $\P$ whenever
    {\em all} salient items in $\P$ are observed $\ge \minobs$ times. In top
    half, items are new when underlying \sd changes, and in the bottom
    half, items are 'recycled' (\sec \ref{sec:syn_multi}). }
\label{tab:devs_multi_non_v3} 
\end{table}

Once the set $S$ of $k$ \prs is generated, GenSD() then makes a \sd from
$S$.  Under the {\bf \em new-items} setting (when $recycle=0$), new
item ids, $|S|$ such, are generated, \eg an item id count is
incremented and assigned, and the items are assigned the probabilities
(and all the old items, salient in $\ott{\P}{j-1}$, get zeroes, so
discarded).  Thus, as an example, the first few \sdn s could be
$\ott{\P}{1}=\{$1:0.37, 2:0.55, 3:0.065$\}$, $\ott{\P}{2}=\{$4:0.75,
5:0.21, 6:0.01, 7:0.017$\}$, and $\ott{\P}{3}=\{$8:0.8, 9:0.037,
10:0.15$\}$.

Under the {\em \bf recycle} setting, with $k$ \prs generated, items
$1$ through $k$ are assigned from a random permutation\footnote{As the
\pr generation process tends to generate smaller \prs with each
iteration of the loop, this random shuffling ensures that the same items are not
assigned consistently high or low \prsn. } $\pi()$ of $S$: $\P(i)
\leftarrow p_{\pi(i)}$ (thus item 1 may get $p_3$, etc.).  For
example, with the previous \prsn, we could have
$\ott{\P}{1}=\{$1:0.065, 2:0.37, 3:0.55$\}$ (and $0.015$ is left for
\ns items), $\ott{\P}{2}=\{$1:0.75, 2:0.01, 3:0.21, 4:0.017$\}$ (a new
item is added), $\ott{\P}{3}=\{$1:0.037, 2:0.8, 3:0.15$\}$. Thus,
under the $recycle$ setting, in addition to changes in probability,
the support of the underlying \sd $\P$ may expand (one or more new
items added), or shrink, with every change (\ie from one stable
subsequence to the next one).



During drawing from a \sd $\P$, when a \ns item is to be generated we
generate a unique \ns id. This is one simple extreme, which also
stress tests the space consumption of the maps used by the various
methods on long sequences. Another extreme is to reuse the same \ns
items, so long as the condition that their true probabilities remain
below (but close to) the boundary $\minprob$ is met.\footnote{Specially
when \ns item probabilities are borderline and close $\minprob$, there
can be some windows or subsequences on which the proportion goes over
$\minprob$. } For simplicity, we present results on the former unique
\ns setting (pure noise). We have seen similar patterns of accuracy
performance in either case.

We note that there are a variety of options for \sd and sequence
generation.  In particular, we also experimented with the option to
change only one or a few items, once they become eligible (their
observation count reaching $\minobs$), and we obtained similar
results. Because this variation requires specifying the details of how
an item's \pr is changed (\eg how it is replaced by zero or more new
or old items), for simplicity, we use GenSD(), \ie changing all items'
\prn s, and only when all become eligible. Note that, under the
$reuse$=1 setting, some items' probabilities may not change much,
simulating a no-change for some items. Experiments on real streams
provide further settings (\sec \ref{sec:real}).

Table \ref{tab:devs_multi_non_v3} presents deviation rates
(generalized to multiple items) as well as \logloss of various
methods. We next describe how the deviation-rates as well the {\em
  optimal loss} (lowest achievable \loglossn) are computed. To define
all these measures, we need the sequence of underlying true
distributions, $\seq{\P}{1}{N}$ ($\oat{\P}$ was used to generate
$\oat{o}$), which we have access to in these synthetic experiments.



The optimal (lowest achievable) loss (\loglossn) at time $t$ is
defined as follows: given item $\oat{o}$ is observed and $\oat{\P}$ is
the underlying distribution (at time $t$), then if $\oat{o} \in
\oat{\P}$ ($\oat{o}$ is salient), the optimal loss at $t$ is
$-\log(\oat{\P}(\oat{o}))$, and if $\oat{o}$ is \ns ($\o \not \in
\oat{\P}$), then the loss is $-\log(1-sum(\oat{\P}))$. The optimal
\logloss is simply the average of this measure over the entire
sequence. Note that if a single \sd $\P$ generated $\seq{o}{1}{N}$,
what we described is an empirical estimate of the \sd entropy, \ie the
expectation:\footnote{When the underlying \sd $\P$ changes from time
to time, the computed optimal \logloss is the weighted average of the
entropies, weighted by the length of the subsequence each \sd was
responsible for.  } $-(1-sum(\P))\log(1-sum(\P))+\sum_{i\in
  \P}-\log(\P(i))$.  Reporting the optimal \logloss allows us to see
how close the various methods are getting to the lowest loss possible.
Note that optimal loss, as seen in Table \ref{tab:devs_multi_non_v3},
does not change whether we use new items or recycle items. We use the
same 50 sequences for the different methods, so the corresponding
optimal losses are identical as well.



The (multi-item) deviation-rate, given the sequence of underlying \sdn
s $\seq{\P}{1}{N}$, is defined as:
\begin{align}
  \nonumber
\hspace*{-0.5cm}\dev(\seq{\Q}{1}{N}, \seq{o}{1}{N}, \seq{\P}{1}{N}, d) = \frac{1}{N}\sum_{t=1}^N \mdvc(\oat{o}, \oat{\Q}, \oat{\P}, d) \hspace*{.07in} \mbox{(multi-item deviation)},
\end{align}
where for the \mdvc() function, we explored two options: under the more lenient 'obs'
setting, we score based on the observation $o$ at time $t$ only:
$\mdvc_{obs}(o, \Q, \P, d) = \dvc(\Q(o), \P(o), d)$.  Under the more demanding 
'any' setting, we count as deviation if any estimate in $\Q$ having high
deviation: $\mdvc_{any}(o, \Q, \P, d) = \max_{i \in \P}\dvc(\Q(i),
\P(i), d)$ ($o$ is not used). $\mdvc_{obs}()$ is closer to \loglossn,
as it only considers the observation, and like \loglossn, is therefore
more sensitive to items with higher \pr in the underlying
$\P$. $\mdvc_{any}()$ is a more strict performance measure.  Table
\ref{tab:devs_multi_non_v3} shows both rate types with $d=1.5$.

Table \ref{tab:devs_multi_non_v3} shows performance results when GenSD
is used with $\maxprob=1.0$, $\minprob=0.01$, and two $\minobs$
settings.  With such generation settings, we get on average just under
17 \sd changes, or subsequences (stability periods), when
$\minobs$=10, and just under 4 \sd changes, when $\minobs$=50. The
support size (number of positive entries in an \sdn) was around 5.
All performances, even for \qu with \qcap=5, improve as $\minobs$ is
increased from 10 to 50, and the performance of \qu with \qcap=10 is
better than \qcap=5, even for high non-stationarity $\minobs$=10, as
we change the underlying \sd $\P$ only when {\em all} salient items of
\sd $\P$ pass the $\minobs$ threshold. Nevertheless,  even with
\qcap of 10, \qu often trails the best of the EMA variants
significantly.  As in the previous section, harmonic EMA can slightly
outperform static EMA, but both underperform \qd at its best. In
particular \qd is not as sensitive to setting the (minimum) rate to a
low value, as seen in \fig \ref{fig:loss_vs_rate1} for several
settings and rate values. We also note that while the \logloss values
can seem close, the deviation rates can explain or reveal better why
\qd often does perform better.

We can also pair two methods and count the number of wins and losses,
based on \logloss over each of the same 50 sequences, and perform
statistical sign tests. With $\minobs$=50 and the new-items setting,
pairing \qd with $\lrmin=0.01$ (best or near best of \qdn) against all
other techniques, EMA static or harmonic (with $\lr$ in $\{0.001,
0.01, 0.02, 0.05, 0.1\}$) or \qun, \qd gets lower a \logloss on {\em
  all} 50 sequences. If we lower the $\minlr$ rate for \qd to 0.001,
we still get dominating performance by \qdn, but harmonic can win on a
few one or two sequences.  With $\minobs$=10 and again the new-items
setting, pairing \qd with $\lrmin=0.01$ against all others, again we
obtain the same dominating results for \qdn.

With the item-recycle setting, the differences between the best of the
EMA variants and \qd variants shrinks somewhat. For instance, with
$\minobs$=50, harmonic EMA with $\lrmin=0.01$ gets 9 wins (\qd,
$\lrmin=0.01$, gets the remaining 41 wins), and if we use
$\lrmin=0.001$ for \qd, \qd loses 42 times to harmonic EMA with
$\lrmin=0.01$. Similarly, with $\minobs$=10, in the recycle setting,
we need to set $\lrmin=0.02$ to get a dominating performance by \qdn,
and setting it lower to $\lrmin=0.01$ (which worked well for
$\minobs$=50) yields mixed performance.  Thus, the choice of rate
$\minlr$ for \qd can make a difference, although the operating range
or the sensitivity to $\minlr$ is substantially lower for \qd than for
other EMA variants (\fig \ref{fig:loss_vs_rate1}), in particular when
setting $\lrmin$ to a low value. We observe this improved sensitivity
on other data sources and settings as well.

The \bx \sma (\sec \ref{sec:box}) performs somewhat better than static
EMA with an appropriate window size, which is about 100 for the
setting of Table \ref{tab:devs_multi_non_v3} (50 and 500 are strictly
inferior, and 200 is slightly worse, in terms of \llns \ and other
measures). Like EMA, \bx is inflexible and requires tuning the fixed
window size for different input streams, and unlike EMA its rigid
space consumption is prohibitive. Table \ref{tab:box1} summarizes
paired comparisons with \qdn, under the two recycle and new-item
settings and $\minobs=10$ and 50. $\dyal$ with $\minlr=0.01$,
consistently outperforms it.

\begin{table}[t]  \center
  \begin{tabular}{ |c|c|c|c|c| }     \hline
    \qdn, 0.01 vs. \bxn, 100    &  $\minobs=10$, recycle &  10, new & $\minobs=50$, recycle & 50, new  \\ \hline
losses, wins of \qd  &  7, 43   & 0, 50 &  8, 42 & 0, 50 \\ \hline
\llns,  opt: 1.064 and 0.988  & 1.120,1.108   & 1.182,1.113 & 1.013,1.008 & 1.025,1.010 \\ \hline
    \end{tabular}
\vspace{.2cm}
\caption{Number of losses and wins, based on \llns , of \qdn, with
  $\minlr=0.01$, pairing it against \bx with window size 100 (which
  performed better than 50 or 200), under the recycle (item-recycle)
  and new-items settings, with $\minobs=10$ and $50$, on 50 sequences
  of 10k each (the setting of Table \ref{tab:devs_multi_non_v3}).
  Thus \dyal has the lower loss on 43 of 50 sequences in the recycle
  setting with $\minobs=10$. Average \llns \ values are also shown
  (1st number is \bxn's), and average optimal \llns \ is 1.064
  ($\minobs$=10) and 0.988 ($\minobs$=50).  All wins are highly
  significant (using the binomial sign test).  With the lower
  $\minlr=0.001$ for \qdn, it has statistically more wins under the
  more challenging new-items settings but has mixed performance under
  the recycle settings. }
\label{tab:box1}
\end{table}


Appendix \ref{app:syn_multi}, Table \ref{tab:devs_multi_non2},
presents performances for a few additional settings, in particular
when $\maxprob$, for GenSD(), is set to a lower value of 0.1
instead of 1.

\co{ Points to make (see tables/figs) (see below for set up):
\begin{itemize}
  \item With longer stability results, meaning deviation violations and/or logloss, improve ($\minobs$ 50 vs 10)

  \item With less stability, queues perform best or near best, and
    higher EMA variants with higher $\minlr$ do better ..

  \item \qu has high variance.. higher capacity helps, but still
    insufficient..

  \item for static EMA variants to perform well, one needs a high
    learning rate...
    
  \item Need to pay attention to $\minobs=50$ which we think is the
    range we are most interested in..

  \item \qd has a higher 'operating range' or regret range (we think)..

  \item Why does harmonic EMA seem to underperform compared to static EMA??

  \item todo: plot as a fn of minlr
  \item todo: include sensitivity analysis here: binomial threshold, too
    low, it's like queues, and too high, like harmonic ema...
\end{itemize}

}

\co{
Multiple items.  A distribution is allowed to change once all of its
(salient) items are observed min\_obs times.  One this criterion is
met, a new distribution is created and sampled from. Other variations
are discussed below.

For example, assume $\minobs=10$. the first distribution created could
be: $\{ \}$, and this is drawn from iid until all items are observed
10 $\minobs$ times (expected x draws). Then a new distribution is
created: $\{ \}$ and drawn from so all items with this distribution
are observed $\minobs$ times (expected x draws).

To evaluate:

We report on 1.5any and 1.5obs to compare/connect to previous binary
experiments.. on the other hand, 1.5obs is closer to log-loss as we
explain below......

for 1.5any, at any time point compute $\max_{items}
(\frac{p}{p}, \frac{p}{p})$, where if p(item) is 0, we define it to be
a violation (\eg the first time the concept is observed)..

for 1.5obs, we don't take the maximum over all the items.. only
consider the ratio of the observed item ...  This is closer to
logloss, in that the items with higher frequency affect it more.
(again, if the item is not seen before, we consider it a violation ..)

So the numbers, 1.5any and 1.5obs, get worse when noise-rate is 0.1,
compared to 0.01...

for log-loss ... we use oov marker ..

}

\begin{figure}[t]
\begin{center}
  \centering
\hspace*{-1cm}  \subfloat[new items, $\minobs=50$]{{\includegraphics[height=5.5cm,width=4.75cm]{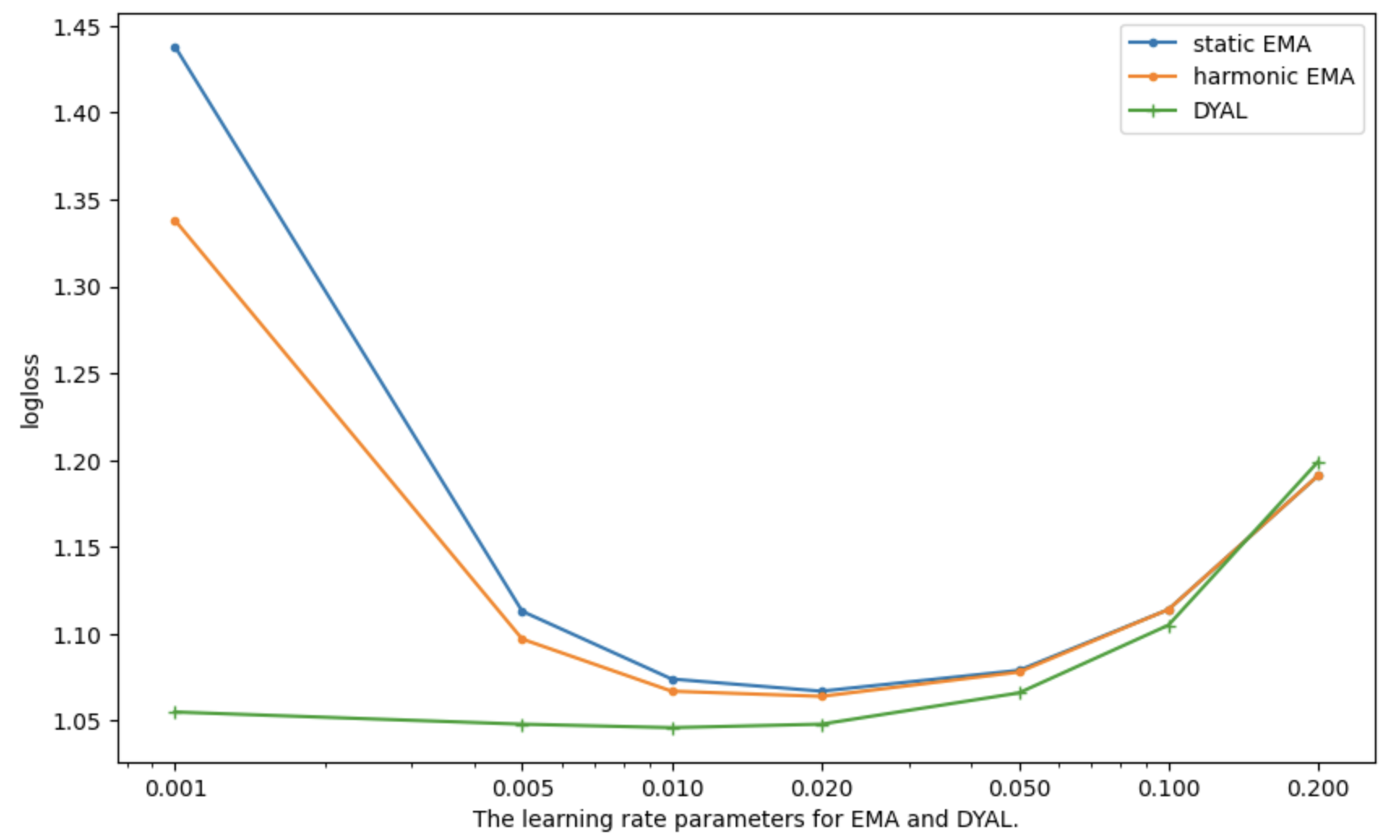}
  }}
  \subfloat[recycle items, $\minobs=50$]{{\includegraphics[height=5.5cm,width=4.75cm]{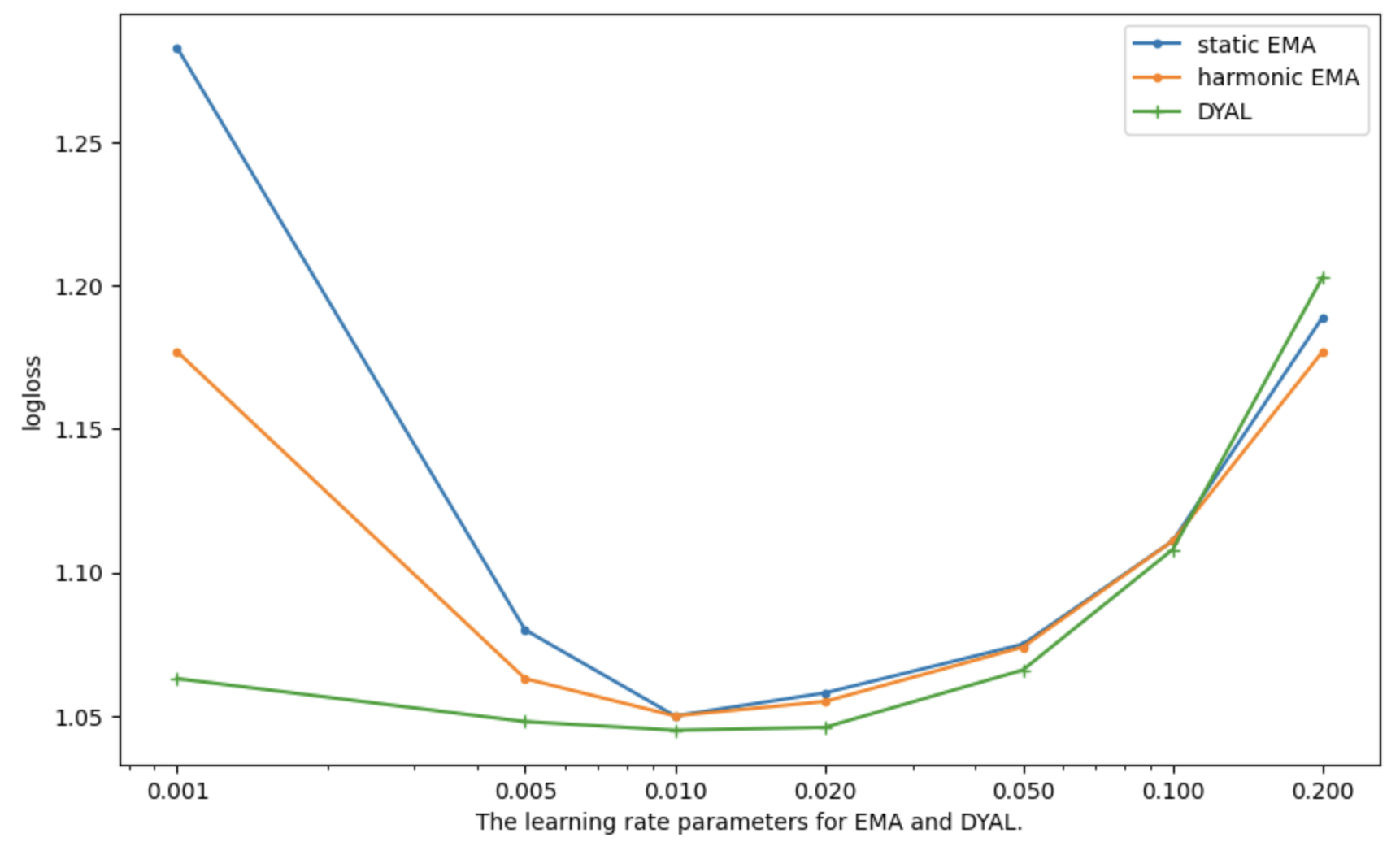}
  }}
  \subfloat[new items, $\minobs=10$]{{\includegraphics[height=5.5cm,width=4.75cm]{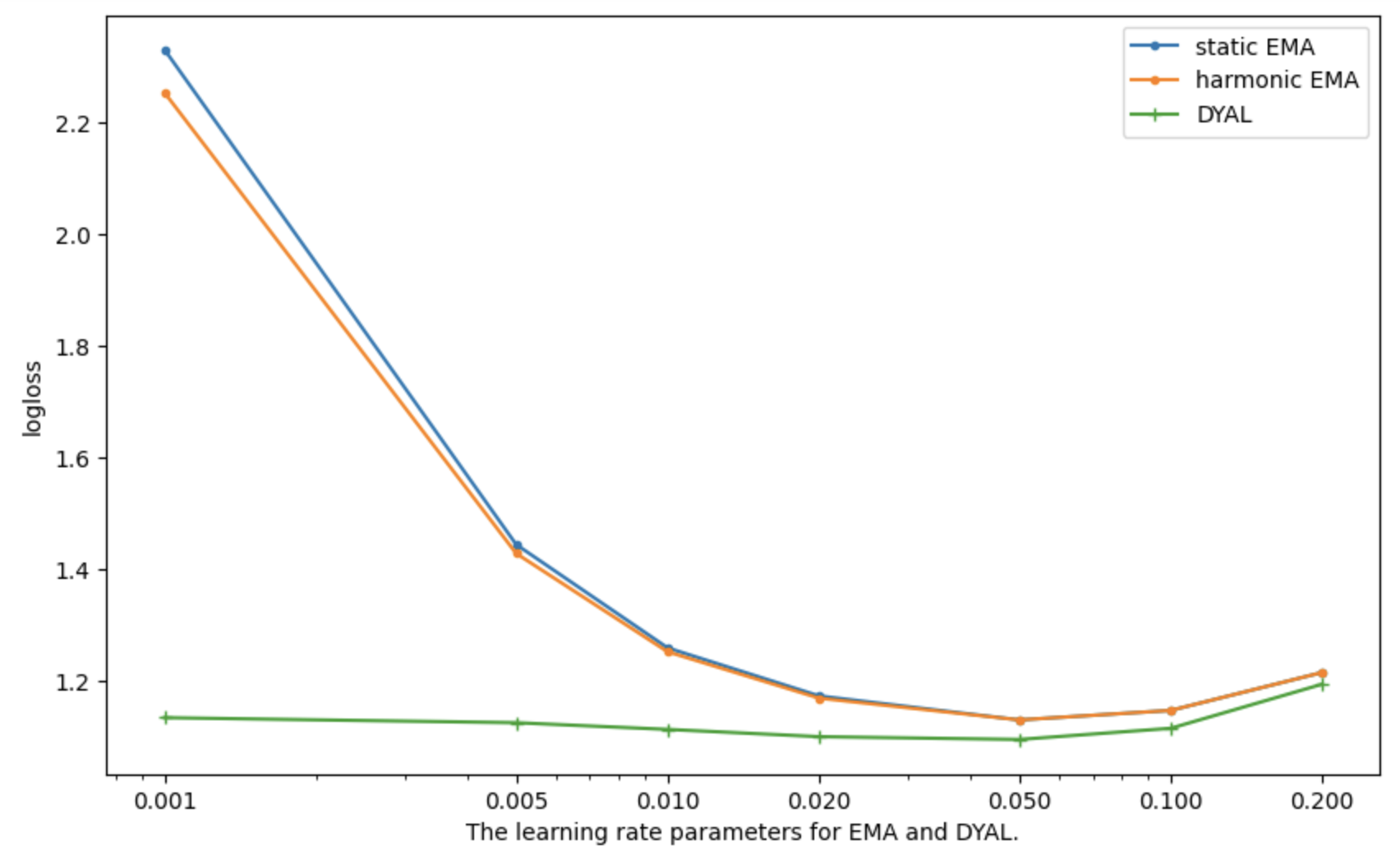}
  }}
\end{center}
\vspace{.2cm}
\caption{\logloss performance, as the learning rate is changed, in
  synthetic multi-item experiments, $\minobs=50$. (a) new items. (b)
  recycle items.  \qd is less sensitive to how low $\minlr$ is set,
  compared to harmonic and static EMA variants.}
\label{fig:loss_vs_rate1}
\end{figure}




\section{Experiments On Real-World Data Sources}
\label{sec:real}
\todo{if you need to: all the issues with the real world (although these tasks below may not
show all the issues of course!) and also when you DON'T know the true
probabilities ( if there is such a thing! )
}

In real-world datasets, a variety of complex phenomena combine to
generate the sequences of observations, and even if we assume
stationary distributions generate the data over some stable durations,
we do not know the actual underlying probabilities in order to compare
methods.  In all the sections below, we obtain sequences of
observations, in different domains, and report the \logloss
performances, in particular \empllnsp (\eq \ref{eq:empllns}), when
comparing different predictors
on these (same) sequences.

Table \ref{tab:exped_seqs_etc} presents the classification of the data
sources we use here according to the type of non-stationarity.
The sequences obtained from the \exped system,\footnote{We used a text
corpus of NSF abstracts in the \exped experiments \cite{Dua:2019}. The
dataset contains approximately 120k research paper abstracts, yielding
2.5 million English text lines, over 20 million term occurrences, and
just under 100 unique characters.} an implementation of ideas in
prediction games \cite{expedition1}, and described next, exhibit what
we have called internal or developmental non-stationarity: new items
(concepts) are generated over time by the system itself.  If we turn
off concept generation and keep predicting at the character level in
\expedn, we do not have any non-stationarity, with a static external
text corpus and in the way we sample lines, and we report comparisons
in this setting as well (\sec \ref{sec:expd_chars}), finding that
\qdn, developed for non-stationarity, continues to enjoy performance
advantages over the others in this setting too.  We provide evidence
that the sequences in our final task, the Unix commands data sequences
(\sec \ref{sec:unix}), exhibit external non-stationarity: each
person's pattern of command usage changes over days and months, as
daily activities and projects change. An example of a task exhibiting
both internal and external non-stationarity is feeding the \exped
system one genre or language (\eg French) for some time, followed by
another (\eg Spanish).  We leave experiments on tasks exhibiting both
internal and external non-stationarity to the future.


\subsection{Expedition: Up to N-Grams}


The \exped system operates by repeatedly inputting a line of text (on
average, about 50 characters in these experiments), {\bf \em
  interpreting} it, and learning from the final selected
interpretation. The interpretation process consists of search and
matching (\fig \ref{fig:example_expd}): it begins at the low level of
characters, which we call {\em primitive concepts},
and ends in the highest-level concepts in its current {\bf \em concept
  vocabulary} $V$ that both match the input well and fit with one
another well. Initially, this vocabulary $V$ corresponds to the set of
characters, around 100 unique such in our experiments, but over many
episodes the vocabulary with which the system interprets grows to
thousands of concepts and beyond. In our experiments, concepts are
{\bf \em n-grams} (words and words fragments).  Higher level n-grams
help predict the data stream better, in the sense that bigger chunks
of the stream (better predictands) can be predicted in one shot and
fewer independence assumptions are made. In the \exped system,
concepts are {\em both the predictors and the predictands}.

\subsubsection{An Overview of Simplified \exped}

\begin{table}[t]  \center
  \begin{tabular}{|c|c|c|}    \hline
  Non-Stationarity $\rightarrow$  &  Internal  & External \\ \hline 
  \expedn, up to n-grams  &  \checkmark & -- \\ \hline
  \expedn, only characters  & -- & -- \\ \hline
  Unix commands  & $-$   & \checkmark \\ \hline
  \end{tabular}
  \hspace*{.31cm}  
    \begin{tabular}{ |c|c|c|c| }     \hline
  Number  & Median & Minimum & Maximum \\ \hline 
  104 & 1.2k & 75 & 48k  \\ \hline
  \multicolumn{4}{|c|}{Examples of predictors, and seq. lengths: } \\
  \multicolumn{4}{|c|}{ (``ht'', 91), (``t s'', 75), (``and'', 144), (``ron'', 189), }\\
  \multicolumn{4}{|c|}{ (``th'', 3066), (``t'', 25k), (`` '', 48k) }\\ \hline
%
  \end{tabular}
  \vspace*{.2cm}
  \caption{Left: Real datasets and types of non-stationarity
    exhibited in the experiments. Right: Statistics on 104 \exped
    sequences, with median sequence size of 1.2k. Blank space is the
    most frequent character leading to a sequence length of 48k, and
    't' is the next most common at 25k.  }
\label{tab:exped_seqs_etc}
\end{table}





In previous work, we developed and motivated an information theoretic
concept quality score we named CORE. CORE is a measure of {\em
  information gain}, a reward or score that promotes discovering and
use, within interpretations, of larger or higher-level (higher reward)
concepts \cite{pgs3}.  In this work, we ignore the quality of the
concepts (the predictands), and focus only on the more ``pure''
prediction task of better predicting a sequence of (unit-reward) items
in the face of non-stationarity. We also leave it to future work to
carefully assess the impact of better prediction algorithms on the
learning and development of the entire system, \ie the performance of
the overall system with multiple interacting parts. Here and in
Appendix \ref{app:real}, we briefly describe how a simplified
Expedition system works, and how we extract sequences, for learning
and prediction, and for comparing different predictors. There are
three main tasks in \expedn:
\begin{enumerate}
  \item Interpreting (in every episode): using current concept
    vocabulary $V$, segment a line of text (the episode input) and map
    the segments into highest-level concepts, creating a concept
    sequence.
  \item Predicting and updating (learning) prediction weights (every
    episode): Using a final selected interpretation $s$, update
    prediction weights among concepts in interpretation $s$.
  \item Composing (every so often): Make new concepts out of existing
    ones (expand $V$).
\end{enumerate}




In this paper, we focus on the prediction task 2, and to simplify the
comparison of different prediction algorithms, we divorce prediction
from interpreting and composing,
so that the trajectory the system takes, in learning and using
concepts, is independent of the predictor (the learner of the
prediction weights), and different predictors can be directly
compared.  Otherwise, one needs to justify the intricacies of
interpretation and new concept generation and that such do not
unfairly work well with one technique over another.  The (simplified)
interpretation and composing methods we used, for tasks 1 and 3, are
described in Appendix \ref{app:real}.



\begin{figure}[tb]
\begin{center}
  \centering
  \subfloat[Bottom to top interpretation (search) paths: invoking
    compositions (upward) and matching attempts (downward), in finding
    an interpretation.]
           {{\includegraphics[height=3cm,width=5.5cm]{
                 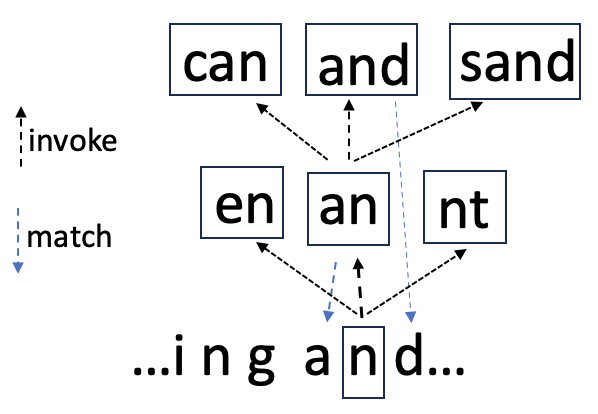} }}\hspace*{1cm}
           \subfloat[An example final interpretation, given the input
             fragment "running and playing".]
                    {{\includegraphics[height=3cm,width=6cm]{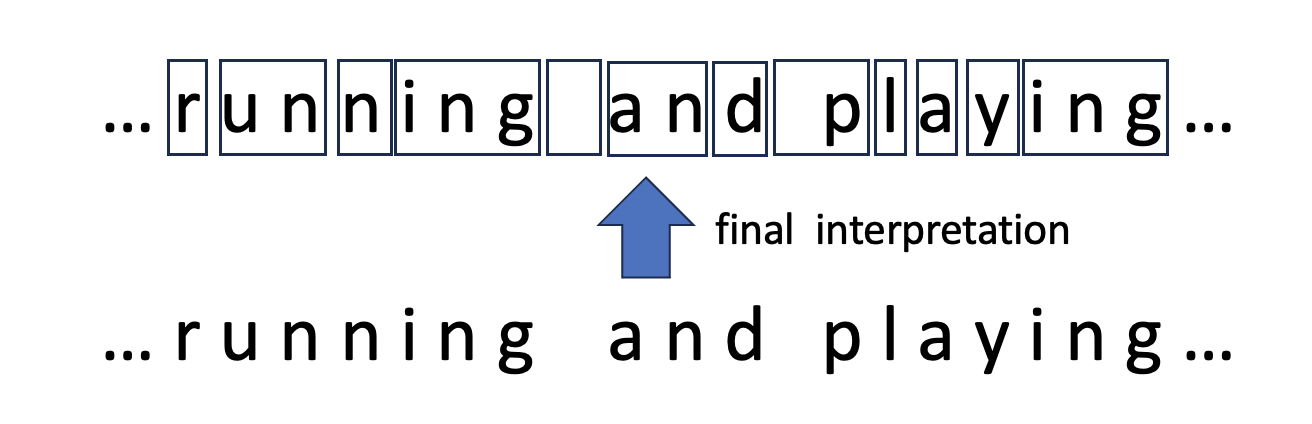}
                    }}
\end{center}
\vspace{.2cm}
\caption{(a) The interpretation search process consists of invoking
  (``upward'') and matching attempts (``downward''), until no more
  compositions that match the input remain. (b) The final selected
  interpretation is a sequence of highest level available concepts,
  n-grams, that match the given input line well. }
\label{fig:example_expd} 
\end{figure}


\subsubsection{Prediction Tasks, and Collecting Sequences}

\co{
\begin{table}[t]\center
  \begin{tabular}{ |c|c|c|c|c| }     \hline
  Num.  & Median & Min & Max & Examples of recorded predictors, and sequence lengths \\ \hline
  104 & 1.2k & 75 & 48k & (``ht'', 91), (``t s'', 75), (``and'', 144), (``ron'', 189),
  (``th'', 3066), (``t'', 25k), (`` '', 48k)\\ \hline
  \end{tabular}
  \vspace*{.2cm}
  \caption{Statistics on 104 \exped sequences, with median sequence
    size of 1.2k. Blank space is the most frequent character leading
    to a sequence length of 48k, and 't' is the next most common at
    25k. }
\label{tab:exped_seqs}
\end{table}  
}

In every episode, once an interpretation, \ie a final concept
sequence, is selected, each concept in the sequence, except the right
most, acts as a predictor and updates its weight for predicting what
comes immediately after it to the right. 
For example, in the concept sequence of \fig
\ref{fig:example_expd}(b), the predictor (concept) 'r' observes the
(concept) 'un' to its right and updates accordingly, and (the
predictor) 'un' observes 'n' to its right and updates accordingly, and
so on.  Note that only the final selected interpretation is used for
updating prediction weights here, and the interpretation search path
leading to it, as well as other search paths leading to unselected
final interpretations, are discarded.

We ran the \exped system for four trials, in each trial starting from
scratch and with a different random seed, \ie different sequences of
lines are input to the system. Each trial was stopped after 20k to 30k
episodes, each taking under an hour. In the course of a trial,
thousands of new concepts are generated and used.  Id generation for
concepts is incremental via a simple counter, and as characters are
seen first, before any higher-level n-gram is generated, concepts with
lower ids correspond to characters (unigrams) and after id 95 we get
bigrams and trigrams.

In each trial, for each of a few arbitrary concept ids being
tracked\footnote{The same character gets different ids in different
trials.} (a few below id 20, a few around 100 and 500), we collect and
create a sequence from what comes after the corresponding concept in
the episodes it is active in, \ie it appears in one or more places in
the final selected interpretation.  For instance, in one trial, the
concept (corresponding to) 't' was tracked, and 't' got a sequence of
25k observations long, each observation being a concept id. Thus 't'
was active in up to 25k episodes (in some episodes, 't' may appear
more than once). Initially, in the first few 100s of episodes say, 't'
will only 'see' single characters (unigrams) immediately next to its
right (concept ids below 95), and later on, as bigrams and higher
n-grams are generated, a mix of unigrams with higher level n-grams is
observed. With tracking of a few ids, we collected 104 sequences over
the different trials, with median sequence size of 1.2k observations,
minimum size of 75, and maximum sequence size of 48k (Table
\ref{tab:exped_seqs_etc}). For the less frequent and newer concepts we
get shorter sequences. There are nearly 1000 unique concepts (concept
ids) in the longest sequences, and 10s of unique concepts in the
shorter ones. A few example concepts with their sequence lengths are
given in Table \ref{tab:exped_seqs_etc}.


\subsubsection{Internal Non-Stationarity in \exped}
\label{sec:internal}

%


Initially, every character (as a predictor) sees single characters in
its stream of observations.  For instance the character 'b', as a
predictor, sees 'a', or 'e', or ' ', etc. immediately occurring
afterwards in an episode, for some time. Later on, as the higher level
concepts, bigrams and trigram, are generated and used in
interpretations, the predictor 'b' also sees bigrams and trigrams (new
concepts) in its input stream, and sees the unigrams characters less
frequently. Furthermore, when concept 'b' joins another concept, such
as 'e', to create a new concept 'be', the distribution around 'b'
changes too, as 'e' is not seen to follow 'b' as much as before the
creation and use of 'be': a fraction of the time when 'b' occurs in an
episode at the lowest level, 'be' is observed, at the highest level of
interpretation, instead of the unigram 'b' followed by 'e'.  We also
note that the frequency of the occurrence of a lower-level concept, at
the highest interpretation level, tends to decrease over time as
higher n-grams, that use that concept as a part, are generated and
used.

The input corpus of text as well as the manner in which a text line is
sampled (uniformly at random) to generate an episode is not changing
here, \ie no change in the occurrence statistics of individual
characters, nor in their co-occurrences, or there is no external
non-stationarity here, but there is internal non-stationarity, as the
interpretations, over time, use highest level matching n-grams.  Thus,
the nature of the prediction task, at the highest interpretation
level, can change gradually, even when the external input stream is
stationary.


\co{ We measure the prediction performance of a few such concepts,
  concepts corresponding to single characters as well bigrams and
  trigrams, according to \ref{eq:lleval}.

We report on performance of different techniques on the 104 sequences.
}


\subsubsection{Overall Performance on 104 \exped  Sequences}

Table \ref{tab:exped1} shows the \logloss (\empllns()) scores of our 4
predictors, averaged over the 104 \exped sequences.  All the
parameters are at their default when not specified, thus \qd is run
with binomial threshold of 5 and a queue capacity of 3. We observe
that \qd does best on average, and as Table
\ref{tab:loss_wins_expedition} shows, pairing and performing sign
tests indicates that \logloss of \qd (with $\minlr=0.001$) outperforms
others (yields the lower loss) over the great majority of the
sequences.

Table \ref{tab:loss_wins_expedition} also changes the $\nstrsh$ to
assess sensitivity to what is considered noise.  Lowering the
$\nstrsh$ makes the problem harder, and we have observed here and in
other settings, that \logloss goes up. For instance, the \logloss
performance of \qd goes from $2.93$ at $\nstrsh=0$ down to $2.3$ with
$\nstrsh=3$.  Of course, with $\nstrsh=1$, we are expecting a
technique to provide a good \pr estimate even though the item has
occurred only once before! Without extra information or assumptions,
such as making the stationarity assumption and assuming that there are
no noise items, or using global statistics on similar situations (\eg
past items that were seen once), this appears impossible.

We next go over the sensitivity and behavior of different techniques
under parameter changes, as well as looking at performance on long \vs
short sequences (subsets of the 104 sequences).




%

\begin{table}[t]  \center
  \begin{tabular}{ |c?c|c|c|c?c|c?c?c| }     \hline
    & \multicolumn{4}{c?}{\qun} &  \multicolumn{2}{c?}{static}  & harmonic & \qd \\ 
    &  2 &  3 &  5 &  10 &  0.001 &  0.01 &  0.01 &  0.001\\ \hline
    logloss & 2.65 & 2.60 & 2.61 & 2.70 & 2.80 & 2.56 & 2.71 & 2.39 \\ \hline
  \end{tabular}
  \vspace*{.2cm}
  \caption{\logloss of various methods on 104 sequences extracted from
    Expedition (median sequence size of 1.2k).}
\label{tab:exped1}
\end{table}

\co{
\begin{table}[t] 
  \begin{tabular}{ |c|c|c|c|c|c|c|c|c| }     \hline
    & \qun, 2 & \qun, 3 & \qun, 5 & \qun, 10 & static, 0.001 & static, 0.01 &
    Harmonic, 0.01 & \qd, 0.001\\ \hline
    logloss & 2.65 & 2.60 & 2.61 & 2.70 & 2.80 & 2.56 & 2.71 & 2.39 \\ \hline
  \end{tabular}
  \vspace*{.2cm}
  \caption{\logloss of various methods on 104 sequences extracted from
    Expedition (median sequence size of 1.2k).}
\label{tab:exped1}
\end{table}
}
  
\begin{table}[t]  \center
  \begin{tabular}{ |c|c|c|c|c| }     \hline
    \qdn, 0.001 vs. $\rightarrow$   &  \qun, 3 &  static, 0.01
    & static, 0.005 & harmonic, 0.01 \\ \hline
 $\nstrsh=3$ & 1, 103   & 0, 104 &  17, 87  & 0, 104 \\ \hline
 $\nstrsh=2$ (default) & 1, 103   & 0, 104 &  16, 88 & 0, 104  \\ \hline
 $\nstrsh=1$  & 13, 91 & 22, 82 &  13, 91 & 8, 96  \\ \hline
 $\nstrsh=0$ & 17, 87 & 13, 91  &  2, 102 & 27, 77 \\ \hline
    \end{tabular}
\vspace{.2cm}
\caption{Number of losses and wins of \qdn, with $\minlr=0.001$,
  pairing it against a few other techniques, on the 104 \exped
  sequences, as we alter the $\nstrsh$ threshold. If observation count
  $\le \nstrsh$ then it is marked \ns (\sec \ref{sec:evalns}). Thus \qd wins
  over \qu with (\qcap$=3$) on 103 of 104 sequences (top left). The
  number of wins of \qd is significant at over 99.9\% confidence level
  in all cases.  Also, \qd wins over static with $\lr \in \{0.01,
  0.005\}$, on all the 19 longest sequences, at default $\nstrsh=2$.}
\label{tab:loss_wins_expedition}
\end{table}

\co{
  %
\begin{table}[t]\center
  \begin{tabular}{ |c|c|c|c|c|c|c|c| }     \hline
    & \qun, 2 & \qun, 3 & \qun, 5 &  \qun, 10 & static, 0.01 & Harmonic, 0.01 & \qd, 0.001\\ \hline
logloss  & 2.59  & 2.54 & 2.55 & 2.64 & 2.51 & 2.71 & 2.32 \\ \hline
  \end{tabular}
  \vspace*{.2cm}
  \caption{Loglosses on the streams of 108 predictors (median size of 1.1k).}
\label{tab:exped1_old}
\end{table}
}

\subsubsection{Sensitivity to Parameters}


\fig \ref{fig:loss_vs_params_expedition}(a) shows the sensitivity to
$\lrmin$ for \qd and harmonic EMA, and $\lr$ for static EMA. We
observe, as in \sec \ref{sec:syn_multi} for the case of multi-item
synthetic experiments, that \qd is less sensitive than both of
harmonic and static EMA, while harmonic is less sensitive than static
when $\lr$ is set low (and otherwise, similar performance to static).
In particular, for longer sequences, lower rates can be better (see
next Section), but for other EMA variants, low rates remain an issue
when faced with non-stationarity (\ie new salient items).

\fig \ref{fig:loss_vs_params_expedition}(b) shows the sensitivity of
\qd to the choice of the binomial threshold, which controls when \qd
uses the queue estimates. There is some sensitivity, but we posit that
at 3 and 5, \qd performs relatively well. We also ran \qd using
different queue capacities
with $\minlr=0.001$ (default is $\qcap=3$), and obtained similar
\logloss results (\eg $2.38$ at $\qcap=2$, and $2.4$ for $\qcap=5$).

\begin{figure}[]
\begin{center}
  \centering
  \subfloat[Logloss vs learning rate.]{{\includegraphics[height=5cm,width=6.5cm]{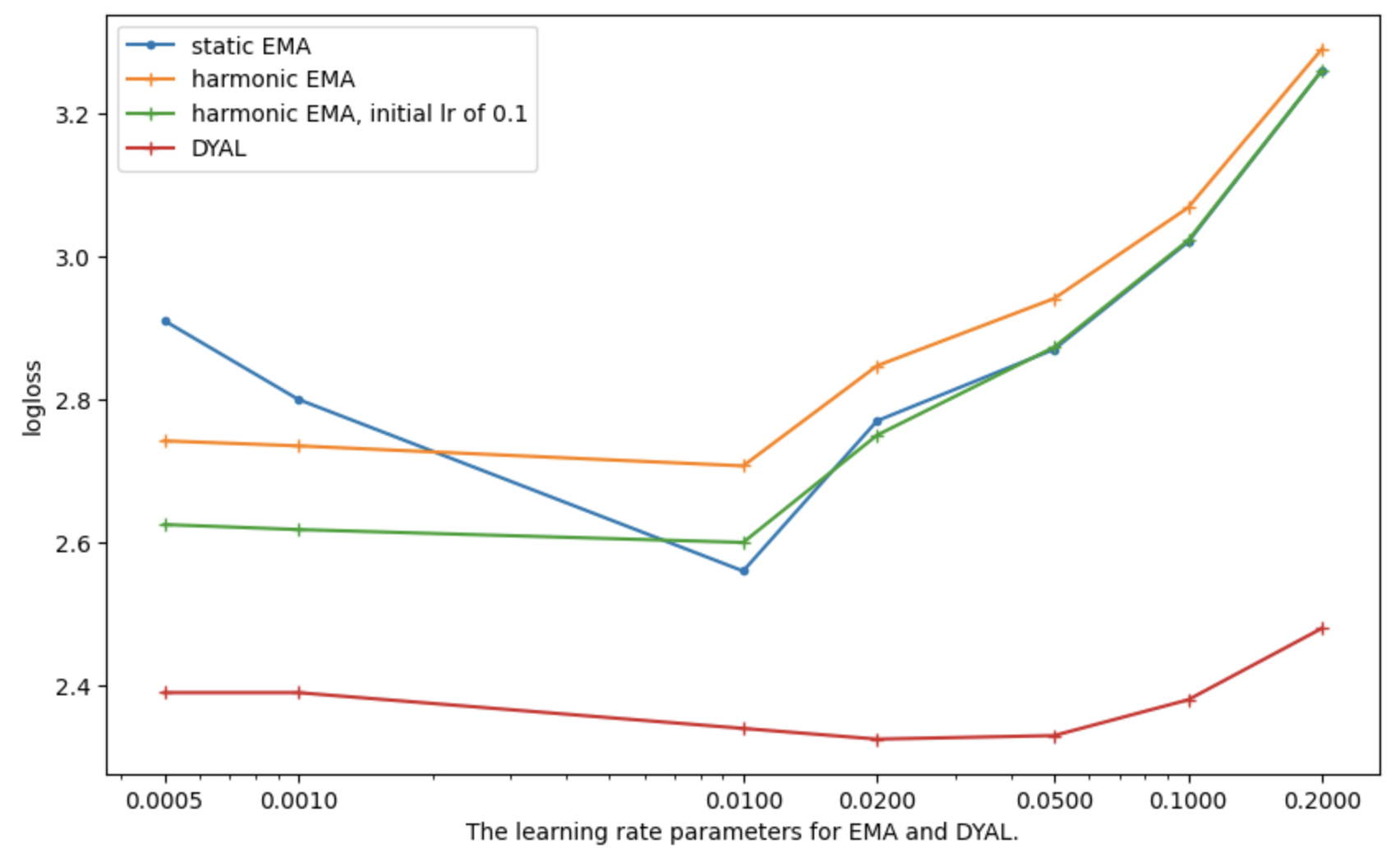}
  }}
  \subfloat[Logloss vs binomial threshold, \qdn.]{{\includegraphics[height=5cm,width=6.5cm]{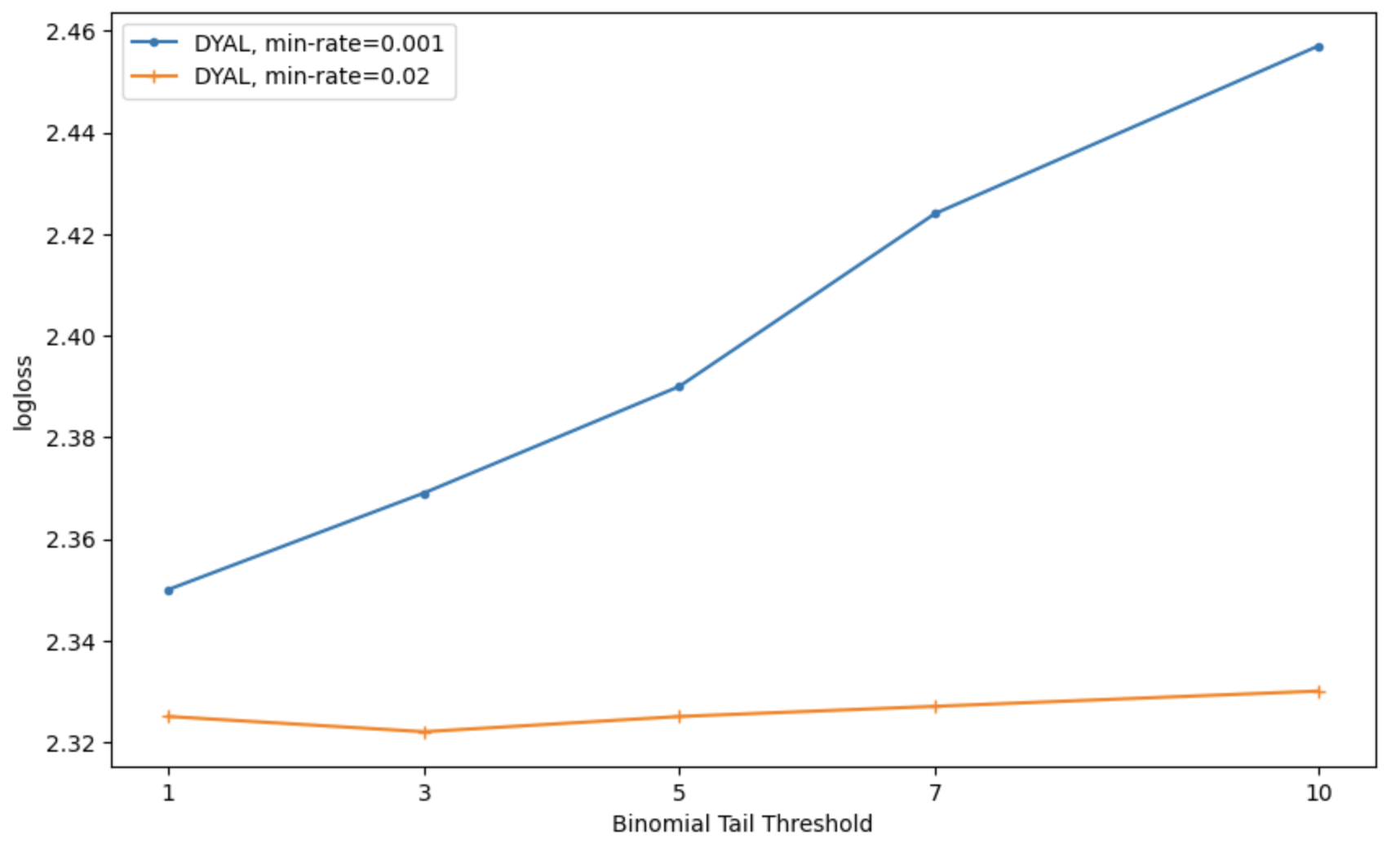}
  }}
\end{center}
\vspace{.2cm}
\caption{Prediction sequences from Expedition: Changing the learning
  rate (left), and the binomial threshold (right) in \qdn, and
  plotting the \logloss (\empllnsp()) performance. }
\label{fig:loss_vs_params_expedition}
\end{figure}



\subsubsection{Longer \vs Shorter Sequences}
\label{sec:long_short}

Of the 104 sequences, there are 19 sequences with length above 5k,
median length being 13k.  We averaged the \logloss performances of
\qdn, as we change its $\minlr$, over these 19, as well as over the
sequences with length below 1k, of which there are 46 such, with
median of 280 observations.  \fig \ref{fig:long_vs_short}(a) shows
that lowering the $\minlr$ works better or no worse, for longer
sequences, as expected, while the best performance for the shorter
sequences occurs with higher $\minlr > 0.01$. A similar patterns is
also seen in \fig \ref{fig:long_vs_short}(a) for \qu technique. Larger
queue capacities work better for longer sequences, but because shorter
sequences dominate the 104 sequences, overall we get the result that a
\qcap of 3 works best for this data overall.

We also note that the shorter sequences yield a lower \logloss (both
figures of \ref{fig:long_vs_short}). This is expected and is due to
our policy for handling \nsn: for methods that allocate most their
initial mass to noise and when this agrees with the \ns marker
referee, one gets low loss.  For instance, at $t=1$, the \ns marker
marks the next item as \nsn, and with a predictor that has all its \pr
mass unallocated, \logloss is 0 (\llnsr() in \fig
\ref{code:norming_etc}(b)).  As the sequence grows longer, and more
salient items are discovered, the average loss can go up.  This is
also observed in the next section when we do character-based
prediction (\sec \ref{sec:goes_up}).

\begin{figure}[]
\begin{center}
  \centering
\hspace*{-0.5cm}  \subfloat[Loss on longer \vs shorter sequences (\qd and \qun).]{{
      \includegraphics[height=5.5cm,width=6.5cm]{
        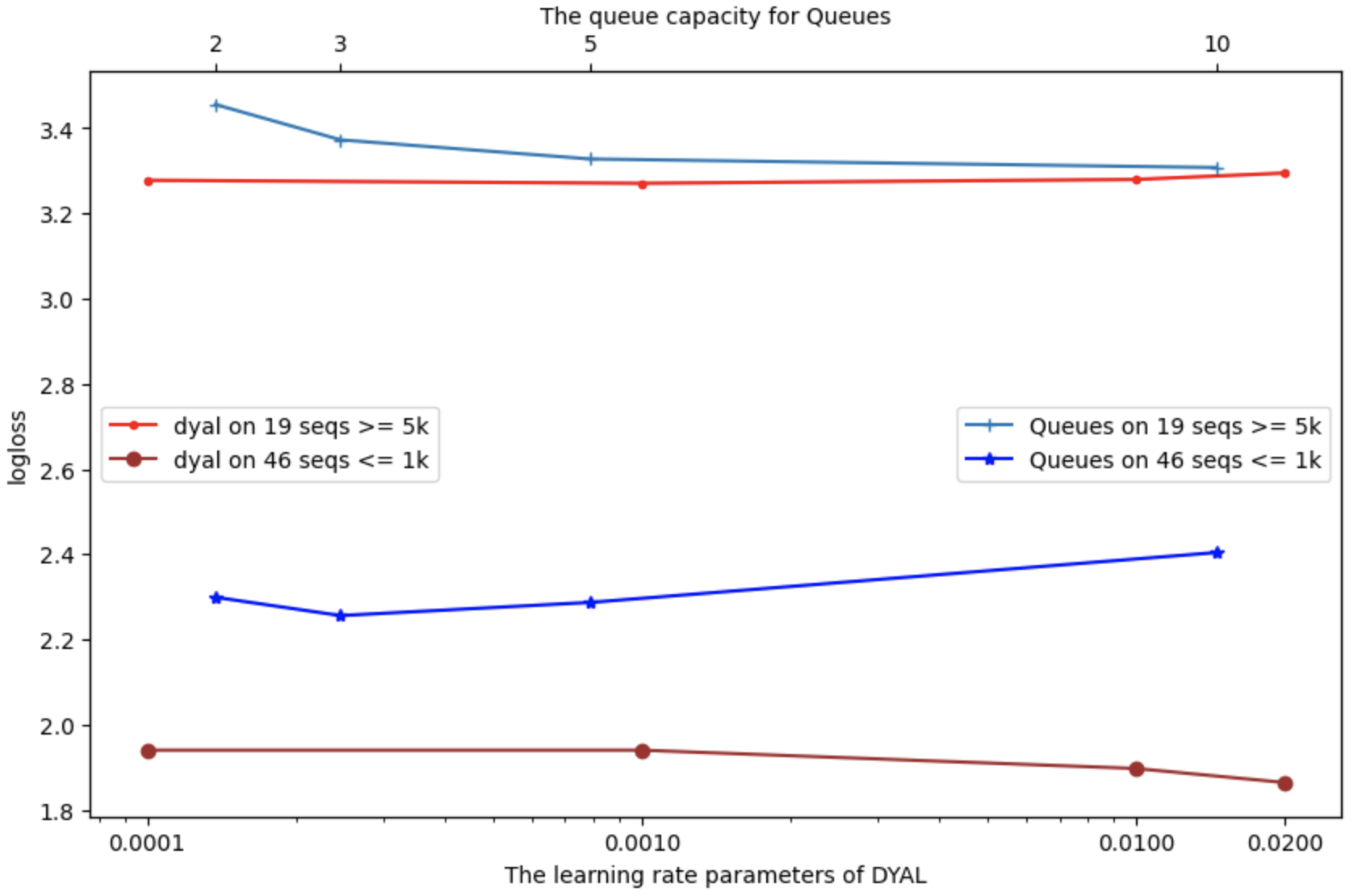}
  }}
\hspace*{.5cm}  \subfloat[Loss on 19 longest sequences \vs binomial threshold.]{{
      \includegraphics[height=5.5cm,width=6.5cm]{
        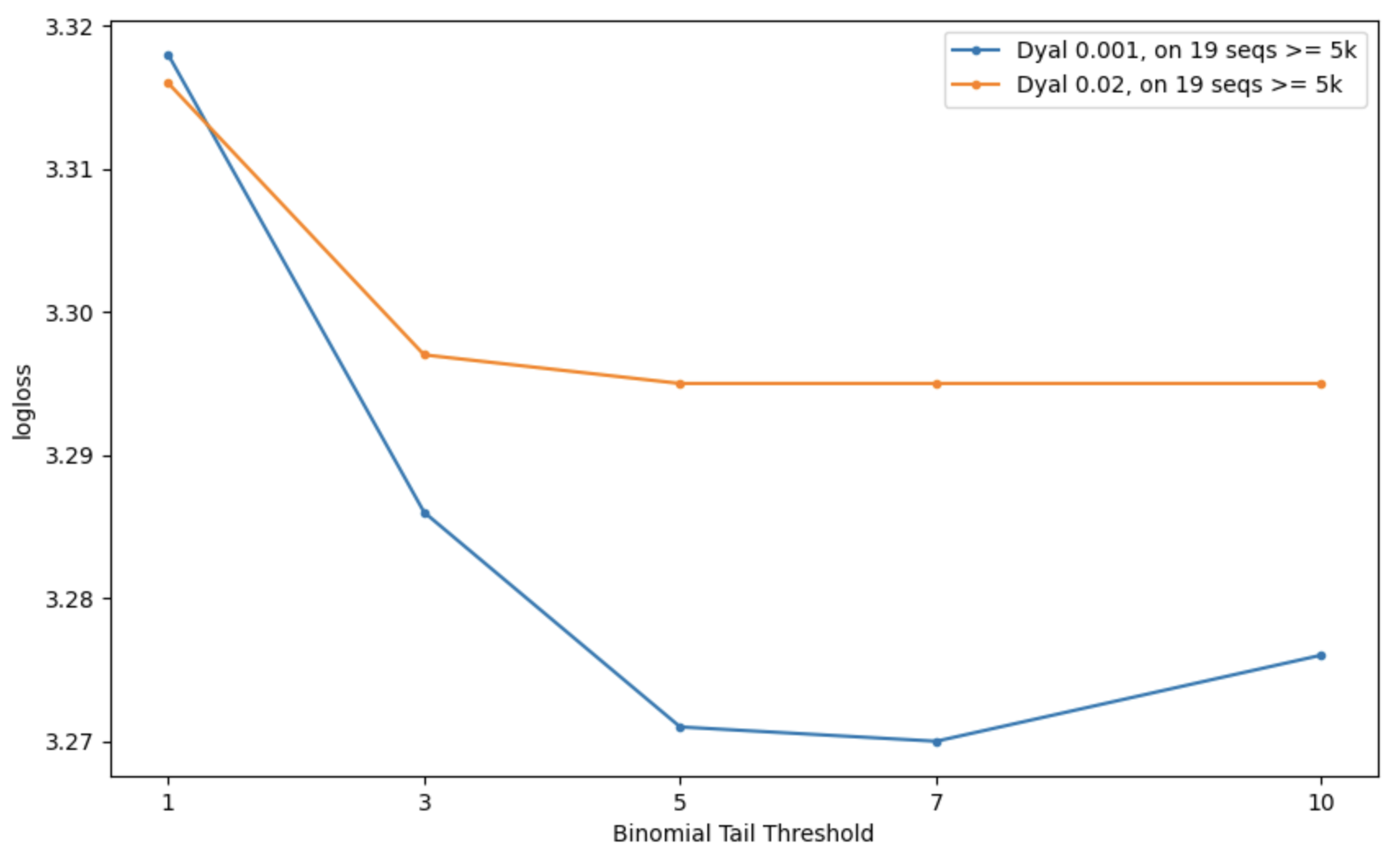}
  }}
\end{center}
\vspace{.2cm}
\caption{\logloss (a) on 19 longest sequences (above 5k) vs 46 shorter
  sequences (below 1000). On the longer sequences and lower $\minlr$
  for \qdn, and a higher queue capacity for \qun, helps. (b) loss vs
  binomial threshold on 19 longest sequences. Setting the value in 3
  to 7 works well, and $\minlr=0.001$ does better than $\minlr=0.02$
  on these sequences. }
\label{fig:long_vs_short}
\end{figure}

\subsubsection{Evolution of the Learning Rates, Degrees, etc. }

\fig \ref{fig:evolution_2_seqs} shows plots of the evolution of
maximum (max-rate) and median of the learning rates in the $\lrmap$
\ of \qd for two predictors (two sequences), the concept "ten", with
just over 200 episodes, and the concept ``l'' with over 12000
episodes. The number of entries (edges) in the $\lrmap$ \ (and the
$\emamap$), or the out-degree, is also reported.  We observe that the
maximum over the learning rates contain bursts every so often
indicating new concepts need to be learned, while the median rate
converges to the minimum, indicating that most predictands at any
given time are in a stable state. On the long 12k sequence, we also
see the effect of pruning the map every so often: the number of map
entries remain below 100 as old and low \pr items are pruned (see \sec
\ref{sec:qspace} and \ref{sec:pruning}).

Appendix \ref{sec:evid} also includes plots of the max-rate on a few
additional sequences (self-concatenations), in exploring evidence for
non-stationarity.

\todo{ Percent of items marked \ns was ? (averaged over all
  sequences?? with std ?) as a function of $\nstrsh$? }

\begin{figure}[t]
\begin{center}
  \centering
\hspace*{-.65cm}  \subfloat[Learning rates, max and median, and out-degree, for concept 'ten' (via
    \qdn).]{{ \includegraphics[height=5cm,width=6.8cm]{
        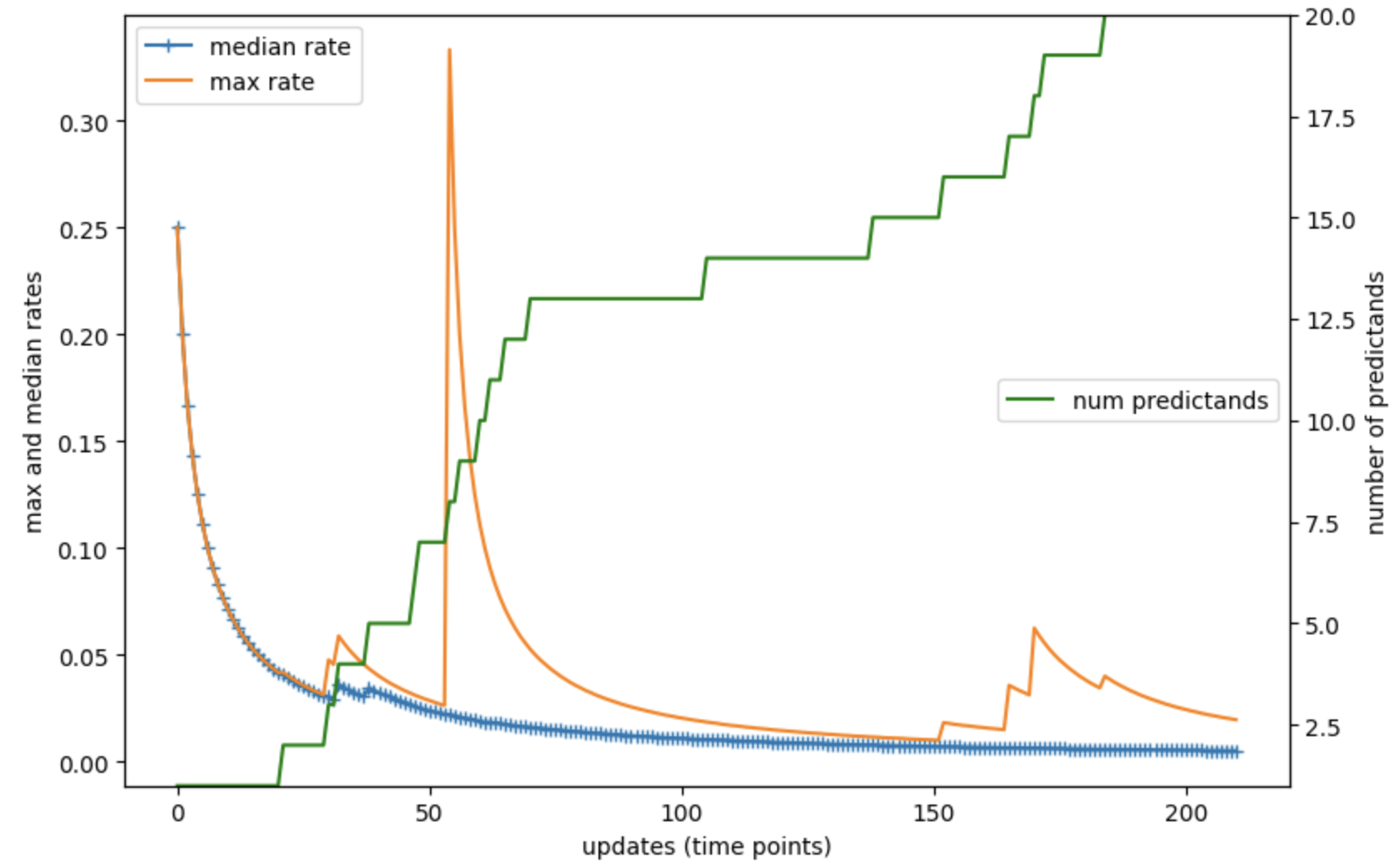} }}
\hspace*{.25cm}  \subfloat[Max rate and out-degree for 'l'.]{{
      \includegraphics[height=5cm,width=6.8cm]{
        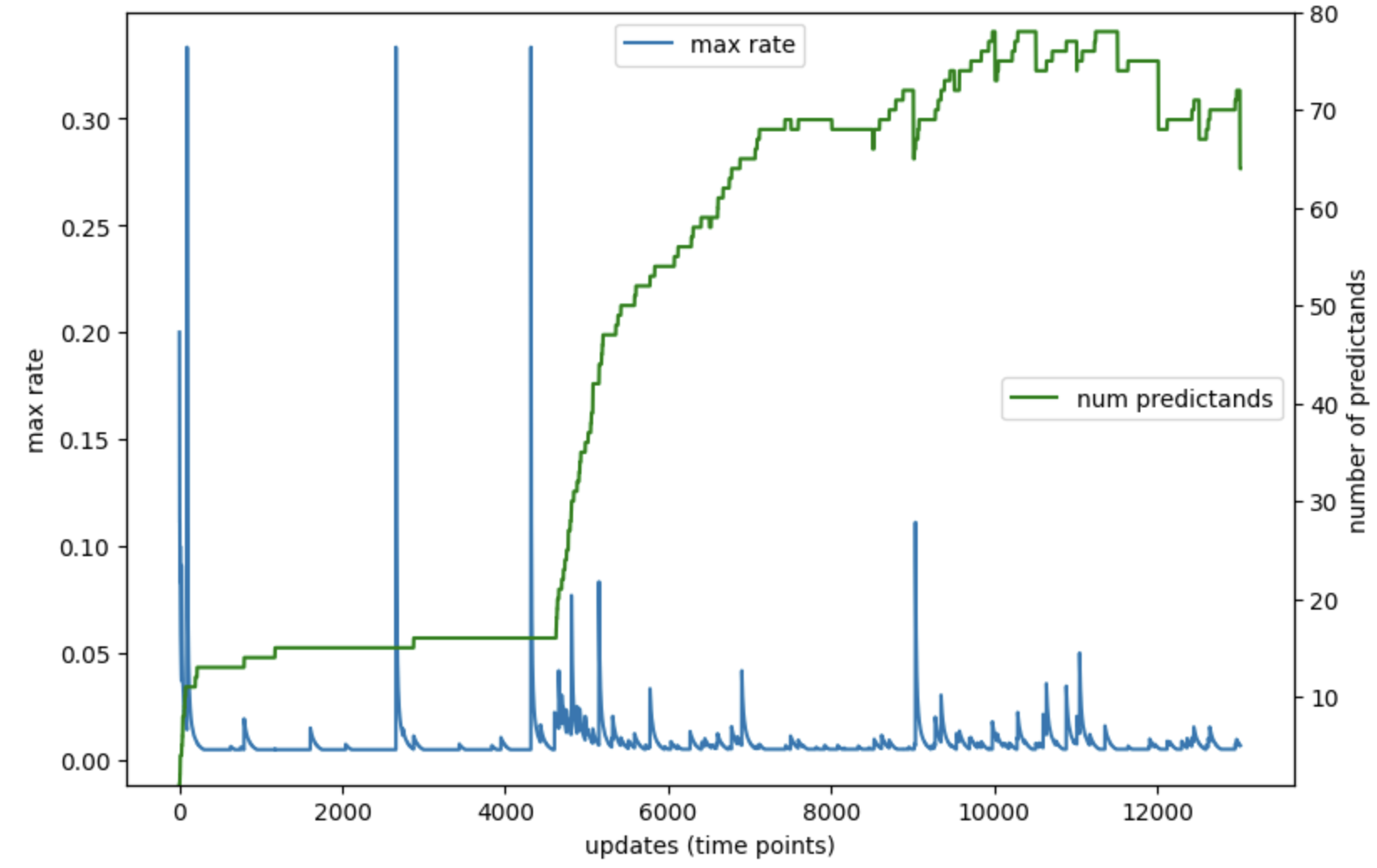} }}
\end{center}
\vspace{.2cm}
\caption{Examples of evolution of the learning rates of \qd and the
  number of predictands, or out-degree (number of entries in
  $\emamap$), with time. (a) Maximum (max-rate) and median learning
  rates and number of predictands for the concept 'ten'. (b) Maximum
  learning rates and number of predictands for the 13k-long sequence
  for the concept 'l'.}
\label{fig:evolution_2_seqs}
\end{figure}

\subsection{\exped at the Character level}
\label{sec:expd_chars}

Here, we do not generate new concepts, thus the prediction task
remains at the (primitive) character level. We track all the
characters as predictors: each character, when observed, predicts the
next character, and, for evaluation, each has its own referee (\ns
marker). We report \logloss (\empllns()) at every time point, meaning
that with each input line, from left to right as in the previous
section, for each character (predictor), we predict (what character
comes to the right) and record the \logloss, then observe and update
the predictor. At certain times $t$, \eg after observing 1k characters
(1000 observe and update events), we report the average of the
\loglossn es upto $t$. For instance, on the line ``abcd'', we collect
the \logloss performance of the predictors corresponding to 'a', 'b'
and 'c' (and after the loss is collected, each predictor updates
too\footnote{In these experiments, we do not update for the last
concept of a line, 'd' in this example, as there is no next character
for such, though one could use a special end of line marker.}).  Thus,
in this section, unlike the previous sections, we are reporting an
average of the \logloss performance of {\em different} predictors,
over time.  Note further that, in this manner of reporting, the more
frequent characters (predictors), such as 'e', 'a', and the blank
space ' ', will have more of an impact on the reported performance.



In this character-prediction setting, there is neither external nor
internal non-stationarity, as described in the previous section, and
even though there is no non-stationarity, we find that \qd
out-performs the other predictors, as seen in Table
\ref{tab:exped_chars} and \fig
\ref{fig:loss_vs_time_expedition}. Table \ref{tab:exped_chars} shows
\logloss averaged over 10 runs for a few choices of parameters, where
a run went to 1k, 2k and 10k time points (prediction episodes).  \fig
\ref{fig:loss_vs_time_expedition} shows \logloss performance for a few
learning rates, from 10k to 300k time points. We see that \qd with one
choice of $\lrmin=0.001$, does best over all snapshots.  The \qu
technique does best with the highest capacity of 10 (in this
stationary setting), and EMA variants require playing with the
learning rate as before, and still lag \qdn. Plain EMA variants,
static and harmonic, underperform for a combination of the way we
evaluate and their insufficient inflexibility with regards to the
learning rate: when an item is seen for the first time, harmonic may
give it a high learning rate, but then is punished in subsequent time
points, as the item may be a low probability event. Note that in our
experiments, we started the harmonic with a initial learning rate of
1.0 (and experimenting with the choice of initial rate $\lrmax$ may
improve its performance, see \fig \ref{fig:loss_vs_rate_unix1} on Unix
commands). Static EMA assigns whatever its fixed learning rate is, to
a new item, which could be too low or too high.





\subsubsection{\llns \ Initially Goes Up}
\label{sec:goes_up}

Consistent with our previous observation on lower (better) \logloss
performance on shorter \vs longer sequences (\sec
\ref{sec:long_short}), here, as more seen items become salient (not
marked \nsn), and with the manner we evaluate with a referee, \logloss
increases over time for most methods but approaches a plateau and
converges. For a 'slow' method such as static EMA with a low $\lr$,
\logloss peaks before it starts going down.  \qd is not slowed down or
is not as sensitive to setting the (minimum) rate low ($\lrmin$ set to
0.001 or lower) (in \fig \ref{fig:loss_vs_time_expedition}, the plots
for \qd with $\lrmin$ from 0.0001 to 0.01 appear identical), and may
actually benefit from a low rate in the long run (Table
\ref{tab:exped_chars_dyal}).


We developed the bounded \logloss primarily for comparing different
techniques.
Here, we observe that the loss may go up on the initial segments of a
sequence, which can be counter-intuitive and undesired (loss should go
down, in general, with more learning).  Appendix \ref{sec:alts}
briefly discusses the alternatives we considered.  It is also possible
that assuming items have rewards and taking into account the rewards
may address this issue as well (mentioned as a future direction in
Conclusions \sec \ref{sec:summary}), and the motivation for the
reward-based CORE score in \cite{expedition1}.



\begin{table}[t] \center
  \begin{tabular}{ |c?c|c|c?c|c|c?c|c| }     \hline
time $\downarrow$    & \multicolumn{3}{c?}{\qun} & \multicolumn{3}{c?}{static (EMA)} &
    harmonic & \ \qdn \ \\ 
    & 3 & 5  & 10 & 0.005 & 0.01 & 0.02  & 0.01 & 0.001 \\  \hline
1000  &  1.69$\pm$0.05  & 1.66 & 1.64 & 1.6 & 1.87 & 2.24 & 2.59 & 1.20 \\ \hline
2000  & 2.01$\pm$0.02 & 1.96 & 1.93 & 2.11  & 2.16 & 2.34 & 2.50 & 1.60 \\ \hline
10000 & 2.39$\pm$0.01 & 2.32 & 2.27 & 2.43 & 2.37 & 2.41  & 2.44 & 2.15 \\ \hline
  \end{tabular}
  \vspace*{.2cm}
  \caption{\logloss averaged over 10 runs of character-level
    \exped (no new concepts generated). \qd does best over
    all time snapshots. }.
\label{tab:exped_chars}
\end{table}

\begin{table}[h]  \center
  \begin{tabular}{ |c|c|c|c|c|c| }     \hline
    $\lrmin$ $\rightarrow$ & 0.0001 & 0.001 &  0.01 & 0.05 & 0.1 \\ \hline
   \logloss & 2.408$\pm$0.002 & 2.398$\pm$0.001 & 2.415$\pm$0.000 & 2.512$\pm$0.002 & 2.652$\pm$0.001 \\ \hline
  \end{tabular}
  \vspace*{.2cm}
  \caption{Character-level \exped using \qd over longer runs: \logloss
    averaged over 5 runs till 300k time points (character
    predictions). $\lrmin$ of $\ge 0.05$ is too high (indicating there
    are salient predictands with probability below $0.1$), while low
    $\lrmin \le 0.001$ do slightly better than $0.01$.  }
\label{tab:exped_chars_dyal}
\end{table}

\begin{figure}[tb]
\begin{center}
  \centering
\hspace*{-.9cm}  \subfloat[\exped at the character level (default $\nstrsh=2$).]{{\includegraphics[height=5.6cm,width=5.6cm]{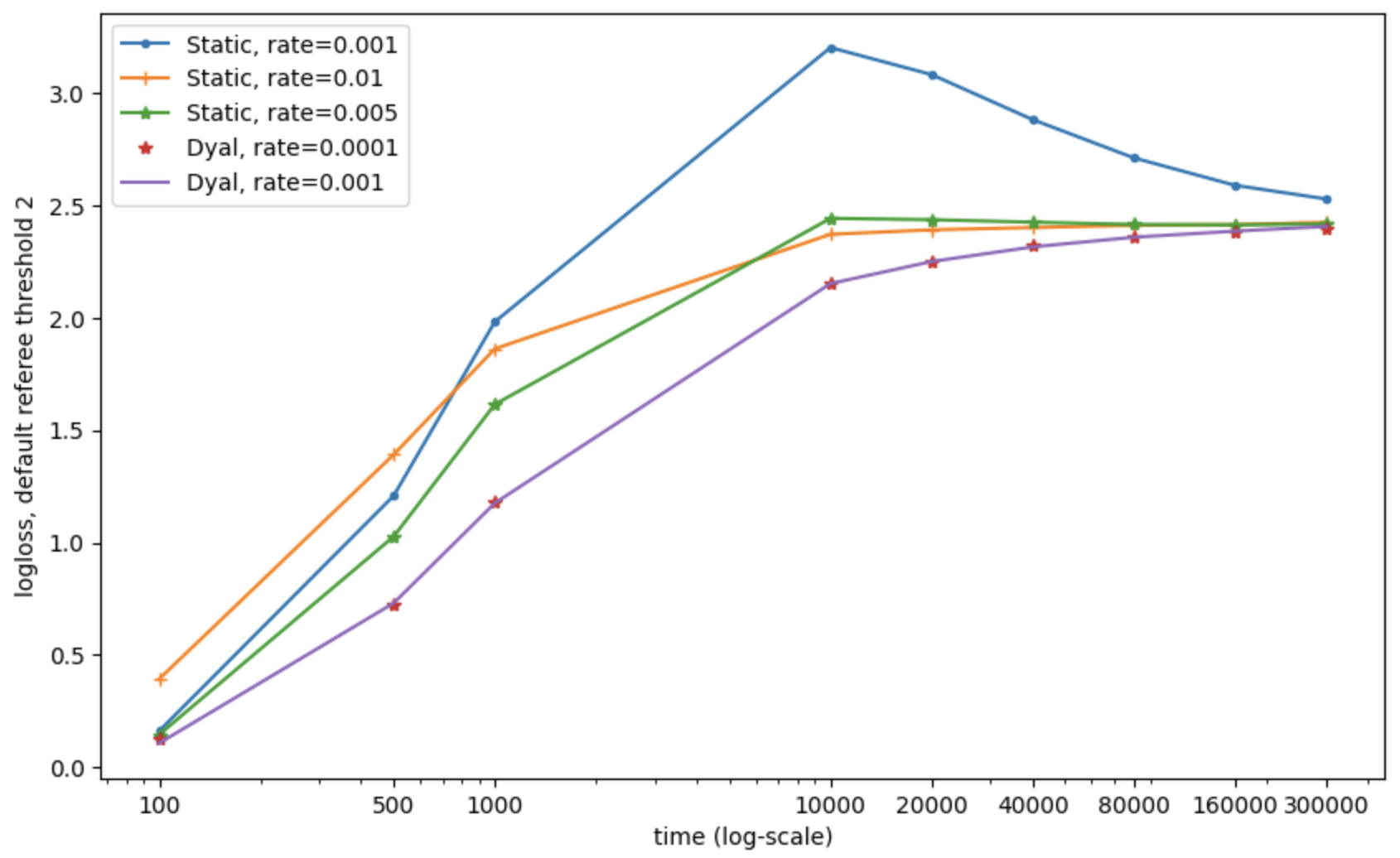}
  }}
\hspace*{0.1cm}   \subfloat[\exped at the character level, with relaxed $\nstrsh=5$.]{{\includegraphics[height=5.6cm,width=5.6cm]{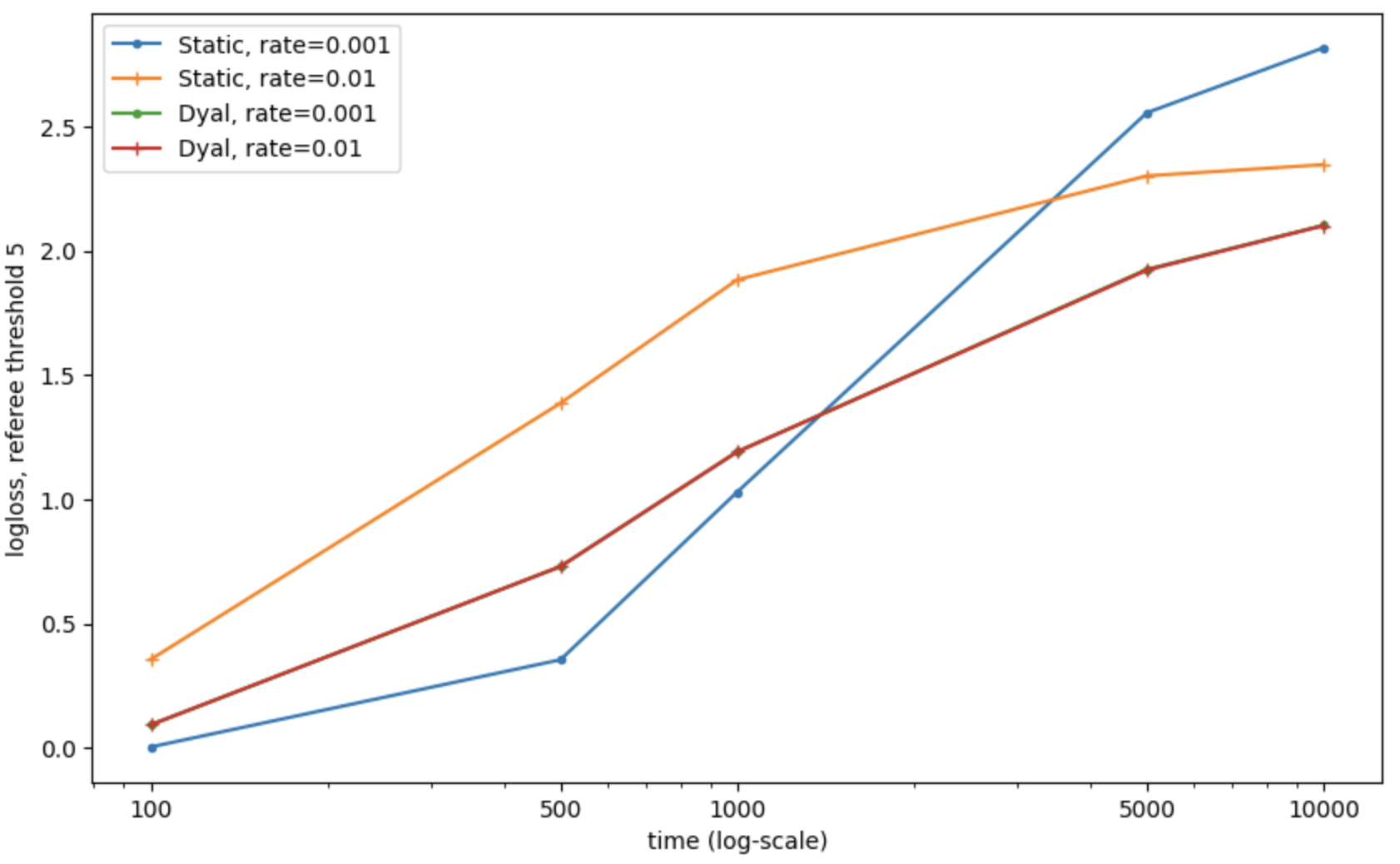}
  }}
\end{center}
\vspace{.2cm}
\caption{\logloss vs time, for \exped at the character level. (a) With
  default $\nstrsh=2$, each point is average over 5 sequences. (b) With the more
  relaxed $\nstrsh=5$ (shown up to 10000) initially all curves
  improve, lowered \loglossn, compared to $\nstrsh=2$. In particular
  static EMA with a low rate benefits most, as it only slowly
  allocates probability to salient items, agreeing more with the
  relaxed referee.  In both cases, \qd with two different rates
  ($\minlr\in \{0.0001, 0.001\}$ in (a) and $\minlr\in \{0.001,
  0.01\}$ in (b)) results in almost identical losses. }
\label{fig:loss_vs_time_expedition}
\end{figure}



\subsection{Unix Commands Data}
\label{sec:unix}

\begin{table}[t]  \center
  \begin{minipage}[t]{0.45\linewidth}
  \begin{tabular}{ |c|c|c|c|c|c| }      \hline
    \multicolumn{6}{|c|}{ 52-scientists (52 sequences)  } \\ \hline
    \multicolumn{3}{|c}{ sequence length } & 
    \multicolumn{3}{|c|}{\# unique commands per user }  \\ \hline 
     median & min & max &  median & min & max \\ \hline
     1.8k & 205 & 12k & 106 & 22 & 359 \\ \hline
  \end{tabular}
  \end{minipage}
    \begin{minipage}[t]{0.45\linewidth}
  \hspace*{1.5cm}
  \begin{tabular}{ |c|c|c| }     \hline
    \multicolumn{3}{|c|}{ Masquerade (50 sequences, 5k each)} \\ \hline 
    \multicolumn{3}{|c|}{\# unique commands per user } \\ \hline 
    median & min & max  \\ \hline
    101 & 5 & 138  \\ \hline
  \end{tabular}
  \end{minipage}  
  \vspace*{.2cm}
  \caption{Statistics on Unix commands data. Left: 52-scientists,
    median sequence length is 1.8k commands, and the median number of
    unique commands per sequence is 106, while one user (sequence) has
    138 unique commands (the maximum).  Right: Masquerade, 50
    sequences, each 5k long. Median number of unique commands per
    sequence is 101, while one user used only 5 unique commands (in 5k commands). }
\label{tab:unix1}
\end{table}

\todo{ describe each data sets, ie: ie how it was recorded, how many days,
  etc}

We look at two sequence datasets in the domain of user-entered Unix
commands, which we refer to as the {\em 52-scientists} data, a subset
of data collected by Greenberg \cite{greenberg88}, and the Masquerade
sequences \cite{masq2001}. These sequences are good examples of high
external non-stationarity: for instance, it is observed in our
previous work on Greenberg sequences \cite{ijcai09} that if the online
updates are turned off mid-way during learning, when it appears that
the
accuracy has plateaued, the prediction performance steadily and
significantly declines. Non-stationarity in daily activities can be
due to a variety of factors, and changes may be permanent, periodic
and so on (project changes, taking vacations, illness, and so on), and
timely adaption is important for automated assistants. We also present
further evidence of non-stationarity below, as well as in Appendix
\ref{sec:evid}. The data likely includes a mix of other non-iid
phenomena as streaks (repeating commands) and certain hidden periodic
behaviors.  Table \ref{tab:unix1} presents basic sequence statistics
for the two sources.


\subsubsection{Task and Data Description}
\label{sec:unix2}

In our previous work \cite{ijcai09}, we were interested in the
ranking performance, for a recommendation or personalization task, and
we used several features of context, mostly derived from previously
typed commands, as predictors of the next command.  We updated and
aggregated the predictions of the features using a variation of EMA
that included mistake-driven or margin-based updates, and showed
significant performance advantage over techniques such as SVMs and
maximum entropy \cite{ijcai09}. Here, we are interested in the extent
to which the sequences are stable enough to learn probabilities
(\prsn), and the relative performance of different probabilistic
prediction techniques.  As the sequences would be relatively short if
we focused on individual commands as predictors (100s long for almost
all cases), we look at the performance of the ``always-active''
predictor that tries to learn and predict a ``moving prior'' of which
next command is typed.\footnote{The Greenberg data has 3 other user
types, but the sequences for all the other types are shorter. We
obtained similar results on the next collection of long sequences, \ie
'expert-programmers', and for simplicity only include the
52-scientists subset. } Therefore here, each command in the sequence
is an item.  For both datasets, we use only the command stubs, \ie
Unix commands such as ``ls'', ``cat'', ``more'', and so on, without
their arguments (filenames, options, etc.). The Masquerade data has the
stubs only, and our experiments with full commands, \vs stubs only, on
52-scientists yields similar findings.  The sequences have been
collected over the span of days to months for both datasets, depending
on the level of the activity of a user. For Masquerade, we only use the
first 5000 commands entered by each user, as the remainder can have
other users' commands interspersed, designed for the Masquerade
detection task\footnote{It may be fruitful to develop an application
of the online predictors we have developed, to the Masquerade detection
task and compare to other methods.  } \cite{masq2001}.




 \subsubsection{Performance on Unix Data}
 \label{sec:unix3}

\begin{figure}[thb]
\begin{center}
  \centering
  \subfloat[\logloss vs rate, Unix 52-scientists sequences.]{{\includegraphics[height=6cm,width=6.7cm]{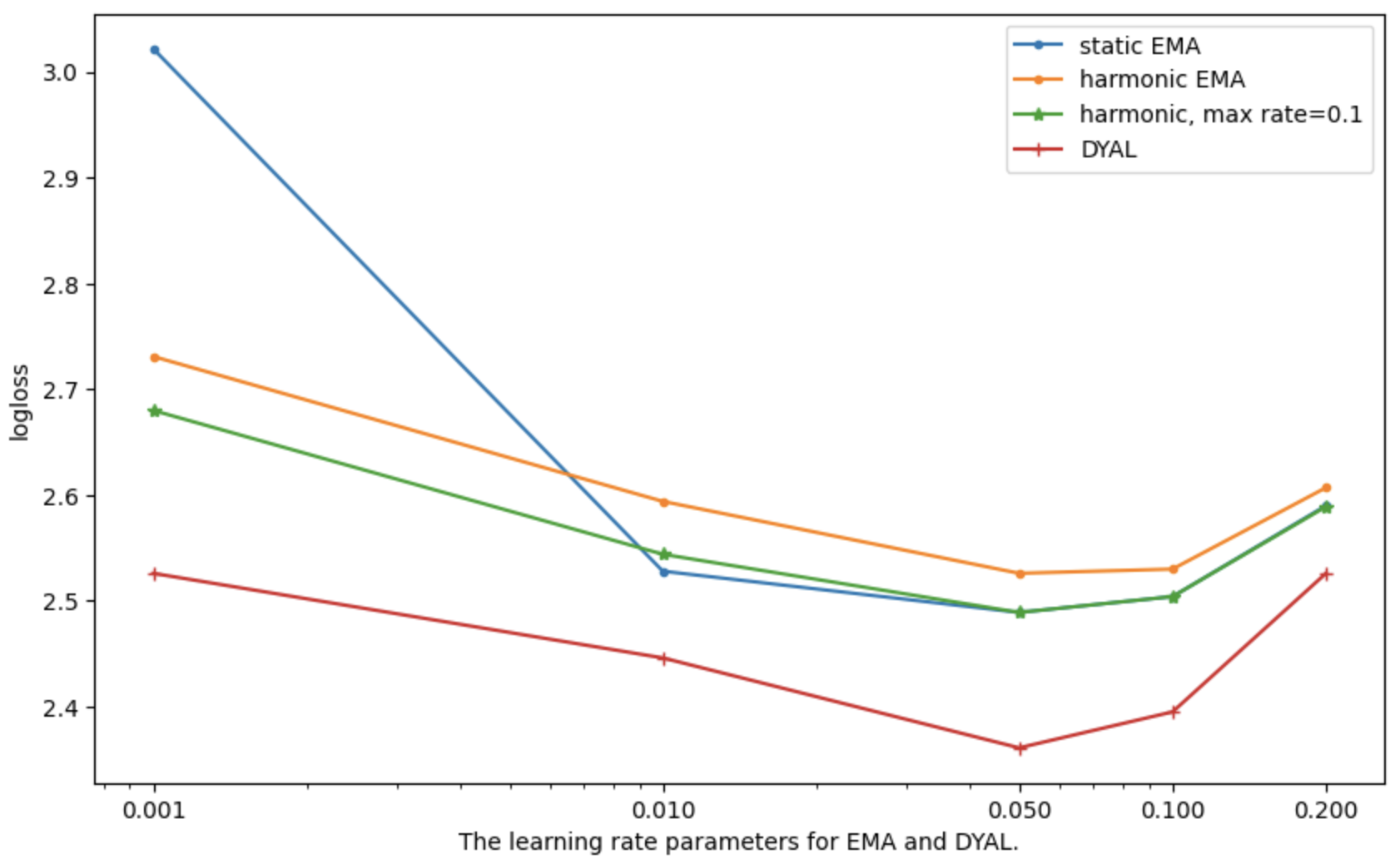}
  }}
  \subfloat[\logloss vs rate, Unix Masquerade sequences.]{{\includegraphics[height=6cm,width=6cm]{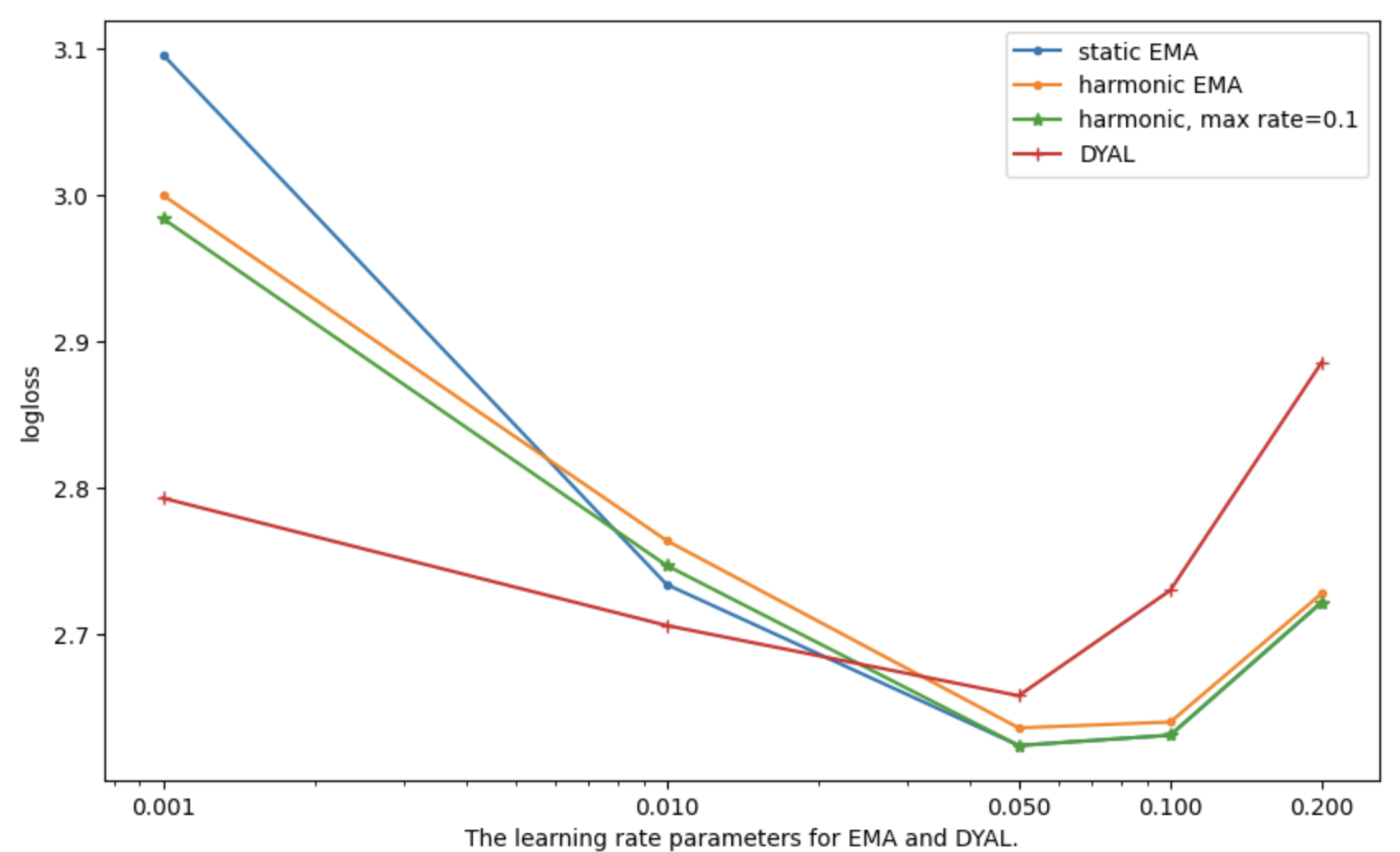}
  }}
\end{center}
\vspace{.2cm}
\caption{\logloss performances (\empllns()) \vs $\lr$ parameters: (a)
  52-scientists (b) Masquerade. }
\label{fig:loss_vs_rate_unix1}
\end{figure}

\fig \ref{fig:loss_vs_rate_unix1} shows the performance of EMA
variants as a function of the learning rate, and Table
\ref{tab:qu_on_unix} reports on \qu performance with a few \qcap
values, and includes the best of EMA variants.  We observe a similar
v-shape pattern of performance for static and harmonic EMA, while for
\qdn, as before, the plots show less of a dependence on $\minlr$
compared to other EMA variants. However, the degradation in
performance as we lower $\minlr$ is more noticeable here. \qd performs
better on 52-scientists compared to others, but is beaten by static
EMA on Masquerade. Below, we report on paired comparisons and the
effect of changing the referee threshold $\nstrsh$.

We note these data indeed have higher non-stationarity compared to
previous ones when we consider a few indicators: the best performance
occurs when $\minlr$ is relatively high at $\approx 0.05$, larger than
our previous datasets (\fig \ref{fig:loss_vs_rate_unix1}). Table
\ref{tab:qu_on_unix} also shows that \qu does best here with smaller
\qcap values.  Appendix \ref{sec:evid} presents additional experiments
showing further evidence of external non-stationarity.




\begin{table}[t]  \center
  \begin{tabular}{ |c|c|c|c|c?c|c|c|c|c| }     \hline
 \multicolumn{5}{|c?}{ 52-Scientists } &  \multicolumn{5}{c|}{Masquerade}   \\ \hline
      \multicolumn{4}{|c|}{ \qun } & \qdn & \multicolumn{4}{c|}{ \qun }   & static \\ 
      2 &  3 &  5 &  10 &  0.05 &  2 &  3 &  5 &  10  & 0.05 \\ \hline
   2.581$\pm$0.28 & 2.586 & 2.629 & 2.686 & 2.361 & 2.769$\pm$0.5 & 2.735 & 2.754 & 2.830 & 2.624 \\ \hline
    \end{tabular}
\vspace*{.2cm}
\caption{On Unix sequences, \qu \logloss performance, and the best of
  EMA variants. Lower \qcapn s work better, suggesting significant
  non-stationarity. }
\label{tab:qu_on_unix}
\end{table}

\subsubsection{Pairing and Sign-Tests on Unix Sequences}
\label{sec:unix_pairing}

We compare \qd to static EMA as harmonic and static behave similarly.
At $\lr$ of 0.05, best for both \qd and static EMA, on the
52-scientist sequences, we get 46 wins for \qd over static (lower \logloss
for \qdn), and 6 wins for static. On average, 13\% of a sequence
is marked \ns by the referee.  As we raise the referee threshold $\nstrsh$ from 2 to
3 to 4, we get additional wins for \qd and the \logloss performances
improve for all methods, and the fraction of sequence marked \ns goes
up, reaching to 18\% at $\nstrsh=4$. Conversely if we lower the
referee threshold to 1, we get fewer 34 wins for \qd over static EMA
(and 11\% marked \nsn).



On Masquerade's 50 sequences, we get only 10 wins for \qd \vs 40 for
static, again at $\lrmin=0.05$, where both do their best, which is
statistically highly significant. With the default of $\nstrsh=2$ only
5\% of a sequence is marked \ns on average. As in the case of
52-scientists, when we increase the referee threshold from 2 to 3 to
4, we get additional wins for \qdn, and at $\nstrsh=4$, \qd has 37
wins over 13 wins for static (8\% marked \ns with
$\nstrsh=4$). Similar to the above 52-scientists case, lowering the
threshold to 1 leads to more wins for static. Importantly, if use a
referee that is based on a short window of the last 200 (rather than
our simple unlimited window size) and require seeing an item at least
twice in the last 200 timepoints to mark it as salient (not \nsn),
again \qd becomes superior (47 wins for \qdn, and 9\% marked \nsn).

Why a fixed and relatively high learning rate of 0.05 does relatively
well here compared to the more dynamic \qdn, on Masquerade sequences?
Any assumption behind the design of \qdn, in particular the sufficient
stability assumption, may be partially failing here.  For many items,
their \prn, or appearance frequency, may be high once they appear, but
their stability period may be too short, and a simple high
learning-rate of 0.05 may work just as well, compared to the slow
two-tiered approach of first detection and estimation via the \qu
technique, and at some point, switching to the queues estimate. As we
increase the referee threshold $\nstrsh$, we focus or bias the
evaluation further on the more stable items in the sequence, and we
get better relative results for \qdn. This observation gives more
credence to the stability explanation.








\section{Related Work}
\label{sec:related}








Our problem lies in the areas of online learning \cite{onlineml18},
non-stationarity and learning under concept drift
\cite{nonstationarity2015,survey_on_concept_drift_2014gama,conceptDrift2023}, change detection
\cite{cpEval2022,nonstationarity2015}, probability forecasting and
assessing quality of output probabilities (propriety, calibration)
\cite{ml_calibration_survey21,rev2022,brier1950,good1952,prequential_dawid1984},
(non-parametric) density estimation or distribution learning
\cite{driftingDistros2024,distroDrift2023Nonparametric,discreteDistrosInfiniteSupport2020},
with connections to streaming data structures and algorithms
\cite{streams_review,Gama2009KnowledgeDF}
and time-series analysis
\cite{intro_time_series2018,coherent_forecasting2015}.  To the best of
our knowledge, this combination of open-ended categorical (nominal)
probability prediction under non-stationarity,
%
with attention to practical efficiency,
has not been studied before.  We provided pointers on relevant work
for the tools and techniques we used throughout the paper.  In this
section, we further situate our work within the broader context of
similar tasks and problem domains, and provide a short history.





Our task and the \qd solution involves a kind of {\em
  implicit}
change detection (CD)
\cite{survey_on_concept_drift_2014gama,nonstationarity2015,incrementalNB_2020}:
%
Two main categories of adaptation strategy to change are blind (or
implicit) and informed (or explicit)
\cite{survey_on_concept_drift_2014gama,nonstationarity2015}.  In all
of \qun, EMA (static or harmonic-decay), and \bxn, adaptation to
non-stationarity is done automatically or without explicit
detection. \dyal comes closest to explicit detection, and it uses the
implicit \qs to do so. In our task, the system need only adjust its
output probabilities in a timely manner.  In some tasks, such as
monitoring for safety and potential attacks, explicitly pinpointing
the (approximate) time of change can be important too
\cite{nonstationarity2015,multicp19}.
%
%
%
Change detection is a diverse subject
studied in several fields such as system monitoring, psychology (\eg
within human vision) and
image processing and time series analysis (\eg
\cite{multicp19,tartakovsky14,cpEval2022,atto2021change}). Here, we
seek a (timely) response and adaptation to a change.
%
Many variants of moving averages are also used in time-series analysis
(\eg the ARMA model). There, the observations are ordinal (such as
counts) even if discrete (\eg
\cite{intro_time_series2018,coherent_forecasting2015}), and
stationarity or limited non-stationarity is typically assumed.
Similarly, much past work, \eg on variable window sizes, has addressed
numeric data, for instance tracking means and variances
\cite{adaptive_windows_bifet2007,histo_change2008} (\sec
\ref{sec:box}).
It has been observed that CD is understudied for categorical data
\cite{multicp19}, and the authors develop efficient model-based
explicit CD techniques for streaming data: the set of items is known
and fixed in this work (a multinomial distribution).  Their evaluation
is based on mearsuring change-detection rates (\eg true-positive rates
for an acceptable fixed false-positive rate).

The online observe-update cycle, updating a semi-distribution, has a
resemblance to online (belief) state estimation, \eg in Kalman filters
and partially observed Markov models for control and decision making,
in particular with a discrete state space
\cite{Dean1991PlanningAC,wikip1}. Here, the goal is pure prediction
(\vs action selection or control), though some of the techniques
developed in this work may be useful to that setting, specially when
one faces a changing external world. Extensions of the Karlman filter
with dynamic window sizes (for numeric sequences) has proved useful
for adapting supervised and unsupervised learning techniques, such as
Naive Bayes, to non-stationarity
\cite{adaptive_windows_bifet2007,kalmanFilteringEvolvingData_2021}.
More generally, streaming algorithms aim to compute useful summary
statistics, such as unique counts and averages, while being space and
time efficient, in particular often requiring a single pass over a
large data set or sequence, such as the count-min sketch algorithm
\cite{streams_review,Gama2009KnowledgeDF}. Here, we have been
interested in computing recent proportions in a non-stationary
setting, for continual prediction, implemented by each of many
(severely) resource-bounded predictors.
%
%
Learning finite state machines in a streaming manner shares similar
philosophy and a similar subproblem of change detection (in
particular, \fig 1, ``system for continuous learning''
\cite{Balle2014AdaptivelyLP}).

Non-parametric density estimation techniques are often
based on kernels or
keeping track of specific instances and assume stationarity (\eg
\cite{discreteDistrosInfiniteSupport2020,nonparametric_1987}).  Recent work, in
computational-learning theory, comes closest to ours and extends the
density estimation task to an open-ended, in particular with
infinite-support, non-stationary setting \cite{driftingDistros2024},
but (space) efficiency and empirical performance of the predictor is
not a focus,
and the kept history grows logarithmically in stream length.
%
Furthermore, absolute loss (total-variational distance) is used to
assess the theoretical quality of algorithms (limiting applicability
to estimately only high probabilities well, above say 0.1, see \sec
\ref{sec:box}).

This paper extends our ongoing work on developing large-scale online
multiclass learning techniques, for lifelong continual learning when
the set of classes is not known apriori, and often large and dynamic
(\ie open-ended) \cite{abound09,updateskdd08}, in particular in the
framework of Prediction Games for learning a growing networked
vocabulary of concepts \cite{pgs3,pgs1,expedition1}. In online
(supervised) learning, \eg \cite{onlineml18,perceptron,winnow1}, the
focus is often on the interaction of the predictors (features) and,
for example, on learning a good weighting combination for a linear
model, while we have focused on learning good {\em independent}
probabilistic predictors, akin to the (multiclass) Naive Bayes model
\cite{lewis-naive40}, but in an unsupervised (or self-supervised)
setting, as well as the counting techniques for n-gram language models
\cite{slm}, but in a non-stationary setting.  Another avenue of work
is calibrating classifiers after training (trained on a ranking or
accuracy-related objective) \cite{ml_calibration_survey21}, and there
may be extensions applicable to our non-stationary open-ended and
online setting.  Similarly, investigating online techniques for
learning good mixing weights, but also handling the
non-stationarities, may prove a fruitful future direction, for
instance in the mold of (sleeping) experts algorithms
\cite{sexperts1,onlineml18,experts_uai}.


\co{
Learning to predict new items has some resemblance to cold-start
problems, in particular the cold-item problem in
recommender systems \cite{wikip1}, though the setting (constraints and
goals) are rather different (the goal there is providing a good
ranking rather then good distributions and probabilities).
}



Learning-rate decay has been shown beneficial for training on
backpropagation-based neural networks, and there is research work at
explaining the reasons \cite{You2019HowDL,Smith2018ADA}. Here, we
motivated decay variants in the context of sparse EMA updates and
learning good probabilities fast, and motivated keeping
predictand-specific rates.




\section{Summary and Future Directions}
\label{sec:summary}


%

Dynamic worlds require dynamic learners: often there is change, but we
typically cannot predict what will change (or when).  In the context
of online open-ended probabilistic categorical prediction, we
presented a number of sparse moving average (\sman) techniques for
resource-bounded (finite-memory) predictors.
We described the challenges of assessing probability outputs and
developed a method for evaluating the probabilities, based on
bounding \loglossn, under noise and non-stationarity.
We showed that
different methods work best for different regimes of non-stationarity,
but provided evidence that in the regime where the probabilities can
change substantially but only after intermittent periods of stability,
the \qd \sma which is a combination of the sparse EMA and the \qu
technique, is more flexible and has advantages over either of the
simpler methods: the \qu predictor has good sensitivity to abrupt
changes (can adapt fast), but also has higher variance, while plain
EMA is slower but is more stable, and combining the two, with
predictand-specific learning rates, leads to faster more robust
convergence in the face of non-stationarities. The use of
per-predictand rates can also serve as indicators of current
confidence in the prediction estimates.




This work originated from the problem of
assimilating new concepts (patterns), within the \pgs
framework, where concepts serve as both the predictors and the
predictands in the learning system.  In particular, here we made a
distinction that even if the external world is assumed stationary, the
development of new concepts, explicitly represented in the system,
results in internal, in particular developmental, non-stationarity, for the
many learners within the system. This means sporadic abrupt changes in
the co-occurrence probabilities among the activated concepts
as the system changes and evolves (over many episodes) its
interpretations of the
raw input stream. However, we expect that the rate of the generation
of new concepts, which should be controllable, can be such that there
would be stability periods, long enough
to learn the new probabilities, both to predict well and to be
predicted well. Still, we strive for predictors that are
fast in adapting to such intermittent patterns of change.






We plan to further evaluate and develop the \smas
within the \pgs framework,
in particular under longer time scales and as
%
the concept
generation and interpretation techniques are advanced,
and we seek to better understand the interaction of the various system
components.  We touched on a variety of directions in the
the paper. For instance,
it may be fruitful to take into account item (predictand) rewards when
rewards are available,
and this may lead to different and interesting design changes in the
prediction algorithms, or how we evaluate.


\section*{Acknowledgments}
\label{sec:acks}

Many thanks to Jana Radhkrishinan for granting the freedom conducive
to this work at Cisco, and to Tom Dean's SAIRG reading group, including Brian
Burns, Reza Eghbali, Gene Lewis, and Justin Wang, for discussions and
valuable pointers and feedback: we hope to crack portions of the
hippocampal and frontal neocortex code and loops!  Many thanks to
James Tee for his assistance, and discussions on online learning and
state estimation. I am grateful to John Bowman for providing me the
valuable pointer and a proof sketch based on Rao-Blackwellization.


\bibliographystyle{plain}  


\bibliography{global}


\vspace*{3in}


\appendix


\vspace*{-1in}
\section{Further Material for the Evaluation Section}
\label{sec:app_eval}

In this section, we provide proofs and further details for the
evaluation section \ref{sec:eval}. We begin with quadratic loss and
its insufficient sensitivity (to estimating low probabilities), and we
develop properties that point to the near propriety of \llns() in \sec
\ref{sec:np}, including extensions of \kl() to \sds and how scaling
and shifting \sds affects optimality of \sds under \kl \ scoring. \sec
\ref{sec:alts} discusses a few alternatives that we considered for
evaluating sequence prediction when using \logloss and handling noise
(\nsn) items.



\subsection{The Quadratic Loss, and its Insensitivity to \pr Ratios}
\label{sec:brier}





The Quadratic loss (Brier score) of a candidate \sdn, \bl, is defined as: 
\begin{align}
\hspace*{-.35in}  \label{eq:brier1}
  \bl(\Q|\P) := \expdb{o \sim \P}{\bitem(\Q, o)}, \mbox{ where \ } \bitem(\Q, o) :=
  \sum_{i\in \I} (\delta_{i,o} - \Q(i))^2 , 
\end{align}


where $\delta_{i,o}$ denotes the Kronecker delta, \ie $\delta_{i,o}=1$
when $o=i$ and 0 otherwise ($\delta_{i,o}=0$ for $i\ne o$) and an
equivalent expression is $\bitem(\Q, o)=(1-\Q(o))^2 + \sum_{i \in \I,
  i \ne o} \Q(i)^2$.  Note that at each time point $t$, to compute the
value of \bitem() \ (the loss or cost over an observation), we go
over all the items in $\I$, and we have $0\le \bitem(\Q, o) \le 2.0$.
One can view the scoring rule \bitem() as taking the (squared)
Euclidean distance between two vectors: the {\em Kronecker vector},
\ie the vector with Kronecker delta entries (all 0, except the
dimension corresponding to observed item $o$), and the probability
vector corresponding to the \sd $\Q$.





\co{
  
Note that at each time point $t$, to compute \bitem \ (the Brier loss
or cost over an observation), we go over all the items in $\I$.
As presented in \cite{selten98}, \bitem \ can also be written as
follows, which better shows its connection to Euclidean distance
discussed next:

\begin{align}
  \label{eq:brier2}
  \bitem(\Q, o) := \sum_{i\in \I} (\delta_{i,o} - \Q(i))^2 ,
\end{align}


where $\delta_{i,\oat{o}}$ is the Kronecker delta, \ie
$\delta_{i,\oat{o}}=1$ when $\oat{o}=i$ and $\delta_{i,\oat{o}}=0$
when $\oat{o}\ne i$ (maybe the defn of Kronecker delta is moved
earlier..)..
}

In particular, for instance in \cite{selten98}, 
when both $\P$ and $\Q$ are \disn, it is established that
the \brier score is equivalent to (squared) Euclidean distance to the
true probability distribution, where the distance is defined as:
\begin{align*}
  \bdist(  \Q_1, \Q_2 ) := \sum_{i\in \I} (\Q_1(i) - \Q_2(i))^2
  \mbox{\ \ \ \ \ \ \ (defined for \sds $\Q_1$ and $\Q_2$)}
\end{align*}





The equivalence is in the following strong sense (extended to \sdsn): 

\begin{lemma} (sensitivity of \brier to the magnitude of \pr shifts only) 
Given \di $\P$ and \sd $\Q$, defined over the same finite set $\I$, then 
\hspace*{.5in} $\bl(\Q|\P) = \bdist(\Q, \P)$ 
\label{lem:quad_eq}
\end{lemma}

\begin{proof}
This property has been established when both are \dis (\ie when
$\sm{\Q}=1$) \cite{selten98}, and similarly can be established for \sds by
writing the expectation expressions, rearranging terms, and
simplifying. We give a proof by reducing to the \di case: when $\Q$ is
a strict \sdn, we can make a \di variant of $\Q$, call it $\Q'$, via
adding an extra element $j$ to $\I$ with remainder \pr
$\Q'(j)=\u{\Q}$ (note: $\P(j)=0$). Observing the following 3
relations establishes the result: 1) $\bdist(\P, \Q) = \bdist(\P, \Q')
- \Q'(j)^2$, 2) $\bl(\Q'|\P) = \bdist(\Q', \P)$ (as both are \disn),
and 3) $\bl(\Q'|\P) = \bl(\Q|\P) + \Q'(j)^2$, as with every draw from
$\P$ we incur the additional cost of $\Q'(j)^2$ compared to the cost
from $\Q$): by definition, we have, $\bl(\Q'|\P)=\sum_{i\in \I,
  i\ne j}\P(i)\left(\Q'(j)^2+(1-\P(i))^2+\sum_{u\in \I, u\ne i, u\ne
  j} \P(u)^2\right) = \sum_{i\in \I, i\ne j}\P(i)\Q'(j)^2 + \sum_{i\in
  \I, i\ne j}\P(i)\left((1-\P(i))^2+\sum_{u\in \I, u\ne i, u\ne j}
\P(u)^2\right) = \Q'(j)^2 + \bl(\Q|\P)$ (the last equality in part
follows from $\sum_{i\in \I, i\ne j}\P(i)=1$).
\end{proof}


Examining the distance version of the loss, we first note that \brier
has a desired property of symmetry when $\Q$ is also a \din, \ie
$\bl(\Q_1|\Q_2)=\bl(\Q_2|\Q_1)$ (for two \dis $\Q_1$ and $\Q_2$).  We
can also observe that \brier is only sensitive to the magnitude of
shifts in \prsn:


\begin{corollary}
  \label{quad_cor}
  (sensitivity of \brier to the magnitude of \pr shifts only) With \di $\P$
  and \sd $\Q$ (on the same items $\I$), let
  $\Delta_i:=\P(i)-\Q(i)$. Then $\bl(\Q|\P)=\sum_{i\in \I}
  \Delta_i^2$.
\end{corollary}

\brier is not sensitive to the size of the source (or destination) of
the original probabilities of items those quantities are transferred
from: it does not particularly matter if a positive \pr is reduced to
0 in going from $\P$ to $\Q$.  For instance, assume $\P=\{1$:$0.9$,
$2$:$0.1\}$. Consider the two candidate \sdn s $\Q_1=\{1$:$0.8$,
$2$:$0.1\}$, $\Q_2=\{1$:$0.9$, $2$:$0.0\}$ ($\Q_i$ differing from $\P$
only on item $i$). In terms of violating a deviation threshold (\sect
\ref{sec:eval}), $\Q_2$ violates all ratio thresholds on item 2, and
the \logloss (\sec \ref{sec:logloss2}) is rendered infinite on it,
while \sd $\Q_1$ has a relatively small violation. However, for both
cases $\Delta=0.1$, and we have $\bl(\Q_1|\P) = \bl(\Q_2|\P) = 0.1^2$,
and $\Q_1$ and $\Q_2$ would have similar empirical losses using the
\brier score.
\co{
Thus for a reference \di $\P$ and two candidate \sds $\Q_1$ and
$\Q_2$, with $\Q_i$ differing from $\P$ only on item $i$, and
furthermore $\Delta=\P(1)-\Q_1(1)=\P(2)-\Q_2(2)$ (same change in
magnitude, but on different dimensions), we then have
$\bl(\Q_1|\P)=\bl(\Q_2|\P)$ (\brier is indifferent, as long as
magnitude of change is the same), while for \loglossn, discussed next,
this depends (the losses can be very different), as \logloss is
sensitive to the magnitudes of $\P(1)$ and $\P(2)$ as well.
}

\co{We give a proof by reducing to the \di case: when $\Q$ is
a strict \sdn, we can make a \di variant of $\Q$, call it $\Q'$, via
adding an extra element $j$ to $\I$ with remainder \pr
$\Q'(j)=1-\sm{\Q}$ (note: $\P(j)=0$). Observing the following 3
relations establishes the result: 1) $\bdist(\P, \Q) = \bdist(\P, \Q')
- \Q'(j)^2$, 2) $\bl(\Q'|\P) = \bdist(\Q', \P)$ (as both are \disn),
and 3) $\bl(\Q'|\P) = \bl(\Q|\P) + \Q'(j)^2$ (as with every draw from
$\P$ we incur the additional\footnote{By definition, we have
$\bl(\Q'|\P)=\sum_{i\in \I, i\ne
  j}\P(i)\left(\Q'(j)^2+(1-\P(i))^2+\sum_{u\in \I, u\ne i, u\ne j}
\P(u)^2\right) = \sum_{i\in \I, i\ne j}\P(i)\Q'(j)^2 + \sum_{i\in \I,
  i\ne j}\P(i)\left((1-\P(i))^2+\sum_{u\in \I, u\ne i, u\ne j} \P(u)^2\right) =
\Q'(j)^2 + \bl(\Q|\P)$ (the last equality in part follows from
$\sum_{i\in \I, i\ne j}\P(i)=1$).} cost of $\Q'(j)^2$ compared to
the cost from $\Q$).
}









\co{
\begin{corollary}
  (sensitivity of \brier to the magnitude of \pr shifts only) With \di $\P$
  and \sd $\Q$ (on the same items $\I$), let
  $\Delta_i=\P(i)-\Q(i)$. Then $\bl(\Q|\P)=\sum_{i\in \I}
  \Delta_i^2$.
\end{corollary}
}
 \co{
\begin{prop}
  (sensitivity of \brier to the magnitude of a shift) Let $\P$ be a
  \din, and let \sd $\Q$ be identical to $\P$ except on one
  dimension (item) $j$ ($\forall i\in \I, i \ne j, \P(i)=\Q(i)$),
  and let $\Delta=\P(j)-\Q(j)$. Then $\bdist(\P,
  \Q)=\Delta^2$.
\end{prop}
}





Compared to the \loglossn, the Brier score can be easier to work with,
as there is no possibility of ``explosion'', \ie unbounded values (see
\sect \ref{sec:logloss2}). However, in our application we are
interested in probabilities that can widely vary, for instance,
spanning two orders of magnitude ($\P(i) \ge 0.1$ for some $i$, and
$\P(j)$ near $0.01$ for other items),
and as mentioned above, different items are associated with different
rewards and we are interested in optimizing expected rewards. Thus
confusing a one-in-twenty event with a substantially lower probability
event, such as a zero probability event, can incur considerable
underperformance (depending on the rewards associated with the items).
The utility of a loss depends on the application. For instance, for
clustering \disn, symmetry can be important, and thus quadratic loss
may be preferable over \loglossn. In our prior work, in an
application where events were binary, and \prs were concentrated near
0.5 (outcomes of two-team professional sports games), quadratic loss
was adequate \cite{experts_uai}.



\co{
This property has been established when both $\P$ and $\Q$ are \dis (\ie when
$\sm{\Q}=1$) \cite{selten98}, and similarly can be established by
writing the expectation expressions, rearranging terms, and
simplifying (proof provided in the Appendix \ref{app:quad}).

}



\co{
The above property was described in terms of a subdistribution
$\Q$. A similar property holds when $\Q$ is a proper
distribution wherein one dimension, say $j$, loses a probability, and
another $k$, gains it: $\Q(j) = \P(j)-\Delta$, and $\Q(k) =
\P(k)+\Delta$. The distance then becomes $2\Delta^2$. Finally, with
$\Delta_i$ denoting the difference in each dimension (item), the
\brier score is $\sum_i \Delta_i^2$.
}

\co{
\begin{prop}
  (sensitivity of \brier to the magnitude of the shifts) Let $\P$ be
  the true $\di$ and $\Delta \ge 0$ be the total deductions over $i
  \in \I$ in going from $\di$ to $\Q$. Then $\bdist(\P,
  \Q)=\Delta^2$.
\end{prop}
}
  

\co{
\label{app:quad}

\begin{lemma*} (sensitivity of \brier to the magnitude of \pr shifts only) 
Given \di $\P$ and \sd $\Q$, defined over the same finite set $\I$, \\ \\
\hspace*{.5in} $\bl(\Q|\P) = \bdist(\Q, \P)$ 
\end{lemma*}
}

\subsection{The Sensitivity of \logloss}

The following corollary to Lemma \ref{lem:ll_kl} (on the connection of
\logloss to relative entropy) highlights and quantifies the
sensitivity (or bias) of \logloss towards larger \prs (where, without
loss of generality, we are using items, or dimensions, 1 and 2):

\vspace*{.2in}
\begin{corollary*} (\llsp is much more sensitive to relative changes in larger \prsn)
  Let $\P$ be a \di over two or more items ($\I=\{1,2,\cdots\}$,
  $|\I|\ge 2$), with $\P(1) > \P(2) > 0$, and let the multiple
  $m=\frac{\P(1)}{\P(2)}$ (thus $m > 1$).  Consider two \sdn s $\Q_1$
  and $\Q_2$, with $\Q_1$ differing with $\P$ on item 1 only ($\forall
  i\in I$ and $i \ne 1$, $\Q(i)=\P(i)$), and assume moreover $\P(1) >
  \Q_1(1) > 0$, and let $m_1=\frac{\P(1)}{\Q_1(1)}$ (thus, $m_1 >
  1$). Similarly, assume $\Q_2$ differs with $\P$ on item 2 only, and
  that $\P(2) > \Q_2(2) > 0$, and let $m_2=\frac{\P(2)}{\Q_2(2)}$.  We
  have $\lln(\Q_2|\P) < \lln(\Q_1|\P)$ for any $m_2 < m_1^{m}$.
\end{corollary*}

\begin{proof}  
We write the difference $\Delta$ in the losses and simplify:
$\Delta=\lln(\Q_1|\P) - \lln(\Q_2|\P) = \P(1)\log\frac{\P(1)}{Q_1(1)}
- \P(2)\log\frac{\P(2)}{Q_2(2)}$ (all other terms are 0, \eg
$\Q_1(2)=\P(2)$, and the entropy terms cancel). We thus have $\Delta =
m\P(2)\log(m_1)-\P(2)\log(m_2)$, or we have $\Delta > 0$ as long as
$\log(m_1^m) -\log(m_2) > 0$, or whenever $m_1^m > m_2$. \end{proof}



\subsection{On the (Near) Propriety of \llns() }
\label{sec:np} 

In the iid generation setting based on \sd $\P$, when using a perfect
\nsm and using \fcap() parameterized by $\pns = \pmin > 0$, we
explore and establish how KL() comparisons change
under various transformations.  Under a few assumptions such as small
$\pns$, we describe when $\P$ remains optimal or near optimal, and for
any $\Q$ scoring near or better than $\P$ (using \llns()), when we can
expect that such a \sd must remain close to $\P$.  See \fig
\ref{fig:distortion}.  We note that there are two related but distinct
questions, when using \llns(): (1) how far the loss of $\P$ is from
the minimum loss $\llns(\Q^*|\P)$ (or {\em {\bf distortion} in the
  loss}), and (2) how far can any minimizer $\Q^*$ be from $\P$ (item
distortion).
When comparing different candidate \sdsn, we care about distortion in
items (\ie the second sense\footnote{As comparisons converge onto or
prefer a minimizer $\Q^*$, and we want to see how different a $\Q^*$
can be from the true $\P$.}).  There can be multiple minimizers and we
can have no distortion in loss but some distortion in items. We will
focus on distortion on items. There are 4 causes or sources of
distortion: filtering, capping, bounding the loss, and use of an
imperfect \nsmn. At the end of this section, we briefly discuss the
last cause, that of an imperfect practical \nsm (\sec
\ref{sec:practical}).


\begin{figure}[t]
\hspace{-0.3cm}  \subfloat{{\includegraphics[height=4cm,width=14cm]{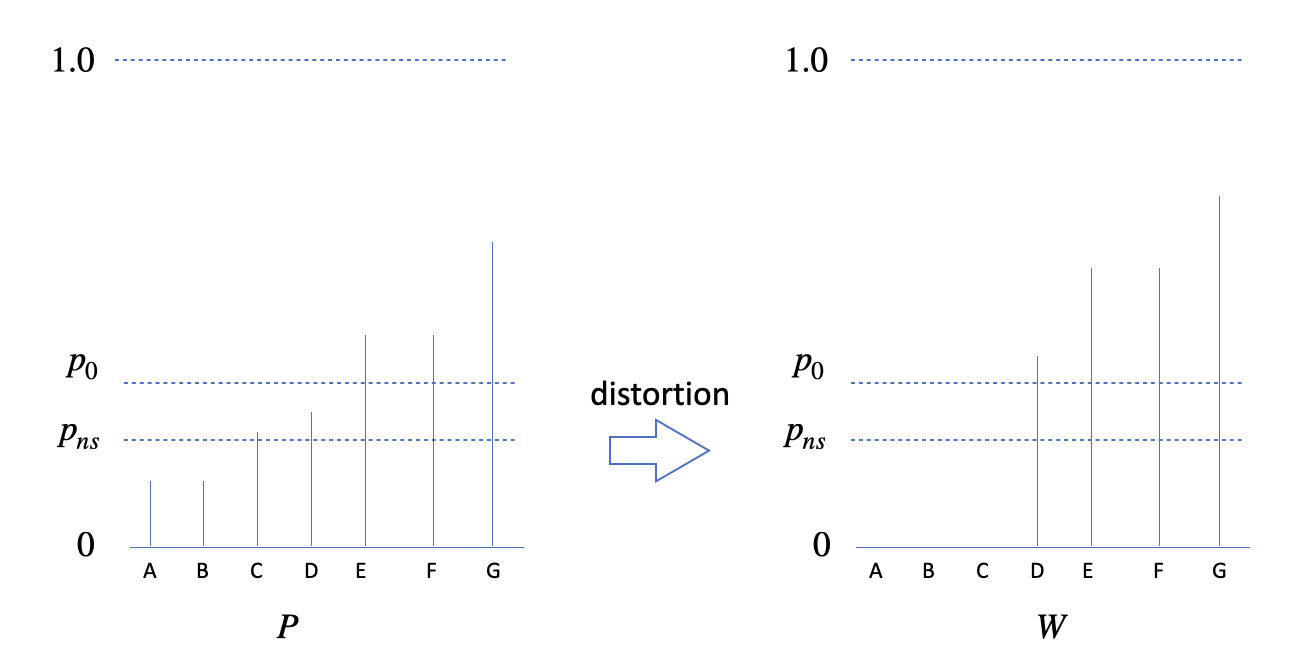} }}\\
  \caption{When using $\fcap()$ and \llnsr(), {\em distortion} can
    occur, \ie a \sd $\Q$ may obtain a lower loss than the true
    generating \sd $\P$ (due to filtering, scaling, and bounding the
    loss).  Sections \ref{sec:np} through \ref{sec:practical} shed
    light on the conditions for distortion and its extent. In
    particular, Lemma \ref{lem:all_or_nothing} characterizes the
    properties of a minimizing \sd $\Q^*$: smaller items in $\P$
    (those with lower \prn, such as $A$, $B$ and $C$ above) may be
    zeroed, and have their \pr mass transferred (spread,
    proportionately) onto higher items in an optimal $\Q^*$. }
\label{fig:distortion}
\end{figure}


Below we explain that \llns() is equivalent to a {\em bounded}
quasi-divergence denoted $\klns$(), similar to plain \logloss being
equivalent to plain \kl \ divergence: the quasi-divergence $\klns$() is based
on transformations of \sds (scaling and filtering, augmenting). In
particular, first we define {\em augmentation} of an \sd to a
corresponding \di that has one extra item with all remaining \pr
mass. This allows us to draw from $\P$ (as if it were \din, in
particular for the purpose of evaluating \llns(): We show computing
\llns($\Q|\P$) is equivalent to computing a bounded \kl() on the
augmented versions of $\P$ and \fcap($\Q$) (Lemma \ref{lem:klns}).

\co{
In the iid generation setting based on \sd $\P$, when using a perfect
\nsm and using \fcap() parameterized by $\pns$ and $\pmin$, $\P$ is
optimal (scores lowest) for \llns() in several cases detailed below,
and otherwise, an optimal $\Q\ne \P$ is close to $\P$ in a certain
sense, when $\pns$ and $\pmin$ are small. At the end of this
subsection, we  briefly discuss the last case of an imperfect practical \nsmn.

Below we show that \llns() is equivalent to a \kl() divergence, as in
plain \logloss but the \kl() is based on transformations of \sds
(scaling and filtering, augmenting). In particular, first we define
{\em augmentation} of an \sd to a corresponding \di that has one extra
item with all remaining \pr mass. This allows us to draw from $\P$ (as
if it were \din, in particular for the purpose of evaluating \llns():
We show computing \llns($\Q|\P$) is equivalent to computing \kl() on
the augmented versions, of $\P$ and \fcap($\Q$).
}


\vspace*{.1in}
\begin{defn}\label{augment} Definitions of augmenting an \sd to a \din, and bounded \klsp
  corresponding to \llns():
  \begin{itemize} 
  \item Given a non-empty \sd $\P$ defined on $\I=\{1, \cdots, k\},
    k\ge 1$, its augmentation (operation), denoted $\dix{\P}$, is a
    corresponding \di $\P'$, where $\P'$ is defined on $\I'$,
    $\I'=\I\cup\{0\}$, $\P'(0)=\u{\P}$, and $\P'(i)=\P(i)$,
    $\forall i\in I$.
  \item {\em({\bf Bounded \kl})} For non-empty \sd $\P$ and \sd $\Q$:
    \begin{align*}
\hspace*{-.65in}    \klb(\P||\Q) := \sum_{i\in \sup(\P)}
    \P(i)\log(\frac{\P(i)}{\max(\Q(i), \pns)})\ \ \ \ \mbox{(where }\pns \in
      [0,1), \mbox{bounded when } \pns > 0).
    \end{align*}
  \item {\em({\bf \kl() for \llns})} $\klns(\P||\Q) := \klb(\dix{\P} ||
    \dix{\fcap(\Q)})$ \ \ (where the $\pns$ of 
    \fcap() is used in $\klb()$).
  \end{itemize}
\end{defn} 

For a \di $\P$, $\dix{P} \equiv \P$ in the sense that all positive
\prs match,\footnote{We are abusing notation in using $\dix{}$ to also
denote the augmentation operator applied to a \sd $\P$ to make a
distribution, in addition to the related notion of referring to a
distribution.} and for two \sds $\P$ and $\Q$, $\P\ne\Q
\Leftrightarrow \dix{\P}\ne\dix{\Q}$. Unlike $\kl()$, $\klb()$ is
bounded when $\pns > 0$ (no greater than $-\log(\pns)$), and both
$\klb(\P||\Q)$ and $\klns(\P||\Q)$ can be negative even when both $\P$
and $\Q$ are \dis (the minimum is not 0 any more), and so they are not
divergences, but we argue still useful for scoring and comparisons. In
particular, $\klns(\P||\P)$ is not necessarily 0, as $\fcap()$ used on
the 2nd parameter of $\klns$, here $\P$, can alter $\P$, \ie can
filter and scale it, while $\P$ provided as first parameter is only
augmented. Furthermore, $\klb(\P||\P)$ can grow unbounded in the
negative direction: imagine $\P$ being the uniform distribution with
$k$ items, each item with \pr $\frac{1}{k}$ contributing
$\frac{1}{k}\log(\frac{1}{k\pns})$, then $\klb(\P||\P)=-\log(k\pns)$
(with $k\pns > 1$), and with $k$ growing, $\klb(\P||\P)\ra
-\infty$. The same condition holds for $\klns()$ ($\fcap()$ does not
change this), but as the lemma below shows, the addition of the
positive entropy counteracts this growth and \llns() is always in $[0,
  -\log(\pns)]$.

We also note that there is another distinct way of obtaining a \di
from a strict \sd $\P$, that of scaling it, to $\alpha\P$, where
$\alpha=\frac{1}{\sm{\P}}$. Both operations are useful, and used
below.






  
\vspace*{.1in}
\begin{lemma}\label{lem:klns}
  For any two \sds $\P$ and $\Q$ defined on $\I=\{1, \cdots, k \}$,
  where $\P$ is non-empty, and using a perfect \nsm, $isNS_{\P}$(), wrt
  to $\P$:
  \begin{align}
    \nonumber
    \llns(\Q|\P) = \ent(\dix{\P}) +   \klns(\P||\Q).
  \end{align}
\end{lemma}

\begin{proof}
  We first explain the different parts of $\llns(\Q|\P)$, \ie
  $\expdb{o\sim \P}{\llnsr(o, \Q, isNS_{\P}(o))}$, to show that it is
  equivalent to comparing two augmented \dis under \kl().

  Let $\Q'=\fcap(\Q)$ and $\P'=\dix{\P}$. An item $o\sim\P$ (drawn
  from $\P$), is salient with probability $\sm{\P}$ ($o \in
  \sup(\P)$), and otherwise is marked \ns by $isNS_{\P}()$ (perfect
  \nsmn), and it is important to note that $o$ does not occur in $\I$
  when marked \ns ($\Q'(o) = \Q(o) = \P(o) = 0$ when $isNS_{\P}(o)$ is
  true), with our uniqueness assumption when generating \ns
  items. When we use \llnsr(), the score in the \ns case is
  $-\log(\u{\Q'})$. This case occurs $\u{\P}$ of the time, and a
  salient item $i$, $i \in \I$, occurs $\P(i)$ of the time, thus
  $\llns(\Q|\P) = -\u{\P}\log(\max(\u{\Q'}, \pns)) -
  \sum_{i\in I}\P(i)\log(\max(\Q'(i), \pns))$.  We have that
  $\u{\Q'} \ge \pns$ (from the scaling down in \fcap()), or
  $\max(\u{\Q'}, \pns)=\u{\Q'}$.  Adding and subtracting
  $\ent(\P')$ establishes the equivalence:
  \begin{align*}
   \hspace*{-.6in} \llns(\Q|\P) & = -\u{\P}\log(\u{\Q'}) - \sum_{i\in
    I}\P(i)\log(\max(\Q'(i), \pns)) - \ent(\P') + \ent(\P') \\
    & = -\u{\P}\log(\u{\Q'}) - \sum_{i\in
    I}\P(i)\log(\max(\Q'(i), \pns)) + \sum_{i\in
    I'}\P'(i)\log(\P'(i)) + \ent(\P') \\
    & = \u{\P}\log\frac{\u{\P}}{\u{\Q'}} + \sum_{i\in
    I}\P(i)\log\frac{\P(i)}{\max(\Q'(i), \pns)}  + \ent(\P') \\
    & = \kl(\P'||\dix{\Q'})  + \ent(\P') = \klns(\P||\Q)  + \ent(\dix{\P}).
  \end{align*}
  The 2nd to 3rd line follow from $\ent(\P')=-\sum_{i\in
    I'}\P'(i)\log(\P'(i))=-\u{\P}\log(\u{\P}) - \sum_{i\in
    I}\P(i)\log(P(i))$). \end{proof}


\co{
  This data generation and evaluation of the expectation,
  $\expdb{o\sim \P}{\llnsr(o, \Q, isNS(o))}$, is therefore equivalent
  to computing \logloss when drawing from the \di augmentation of
  $\P$, $\P'=\dix{\P}$.  For any \sd $\Q'=$\fcap($\Q$), where we have
  $\Q'(0)=0$, similarly we get the corresponding $\Q''=\dix{\Q'}$
  defined on $\I'$, and we have the equivalence $\expdb{o\sim
    \P}{\llnsr(o, \Q, isNS(o))} \equiv \kl(\P'||\Q'')$.  The lowest
  loss is achieved with \di $\P'$ itself (when $\Q''=\P'$),
  and the lowest loss is the entropy $-\sum_{i \in
    \I'}\P'(i)\log(\P'(i))$.
}

\co{
  
The following (near) propriety properties follow as corollaries of
Lemma \ref{lem:klns}:

\vspace*{.1in}
\begin{lemma*} Given a non-empty \sd $\P$ generating sequences,  using a perfect
  \nsmn, and using \fcap() to constrain candidate \sdsn, with $\pns >
  0$ but assuming no \prs from $\P$ are dropped when applying \fcap()
  to $\P$ (\eg $\pmin=0$, or more generally $\pmin < \min(\P)$ even
  after possible scaling within \fcap()):
  \begin{enumerate}
  \item If $\sm{\P} \le 1-\pns$, then $\P$ is a minimizer of \llns(),
    and it is the unique minimizer when $\pmin=0$ and $\sm{\P} < 1-\pns$.
  \item If $\sm{\P} > 1-\pns$, then $\klns(\P||\P) \le -\log(1-\pns)$.
  \end{enumerate}
\end{lemma*}
}
\co{
\begin{proof} We let $\P'=\dix{\P}$, defined over $\I'=\{0\}\cup \I$.
  Proof of claim 1: In this case, we note that $\fcap(\P)=\P$ (because
  $\sm{\P} \le 1-\pns$ and the assumption that no \prs are dropped from
  $\P$), (thus, $\dix{\fcap(\P)} = \P'$), so setting $\Q$ to $\P$
  yields $\kl(\P'||\dix{\fcap(\P)}) = \kl(\P'||\P') = 0$ (the
  minimum). It is the unique  minimizer because when $\pmin=0$ and
  $\sm{\P} < 1-\pns$, we can show that if $\Q \ne \P$, $\fcap(\Q) \ne
  \P$, and thus we must have $\dix{\fcap(\Q)} \ne \P'$: first $\Q$
  must have an identical support set as $\P$, otherwise $\fcap(\Q)\ne
  \P$ ($\pmin=0$ and scaling alone does not change the support
  set). Second (with the same support), if $\sm{\Q} < 1-\pns$, it is
  not altered by \fcap, therefore $\Q'=\Q \ne \P$. Finally, if
  $\sm{\Q} \ge 1-\pns$, we must have $\sm{\fcap(\Q)}=1-\pns
  \Rightarrow$ $\fcap(\Q)\ne \P$.
  
  Proof of claim 2: Here $\sm{\P} > 1-\pns$, and for evaluating
  $\llns(\P|\P)$, $\P$ is scaled in $\fcap()$, \ie \fcap($\P$)$ =
  \alpha \P$, where $\alpha =\frac{1-\pns}{\sm{\P}}$ (note: $1-\pns
  \le \alpha < 1$), and
  $\llns(\P|\P)=\kl(\P'||\dix{\fcap(\P)})=\sum_{i \in \I} \P(i)\log
  \frac{\P(i)}{\alpha \P(i)} + (1-\sm{\P})\log\frac{1-\sm{\P}}{\pns}
  \le \sum_{i \in \I} \P(i)\log \frac{\P(i)}{\alpha \P(i)} \le -\log
  \alpha \le -\log (1-\pns) $. \end{proof}

We note that, for the first part above, when $\sm{\P}=1-\pns$, then
one can verify that for any $\Q=\alpha \P$, where $\alpha \in [1,
  \frac{1}{1-\pns}]$, $\Q$ also minimizes $\llns(\Q|\P)$, as
$\fcap(\Q)=\P$. However, the \pr of no item in such $\Q$ deviates more
than $\alpha \le \frac{1}{1-\pns}$ (\ie $\forall i,
\frac{\Q(i)}{\P(i)} \le \frac{1}{1-\pns})$, and with $\pns$ small (\eg
$\pns=0.01$), the deviations are therefore small. Similarly, for the
second claim, we observe that $\P$ scores near optimal, \ie has near 0
\kl() divergence, for small $\pns$ (no more than
$-\log(0.99)\approx0.01$ when $\pns=0.01$).

}

\co{
As an example where $\P$ may not minimize the loss in the setting of
claim 2, let $\pns=0.2$, $\P=\{$1:0.8, 2:0.2$\}$ (thus $\P$ is a \di
and $\dix{\P}=\P$), and $\Q=\{$1:1.0$\}$, then (the \prs of both $\P$
and $\Q$ are scaled by $1-\pns=0.8$, and
$\llns(\Q|\P)=-0.8\log((1.0)(0.8))-0.2\log(0.2)\approx 0.5$, while
$\llns(\P|\P)=-0.8\log((0.8)(0.8))-0.2\log((0.2)(0.8))\approx 0.72$,
in fact, term by term, $\llns(\Q|\P) < \llns(\P|\P)$. The next lemma
sheds further light on this scenario.
}


The following example describes a case when $\P$ may not minimize
$\llsn()$ even when $\P =$ \fcap($\P$) (\ie even when $\P$ is not
altered by \fcap()). Let the underlying $\P :=\{A$:0.78, $B$:0.02$\}$,
thus $0.2$ of the time, a noise item is generated. Then with
$\pmin=\oovprob=0.01$ in \fcap(), $\P$ scores at
$\llsn(\P|\P)=-0.78\log(0.78)-0.02\log(0.02)-0.2\log(0.20)\approx
0.594$, while $\Q_1=\{A$:0.78$\}$ and $\Q_2=\{A$:0.8$\}$ have a lower
loss:
$\llsn(\Q_1|\P)=-0.78\log(0.78)-0.02\log(0.01)-0.2\log(0.22)\approx
0.589$, and similarly
$\llsn(\Q_1|\P)=-0.78\log(0.8)-0.02\log(0.01)-0.2\log(0.20)\approx
0.588$. An intuitive reason is that when $\P(B)$ is sufficiently close
to $\pns$, then it is advantageous to move that mass to $A$ or to the
allocation to noise (or spread a portion on both), and the penalty
from this reduction (and transfer) is not as much as the reduced loss
(the gain) of a higher \pr for A or for the noise allotment. This can
be seen more clearly when $\P(B) \le \pns$ (\ie right at or below the
boundary). Section \ref{sec:all_or_none} develops this near boundary
property. We first need the (scaling) properties developed next.



\co{
we have
$\llns(\Q|\P) < \llns(\P|\P)$:
$\llns(\Q|\P)=-0.8\log((1.0)(0.8))-0.2\log(0.2)\approx 0.5$ (the 2nd
term is when item 2 is drawn, $2\not\in\Q$, for which $-\log(\pns)$ is
used), while
$\llns(\P|\P)=-0.8\log((0.8)(0.8))-0.2\log((0.2)(0.8))\approx 0.72$
(the \prs of both $\P$ and $\Q$ are scaled by $1-\pns=0.8$).
}



\subsubsection{Properties of \kl() on \sds (Under Scaling and Other Transformations)}
\label{sec:klprops}
We now show how plain \kl() comparisons (defined for \sds 
in Defn. \ref{def:kl}) are affected by scaling one or both \sdsn. This is
useful in seeing how \fcap(), with its scaling and filtering, can
alter the minimizer of $\klns()$, and to see when distortion can
happen when using $\klb()$.  Recall that, as in \sec
\ref{sec:scoring}, we assume that the set $\I$ over which both $\P$
and $\Q$ are defined is finite, \eg the union of the supports of both.
We have defined \kl() for sds $\P$ and $\Q$ (definition \ref{def:kl}),
and for the following lemma, we can extend the definition of \kl() to
non-negative valued $\P$ and $\Q$ (no constraint on the sum, unlike
plain \sdsn), or assume $\alpha$ is such that $\alpha\Q$ and
$\alpha\P$ remain a \sdn.

\begin{lemma}\label{lem:klscale}
  (plain \kl() under scalings) For any non-empty \sd $\P$ and \sds $\Q$, and any $\alpha > 0$:
  \setlength{\leftmargini}{1in} 
  \begin{enumerate}
  \item   $\kl(\P||\alpha \Q) = \kl(\P||\Q)+\log(\alpha^{-1})\sm{\P}$. 
    \item $\kl(\alpha\P||\Q) =  \alpha\kl(\P||\Q)+\alpha\log(\alpha)\sm{\P}$. 
    \item $\kl(\alpha\P||\alpha \Q) = \alpha\kl(\P||\Q).$
  \end{enumerate}
\end{lemma}
\begin{proof}
  The $\alpha$ multiplier comes out, in all cases, yielding a fixed
  offset for first 2 cases, and a positive multiplier for the 2nd and
  3rd cases:
\hspace*{-1.0cm}\vbox{\begin{align*}
   \kl(\P||\alpha \Q)&=\sum_{i\in \I}\P(i)\log(\frac{\P(i)}{\alpha\Q(i)})=\sum_{i\in
    \I}\P(i)\log(\frac{\P(i)}{\Q(i)}\frac{1}{\alpha})=\sum_{i\in
    \I}\P(i)(\log(\frac{\P(i)}{\Q(i)}) + \log(1/\alpha))   \\
  &=\sum_{i\in\I}\P(i)\log\frac{\P(i)}{\Q(i)}+
  \log(\alpha^{-1})\sum_{i\in\i}\P(i) = \kl(\P||\Q)+\log(\alpha^{-1})\sm{\P}. \\
  \kl(\alpha\P||\Q)&=\sum_{i\in \I}\alpha\P(i)\log(\frac{\alpha\P(i)}{\Q(i)})=
  \alpha\left(\sum_{i\in\I}\P(i)\log\frac{\P(i)}{\Q(i)}  + \log(\alpha)\sum_{i\in\I}\P(i) \right).\\
  \kl(\alpha\P||\alpha \Q)&=\sum_{i\in \I}\alpha\P(i)\log(\frac{\alpha\P(i)}{\alpha\Q(i)}) = \alpha\sum_{i\in \I}\P(i)\log(\frac{\P(i)}{\Q(i)})= \alpha\kl(\P||\Q). 
  \end{align*}}\end{proof}

  
As a consequence of the above lemma, we conclude that scaling does not
change \kl() comparisons, \eg with $\alpha > 0$, if $\kl(\P||\Q_1)$ <
$\kl(\P||\Q_2)$ then $\kl(\P||\alpha\Q_1) < \kl(\P||\alpha \Q_2)$, and
we can use the above lemma and the properties of $\kl()$ to conclude
the following {\em proportionate} properties. These help inform us
on how a best scoring \sdn, under scaling and filtering, is related to
the \sd generating the sequence:
\begin{corollary}\label{lem:scale_corr} Scaling and spreading
  (adding or deducting mass)
  should be proportionate to \sd $\P$ to minimize $\kl(\P||.)$:
  \setlength{\leftmargini}{0.5in} 
  \begin{enumerate}
  \item ($\P$ is the unique minimizer over appropriate set) Given non-empty \sd $\P$,
    $\kl(\P||\P)=0$, and for any $\Q \ne \P$ with $\sm{\Q}\le\sm{\P}$,
    $\kl(\P||\Q)>0$.
  \item (same minimizer for $\P$ and its multiple) Let non-empty \sds
    $\P_1$ and $\P_2$ be such that $\P_1 = \alpha \P_2$ for a scalar
    $\alpha > 0$, and consider any non-empty set $S$ of \sds. For any two
    \sds $\Q_1, \Q_2 \in S$, $\kl(\P_1||\Q_1) < \kl(\P_1||\Q_2)
    \Leftrightarrow \kl(\P_2||\Q_1) < \kl(\P_2||\Q_2)$
    (thus, a \sd $\Q \in S$ is a minimizer for both or for neither).
    \item Given a non-empty \sd $\P$, among \sd $\Q$ such that
      $\sm{\Q}=s$, the one that minimizes $\kl(\P||\Q)$ is
      proportionate to (or a multiple of) $\P$, \ie $\forall i\in \I,
      \Q(i)=\frac{s}{\sm{\P}}\P(i)$ (when $s=0$ this becomes vacuous).
      \co{
    \item Given non-empty \sds $\P$ and $\Q$, where $\Q=\alpha
      \P$ for some $\alpha>0$, and a total mass $s > 0$ to spread over
      $\Q$ to yield $\Q'$, meaning $\Q'(i)=\Q(i)+s(i)$ (and
      $s=\sum_{i\in\I}s(i)$), $\kl(\P||\Q')$ is minimized when the
      spreading is proportionate: $\forall i\in \I$,
      $\frac{s(i)}{s}=\frac{\P(i)}{\sm{\P}}$.
      }
  \end{enumerate}
\end{corollary}


\begin{proof} ({\bf part 1}) When $\P$ is a \din, among $\Q\ne \P$ that are $\di$ too,
  the property $\kl(\P||\Q) > 0$ holds \cite{cover91}.  If $\Q$ is a
  strict \sdn, on at least one item $i$, $\Q(i) < \P(i)$, and we can
  repeat increasing all such $\Q(i)$ in some order until $\Q(i)=\P(i)$
  or $\Q$ becomes a \di (finitely many such $i$), lowering the
  distance (the log ratio $\frac{\P(i)}{\Q(i)}$). We conclude for any
  $\Q\ne \P$, $\kl(\P||\Q) > \kl(\P||\P)=0$. When $\P$ is a nonempty
  strict \sdn: We have
  $\kl(\P||\P)=\sum_{}\P(i)\log\frac{\P(i)}{\P(i)}=0$. We can scale
  $\P$ by $\alpha=\frac{1}{\sm{\P}}$ to get a \di, and from the
  property of $\kl(\P||.)$ for \di $\P$, and using Lemma \ref{lem:klscale}, we
  conclude that for any other \sd $\Q\ne \P$, with $\sm{\Q} \le \sm{\P}$ (and
  thus $\alpha\Q$ remains a \sdn), must yield a higher (positive)
  $\kl(\P||\Q)$: $\kl(\P||\Q)=\frac{\kl(\alpha\P||\alpha\Q)}{\alpha} >
  0$ (using Lemma \ref{lem:klscale}, and $\kl(\alpha\P||\alpha\Q) >
  0$, from first part of this claim).

  ({\bf part 2}) Let $\Delta_1 :=\kl(\P_1||\Q_1)-\kl(\P_1||\Q_2)$ and
  $\Delta_2 :=\kl(\P_2||\Q_1)-\kl(\P_2||\Q_2)$.  $\Delta_1=
  \kl(\alpha\P_2||\Q_1)-\kl(\alpha\P_2||\Q_2) =
  \alpha(\kl(\P_2||\Q_1)-\kl(\P_2||\Q_2))$ (from part 2, Lemma
  \ref{lem:klscale}, $\alpha\ln(\alpha)\ln(\sm{\P_2})$
  canceling). Therefore, $\Delta_1 < 0 \Leftrightarrow \Delta_2 < 0$.

  ({\bf part 3}) This is a consequence of parts 1 and 2 where the set
  $S$ includes $\P_1$, where $P_1$ is proportionate to $P_2$, and
  $P_1$ minimizes $\kl()$ to itself among $\Q \in S$ (part 1). \end{proof}
  
  \co{
  Let \sds $\P'$ and $\Q \ne \P'$ be such that
  $\sm{\Q}=\sm{\P'}=s$, and further $\forall i\in\I,
  \P'(i)=s\frac{\P(i)}{\sm{\P}}$, by setting
  $\alpha=\frac{\sm{\P}}{s}$, we have $\alpha\P' =\P$, and
  $\kl(\P||\Q) - \kl(\P||\P') = \kl(\P||\alpha \Q) - \kl(\P||\alpha
  \P') = \kl(\P||\alpha \Q)$ (using \ref{eq:scale1}), and from first
  part, $\kl(\P||\alpha \Q) > 0$.

  $\P$ is the
  unique minimizer $\kl(\P||\P)=0$, and using Lemma \ref{lem:klscale},
  no other $\Q'$ with $\sm{\Q'}=s$ could have lower $\kl()$.
  }
  
\co{  ({\bf part 4}) Spreading $s$ over $\Q$ changes $\sm{\Q}$ to $\sm{\Q}+s$, but
  from part 3, among such \sdsn, the one proportionate to $\P$ scores
  lowest, therefore, as $\Q$ is already proportionate, proportionate
  spreading over $\Q(i)$ keeps the proportionality: Let
  $r=\frac{\P(i)}{\sm{\P}}=\frac{\Q(i)}{\sm{\Q}}=\frac{s(i)}{s}$, then
  $\frac{\Q(i)+s(i)}{\sm{\Q}+s}=\frac{r\sm{\Q}+rs}{\sm{\Q}+s}=r$.
}

We note that in part 1 above, if there is no constraint on $\sm{\Q}$,
then due to the functional form of \kl(), multiples of $\P$ (and other
$\Q$) can score better than $\P$ and obtain negative scores. In this
and related senses, \kl() on \sds is not a divergence in a strict
technical sense (even if we generalize the definition of statistical
divergence to non-distributions). If we constrain the set of \sds
\kl() is applied to, for instance to $\sm{\Q}=s$, then $\kl(P||.)$ can
enjoy certain divergence properties (such as having a unique
minimizer). Our main aim is to show such losses and distances remain
adequate for comparing predictors.  In part 2 above, we consider the
set $S$ \sds to be general: it may not yield a minimizer (consider
open sets), or may have many (\eg disjoint closed sets).

\co{

  one bin, spread over others: proportionate spreading is enough;
  two bins is enough! moving to one additional bin is enough!
  
}

\subsubsection{All-Or-Nothing Removals and Other Properties of $\klb()$ }
\label{sec:all_or_none}


When using $\klb(\P||.)$ with $p=\pns > 0$, in changing $\P$ to get
a $\Q^*$ that minimizes $\klb(\P||.)$, where $\sm{\Q^*} = \sm{\P}$, we
show that we have all-or-nothing deductions, and when not deducting,
we may increase the mass, shifting from small to larger items, \ie those
with higher \prsn:

\begin{defn}
  With respect to given threshold $\pns, \pns \in (0, 1)$, the \sd
  $\P$ is called {\em \bf non-degenerate} if $\sm{\P} > \pns$,
  otherwise, it is {\em \bf degenerate} (wrt $\pns$).
\end{defn}

\begin{lemma}
  \label{existence}
  (Minimizer existence for the degenerate case) Given threshold $\pns,
  \pns \in (0, 1)$, if a nonempty \sd $\P$ is degenerate, then for any
  \sd $\Q$ with $\sm{\Q} \le \sm{\P}$ (including $\P$ and the empty
  \sd), $\klb(\P||\Q)= -\sm{\P} \log(\pns)$ (all such \sds are
  minimizers).
\end{lemma}
\begin{proof}
  For any such $\Q$, $\klb(\P||\Q)=\sum_{i\in
    \sup(\P)}-\P(i)\log(\pns)$ (as $\forall i, \Q(i)\le
  \sm{\Q} \le \pns$), thus $\klb(\P||\Q)=-\sm{\P} \log(\pns)$\end{proof}

If $\P$ is non-degenerate, we show that a minimizer must exist, but we
first show some of the properties it must have (see also \fig
\ref{fig:distortion}).



\begin{lemma}
  \label{lem:all_or_nothing}
  For any $\pns \in (0, 1)$ and any non-degenerate \sd $\P$, where we
  want to minimize $\klb(\P||\Q)$ over the set $S$ of \sds $\Q$ such
  that $\sm{\Q_1} \le \sm{\P}$ (and wlog we need only consider
  $\sup(\Q) \subseteq \sup(\P)$):
  \begin{enumerate}
  \item (all positive \pr items are above $\pns$) For any $\Q_1 \in S$, if
    there is an item $i$ where $\Q_1(i) \in (0, \pns]$, then there is
    a \sd $\Q_2\in S$, such that $\forall i \in
    \sup(\Q_2), \Q_2(i) > \pns$, and $\klb(\P||\Q_2) <
    \klb(\P||\Q_1)$. Therefore, in minimizing $\klb(\P||.)$, we need
    only consider the set $S_2 = \{\Q | \Q \in S, \sm{\Q}=\sm{\P}, \mbox{ and } \forall i \in
    \sup(\Q), \Q(i) > \pns\}$ ($S_2$ is not empty as $\P$ is
    non-degenerate).
  \item (existence and proportionate increase) The set $S^*\subseteq
    S_2$ of minimizers of $\klb(\P||.)$ is not empty, and for any
    $\Q^* \in S^*$, and for some fixed multiple $r\ge 1$, $\forall i
    \in \sup(\Q^*), \Q^*(i) = r \P(i)$.
  \item (order is respected) For any minimizer $\Q^*$ and any two
    items $i$ and $j$, when $\P(i) < \P(j)$, if $\Q^*(i) > 0$, then
    $\Q^*(j) > 0$ (and, from part 2,  $\Q^*(j) > \Q^*(i)$).
  \end{enumerate}
\end{lemma}
\begin{proof}
  ({\bf part 1}) Say $\Q_1$ has one or more items with low \pr $\le
  \pns$, call the set $\I_2$, $\I_2 :=\{i|\Q_1(i) \le \pns\}$, with
  total \pr $b$ (thus $b:=\sum_{i\in \I_2}\Q(i)$). We can also assume
  $\sm{\Q_1}=\sm{\P}$ (if less, we can also shift the difference onto
  receiving item $j$ in this argument). Then if $\Q_1$ also has an
  item $j$ with \pr above $\pns$, shift all the mass $b$ to item
  $j$. Otherwise, shift all the mass $b$ to a single item $j$ in
  $\I_2$ (pick any item). In either case, we have $\klb(\P||\Q_2) <
  \klb(\P||\Q_1)$: In the first case, the cost (\ie
  $-\P(i)\log(\frac{\P(i)}{\max(Q(i),\pns)})$) is not changed for
  items in $\I_2$ ($-b\log(\pns)$), while for item $j$ the cost is
  lowered. In the second case, we must have $b > \pns$ ($\P$ is
  non-degenerate), and the cost for the receiving item $i$ improves
  (as we must have $Q_2(i)>\pns$), while for others in $\I_2$ it is
  not changed.

  ({\bf part 2}) The set $S_2$ (from part 1), can be partitioned into
  finitely many subsets (ie disjoint sets whose union is $S_2$), each
  partition member corresponding to a non-empty subset of
  $\sup(\P)$. For instance all those that have positive \pr on item 1
  only (greater than $\pns$ by definition of $S_2$) (support of size
  1) define one partition subset (we get $|\sup(\P)|$ such subsets
  corresponding to singletons). The size of the support set $k$ of a
  \sd in $S_2$, $k \le |\sup(\P)|$, can be large to the extent that
  $\frac{\sm{\P}}{k} \ge \pns$ is satisfied ($k=1$ works, but larger
  $k$ may yield valid \sds in $S_2$ as well.  On each such partition
  set $\mathcal{S}_k$), $\klb(\P||.)$ becomes equivalent to
  $\kl(\P||.)$, in the following sense: If $\mathcal{S}_k$ is a
  partition, then $\forall \Q_1, \Q_2 \in \mathcal{S}_k, \klb(\P||Q_1)
  - \klb(\P||Q_2) = \kl(\P||Q_1) - \kl(\P||Q_2)$. From Corollary
  \ref{lem:scale_corr}, the total mass from elements that are zeroed
  (if any), \ie $\sup(\P)-\sup(\Q_1)$, is spread proportionately on
  $\sup(\Q_1)$ to minimize $\kl(\P||.)$ over a partition subset. Since
  we have a finitely many partitions (the size of a powerset at most),
  and each yields a well-defined minimizer of $\klb(\P||.)$, we obtain
  one or more minimizers for the entire $S_2$ and therefore $S$.
\co{
  For a partition set $\mathcal{S}_k$, and $\Q\in\mathcal{S}_k$, let
  $b:=\sum{i\in\sup(\Q)}\P(i)$, and the \pr amount
  $s:=\sum{i\in\sup(\P)-\sup(\Q)}\P(i)$. Then to get the minimizer for
  this partition, $\P(i)$ is increased to:
  $\P(i)+\frac{\P(i)}{b}s=(1+\frac{1}{b}s)\P(i)$ (the multiplier
  $r=1+\frac{1}{b}s$).
}

  ({\bf part 3}) Wlog consider items 1 and 2, $p_1=\P(1)$ and
$p_2=\P(2)$, where $p_1 < p_2$, and assume in \sd $\Q_1 \in S_2$ (with
support size $|\sup(\Q_1)| \ge 1$), item 1 has an allocation $T> \pns$
(as $\Q_1 \in S_2$), while $\Q_1(2)< \Q_1(1)$. We need only consider
the case $\Q_1(2)=0$: if $\Q_1(2) > 0$, then $\Q_1(2) >\pns$ as
$\Q_1\in S_2$, and proportionate increase from part 2 establishes the
result. Assuming $\Q_1(2)=0$, we can 'swap' items 1 and 2, to get \sd
$\Q_2 \in S_2$, and swapping improves $\klb()$, \ie letting $\Delta :=
\klb(\P||\Q_1) - \klb(\P||\Q_2)$, we must have $\Delta > 0$:
\hspace*{-0.6cm}\vbox{  \begin{align*}
      \Delta & = (p_2\log\frac{p_2}{\pns}+p_1\log\frac{p_1}{T}) -
    (p_1\log\frac{p_1}{\pns}+p_2\log\frac{p_2}{T}) \hspace*{.5in} \mbox{(all other terms cancel)} \\
    &= -p_2\log(\pns)-p_1\log(T)+p_1\log(\pns)+p_2\log(T)  = (p_2-p_1)(\log(T)-\log(\pns)) > 0
  \end{align*}}
  The last step (conclusion) follows from our assumptions that
  $p_2>p_1$ and $T > \pns$.  Note that $\Q_2$ can further be improved
  by a proportionate spread (the allotment $T$ to item 2 increased).\end{proof}

The lemma implies that the minimizer is unique if no two items in
$\sp{\P}$ have equal \prs (due to proportionate spread).  It also
suggests a sorting algorithm to find an optimal allocation: With
$k=1,2\cdots,|\sup(\P)|$, use the highest $k$ items (highest $\P(i)$),
and spread $\sm{\P}$ proportionately onto the $k$ items, yielding
$\Q_k$, and compute $\kl(\P||\Q_k)$. One of $\Q_1, \Q_2, \cdots$ is
the minimizer.  The algorithm can be stopped early if
$\frac{\sm{\P}}{k} \le \pns$ and also when proportionate spread leaves
an item with \pr $\le \pns$. One could also start with smallest item,
spreading it on remainder, and repeating. The threshold $\pt$
described next determines when this algorithm should stop (see also
Lemma \ref{lem:ext}).

Some items can have such a large \pr that they are safe from being
zeroed (in a minimizer $\Q^*$). The next two sections derive the
threshold $\pt >\pns$, specifying when a \pr can be above $\pns$ but
sufficiently close to it for the possibility of being zeroed or
distortion (its \pr shifted to higher items to yield lower
losses). Note that even if a \pr $\le \pns$, it may not be zeroed in
the minimizer because shifts from smaller items may raise its mass to
above $\pt$. The smallest item $i$, with $\P(i) < \pt$, is guaranteed
to be zeroed (in some minimizer $\Q^*$, Lemma \ref{lem:pt} below).

\co{

  \item (no partial deductions) For any item $i$, when $\P(i) >
    \pns$ then 
    either $\Q^*(i)\ge \P(i)$, or $\Q^*(i)=0$.
  
The following states that if a \pr in $\P$ is large enough, it is not
reduced in any optimal \sd $\P^*$.
\begin{lemma}
  Given any non-empty \sd $\P$ and any item $i$, if $\P(i) \ge \frac{2
    \pns}{1-\pns}$, then $\P^*(i) \ge (1-\pns)\P(i)$, and when
  $\u{\P}>\pns$ then $\P^*(i) \ge \P(i)$. Similarly, when
  $\u{\P} \ge \frac{2 \pns}{1-\pns}$, then $\u{\P^*} \ge \u{\P}$.
\end{lemma}
\begin{proof}
  We consider two cases for $\P^*(i)$: $\P^*(i) > \pns$ and $\P^*(i)
  \le \pns$. In either case, we can show that spreading the reduction
  in $\P(i)$ over other items can only increase loss.
\end{proof}
}

\subsubsection{The Case of Two Items (and Threshold $\pt$)}
\label{sec:pt}

Suppose the \di $\P$ has two items, with \prs $p$ and $1-p$, \ie
$\P=\{1$:$p, 2$:$1-p\}$, where $p \le 1-p$, and $p > \pns$, and we
want to see how high $p$ can be and yet distortion remains possible,
\ie item 1 is zeroed and its mass shifted to item 2, and
$\Delta:=\klb(\P||\P) - \klb(\P||\Q) > 0$: how high $p$, $p > \pns$,
can be and yet we can get $\Delta > 0$.  $\Delta = 0 -
(p\log\frac{p}{\pns} + (1-p)\log\frac{1-p}{1})$, thus $\Delta > 0$,
when $p\log(\pns) > p\log(p)+(1-p)\log(1-p)$,
or $\pns > p(1-p)^\frac{1-p}{p}$. Thus the threshold $\pt$ is such
that $\pns=\pt(1-\pt)^\frac{1-\pt}{\pt}$. Now, with
$r:=(1-\pt)^\frac{1-\pt}{\pt}$, $r > 1-\pt$ for $\pt > 0.5$ ($r \ra
1.0$ as $\pt \ra 1$), and $r < 1-\pt < 1$ when $\pt < 0.5$ (and
$r=1-\pt=\pt=0.5$ when $\pt=0.5$) (note that in our set up, $\pt \le
0.5$). Thus, indeed $\pns < \pt$. We can set $\pt$ and see how low
$\pns$ should be (or, otherwise, solve for $\pt$).  Considering the
ratio $\frac{\pns}{\pt}=\frac{\pt(1-\pt)^{(1-\pt)/\pt}}{\pt}=r$, where
we have defined $r=(1-\pt)^\frac{1-\pt}{\pt}$, and as $\pt \ra 0$,
$\frac{\pns}{\pt} \ra \frac{1}{e}$ ($e$ denotes the base of the
natural logarithm), while as $\pt \ra 0.5$, $\frac{\pns}{\pt} \ra
\frac{1}{2}$, or:
\begin{align}\label{eq:pt}
\hspace*{1in}  2\pns\le \pt \le e \pns \approx 3\pns.
\end{align}

Thus, as a rule of thumb, as long as $p \ge 3\pns$, $p$ is not
zeroed. With $\pns=0.01$, we can get a positive distortion for $p <
\pt\approx0.027$ (no distortion if $p \ge 0.0270$). With $\pns=0.001$,
no distortion if $p \ge \pt\approx 0.00272$, and with $\pns=0.1$, the
threshold $\pt\approx 0.24$ (no distortion if $p \ge \pt=0.24$).



When $\P$ is a strict \sdn, the possibility of distortion becomes more
limited: assume again that $\P$ has two items with equal probability
$p, p > \pns$, then $\Delta > 0$ (defined above) implies
$0-(p\log\frac{p}{\pns}+p\log\frac{p}{2p}) > 0$, or $\log(\pns) >
\log(p/2)$, or $2\pns = \pt$, implying distortion possibility only
when $p \le 0.02$ with $\pns=0.01$.

\co{
  
\ie $\llns(\P|\P) > \llns(\Q|\P)$, when we move all of $p$ to item 0,
\ie $\Q=\{0:1.0\}$.

$\llns(\P|\P) = -p\log(p) - (1-p)\log(1-p)$, and with $\Q=\{0:1.0\}$, 
$\llns(\Q|\P) = -p\log(\pns)-(1-p)\log(1)=-p\log(\pns)$.
}

\co{
$\llns(\P|\P) > \llns(\Q|\P)$, or $-p\log(p) - (1-p)\log(1-p) >
  p\log(\pns)$ or when $p\log(p) + (1-p)\log(1-p) < p\log(\pns)$ or }

\subsubsection{More  Items (and Threshold $\pt$)}

The case of more items is similar to two items, and yields the same
distortion threshold $\pt$: given \di $\P$ assume all its items have \pr
above $\pns$, and $\Delta:=\klb(\P||\P) - \klb(\P||\Q^*) = 0 -
\klb(\P||\Q^*)$, and we are wondering about the relation of say
$p_1=\P(1)$ and $\pns$ when $\Delta > 0$.

\hspace*{-.5in}\vbox{\begin{align*}
  \Delta &= - \left(p_1\log(p_1/\pns) + \sum_{i\ge 2 } p_i\log(p_i / (p_i + \frac{p_i}{1-p_1}p_1 ))\right)\\
  & \ \ \mbox{(the above is the proportionate spread of $p_1$ onto others in $\Q^*$)} \\
  &= p_1\log(\pns) - p_1\log(p_1) - \sum_{i\ge 2 } p_i\log(p_i) +
  \sum_{i\ge 2 } p_i\log(p_i + \frac{p_i}{1-p_1}p_1 ) \\
  &= p_1\log(\pns) - \sum_{i \ge 1} p_i\log(p_i) +
  \sum_{i\ge 2 } p_i\log(\frac{p_i}{1-p_1}) \\
  &= p_1\log(\pns) - \sum_{i \ge 1} p_i\log(p_i) +
  \sum_{i\ge 2 } p_i\log(p_i) - \log(1-p_1)\sum_{i\ge 2 } p_i  \\
  &= p_1\log(\pns)  -
  \log(1-p_1)\sum_{i\ge 1 } p_i - p_1\log(\frac{p_1}{1-p_1}) =  p_1\log(\pns) - \log(1-p_1) - p_1\log(\frac{p_1}{1-p_1})
  \end{align*}}

And $\Delta > 0$ implies $p_1\log(\pns) > \log(1-p_1)
+p_1\log(\frac{p_1}{1-p_1})$, or when $p_1$ is low enough such that
$\pns > p_1(1-p_1)^\frac{1-p_1}{p_1}$. This is the same bound or
threshold as the two-item case, and we summarize the properties in the
following lemma.

\begin{lemma}\label{lem:pt}
With a \di $\P$ and $\pns \in (0,1)$, if item $i$ has \pr $\P(i) \ge
\pt$, where $\pt$ is such that $\pns = \pt(1-\pt)^\frac{1-\pt}{\pt}$,
then $\Q^*(i)\ge \P(i)$ (item $i$ is not zeroed) in any minimizer
$\Q^*$ of $\klb(\P||)$. If $i\in\sp{\P}$ has the smallest \pr in
$\sp{\P}$ and $\P(i) < \pt$, then $i$ is zeroed in some minimizer
$\Q^*$ ($\Q^*(i)=0$) and if it is the unique minimum, then $\Q^*(i)=0$
in  any minimizer $\Q^*$.
\end{lemma}

\begin{proof}
  The result follows from the above derivation and the ordering
  properties specified in Lemma \ref{lem:all_or_nothing}. \end{proof}
  

For the analysis of extent of distortion below, we can assume $\P$ is
a \din, as when using $\klns()$ the first argument is augmented to a
distribution. However, we expect the above analysis and bound can be
extended to strict \sds as well.




\co{
\begin{align*}
\llns(\Q|\P) &= -p_1\log(\pns) - \sum_{i\ge 2 } p_i\log(p_i + \frac{p_i}{1-p_1}p_1 ) \\
&= -p_1\log(\pns) - \sum_{i\ge 2 } p_i\log(p_i + \frac{p_1}{1-p_1}p_i)\\
&= -p_1\log(\pns) - \sum_{i\ge 2 } p_i\log(\frac{1}{1-p_1}p_i) \\
&= -p_1\log(\pns) - \sum_{i\ge 2 } p_i\log(p_i) + \sum_{i\ge 2 } p_i\log(1-p_1) \\
&= -p_1\log(\pns) - \sum_{i\ge 1 } p_i\log(p_i) + p_1\log(p_1) + \log(1-p_1)\sum_{i\ge 2 } p_i \\
&= -p_1\log(\pns) + H(\P) + p_1\log(p_1) + \log(1-p_1)(1-p_1) \\
\end{align*}

$\llns(\P|\P) \ge \llns(\Q|\P) \Rightarrow H(\P) \ge -p_1\log(\pns) +
H(\P) + p_1\log(p_1) + \log(1-p_1)(1-p_1)$ or $-p_1\log(p_1) -
\log(1-p_1)(1-p_1) \ge -p_1\log(\pns)$.
}

\subsubsection{Extent of Distortion}
\label{sec:extent}

Lemma \ref{lem:pt} and \ref{lem:all_or_nothing} explain which items
are zeroed, and the following lemma specifies extent of increase in an
item's \pr in any minimizer of $\klb(\P||.)$, and further
characterizes the properties of a minimizer.

\begin{lemma}\label{lem:ext} Given any \di $\P$ and $\pns \in (0,1)$, and
  any \di $\Q$ with $\sp{\Q}\subseteq \sp{\P}$ and with proportionate
  spread according to $\P$, let $Z:=\sp{\P}-\sp{\Q^*}$ (the zeroed
  items or the difference of the two support sets), and
  $\ps(\Q):=\sum_{i\in Z}\P(i)$, thus $\ps(\Q)$ is the total mass of
  the zeroed items (the  shifted mass), $\ps(\Q)\ge 0$.  ({\bf
    part 1}) We have for any $i \in \sp{\Q},
  \Q(i)=\frac{\P(i)}{1-\ps(\Q)}$. ({\bf part 2}) Furthermore, let
  $S_g$ be the set of all such proportionate $\Q$ with
  $min_{i\in\sp{\Q}}\P(i) \ge \pt$. Then any minimizer \di $\Q^*$ of
  $\klb(\P||.)$ is in $S_g$, and has the largest support and the
  smallest shifted mass among such (\ie $\ps(\Q^*)\le \ps(\Q)$ and
  $|\sp{\Q^*}| \ge |\sp{\Q}|$  for any $\Q\in S_g$).
\end{lemma}
\begin{proof}
  For a \di $\Q$, let $\ps$ denote $\ps(\Q)$, and let $s:=\sum_{i\in
    \sp{\Q}}\P(i)$. From the definition of proportionate spread, we
  add $\frac{\P(i)}{s}\ps$ to $\P(i)$. We have $s=1-\ps$, therefore,
  $\Q(i) = \P(i) + \frac{\P(i)}{1-\ps}\ps =
  \P(i)(1+\frac{\ps}{1-\ps})=\frac{\P(i)}{1-\ps}$. ({\bf proof of part
    2}) From Lemma \ref{lem:all_or_nothing}, $\ps$ is proportionately
  spread onto non-zeroed items in $\Q^*$ as well.  Consider $Q_1$ and
  $\Q^*$, both in $S_g$, and assume $\sp{\Q_1} > \sp{\Q^*}$ (proof by
  contradiction). By shifting the lowest \pr in $S_1$ to remaining
  (higher \prn) items in the support, and repeating, we should get to
  $\Q^*$ (or an equivalent, in case of ties). But each shift results
  in an inferior $\Q$ because of our assumption that all \prs are $\ge
  \pt$ and the shifts only increase the \pr on items that remain
  (non-zeroed items).\end{proof}



We now have the tools
for understanding the types and extents of distortion from using
$\klns(\dix{\P}||.)$ under different scenarios.  We know that an item
with \pr below $\pt$ can be zeroed, \ie completely distorted, in
$\Q^*$. From the outset, working with finite-space predictors, we have
accepted the possibility of poor or no estimation of small \prs beyond
a point. Certain such low-\pr items may have their \pr multiplied by
many folds in $\Q^*$ (yielding a high distortion ratio
$\frac{\Q^*(i)}{\P(i)}$): imagine the uniform \di $\P$ with $k$ items,
each with \pr $1/k$. A few of these items attain a high \pr of $\pt$
or higher ($\approx 3\pns$) in $\Q^*$ (a fixed $\approx
\frac{1}{3\pns}$ such non-zeroed items). For instance, with
$\pns=0.01$, the number of nonzeroed items is fixed at $\approx
\frac{1}{0.03}\approx 33$. The original total \pr mass on these items
(in $\P$) is $\frac{33}{k}$, shrinking with increasing $k$. Thus, as
we imagine increasing $k$, the mass $\ps$, $1-\frac{33}{k}$, shifting
from the zeroed items to nonzeroed items increases, approaching 1, and
using Lemma \ref{lem:ext}, the relative increase in the \pr of
nonzeroed items grows unbounded:
$\frac{\Q^*(i)}{\P(i)}=\frac{1}{1-\ps} \ra \infty$. Note also that
\fcap() does not alter $\Q^*$ in this example by much (so our argument
holds for both $\klb$() and $\klns$()).  This example showing large
distortion ratio on small items (items with tiny \pr below $\pns$)
holds when we use a perfect marker (we explain how using a practical
imperfect marker affects distortion in the next section, in
particular, see Lemma \ref{lem:practical}).


When the total \pr shift $\ps$ is small, the increase
$\frac{1}{1-\ps}$ in \pr of any non-zeroed item is also small.
Let $S_g$ be the set of items in $\dix{\P}$ (including item 0)
with \pr above $\pt$ (\ie well above $\pns=\pmin$), and let $\pg:=
\sum_{i\in S_g}\P(i)$.  By Lemma \ref{lem:pt}, no item in $S_g$ is
zeroed (no mass from it is shifted), so $\ps \le 1-\pg$. For example,
a $\pg$ value of $0.9$ means $\ps \le 0.1$, and therefore no item is
increased by no more than $\frac{1}{0.9}\approx 1.1$.
In particular, if there are one or more items in $\P$ with \pr above
$\pt > \pns$, and the remaining is unallocated and $1-\sm{\P}=\u{\P}
\ge \pns$, \ie there is no normalizing and removal in $\fcap$, then
$\P$ remains the unique optimal (minimizer) of $\klns(\P||\fcap(.))$
(and of $\klb(\P||.)$). For instance, this is the case for \sd
$\P=\{1:0.5, 2:0.1\}$ (with $\pns=0.01$) ($\fcap()$ does not change
$\P$ in this example).  If $\u{\P} < \pns$, then some normalizing will
occur. However, if the normalizing factor doesn't change (reduce) the
\prs by much, the distortion will remain low. For example, with the
\sd $\P=\{1:0.5, 2:0.1, 3:0.4\}$,
$\fcap(\P)=0.99\P=\{1:0.495,2:0.099,3:0.396\}$, and $\P$ remains
optimal for $\klns(\P||\fcap(.))$, and so are the close multiples of
$\P$, $\alpha\P$ for $\alpha \in [0.99, 1]$
($\klns(\P||\P)=\klns(\P||\alpha\P)$ as long as $\alpha\ge 0.99)$. In
this example, there are multiple minimizers for $\klns(\P||\fcap(.))$
but a unique one for $\klns(\P||.)$ ($\fcap()$ maps them to one unique
minimizer).

\co{
: $\P$ will be
optimal or will score close to optimal (the minimum loss)
(both distortion in loss and distortion in items of $\P$ would be
low). This is because filtering and normalizing would not remove or
alter much mass (most mass is well above $\pt> \pns$ and so is not
removed nor reduced via normalization by much).  }


\co{
Given a \sd
$\Q^*$ minimizing $\klns(\P||\Q)$,
one could define distortion in loss as $\klns(\P||\P) -
\klns(\P||\Q^*)$ (positive if there is distortion).

We define distortion in items as the maximum over high \pr items in
$\P$, $\max_{\Q^*}\max_{i\in \I, \P(i) \ge \tau}
(\frac{\Q^*(i)}{\P(i)}-1)$, where $\tau$ is a threshold, \eg $\tau =
\pt$.

}



\co{
  
If $\P$ has items below $\pns$, these items are removed by $\fcap$,
and in the item-sense of distortion, these items are distorted
substantially. If $\P$ has many items at or below $\pns$, with total
mass $s$, \eg $s > 1/2$, all of $s$ could be spread over remaining
salient items and thus one or more salient items can be distorted
substantially. The original $\P$ may still remain a minimizer (in that
sense, distortion in loss can be low) but any unfiltered item can be
substantially increased to minimize $\klb(\P||.)$.
}

Note also that an \sd $\Q$ can have more distance from $\P$ than the
minimizer $\Q^*$, and still score lower than $\P$ itself (\ie $\Q^*$
does not determine the largest distance from $\P$ under the constraint
of scoring better). We have focused on the distance of a minimizer
because of our intended use of $\klns()$ for comparisons (where a
minimizer $\Q^*$ is preferred).

\co{
\subsection{ Sketches  }

We finish this section and only sketch how the above properties could
be used to limit the extent of distortion that applying \fcap() and
using $\klb()$ can cause. By distortion, we mean a \sd $\Q$ different
from the underlying $\P$ scoring better than $\P$: $\klns(\P||\Q) \le
\klns(\P||\P)$. We are specially interested how a \pr $p$, $p \ge
\pmin$ in $\P$ can change in such $\Q$.

Consider the earlier case where $\pmin = 0$, $\pns>0$, and
$\sm{\P}>1-\pns$, thus $\P$ is scaled down by an
$\alpha=\frac{1-\pns}{\sm{\P}}$ in \fcap(). A \pr $p$ in $\P$ that is
smaller than $\pns$ or sufficiently close to it could have its mass
transferred to  other higher \prs in $\P$ (as the above
example showed), and for the resulting \sd $\Q$, $\klns(\P||\Q) \le
\klns(\P||\P)$. This is because when computing $\llns(\Q|\P)$, an
observed item that has 0 or tiny \pr can get a loss of
$-\log(\pns)$. However, there is a limit on how much this shifting can
help. For instance, if all \prs in $\P$ are sufficiently large
(distant from $\pns$), no shifting helps, and $\P$ remains the
minimizer of $\klns(\P||\P)$. For example, with $\P=\{1:0.9, 2:0.1\}$
and $\pns=0.01$, $\P$ is the minimizer of $\klns(\P||\P)$. More
generally, with $\P'=\alpha\P$ or $\alpha=\frac{\P'(i)}{\P(i)}$, if
for item $j$, if $\frac{\pns}{\P(j)} \ge \alpha$, or
$\alpha\P(j)\le\pns$, then the \pr of that item could be shifted, for
a better score. With multiple small \prs, this shifting can be done in
order of smallest first, and the sum of shift should not exceed
$\frac{\pns}{\alpha}$. As $\pns$ is set small, and $\alpha$ is near 1,
such shifts result in small change.  Similarly, when $\pmin > 0$ but
is small, a single small \pr below $\pmin$ can be shifted onto others,
in the manner of part 3 of Lemma \ref{lem:scale_corr}, yielding better
scores than original $\P$. As a small amount is spread, the larger
\prs well above $\pmin$ (if any) are not affected much.

}

\co{
\subsection{Extents of Distortions}

Consider \di $\P$, and a positive but small $\pmin > 0$ (\eg
$\pmin=0.01$), and $\forall i \in \sup(\P), \P(i)\le \pmin$ (all the
\prs are small). Then the distortion in $\P$ will be large, for
instance the empty \sd $\Q$ and any other \sd scores the same as $\P$,
$-\log(\pns)$, so $\P$ remains an optimal (no distortion in loss), but
other distant \sds score low too.  Now consider \di $\P$, where
$\P(1)=1/2$, and all other $i > 1$ have low \prn: $\forall i \in
\sup(\P), i\ne 1, \P(i)\le \pmin$. This case leads to both a fairly
large distortion in $\P$ as well as a noticeable distortion in the
minimum loss, when we see that the \di $\Q(1)=1$ scores lower (better)
than $\P$ ($\approx -1/2\log(1)-1/2\log(0.01)$ for $\Q$ vs $\approx
-1/2\log(1/2)-1/2\log(0.01)$ for $\P$, and is distant from $\P$.
}

\co{
More generally, we know that $\sm{\fcap(\P)} \le \sm{\P}$ for
non-empty $\P$. The amount of distortion is a function of the ratio
$r=\frac{\sm{\fcap(\P)}}{\sm{\P}}$ as well as $\pns$. If $r$ is high,
this means that either the underlying $\P$ is empty or the items with
\pr above $\pmin$ have a total mass that is a large fraction of
$\sm{\P}$. The ratio $r$ can be as small as 0 when $\forall i \in
\sup(\P), \P(i)\le \pmin$. If $\pns$ is small and $r$ is high, then we
should get low distortion.

}
  
\subsubsection{Imperfect \nsm and Setting the $\pns$ threshold }
\label{sec:practical}
We have assumed a perfect \nsm (referee) in all the above. Assuming
$\pmin=0.01$, a simple practical \nsmn, \eg keeping a history of the
last 100 time points (a box predictor), will have some probability of
making false positive markings (a salient item marked \nsn) and false
negative errors (items below $\pmin$). The error probability goes down
as an item becomes more salient or more noisy (its \pr gets farther
from the $\pns$ threshold), and for items with \pr near the boundary,
the loss would be similar whether or not $\log(\pns)$ is used. We
assume such a box predictor as a practical \nsm here. Note also that,
just as in the case for the prediction algorithms, there is a tradeoff
here on the size of the horizon used for the practical \nsmn: change
and non-stationarity motivate shorter history. We also note that with
a practical marker, marking items with \pr below $\pns$ as \ns with
high probability, the example of previous section with tiny items
yielding high $\frac{\Q^*(i)}{\P(i)}$ ratio would not occur: if a \di
$\P$ has all its items well below $\pns$, the best scoring $\Q^*$ is
the empty \sd ($\u{\Q^*}=1$), with a practical \nsmn. We formalize the
extent of increase in this imperfect setting next.

Given a \sd $\P$ on $\I=\{1,2, \cdots\}$, the {\em ideal-threshold}
\nsmn, or the {\bf \em threshold marker} for short, marks an item $i$
\ns iff $\P(i)\le \pns$. Considering how \llns() (or \llnsr()) works
when the threshold marker is used, the marker in effect converts a \sd
$\P$ to a \di $\P'$ where for any item $i$ with $\P(i)\le \pns$, its
\prn, together with $\u{\P}$, is transferred to item $0$ when
computing $\klns(\P'||.)$ (recall that $\dix{\P}$, in definition
\ref{augment}, only transferred $\u{\P}$ to item 0). Thus $\P'$ is a
\di where $\forall i \in \sp{\P'}$ if $i\ne 0$, then $\P'(i) >
\pns$. It is possible that $\P'(0) < \pns$, but as we use \fcap(), we
could limit our analysis to \di $\P'$ where $\forall i\in \sp{\P'},
\P'(i) \ge \pns$.

The next lemma shows that the distortion ratio $\frac{\Q^*(i)}{\P(i)}$
is upper bounded by about 3 in this setting where the \prs are above
$\pns$.

\begin{lemma}\label{lem:practical}
  Given any \di $\P$ with $\min_{i\in\sp{\P}}\P(i)\ge \pns$, for any
  minimizer $\Q^*$ of $\klb(\P||.)$,
  the distortion ratio $\frac{\Q^*(i)}{\P(i)}\le
  \frac{\pt}{(1-\pt)\pns} < \frac{3}{(1-\pt)}$.
\end{lemma}
\begin{proof}
  Let $m$ be the mass of non-zeroed items in $\P$ ($m=1-\ps(\P)$,
  $\ps$ was also used in Lemma \ref{lem:ext}), then we need to bound
  the distortion ratio $\frac{1}{m}$ which is the ratio by which each
  item $i$ in $\sp{\Q^*}$ goes up by ($\frac{\Q^*(i)}{\P(i)}$). Let
  $j$ be any item that is zeroed (at least one such item exists,
  otherwise the distortion ratio is 1). Let $x:=\P(j)$, and we have $x
  \ge \pns$. Then spreading 1.0 over one additional item,
  $\sp{\Q^*}\cup\{j\}$, proportionately, would fail to take the \pr of
  $j$ to $\pt$ (by Lemma \ref{lem:ext}, $\Q^*$ would not be optimal).
  Proportional allocation over $\sp{\Q^*}\cup\{j\}$ yields
  $\frac{x}{x+m}$ for item $j$ and we must have $\frac{x}{x+m} < \pt
  \Rightarrow x < \pt x +\pt m \Rightarrow \frac{(1-\pt) x}{\pt} < m$
  and as $x > \pns$, we get $\frac{1}{m} < \frac{\pt}{(1-\pt)x} \le
  \frac{\pt}{(1-\pt)\pns}$.\end{proof}

When $\pns=0.01$, then $\pt < 0.03$, or $(1-\pt)^{-1} < 1.031$, and the
distortion ratio is bounded by $3.1$. A practical marker can be viewed
as a noisy version of the (ideal) threshold marker.

\co{
  Take any minimizer $\Q^*$ and let $i_1$ be largest item in $\sp{\P}$
  that is zeroed ($\Q^*(i_1) = 0$) and $i_2$ be a smallest item with
  $\Q^*(\i2) > 0$ (we must have $\Q^*(\i2) \ge \pt$). Then the extent
  of the ratio $\frac{\P(i_2)}{\P(i_1)}$ determines
  $\frac{\Q^*(i_2)}{\P(i_2)}$ (how large the distortion ratio is).
}
\co{  
  In $\Q^*$ all (positive or support) items are increased by the same
  multiple, $(1-\ps(\Q^*))^{-1}$. Let $\Delta$ denote the total
  increase: $\Delta:=\sum_{i\in \sp{\Q^*}} \Q^*(i)-\P(i)$. To maximize
  the ratio $\frac{\Q^*(i)}{\P(i)}$ and simplify analysis, we can make
  all the $k=|\sp{\Q^*}|$ items in $\sp{\Q^*}$ even via removing and
  discarding the extra \pr in larger items in $\sup{\Q^*}$ from
  $\P(i)$ as follows. Let $i1$ be the item with smallest \pr $\P(i)$
  in $\sp{Q^*}$. For every other item $i2\in\sup{\Q^*}$,
  $\P(i2)-\P(i1)$ is discarded, thus $\P'(i2)=\P(i1)$. We can now
  assume $\Delta$ is spread evenly on all $k$. Note that, as long as
  for some $i2\in\sp{Q^*}$, $\P(i2) > \P(i1)$, then
  $\frac{\Q^*(i)}{\P'(i)}>\frac{\Q^*(i)}{\P(i)}$ (the ratio has
  increased) for {\em all} items with this change (not just for $i1$,
  even though $\P'(i) < \P(i)$).

  Let $x:=\P(i1)$ and we have that spreading $\Delta+kx$ over $k+1$
  does not meet (or exceed) $\pt$. We have $\frac{\Delta+kx}{k+1} <
  \pt$.  \pt$.
  
}




Given threshold $\pns > 0$, if we find that the predictions $\oat{\Q}$
have mostly unallocated mass or \prs close to $\pns$, then $\pns$ and
$\pmin$ may need to be lowered (\eg from $0.01$ to $0.001$) for
evaluation as well as better (finer) prediction. For example, with
$\pmin=\pns=0.01$, we may find that many learned \prs are below
$0.05$. Or may observe that a sizeable portion of the predicted \sds
is unallocated ($\u{\oat{\Q}}$ well above $\pns$).  This lowering of
thresholds can lead to extra space consumption by the predictor. It
does not guarantee that further salient items will be learned, as any
remaining salient items may have probability well below the new $\pns$
that we set, or all that is left could be truly pure noise (measure
0). All this depends on the input stream (rate of change, rate of pure
noise,...), in addition to the quality of the predictor.

\co{
  
Ideas on how to limit ``distortion'' when there is scaling and/or
$\pmin > 0$:

\begin{itemize}
\item Any optimal $\Q$ which may be different from $\P$ is not far
  away from $\P$.
\item As we change an optimal $\Q$, we lower its performance until we
  match $\P$'s, and as we do this, we get further away from $\P$, but
  this change and the distortion it induces is limited too.
\end{itemize}

An example of a result limiting the distance of the optimal $\Q$ from
$\P$.

\begin{lemma*}
  For a non-empty \sd $\P$, let $\Q^* = \argmax_{\Q} \expdb{o \sim
    \P}{\llns(\Q, o$, isNS($o$)$ )}$. Then $\forall i \in \I,$ if
  $\P(i) \ge \pmin$, then $\Q^*(i) \ge \P(i)$, and if $\Q^*(i) \ge \pmin$,
  then $\frac{\Q(i)}{\P(i)} \le 1+\pmin$. 
\end{lemma*}

}

\subsection{Alternative Scoring Schemes for Noise}
\label{sec:alts}

We have considered a few other options for handling \ns items (when
using \loglossn),
including
not having a
referee: One possibility is reporting two numbers for each prediction
method: the number of items it treated as \ns on a given sequence
(assigned 0 or sufficiently small probability), and otherwise average
the \llsp on the remainder of the input sequence.  A method is
inferior if it marks a large fraction of items as \ns compared to
others, but if the fractions are close, one looks at the \llsp
numbers. However, we thought that comparing methods with two numbers
would be difficult, and it is also hard to combine two very different
numbers to get a single understandable number to base comparisons on.
Another no-referee option is to pair two methods (or multiple methods)
and specify a policy when they disagree (on items deemed \ns by one
but not the other). Here, one needs to specify a scoring policy
handling disagreements. This option does not provide a single
understandable quality score for a method in isolation (making it hard
to pick parameters when improving a single algorithm as well), and may
not easily generalize to comparing more than two methods.

We could also use a referee, but ignore (skip) the \NS items when
averaging the loss. This measure can provide insights, and we have
looked at it in some experiments, but it can also lead to impropriety
(incentive incompatibility or deviations from truth-telling): a method
that puts all of its \pr mass toward salient items, than leaving some
for the noise portion of the stream, would score better. Assume the
portion of noise items in the stream is 0.2 (20\% of the times a noise
item appears), and otherwise a salient item, say \itm{A}, occurs with
0.8 probability. If we ignore \NS items during evaluation, the method
that shifts the 0.2 \pr mass over the salient items, in this case
always outputting \itm{A} with \pr 1, would score better than the
honest method that reports \itm{A} with its true \pr of 0.8.



\co{
One way to 

let a prediction method report the
number of \NS (the ones it didn't have a probability for)
independently, but we thought





\begin{itemize}

  \item handling unseen items: they can be infrequent or noise items
    and/or new items

  \item our methods distributions also either assign 0 probability,
    meaning it is treated as unseen, or a probability $p \ge \pmin$ to
    an item.

  \item We use a '3rd party referee' a simple method that marks items
    as unseen (also called OOV elsewhere).. on unseen we have a choice
    of how to score too, and will need to define the scoring policy as
    well..

  \item other choices, besides referee? pair the two predictors each
    time, and define a scoring policy based on whether they agree of not

\end{itemize}
}

\todo{describe the code format: mostly very similar to python
  language.. map.get(o, 0.0)..  but comments begin with '//'
  .. Sometimes, if a parameter list would be too long, or if it's just
  best to think of certain parameters as global (constants, etc), we
  put it in the comment section right after the 1st line (after the
  function name)... the pseudo code is written for readability rather
  than most efficient ..  }






\section{Convergence Properties of Sparse EMA}
\label{app:ema}

Lemma \ref{lem:ema_shirnk}, and its proof follows, where we are in the
stationary binary setting (\sec \ref{sec:binary}), where the target
\pr to learn, $\P(1)$, is denoted $\tp$, and the estimates of EMA for
item 1, denoted $\oat{\ep}$, form a random walk ($\oat{\ep}$ is the
estimate immediately after the update at time $t$).  \fig
\ref{fig:proof} shows the main ideas of the proofs, \eg the expected
step size is $\lr^2$ towards $\tp$ when $\ep$ is not too close.

\begin{lemma*} EMA's movements, \ie changes in the estimate $\oat{\ep}$,
  enjoy the following properties, where $\lr \in [0, 1]$:
\setlength{\leftmargini}{15pt} 
  \begin{enumerate}
  \item Maximum movement, or step size, no more than $\lr$: $\forall t,
    |\ott{\ep}{t+1}-\oat{\ep}| \le \lr$.
  \item Expected movement is toward $\tp$: Let $\oat{\Delta} := \tp - \oat{\ep}$. Then,
     $\expd{\Delta^{(t+1)}|\oat{\ep}=p} =
    (1-\lr)(\tp-p) = (1-\lr) \oat{\Delta}$.
  \item Minimum expected progress size: With $\ott{\delta}{t} :=
    |\oat{\Delta}|-|\ott{\Delta}{t+1}|$, $\expd{\oat{\delta}}\ge \lr^2$
    whenever $|\oat{\Delta}|\ge\lr$ (\ie whenever $\ep$ is sufficiently
    far from $\tp$).
  \end{enumerate}
\end{lemma*}
\begin{proof}
  (proof of part 1) On a negative update,
  $\oat{\ep} - \ott{\ep}{t+1} = \oat{\ep} - (1-\lr)\oat{\ep} =\lr\oat{\ep}\le \lr$, and on a positive update,
  $\ott{\ep}{t+1}-\oat{\ep} = (1-\lr)\oat{\ep}+\lr-\oat{\ep}=\lr-\oat{\ep}\lr\le \lr$ (as $\oat{\ep}\in[0,1]$).

  (part 2) We write the expression for the expectation and simplify:
 $\tp$ of the time, we have a positive update, \ie both weaken and
  boost ($(1-\lr)p+\lr$), and the rest, $1-\tp$ of the time, we have
  weaken only ($(1-\lr)p$). In both cases, the term $(1-\lr)p$, is
  common and is factored:
  \begin{align*}
    \hspace*{-.4in}\expd{\Delta^{(t+1)}|\oat{\ep}=p} & = \tp\left(\tp - ((1-\lr)p+\lr) \right)+(1-\tp)\left(\tp-(1-\lr)p)\right) \\
    & \mbox{\ \ (next, the term, $\tp - (1-\lr)p$, is common and is factored )}\\
    &= (\tp+(1-\tp))(\tp - (1-\lr)p) - \tp\lr \hspace*{.3in} \\
    &= \tp - (1-\lr)p - \tp\lr = \tp(1-\lr) - (1-\lr)p  \\
    &= (1-\lr)(\tp-p)
  \end{align*}

  (part 3) Note from our definition of $\oat{\delta}$, $\oat{\delta} >
  0$ when distance to $\tp$ is reduced (when $|\ott{\Delta}{t+1}| <
  |\oat{\Delta}|$). When $\ep$ is close to $\tp$, \eg $\ep=\tp$, the
  expected distance may not shrink, but when outside the band, we can
  show a minimum positive progress: Assume $\oat{\ep} \le \tp-\lr$,
  then $\oat{\delta}=\tp-\oat{\ep} - (\tp-\ott{\ep}{t+1})$
  ($\ott{\ep}{t+1}$ is also below $\tp$) or
  $\oat{\delta}=\ott{\ep}{t+1}-\oat{\ep}$, and:
%
\begin{align*}
  \hspace*{-.35in}\expd{\ott{\ep}{t+1}-\ott{\ep}{t}}=
  \expd{\ott{\delta}{t}}&=
\expd{\tp-\ott{\ep}{t}-(\tp-\ott{\ep}{t+1})} =
\expd{\tp-\ott{\ep}{t}} - \expd{\tp-\ott{\ep}{t+1}}\\
& \mbox{\ \ (the above transformation used the linearity of expectation)} \\
&= \expd{\tp-\ott{\ep}{t}} - (1-\lr)\expd{\tp-\ott{\ep}{t}}
= \lr \expd{\tp-\oat{\ep}} \mbox{\ \  (from part 2)} \\
& \ge \lr^2 \mbox{\ \  (from the assumption $|\tp - \oat{\ep}|\ge\lr$)} \end{align*}
And similarly for when $\oat{\ep} \ge \tp+\lr$, then
$\oat{\delta}=\oat{\ep}-\ott{\ep}{t+1}$.
\end{proof}


\begin{figure}[t]
\begin{center}
  \centering
\hspace*{-0.3cm}  \subfloat[Convergence to within the band (from below in this picture). ]{{
      \includegraphics[height=5cm,width=6.5cm]{
        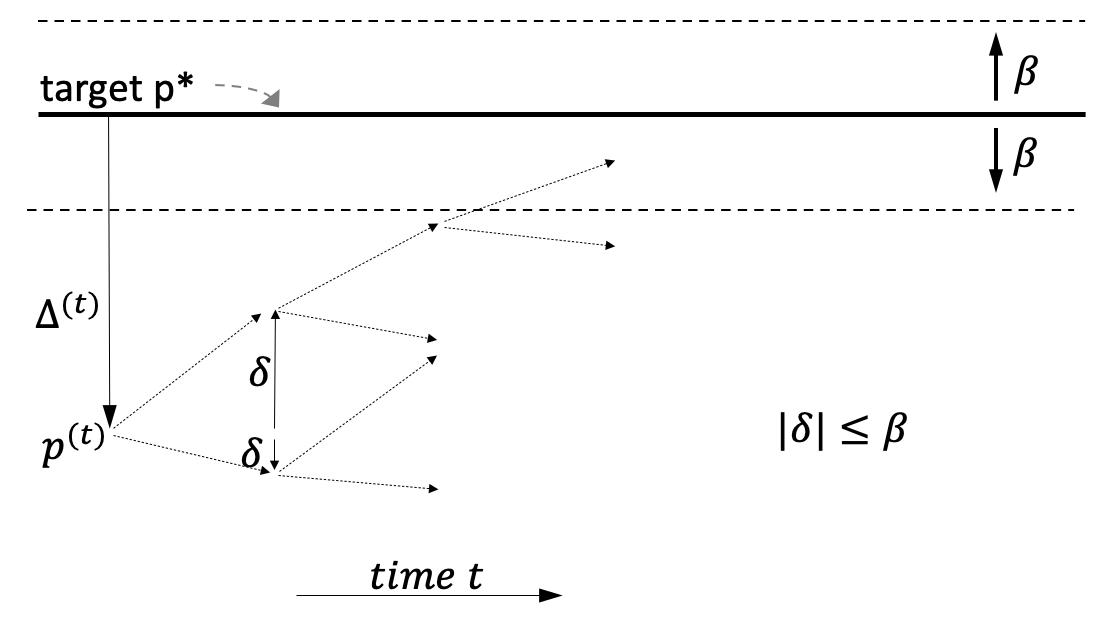}
  }}
  \subfloat[Expected steps and expected positions]{{
      \includegraphics[height=5cm,width=6.5cm]{
        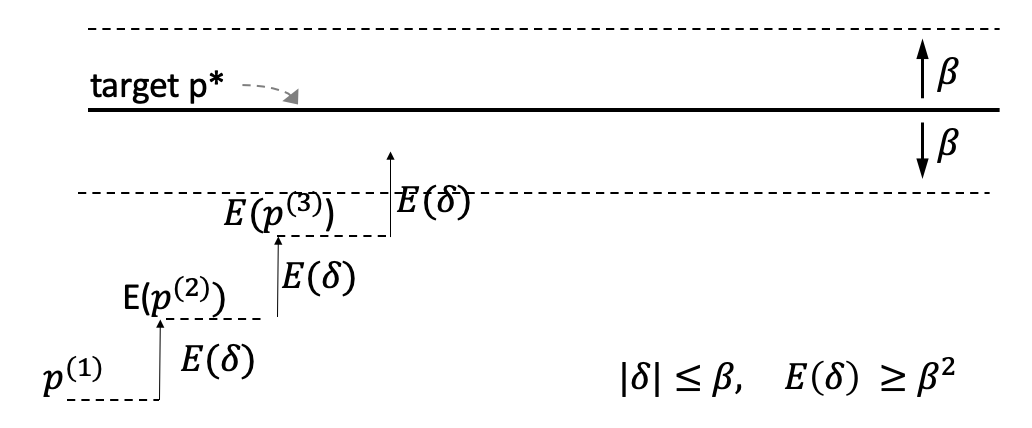}
  }}
  \end{center}
\vspace{.2cm}
\caption{A picture of the properties in Lemma
  \ref{lem:ema_shirnk}, upper bounding the expected number of time
  steps to enter the band $\tp\pm\lr$ by $\lr^{-2}$, whether from
  below or above the band
  (Theorem \ref{thm:converge}), where $\oat{\Delta}:=\tp-\oat{\ep}$
  and $\oat{\delta}:=|\oat{\Delta}|-|\ott{\Delta}{t+1}|$. (a) At any
  point, there are two possible outcomes after an update (weaken or boost), and
  the movement, $\delta$, could be toward or away from $\tp$ (\eg
  $\delta < 0$), but always $|\delta|\le \lr$. (b) As long as $\ep$ is
  not in the band, there is an expected movement, $\expd{\delta}$, of
  at least $\lr^2$, toward $\tp$. $\expd{\delta}$ is smallest when
  $\ep$ is near $\tp$ and largest, up to $\lr$, when $\ep$ is farthest
  ($\ep=0$ or $\ep=1$). }
\label{fig:proof} 
\end{figure}

Note that from property 2 above, there is always progress towards
$\tp$ {\em in expectation}, meaning that if $\oat{\ep} < \tp$, then
$E(\ott{\ep}{t+1}) > \oat{\ep}$, and if $\oat{\ep} > \tp$, then
$E(\ott{\ep}{t+1}) < \oat{\ep}$. This is the case even if the
probability of moving away is higher than $0.5$ ($\ep \le \tp < 0.5$).
We note however, that the result being in expectation, both the actual
outcomes for $\ott{\ep}{t+1}$ can be father from $\tp$ than
$\ott{\ep}{t}$ (consider when $\ott{\ep}{t}=\tp$). Property 3 puts a
floor (a minimum) on amount of the progress towards $\tp$ in
expectation, when $\ep \not \in [\tp-\lr, \tp+\lr]$, For instance,
when $\ott{\ep}{t}\le \tp-\lr$ , it puts a minimum on the expected
positions $\expd{\ott{\ep}{t+1}}, \expd{\ott{\ep}{t+2}},
\expd{\ott{\ep}{t+3}}, \cdots$, until one such point crosses the band,
\fig \ref{fig:proof}(b), and Theorem \ref{thm:converge} follows.

\begin{theorem*}
  EMA, with a fixed rate of $\lr \in (0, 1]$, has an expected
    first-visit time bounded by $O(\lr^{-2})$ to within the band
    $\tp\pm\lr$.  The required number of
    updates, for first-visit time, is lower bounded below by
    $\Omega(\lr^{-1})$.
\end{theorem*}

\begin{proof}
We are interested in maximum of first-time $k$ when the expected
$\expd{\ott{\ep}{k}} \in [\tp-\lr, \tp+\lr]$.  Using the maximum
movement constraint, as long as $|\tp - \oat{\ep}| > \lr$, an EMA
update does not change the sign of $\tp - \oat{\ep}$ ($\ep$ does not
switch sides wrt $\tp$, \eg if greater than $\tp$, it remains greater
after the update).  Thus, before an estimate $\oat{\ep} < \tp$ changes
sides, and exceed $\tp$, it has to be within or come within the band
$\tp\pm\lr$. Therefore, start with an arbitrary location
$\ott{\ep}{1}$ outside the band, say $\ott{\ep}{1} < \tp - \lr$
(similar arguments apply when $\ott{\ep}{1} > \tp + \lr$), and
consider the sequence, $\ott{\ep}{1}, \ott{\ep}{2}, \ott{\ep}{3},
\cdots, \ott{\ep}{k}$, where $\forall t, 1\le t \le k, \ott{\ep}{t} <
\tp-\lr$.  We can now lower bound the expected position of $\ott{\ep}{k}, k\ge 2$
wrt $\ott{\ep}{1}$, to be at least $(k-1)\lr^2$ above $\ott{\ep}{1}$:

\begin{align*}
\expd{\ott{\ep}{k} - \ott{\ep}{1}| \ott{\ep}{1}=p} &=
\expd{\ott{\ep}{k} - \ott{\ep}{k-1} + \ott{\ep}{k-1} - \ott{\ep}{1} | \ott{\ep}{1}=p} \\
 & \mbox{\ \  (insert all intermediate sequence members)} \\
&= \expd{\sum_{2\le t\le k}\ott{\ep}{t} - \ott{\ep}{t-1} | \ott{\ep}{1}=p}\\
&= \sum_{2\le t\le k} \expd{\ott{\ep}{t} - \ott{\ep}{t-1} | \ott{\ep}{1}=p} \ge (k-1)\lr^2,
\end{align*}
where we used the linearity of expectation, and the $\lr^2$ lower
bound for the last line.  With $\tp\le 1$, an upper bound of
$\frac{1}{\lr^2}$ on maximum first-visit time follows: $k$ cannot be
larger than $\frac{1}{\lr^2}$ if we want to satisfy $\forall 1\le t\le
k, \expd{\ott{\ep}{t}}< \tp-\lr$, or one of $t \in \{1, \cdots,
\frac{1}{\lr^2}+1\}$ has to be within the band $\tp\pm\lr$.

The lower bound $\Omega(\lr^{-1})$ on $k$ follows from the upperbound
of $\lr$ on any advancement towards $\tp$.
\end{proof}

The dynamics of the estimates $\oat{\ep}$ can be likened, in some
respects, to oscillatory physical motions such as the motion of a
pendulum and a vertically hung spring: the expected movement is 0 at
target $\tp$, corresponding to the resting length of the spring, or
its equilibrium length, where spring acceleration is 0, and the
expected movement is highest at the extremes (farthest from $\tp$),
akin to the acceleration (vector) of the spring being highest when it's
most stretched or compressed.

\co{

  
Interestingly, the bounds are independent of how far $\tp$ is from
$\ep$, for instance $\ep$ could be 0 and $\tp$ could be $1$. The
intuition is that the higher (farther) $\tp$ also lead to more
positives observed (more chances to make progress toward them).

}

\co{
We now show each time step is a positive progress, in
expectation, towards the band $\tp\pm\lr$ (a step up), in particular
$\expd{\ott{\delta}{t}}=\expd{\ott{\ep}{t+1}-\ott{\ep}{t}}\ge \lr^2$:
\begin{align*}
\hspace*{-.5in}\expd{\ott{\delta}{t}}=\expd{\ott{\ep}{t+1}-\ott{\ep}{t}}&=
\expd{\tp-\ott{\ep}{t}-(\tp-\ott{\ep}{t+1})} =
\expd{\tp-\ott{\ep}{t}} - \expd{\tp-\ott{\ep}{t+1}}
\mbox{\ \ (linearity of expectation)} \\
&= \expd{\tp-\ott{\ep}{t}} - (1-\lr)\expd{\tp-\ott{\ep}{t}}
= \lr \expd{\tp-\oat{\ep}} \mbox{\ \  (use of Lemma \ref{lem:ema_shirnk}, property 2)} \\
& \ge \lr^2 \mbox{\ \  (from our assumption on the sequence:
  $\forall t\le k, \ott{\ep}{t} \le \tp-\lr$)}
\end{align*}
}

\co{
\subsection{EMA with Harmonic Decay}

Here, we show that decreasing the rate $\lr$ of EMA in harmonic manner
is equivalent to plain averaging (shown earlier in
\cite{expedition1}). The equivalence remains the case until the rate
reaches the minimum allowed $\lrmin$, and can be shown by expanding
the sum and noting that it telescopes:
}




\section{Further Material on the \qu (Queues) Method}
\label{sec:qanalysis}

We begin by reviewing a few properties of a single completed cell of a
queue in the \qu method, which is equivalent to the experiment of
tossing a biased two-sided coin, with unknown heads \pr $\tp > 0$ (the
target of estimation), counting the tosses until and including the
first heads outcome.  We can repeat this experiment until we get $k
\ge 1$ heads outcomes. The count in each completed cell follows the
geometric distribution, and the total number of trials, or the total
count over all completed cells, minus the number of heads, has the
more general negative binomial distribution (with parameters $k$ and
$\tp$). In the next section we look at estimators for $\tp$ using
counts from several completed cells, and describe how the general
method of Rao-Blackwellization can be applied there to get a superior
estimator, in a sense described next. In \sec \ref{app:qsum_etc}, we
proceed to multiple queues, and explore the spread of the \pr
estimates in a $\qmap$.

Ideally, we desire estimators that have no or little bias, that is, if
we took the average over many repetitions of the same experiment (\eg
by different people using the same data collection technique, and
assuming $\tp$ is not changing), the average over all the estimates
would converge to $\tp$. We also prefer low variance, implying that
any particular estimate in time has low chance of being far from
target. The two estimation goals are often distinct, for instance an
estimator can have zero variance (fully stable) but nonzero bias, and
different ways of collecting data and estimation techniques can
exhibit these tradeoffs. We observe, for instance, that the bias of a
(MLE) technique below gets worse, in the ratio sense, as the target
$\tp \ra 0$.

\subsection{A Single Completed Queue Cell}
\label{sec:one_qcell}

With $\c$ denoting random variable (\rvn) for the count in a completed
cell of a queue, in the binary iid setting (\sec \ref{sec:binary} and
\ref{sec:qs}), the expectation of the estimator $X=\frac{1}{\c}$ for
$\tp \in (0, 1]$, which is known to be the maximum likelihood estimator
  (MLE) for the geometric and the more general negative binomial
  distributions \cite{math_stats18,math_stats16,wikip1}, can be
  expressed as the infinite series below: $\tp$ of the time, we get a
  heads outcome on the first toss, and the estimate $\frac{1}{\c}$ for
  $\tp$ is 1, and $(1-\tp)\tp$ of the time, we get a single tails and
  then a heads, and the estimate is $1/2$, and so on.  The series is
  called a geometric-harmonic series, and has a closed form in terms
  of the natural logarithm:
\begin{align*}
\hspace{-.37in}  \expd{X} = \expd{\frac{1}{\c}} = \tp\frac{1}{1} + (1-\tp)\tp\frac{1}{2} + (1-\tp)^2\tp\frac{1}{3} +
\cdots = \tp \sum_{i \ge 1} (1-\tp)^{i-1}\frac{1}{i} =
 \frac{-\tp\log(\tp)}{1-\tp}
\end{align*}

From the series, $\tp + (1-\tp)\tp\frac{1}{2} + \cdots$, it is seen
for $\tp \in (0, 1)$, as all series elements are positive, that
$\expd{\frac{1}{\c}} > \tp$, or the MLE $X=\frac{1}{\c}$ is a
biased, upper bound, estimator of $\tp$ in expectation.  The ratio
$\frac{\expd{\frac{1}{\c}}}{\tp} = \frac{-\log(\tp)}{1-\tp}$, and we
can verify that this ratio grows unbounded as $\tp \ra 0$ ( $\lim_{\tp
  \ra 0}\frac{-\log(\tp)}{1-\tp} = +\infty$). At $\tp=1$,
$\expd{\frac{1}{\c}}=1$ and there is no bias, thus the bias, in a
relative or ratio to $\tp$ sense, gets worse as $\tp$ gets
smaller.\footnote{One can also verify that the derivative of the
ratio, $\frac{-1}{\tp(1-\tp)}-\frac{\log(\tp)}{(1-\tp)^2}$, is
negative with $\tp \in (0, 1)$, therefore the ratio is indeed a
decreasing function as $\tp \rightarrow 1$. }


We now look at the variance, $\var{X}$ of \rv $X=\frac{1}{\c}$, and how
its ratio to $\tp$ changes as $\tp$ is reduced. For any \rv $X$,
$\var{X} := \expd{(X-\expd{X})^2}$, or $\var{X} = \expd{X^2}-
(\expd{X})^2$ (linearity of expectation).

\begin{align*}
\frac{\var{X}}{\tp} = \frac{\expd{X^2}}{\tp} - \frac{(-\tp\log(\tp)/(1-\tp)))^2}{\tp}
\end{align*}

The limit of the 2nd term, $\frac{\tp(\log(\tp))^2}{(1-\tp)^2}$, as
$\tp \ra 0$, is 0, and for the first term:

\begin{align*}
& \expd{X^2} = \tp\frac{1}{1^2} + (1-\tp)\tp\frac{1}{2^2} + (1-\tp)^2\tp\frac{1}{3^2} +
\cdots = \tp\sum_{i\ge 1}(1-\tp)^{i-1}\frac{1}{i^2}\\
& \Rightarrow \lim_{\tp \ra 0} \frac{\expd{X^2}}{\tp} = \lim_{\tp
    \ra 0} \sum_{i\ge 1}(1-\tp)^{i-1}\frac{1}{i^2} = \sum_{i\ge 1}\frac{1}{i^2} = \frac{\pi^2}{6}
\end{align*}

The last step is from the solution to the Basel problem.\footnote{The
Basel problem, solved by Euler, and
named after Euler's hometown, is finding a closed form for sum
of the reciprocals of the squares of the natural numbers:
$\sum_{i\ge 1}\frac{1}{i^2}$ \cite{wikip1}.}
%
Therefore, $\lim_{p \ra 0} \frac{\var{X}}{\tp} = \frac{\pi^2}{6}.$ We
note that at $\tp=1$, the variance is 0 (as $\expd{X}=1$, and also
$\expd{X^2}=1$ from the series above).  So the ratio of variance to
$\tp$, the relative variance, grows but is bounded as $\tp\ra 0$. That
the variance goes up is consistent with the observation that rates of
violation also go up, as $\tp$ is reduced, in Table
\ref{tab:plain_counting} (using a similar estimator), and Table
\ref{tab:stat_devs} on performance of \qu (using 5 and 10 queue cells).



\co{
The ratio to p is pure

The expectation a harmonic-geometric series, sums to:
$-p/(1-p)\ln(p))$ and the ratio to p is thus $-ln(p)/(1-p)$, and this
is an increasing function of p, without bound as -ln(p) -> infinity
(and denominator goes to 1) ...

The variance of the sample estimate, and its ratio to p: $p(sum
(1-p)^(i-1)(1/i)^2)$, and its ratio to p (ratio of variance to the
quantity we are estimating) is, as p goes to 0, or 1-p goes to 1, the
sum of reciprocals of natural numbers squared, which is the Basel
problem, see wikipedia. converges to $pi^2/6$ !! (Euler solved it)..

The plain cell estimates are all upper bounds, due to Jensen's
inequality (for any k-cell queue)

Lower bounds... proof (that it's a lower bound)

some asymptotics
}

%


\co{
\subsubsection{A Single Cell}

Estimating the heads probability $p$ of a weighted (biased) coin.

\[
  E(X) = \mu = p\frac{1}{1} + (1-p)p\frac{1}{2} + (1-p)^2p\frac{1}{3} +
\cdots = p\sum_{i\ge 1}(1-p)^{i-1}\frac{1}{i}
\]

The sum is a harmonic-geometric series.

ratio to $p$, and limit as $p \rightarrow 0$:

\[
\lim_{p\rightarrow 0} \frac{E(X)}{p} = (-p/(1-p)\ln(p))/p = \lim_{p\rightarrow 0} \ln(p)= \infty
\]

And for variance and its ratio to $p$:

\[ VAR(X) = E((X-E(X))^2) = E(X^2)- (E(X))^2 \Rightarrow \frac{VAR(X)}{p} = \frac{E(X^2)}{p} -
\frac{(p/(1-p)\ln(p)))^2}{p}
\]

\[
E(X^2) = p\frac{1}{1^2} + (1-p)p\frac{1}{2^2} + (1-p)^2p\frac{1}{3^2} +
\cdots = p\sum_{i\ge 1}(1-p)^{i-1}\frac{1}{i^2}
\]
\[
\Rightarrow \lim_{p \rightarrow 0} \frac{E(X^2)}{p} = \lim_{p
  \rightarrow 0} \sum_{i\ge 1}(1-p)^{i-1}\frac{1}{i^2} = \sum_{i\ge 1}\frac{1}{i^2} = \frac{\pi^2}{6}
\]

\[
\Rightarrow \lim_{p \rightarrow 0} \frac{VAR(X)}{p} =
\frac{E(X^2)}{p} - \frac{(p/(1-p)\ln(p)))^2}{p} = \frac{\pi^2}{6} - \lim_{p \rightarrow 0} p(\ln(p))^2 = \frac{\pi^2}{6}
\]

}

\subsection{Multiple Completed Cells: Rao-Blackwellization}
\label{app:rb}

We have $k\ge 2$ completed queue cells, with $\c_i$ denoting the count
of cell $i, 1\le i\le k$, or repeating the experiment toss-until-heads
$k$ times. We first sketch the Rao-Blackwellization (RB) derivation
for an unbiased minimum variance estimator, then briefly look at a few
related properties and special cases.

Let the (good) estimator $G_k = \frac{k-1}{(\sum_{1\le i\le k} \c_i) -
  1}$. The RB technique establishes that this is an unbiased estimator,
\ie $\expd{G_k}=\tp$, and moreover it is the minimum variance unbiased
estimator (MVUE) \cite{marengo2021}. We begin with the simple and
unbiased but crude estimator $\Theta_1 = [[\c_1=1]]$, where we are
using the Iverson bracket: $\tp$ of the time $\Theta_1$ is $1$ (the
first toss is heads), and otherwise $\Theta_1 = 0$, thus indeed
$\expd{\Theta_1}=\tp$. This estimator is crude (highly variant) as it
ignores much information such as the other counts $\c_i, i\ge 2$, and
we can use RB to derive an improved estimator from $\Theta_1$. Let
$\Y=\sum_{i=1}^k \c_i$ ($\Y$ is the total number of tosses), $\Y\sim$
NB($k, \tp$), \ie $\Y$ has the negative binomial distributions with
parameters $k$ and $\tp$. $\Y$ is a sufficient statistic for
$\Theta_1$ \cite{marengo2021,wikip1}. Then the RB estimator is the
conditional expectation $\E{\Theta_1|\Y}$, and because $\Theta_1$ is
unbiased, this estimator is also unbiased as conditioning does not
change bias status and can only improve the variance, and as $\Y$ is a
sufficient statistic, this estimator is in fact the MVUE
\cite{rao45,blackwell47,lehman50}, and can be simplified to:


\begin{align*}
\hspace*{-.25in}  \E{\Theta_1|\Y=y} = P(\c_1=1|\Y=y) = \frac{k-1}{y-1} \hspace*{.4cm} \mbox{(Rao-Blackwellization of
     $\Theta_1:= [[\c_1=1]]$)}
\end{align*}

where we know that the last coin toss
is always a heads (from our data collection set up)
and
this leaves $k-1$ heads
and $y-1$ unaccounted-for tosses. Take the case of $k=2$ cells or
heads. The last heads is fixed, and the first one has $y-1$ positions
to pick from, all equally likely, thus $P(\c_1=1|\Y=y) =
\frac{1}{y-1}$ when $k=2$. More generally for $k\ge 2$, as the tosses
are exchangeable with this conditioning, \ie throwing $k-1$ balls into
$y-1$ bins (each bin can contain 1 ball only), the probability that
one falls in bin 1 is $\frac{k-1}{y-1}$.\footnote{One can also look at
the process sequentially, and the probability that the first ball
misses ($\frac{y-2}{y-1}$), but the second ball falls in position 1 is
$\frac{y-2}{y-1}\frac{1}{y-2}$, and so on, or $P(\c_1=1|\Y=y) =
\sum_{j=1}^{k-1}(\frac{y-j}{y-1}\frac{1}{y-j}) = \frac{k-1}{y-1}$.}

\co{
the probability that the first of $k-1$ falls in position
1 is $\frac{1}{y-1}$, while the probability that the first misses,
$\frac{y-2}{y-1}$, but the second falls in position 1 is
$\frac{y-2}{y-1}\frac{1}{y-2}$, and so on, or $P(\c_1=1|\Y=y) =
\sum_{j=1}^{k-1}(\frac{y-j}{y-1}\frac{1}{y-j}) = \frac{k-1}{y-1}$.
}




\co{
where the second equality comes from the fact that the last coin toss
is always a head from the structure of our data
collection, and the final equality from the fact that this leaves
$k-1$ heads spread out randomly amongst the $c-1$ unaccounted-for
tosses; as the tosses are exchangeable, the probability that the first
toss - i.e., $\Theta_1$ equals 1 is just $(k-1)/(y-1)$, which this is
also its expected value.
}

With $k=2$ completed cells (thus $\Y\ge 2$), one can show that the
estimator $\frac{1}{\Y-1}$ is unbiased more directly (but the property
of minimum variance is stronger in above):
\begin{align*}
  \E{\frac{1}{\Y-1}}=\sum_{i\ge 2} \frac{P(\Y=i)}{i-1} = \sum_{i\ge 2} \frac{(i-1)\tp^2(1-\tp)^{i-2}}{i-1}
  =\tp^2\sum_{i\ge 2}(1-\tp)^{i-2} = \tp
\end{align*}
where, $P(\Y=i)=(i-1)\tp^2(1-\tp)^{i-2}$, as there are $i-1$
possibilities for the first heads outcome, each having equal
probability $\tp^2(1-\tp)^{i-2}$, and the last equality follows from
simplifying the sum of the geometric series ($\sum_{i\ge 0} r^i =
\frac{1}{1-r}$, for $r\in(0,1)$).

From the RB estimator, it follows that $\up_k = \frac{k}{Y}$ and
$\lo_k = \frac{k-1}{Y}$ (where $Y := \sum_{i=1}^{k} \c_i$ as in above)
with increasing $k$ form a sequence respectively of upper bounds and
lower bounds, in expectation, for $\tp$: $\expd{\lo_k} < \tp <
\expd{\up_k}$.  That the estimator $\up_k$ is biased positive can also
be seen from an application of Jensen's inequality \cite{cover91},
using linearity of expectation on the sum of \rvs, and the fact that
$\E{\c_i}=\frac{1}{\tp}$ (the mean of the geometric distribution,
which can be verified by writing the expectation expression), as
follows: Jensen's inequality for expectation is $f(\expd{X}) <
\expd{f(X)}$, where the strict inequality holds when these two
conditions are met, 1) \rv $X$ has finite expectation and positive
variance (true, in our case, when $\tp < 1$), and 2) $f(x)$ is
strictly convex. In our case, $f(x)=\frac{1}{x}$ (a strictly convex
function). For one completed cell, we have
$f(\expd{X})=f(\expd{\c_1})=f(\frac{1}{\tp})=\tp$, therefore, using
Jensen's inequality, $\tp <
\expd{f(X)}=\E{\frac{1}{\c_1}}$. Similarly, for $k \ge 2$,
$f(\expd{\Y}) = f(\expd{\sum_{}\c_i}) = f(\frac{k}{\tp}) =
\frac{\tp}{k}$. And $\expd{f(\Y)}=\expd{\frac{1}{\Y}}$, therefore (via
Jensen's), $\tp < k\expd{\frac{1}{\Y}}=\expd{\up_k}$.


\co{
Denote the simple estimate of a queue of $k$ completed cells by
$\up_k = \frac{k}{\sum_{1\le i\le k} c_i}$.

\begin{lemma}
For any $i \ge 1$ (natural numbers), $E(\up_i) > p$, and $E(\lo_i) <
p$.
\end{lemma}

\begin{proof}
  The upper bound is Jensen's inequality...
  
\end{proof}

\begin{itemize}
  \item A queue representing higher prob item, say 0.5, of capacity k,
    say k=4, in effect represents a shorter history of events (on
    average) compared to the same capacity queue representing a lower
    prob item (say p=0.01)..
  \item And the probs obtained from each queue don't form a
    probability distribution.. for instance, they can add up to more
    than 1.0: imagine item A is seen one or more consecutive times,
    then item B is seen a few consecutive times, so B would be
    estimated to be 1, but A's estimated prob is positive too.
  \item We could normalize, and we do that in the experiments (eg
    plain normalize), but that doesn't address the issue (why not?) ..
    NOTE: we can easily see that the less frequent items can be easily
    over estimated...
\end{itemize}

Two techniques to ponder (I haven't been able to make them work or
haven't pursued them..):

\begin{itemize}
  
\item what about multiple EMAs with different learning rates? A high
  learning like 0.5 or 0.1 is very agile/adaptive, but can't do
  justice to steady low prob items. But what if we have multiple such
  EMAs, and the higher-rate EMAs ``handoff'' to lower-rate (more
  accurate but slower) EMAs ..

\item what about queue or some moving average that has fixed horizon
  or history length, eg up to last 100 or 1000 steps, and can handle
  multiclass.. ??
  
\end{itemize}

}

  
\subsection{Multiple Queues: \pr Sums and the Spread of the \prs}
\label{app:qsum_etc}


Unlike EMA, with its particular weakening step, even though the \qu
method also has in effect a weakening step (a negative update), the
\pr estimates from all the queues of a \qu predictor do not form a \di
or even an \sdn, as the sum can exceed 1.0. \sec \ref{sec:qu_sds} gave
an example, and here we delve deeper. Let $\Q$ denote the item to
(raw) \pr from the \qu method, \ie before any normalizing (and at any
time point $t\ge 1$).  When we feed $\Q$ to \fcap(), scaling is
performed to ensure a \sd is extracted. In effect, \fcap() normalizes
by the sum, but how large does $\sm{\Q}$ (sum of the raw \prsn) get,
violating the \sd property, in the worst case?

\subsubsection{MLEs via Single-Cell Queues}

Assume each queue had one cell only, and we used the simple MLE,
$\frac{1}{\c_0}$, and the stream is composed of $n$ unique items (and
with no pruning): then the \pr estimates are 1 (for the latest
observed item), 1/2 (next to latest), 1/3, and so on, and $\sm{\Q}$
has the growth rate of a harmonic series, which for $n$ (unique)
items, is approximately $\log(n)+0.577$ (0.577 is called the
Euler-Mascheroni constant).


While $\sm{\Q}$ can be above 1.0, a related question is about the form
and spread of the \prs in $\Q$. For instance, can $\Q$ contain 3 or
more \prs equal to $\frac{1}{2}$, or 4 or more $\frac{1}{3}$ \prs,
violating the \sd property by having too many equal \prsn?
Given a threshold $p$, \eg $p=\pmin$, let
\begin{align}
N(\Q, p) := |\{i | \Q(i) > p \}|. \mbox{}
\end{align}
$N(\Q, p)$ is
more constrained than
$\sm{\Q}$, as we will see below.
$N(\Q, p)$ is of interest when we want to use the raw \pr
values in $\Q$ from a \qu technique, and do not want to necessarily
normalize $\Q$ at every time $t$, for instance for the efficient
sparse-update time-stamp \qu method when keeping many \prs for
millions of items (\sec \ref{sec:timestamps}).  If the \prs formed a
\sd, then $\frac{1}{p}$ would be the bound.  In the simpler case of
\qu with \qcap$ = 1$ and using the
MLE $\frac{1}{\c_0}$, we can show the same constraint $\frac{1}{p}$
holds, in the next lemma below.  Note that when pruning the $\qmap$ we
are in effect using the MLE with \qcap$ = 1$, thus understanding its
properties is motivated from the pruning consideration as well.

Let the denominator of \pr $\Q(i)$ be denoted by $Y_i$. In the case of
MLE with \qcap$ = 1$, $Y_i$ is $\c_0$, and more generally it is the
sum of cell counts. $\oat{Y_i}$ denotes the value at time $t$.



\begin{lemma}\label{lem:qcap1}
  For the \qu technique with \qcap$ = 1$, using the MLE
  $\frac{1}{\c_0}$, we have the following properties at any time point
  $t\ge 2$ (on any input sequence):
  \begin{enumerate}
  \item $\oat{Y_i}=1$ if $i$ was observed at time $t-1$, and otherwise
    $\oat{Y_i}=\ott{Y_i}{t-1}+1$. 
  \item There is exactly one item $i$ with $\oat{Y_i}=1$ and thus \pr
    $\oat{\Q}(i)=1$, the item observed at $t-1$.  For any integer
    $k\ge 2$, there is at most one item with $\oat{Y_i}=k$ or \pr
    $\oat{\Q}(i)=\frac{1}{k}$.
  \item For any threshold $p > 0$, $N(\Q, p) < \frac{1}{p}$.
  \end{enumerate}
\end{lemma}
\begin{proof}
  Property 1 follows from how the \qu technique allocates new queue
  cells and increments counts: at each time $t$ exactly one item is
  observed, its $\c_0$ initialized to one upon update, it thus gets a
  \pr of $1$ at $t+1$. Any other item $i$ in the map $\Q$ gets its
  $\c_0$ incremented, to 2 or higher, or
  $\oat{Y_i}=\ott{Y_i}{t-1}+1$. Thus exactly one item, the observed
  item at $t-1$, has $\oat{Y_i}=1$.  Property 2 completes using
  induction on $k \ge 1$, the base case is 1st half of property 2, and
  for the induction step, we use part 1: Assuming it holds for all
  integer up to $k \ge 1$, the property for $k+1$ can be established
  by contradiction: if there are two or more items with $\oat{Y}=k+1$,
  these items must have had $\ott{Y}{t-1}=k$ (neither could have been
  observed at $t-1$) contradicting the at-most-one property for $k$.

  From part 2, it follows that the maximum number of \prs $N(\Q,
  p)$ is $k$
for probability threshold $p = \frac{1}{k} > 0$, where $k$ is an
integer, $k \ge 1$ (with at most one \pr $\frac{1}{j}$, for each
$j\in\{1,2,\cdots,k\}$). As all \prs in $\Q$ are limited to integer
$\frac{1}{k}$ fractions (the harmonic numbers), the constraint remains
$\frac{1}{p}$ for any threshold $p > 0$ (not just integer fractions).
\end{proof}





\subsubsection{Some Properties when Estimating via {\bf GetPR}() (MVUEs, Several Cells per Queue) }

We next explore the same question of the number of \pr values, and
related properties, when using more cells per queue, and where we use
the {\bf GetPR}() function in the \qu technique (\fig
\ref{fig:queues}).  Let $\q(i)$ be short for $\qmap(i)$, \ie the queue
for $i$ (when the queue exists). Thus for item $i$, the \pr $\Q(i)$ is
GetPR($\q(i)$).
Here, with several cells, we need to use the count for \czn, $\c_0$,
\eg we use the estimate $\Q(i)=\frac{1}{\c_0+\c_1-1}$ for two
cells. Otherwise, if counts from only completed cells, say $\c_1$ and
$\c_2$ are used (skipping \czn), we can have a worst-case sequence
(with non-stationarity) such as $AAABBBCCC\cdots$ (an item never
occurs after appearing a few times) where the \pr estimates (using
only completed cells) for many items are all $1$. Using the partial
\cz breaks this possibility, and we can establish properties similar
to the previous section.
Let $|\q(i)|$ denote the number of cells in the queue, and define
$Y_i$ to be the denominator used in GetPR(), $Y_i:=\sum_{0\le j <
  |\q(i)|}\c_j-1$, thus $\Q(i)$ is either 0, when there is no queue
for $i$, or otherwise $\oat{\Q}(i) =
\frac{|\oat{\q}(i)|-1}{\oat{Y_i}}$ ($\Q(i)$ can still be 0 when
$|\q(i)|=1$), where $\oat{Y_i}$ is the denominator and $\oat{\q}(i)$
is the queue of $i$ at $t$.



\begin{lemma}\label{lem:qu_prs}
  For the \qu technique with \qcap $ \ge 2$, for any item $i$ with
  a queue $\q(i)$, $|\q(i)| \le$ \qcapn, using the \pr estimate
  $\Q(i)=\frac{|\q(i)|-1}{Y_i}$, where $Y_i:=\sum_{0\le j <
    |\q(i)|}\c_j-1$, for any time point $t\ge 1$: 
  \begin{enumerate}
  \item $\oat{\Q}(i)$, when nonzero, has the form $\frac{a}{b}$, where $a$
    and $b$ are integers, with $b\ge a\ge 1$.
  \item If $i$ is not observed at $t$, then $\ott{Y_i}{t+1}=
    \ott{Y_i}{t}+1$ or $i$ is removed from the map $\qmap$. When $i$
    is observed at $t$ (exactly one such), then
    $\ott{Y_i}{t+1}\le\ott{Y_i}{t}$ when $|\oat{\q}(i)| =$ \qcapn, and
    $\ott{Y_i}{t+1}=\ott{Y_i}{t}+1$ when $|\oat{\q}(i)| < $ \qcapn.
  \item If $i$ is observed at $t$, then $\ott{\Q}{t+1}(i) \ge
    \ott{\Q}{t}(i)$. If $i$ is not observed at $t$, then
    $\ott{\Q}{t+1}(i) < \ott{\Q}{t}(i)$ or $\ott{\Q}{t+1}(i) = \ott{\Q}{t}(i) = 0$.
  \end{enumerate}
\end{lemma}

\begin{proof}
  Part 1 follows from $Y_i$ being an integer, and the \qu
  technique outputs a non-zero \pr only when $|q(i)|>1$, in which case
  $Y_i > 0$ (with two or more queue cells), and we also have $|q(i)|-1
  \le Y_i$ (each queue cell has count of at least 1, therefore, $Y_i +
  1 = \sum_{0\le j < |\q(i)|}\c_j \ge |\q(i)|$ ).
  
  Proof of part 2: if an item $i$ in $\qmap$ is not observed at
  $t$ and removed from $\qmap$, its count in \cz is always incremented
  for $t+1$, \ie $\ott{Y_i}{t+1}= \ott{Y_i}{t}+1$. Upon observing item
  $i$ at $t$, when $|\oat{\q}(i)| =$ \qcap, one cell is dropped with
  count $\ge 1$, and one cell is added, thus $\ott{Y_i}{t+1}\le
  \ott{Y_i}{t}$, and when $|\oat{\q}(i)| <$ \qcap, one cell with
  count 1 is added, thus $\ott{Y_i}{t+1} = \ott{Y_i}{t}+1$.
  
  Proof of part 3: On a negative update, the denominator of $\Q(i)$,
  $Y_i$, always goes up from part 2, while the numerator (or number of
  queue cells) doesn't change. On a positive update, when the queue is
  at capacity, the numerator does not change, while denominator may go
  down (part 2). When the queue hasn't reached capacity ($\oat{\q}(i) <$
  \qcapn), both the numerator and denominator go up by 1 each, but as
  the denominator is never smaller than numerator (part 1), the result
  is $\ott{\Q}{t+1}(i) \ge \ott{\Q}{t}(i)$.
\end{proof}

As property 3 above states, an item's \pr may not change after a
positive update: This happens when the item's queue is at capacity and
the last cell of the queue, to be discarded, has a 1 (one cell with
count of 1 is dropped, but another such is added at the back, and
$Y_i$ is not changed). The other case for no change is when $\Q(i)=1$
already initially and without discarding any cell (which can occur
after several consecutive initial observations of the item).

With \qcap$ > 2$, we can have $N(\Q, p) \ge \frac{1}{p}$. For
instance, with \qcap of 4, on the sequence of $AAAABB$, item $A$
reaches $\Q(A)=1$ at $t=3$ (after the update at $t=2$) through $5$,
and goes down to $3/5$ at $t=7$, while $B$ reaches 1 at the $t=7$,
thus with threshold $p=3/5$, $N(\ott{\Q}{7}, p) = 2 >
\frac{1}{p}=5/3$.

\begin{figure}[t]
\hspace*{-.03cm}   \subfloat[Determining the queue contents for an item, \qcap is 5.]{{\includegraphics[height=3.cm, width=8cm]{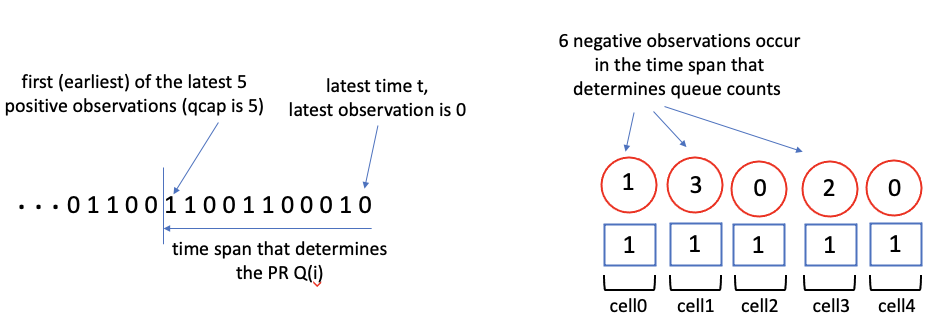} }}
\hspace*{.1in}  \subfloat[A way of generating many high \prs for the non-uniform \qun.]
           {{\includegraphics[height=3.cm, width=7cm]{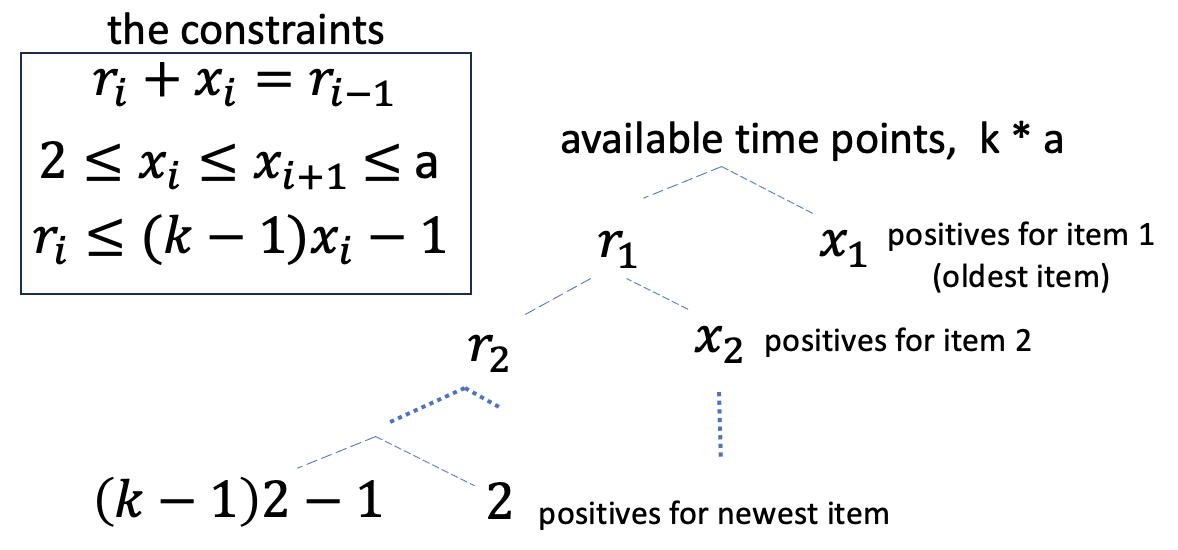} }}
\caption{(a) For the uniform \qu technique, at any time $t$, for any
  item $i$, what determines the queue counts of $i$, in particular the
  negative counts, and therefore the sum $\oat{Y_i}$ and
  $\oat{\Q}(i)$, is the time span $t'$ to $t$, where between $t'$ to
  $t$ (inclusive, or in $[t', t]$) exactly $a$ (\qcapn) positive
  observations occur (the latest $a$ positives). Left: In this example \qcap
  $= a = 5$.
  and an example binary sequence for item $i$ is shown. Here,
  $t'=t-(5+6)+1$, and exactly $a=5$ positives and $6$ negatives occur
  from $t'$ to $t$. Right: The contents of the 5 queue cells at time
  $t$. In the span $[t',t]$, 6 negative observations occur, the count
  in \cz is 2 (one positive, one negative), in cell1 is 3, and so on.
  (b) For non-uniform \qun, a process that generates a sequence
  leading to many items with \pr above $\frac{1}{k}$: for the first or
  oldest (top) item, item 1, allocate $x_1=a$ (or near $a$) positives,
  and remainder $r_1$ negatives, $r_1 \le (k-1)x_1-1$, so that
  $\frac{x_1-1}{x_1+r_1-1} > \frac{1}{k}$. Then, split the remainder
  $r_1$ (the negatives for item 1), into $x_2$ and remainder $r_2 \le
  (k-1)x_2-1$, $x_2$ positives for item 2, and repeat. The constraints
  are that $x_i\ge 2$ and the remainders $r_i$ be as large as possible
  subject to $r_i < (k-1)x_i$, and $x_i+r_i = r_{i-1}$. This process
  is possible in part because, for emitting a positive \prn, we do not
  require exactly $a$ positives ($x_i$ can be less than $a$). With
  high $k$ and $a$, this can lead to many items, roughly
  $\lg_{\frac{k}{k-1}}(ak)$, with \pr above $\frac{1}{k}$. }
  \label{fig:q_spread} 
\end{figure}

Let us call the variant of \qu as the {\em uniform } \qu technique
where for any item $i$, while $|\q(i)| <$ \qcapn, $\Q(i):= 0$, \ie
wait until queue of $i$ has reached its capacity before outputting a
positive \prn.  This variant simplifies the analysis of the worst-case
number of high \prs as we explore below. Note that plain \qu with
\qcap of 2 is already uniform.

For any item $i$, for any time $t$, consider the most recent, or the
last, $a=$ \qcap positive observations of item $i$ and let $t'$ be the
time of the first (earliest) of the last $a$ positives. From the way
the uniform \qu method works, it follows that this time span
determines the queue counts for item $i$, in particular the number of
negatives in $t'$ to $t$ in effect determines the \pr $\oat{\Q}(i)$,
as the number of positives is always $a$. \fig \ref{fig:q_spread}(a)
gives an example for \qcap of 5. We use reasoning about such time
spans to bound the number of items with high \pr:




\begin{lemma}
  \label{lem:q_spread}
  For the uniform \qu technique with $a:= $ \qcap, thus integer $a \ge
  2$, consider items $i$ with $\Q(i) > 0$ (\ie queue $\q(i)$ exists
  and $|\q(i)| = a$), where the \pr estimate
  %
  $\Q(i)=\frac{a-1}{Y_i}$, $Y_i:=\sum_{0\le j < |\q(i)|}\c_j-1$, is
  used.  Then at any time point $t\ge 1$, for any integer $k > 1$,
  $N(\Q, \frac{1}{k}) \le k-1$.
\end{lemma}
\begin{proof}
  At any time $t\ge a$, for an item $i$ (where
  $\Q(i)=\frac{a-1}{Y_i}>0$), go back in time until first time
  $t'=t-\Delta$ where $a$ positives are observed (as it is uniform
  \qun, and $\Q(i) > 0$, we must have observed at least $a$ positives
  from 1 to $t$, \ie $t'$ is well defined).  \fig
  \ref{fig:q_spread}(a) shows an example for $a=5$. Let $c$ be the count
  of the negatives (\ie when item $i$ is not observed) in this time
  span, \ie $[t-\Delta, t]$ (from $t-\Delta$ to $t$ inclusive). We
  have $Y_i$ is $c+a-1$ ($a$ positives and $c$ negatives), and
  $\oat{\Q}(i)=\frac{a-1}{c+a-1}$.

  We will show that we have to upper bound $c$, call it $\bar{c}$, if
  we want $\Q(i)$ to be sufficiently high, and the bound $\bar{c}$ in
  turn upper bounds the length of the time span $\Delta$, \ie $t'$ to
  current time $t$ (how much into the past we can go). As each
  such item requires $a$ positives (positive observations) in the same
  span, the maximum span is $\bar{c}+a$ for any such item, or
  $[t-(\bar{c}+a), t]$, and at any time point we get only one positive
  observation, we deduce we cannot ``fit'', or have too many such high
  \pr items in the same time span.
  
  For example, for threshold $p=\frac{1}{2}$ (showing $N(\Q,
  \frac{1}{2}) \le 1$), we must have $\bar{c} \le a-1$, as if $c \ge a$,
  then $\oat{\Q}(i)\le \frac{a-1}{a+a-1} = \frac{a-1}{2a-1} <
  \frac{a}{2a}$ (for the last, we used $1< a < 2a$). Therefore, any
  such item has $a$ positives in the last $\Delta \le 2a-1$ time
  points. We can have at most one item with $a$ positives in that span
  (leaving $a-1$ positives for all other items implying for any other
  item, its negative count $c$ is at least $a$ or its \pr at $t$ is $\le 1/2$).

  More generally, for any item $i$ with $\Q(i) > \frac{1}{k}$, $k\ge
  2$, then we must have $c\le \bar{c}=(k-1)a-1$ (if $c \ge (k-1)a$,
  $\Q(i)\le \frac{a-1}{(k-1)a+a-1} \le \frac{1}{k}$), or the maximum
  span $\Delta = ka-1$ for all such items. Any such item, within the
  same span of at most $[t-ka+1, t]$ requires $a$ positives. There can
  be at most $k-1$ such items.
\end{proof}


When \qcap $= 2$, like the case of Lemma \ref{lem:qcap1} (\ie using
one cell and the MLE), the (positive) \prs in $\Q(i)$ are the harmonic
fractions, but there is more variation here compared to the case of
MLE with \qcap of 1. For instance, it is possible that no item gets
\pr of 1 at certain time points $t\ge2$, while we can have two items
with \pr $\frac{1}{2}$, and in general, up to $k$ items with \pr
$\frac{1}{k}$, at which case we get a perfect \di (the sum adds to
1). In other cases, such as the sequence $AABBCCDD\cdots$, the sum can
exceed 1 substantially.  The underlying cause for violating the \di
property is that different items have different starting time points
(to keep a bounded memory while reacting to new items, and
non-stationarities), thus one recent item, using its own frame of
reference (starting time point), can have \pr of 1 (item $D$ in the
example $[AABBCCDD]$), while simultaneously, another item, seen
earlier in time, has a positive \pr such as $\frac{1}{2}$ for $C$.

As expected, with larger \qcap $> 2$, the fractions can be more
granular: There can be at most 1 \pr greater than $\frac{1}{2}$ at any
time point (from above Lemma), and with \qcap $=2$, the only
possibility is 1.0, but with \qcap $= 3$, both $1.0$ and $\frac{2}{3}$
are possible.

With the non-uniform \qun, there can still be at most one item with
\pr of 1, but the non-uniform \qu allows for additional degrees of
freedom for integer $k > 1$: \fig \ref{fig:q_spread}(b) shows a way of
building a sequence that results in a large set $S$ of items with \pr
above $\frac{1}{k}$.  For simplicity we can assume the occurrence of
items in $S$, to be described, forms the whole sequence (though the
actual sequence could be longer going further into the past with items
not in $S$).  First take the item with the highest number of positives in
$S$, call it item 1: item 1 will have $x_1=a$ positives in the
sequence and the largest possible negatives $r_1$, such that
$\frac{x_1-1}{x_1+r_1-1} > \frac{1}{k}$ (so roughly $r_1 <
(k-1)x_1$). We can think of the positives of item 1 appearing together
and before the $r_1$ negatives, the occurrence of other items of $S$.
\footnote{Though only the oldest occurrence of item 1 needs to occur
before any other item in $S$, so that presence of other items are
counted as negatives for item 1. A similar consecutive property can be
assumed for other items of $S$.}  The negatives for item 1, $r_1$ is
then split into $x_2$ ($2\le x_2 \le x_1$), $x_2$ being the count of
positives for item 2, and remainder $r_2$ (count of negatives for item
2), and again we could assume positives of item 2 appear before its
negatives (and after item 1's positives).  In general, the $x_i$
(count of positives for item $i$) should be no less than 2, and
$x_i+r_i=r_{i-1}, i\ge 2$. We want each $r_i$ to be as large as
possible, or $r_{i+1}$ be not much smaller than $r_i$, so we can fit
many items in $S$. Thus $x_i$, $i\ge 2$ should be as small as
possible, subject to $\frac{x_i-1}{r_i+x_i-1}>1/k$ and $x_i \ge 2$.
Thus each $r_{i+1} \approx \min(\frac{k-1}{k}r_i, r_i-2)$.

For instance with $k=2$ and $a=8$, we can get $x_1=8$ and $r_1=6$
($\frac{7}{13}>1/2$), and $r_1$ is split into $x_2=4, r_2=2$, and
finally $x_3=2, r_3=0$ ($|S|=3$).  The corresponding sequence would
be: $[11111111222233]$ (item 1 appears 8 times, then item 2 and item
3).  More generally, we can get roughly $\lg_2{a}$ items with \pr
exceeding $\frac{1}{2}$ when $k=2$. Higher $k$ (and capacity $a$)
leads to more items, up to $\lg_{b}{(k-1)a}$, where base
$b=\frac{k}{k-1}$. \footnote{The size $|S|$ will be less than
$\lg_{k/(k-1)}{(k-1)a}$, as $r_{i+1} := min(\frac{k-1}{k}r_i, r_i-2)$
(we need to allocate at least $k=2$ to $x_{i+1}$)}.



\section{An Efficient Time-Stamp Variant of \qs and Queue-Less \qd }
\label{app:qvars}


We describe the slightly different time-stamp method, which allows for
a more efficient $O(1)$ updating for \qun.  We then discuss
alternative (queue-less) possibilities for \qd (\sec \ref{sec:emas}).


\co{
We describe two variants of queuing, the slightly different time-stamp
method, which allows for a more efficient $O(1)$ updating, and the
multiclass box (fixed window) predictor, a static baseline akin to
fixed-rate EMA.
}

%

\subsection{A Time-Stamp Method for Sparse Updates}
\label{sec:timestamps} 

\co{
Another queue approach, again appropriate for nonstationarity, is what
we will referee to as the {\em time-stamp method}. But this works iff
we have a global clock and our predictor has to update on every tick
of the global clock.  Thus this method does not work for the above
scenario where we have many predictors, but only a small subset are
active (\ie are predicting and need to update) in an episode.
}


A small change to the values kept in the queues of the \qu method
makes the updating more efficient. In this {\em time-stamp} variant,
\fig \ref{fig:time_stamp}, each predictor also keeps a single counter,
or its own private {\em clock}.  Upon an update, the predictor
increments its clock. Nothing is done to the queues of items not
observed (thus, yielding a sparse efficient update). For the observed
item, a new \cz is allocated, similar to the case of plain \qun, and
the current
clock value is copied into it (instead of being initialized to 1,
which was the case for plain \qun). Existing cells of this queue, if
any, are shifted right, as before. Thus each queue cell simply carries
the value of the clock at the time the cell was allocated, and the
difference between consecutive cells is in effect the count of
negative outcomes (the gaps), from which proportions can be derived.
When predicting,
any item with no queue or a single-cell queue in $\qmap$, gets 0 \prn,
as before. An item with more than one cell gets a positive \prn, using
a close variant of the GetPR() function of \fig \ref{fig:queues} (\fig
\ref{fig:time_stamp}). The count corresponding to a queue is the
current clock value minus the clock value of \ck (the oldest queue
cell). The count is guaranteed to be positive, and is equivalent to
the denominator used in GetCount().  Then the \pr is the ratio of
number of cells in the item's queue, $k$, to its count.

This variation has the advantage that updating is $O(1)$ only with the
mild assumption that \qu operations and counting take constant time
(\ie incrementing
clock and the operations on a single item's
queue), while for the plain \qu technique, recall that it is
$O(|\qmap|)$. In this sense, this extension is strictly
superior.\footnote{The time-stamps can get large, taking log of stream
length. One can shift and reset periodically (\sec
\ref{sec:overflow}).} However, during prediction, if that involves
predicting all items' \prs in the map, complexity remains
$O(|\qmap|)$, so within the online cycle of predict-observe-update,
time costs would not change asymptotically.  However, the time-stamp
variant is useful in scenarios where, instead of asking for all \prsn,
one queries for the probability of one or a few items only, \eg when
one seeks the (approximate non-stationary) prior of a few items. In
particular, when the number of items for which one seeks to keep a
prior for is large, \eg millions (thus, we are interested in
estimating small probabilities, such as $p \le 10^{-6}$), a sparse
update is attractive.

\subsubsection{Overflow (of Counts)}
\label{sec:overflow}
Incrementing the clock with each update can lead to large clock values
(log of stream length) and require significant memory for a predictor
that frequently updates. One can manage such, and in effect limit the
window size and the prediction capacity, \ie the number of items for
which an approximate \pr is kept at any point, by periodically pruning
and shifting the update count value to a lower value.




\co{
On the other hand, this method is specially suited for computing the
top-level (non-stationary) priors of concepts. Unlike above, the
number of items (predictands) can be large here (while the number of
predictors is small), but only a small number of them are seen in a
given episode: for such observed items, their proportions need to get
updated.

}
\co{
that's incremented each time
the predictor updates (and no other time, so basically the predictor
updates on every clock tick), and this queue method works better with
a universe of many many items: the predictor updates on every tick,
but does sparse updates, ie only a few (from millions) need to be
updated in each episode:

old: Another queue approach, again appropriate for
nonstationarity.. however, works iff we have a clock that's
incremented each time the predictor updates (and no other time, so
basically the predictor updates on every clock tick), and this queue
method works better with a universe of many many items: the predictor
updates on every tick, but does sparse updates, ie only a few (from
millions) need to be updated in each episode:
}

\co{
Description: Each queue cell, for a single item, keep tracks of the
time (global clock value) it was created, and need not keep count: the
number of negatives is the current time minus its time, if its the
newest cell (back of the queue), and otherwise the time of the cell
added immediately after it minus its time. Again, we can ignore the
partial most-recent queue cell, and count estimates from others can be
pulled together, like before, and the numerator is the number of
complete queue cells.
}

\begin{figure}[t]
\hspace*{.1in}\begin{minipage}[t]{0.5\linewidth}
{\bf UpdateTimeStampQs}($\qmap, o$)  \\
\hspace*{0.3cm}  // Each predictor has its own clock (initially 0).\\ 
\hspace*{0.3cm}  // Update the clock.\\ 
\hspace*{0.3cm}  $clock \leftarrow clock+1$ \\
\hspace*{0.3cm} If item $o \not \in \qmap$: // when $o \not \in \qmap$, insert.\\
\hspace*{0.6cm}   // Allocate \& insert q for $o$. \\
\hspace*{0.6cm}   $\qmap[o] \leftarrow Queue()$ \\ 
\hspace*{0.3cm} // Exactly one positive update. \\
\hspace*{0.3cm} // (No explicit negative update) \\
\hspace*{0.3cm} PositiveUpdate(\qmap[o]) \\ \\ 
\end{minipage}
\hspace{0.3cm}
\begin{minipage}[t]{0.55\linewidth}
  \fontsmall{
        {\bf PositiveUpdate($q$) } // Adds a new (back) cell. \\ 
\hspace*{0.25cm} // Existing cells shift one position. Oldest cell \\
\hspace*{0.25cm} // is discarded, in when $q$ is at capacity. \\
        \hspace*{0.3cm}If $q.nc < q.qcap$:\\
        \hspace*{0.7cm}         $q.nc \leftarrow q.nc + 1$ // Grow the queue $q$.\\
        \hspace*{0.3cm}For $i$ in [1, $\cdots, q.nc-1$]: // Inclusive.\\
        \hspace*{0.7cm} $q.cells[i] \leftarrow q.cells[i - 1]$ // shift contents. \\
        \hspace*{0.3cm} $q.cells[0] \leftarrow clock$ // set newest, $\cz$, to clock. \\ \\
        {\bf GetPR}($q$) // Extract a \prn, using $q.nc$  \\
\hspace*{.3cm}         // (number of cells) and the time-stamps. \\
\hspace*{0.3cm}  If $q.nc \le 1$: Return 0 // Too few? (grace period).\\
\hspace*{0.3cm}  Return $\frac{q.nc - 1}{clock - q.cells[q.nc-1]}$ \\ \\ 
  }
\end{minipage}
\caption{Time-stamp queuing, with $O(1)$ updates, showing those
  function that are different from plain \qu technique of \fig
  \ref{fig:queues}: queue fields, \eg $q.nc$, and the {\bf Queue()}
  functions remain the same. There is no {\bf NegativeUpdate()}
  here. Others such as {\bf GetCount()} could be adapted too.  }
\label{fig:time_stamp}
\end{figure}


\subsection{Discussion: Why not multiple EMAs with different rates?}
\label{sec:emas}


Why can't we use two (or more) tiers of EMA, one with high fixed rate,
so agile and adaptive, another with a low and fixed rate, so stable
for fine tuning, rather than the current hybrid two-tier approach in
\qd that also uses queues and predictand-specific $\lr$s?  Queues are
an effective way to differentiate between noise items (below
$\minprob$) from salient items (with good likelihood of success). On
the other hand, EMA with a high fixed rate $\lr$ is too coarse
(non-differentiating) for learning \prs below the $\lr$.

It may be possible to use multiple {\em non-fixed} rate, ``scanning''
EMAs, where the learning rate for each scanner is repeatedly reset to
high but decreased over time (such as via harmonic decay), and using
more than one rate per predictand, to achieve similar goals (change
detection and convergence). We leave exploring this to future work.

\co{
With EMA at a high rate (\eg 0.1), every time an item is seen, its new
\pr is the rate (0.1) even if it is noise (should be 0) or
substantially lower eg 0.02. EMA with a high-rate is too coarse
(non-differentiating) for small probabilities.  With queues, for noise
and non-noise items, we can get a substantially better and nearly
unbiased initial \pr estimate, to base our decision on.  For plain
EMA, initially there is no way of distinguishing the learning rate
from the \pr estimate: initially the two are the same and once the
item is observed more, the \pr may separate and become bigger, and/or
$\lr$ may be lowered (with harmonic decay). So if we set EMA to a high
rate $\lr$, for it to be agile and detect new items or changes, it
fails to estimate and differentiate well for items with \pr lower than
the $\lr$ rate (whether \ns or not).
}



\section{Additional Synthetic Experiments} 
\label{sec:app_syn}

\subsection{Tracking One Item}


Table \ref{tab:devs_nonst_other_paparams} reports under the same
synthetic setting of Table \ref{tab:devs_non1}, \ie tracking the single item
1 in binary sequences, with a few different
parameter values:
deviation-rates with $d=1.1$ (convergence to within 10\%), \qu with
\qcap of 3 (as \qd uses this as its default), and \qdn with lower
rates.  As expected, $d=1.1$ yields high deviation rates, and
$\minobs$ of 50 helps lower the rate for any $d$ and method. Low
minimum rates for \qd (0.0001) help it achieve lowest deviation rates
for $d=1.1$.


\begin{table}[t]  \center
  \begin{tabular}{ |c?c|c|c?c|c|c?c|c|c| }     \hline
    deviation thresh. $\rightarrow$ & 1.1& 1.5& 2 & 1.1 & 1.5 & 2 & 1.1 & 1.5 & 2   \\ \medhline
    & \multicolumn{3}{c?}{\qun, 3 }
    & \multicolumn{3}{c?}{\qdn, 0.1} &  \multicolumn{3}{c|}{\qdn, 0.0001}  \\  \hline
    [0.025, 0.25], 10 & 0.906  & 0.542  & 0.277 & 0.904 & 0.637 & 0.496 & 0.864  & 0.591  & 0.408  \\ \hline
    [0.025, 0.25], 50 & 0.907  & 0.539  & 0.265 & 0.900 & 0.625 & 0.476 & 0.529  & 0.167  & 0.088  \\ \hline
    $\mathcal{U}(0.01,1.0)$, 10 & 0.877  & 0.522  & 0.282  & 0.894  & 0.615  & 0.463 & 0.890  & 0.562  & 0.351  \\ \hline
    $\mathcal{U}(0.01,1.0)$, 50 & 0.873  & 0.507  & 0.263  & 0.849  & 0.563  & 0.424 & 0.750  & 0.303  & 0.156  \\ \medhline 
    & \multicolumn{3}{c?}{Static EMA, 0.1 }
    & \multicolumn{3}{c?}{Static EMA, 0.01} &  \multicolumn{3}{c|}{\qdn, 0.01}  \\  \hline
    [0.025, 0.25], 10  & 0.898 & 0.619 & 0.469 & 0.827  & 0.510  & 0.357  & 0.813  & 0.481  & 0.329  \\ \hline
    [0.025, 0.25], 50 & 0.896 & 0.604 & 0.451 & 0.680  & 0.257  & 0.128  & 0.673  & 0.247  & 0.126  \\ \hline
  $\mathcal{U}(0.01,1.0)$, 10 & 0.892  & 0.608  & 0.446 & 0.926  & 0.684  & 0.478  & 0.896  & 0.583  & 0.360  \\ \hline
  $\mathcal{U}(0.01,1.0)$, 50 & 0.847  & 0.543  & 0.404 & 0.830  & 0.400  & 0.206  & 0.786  & 0.315  & 0.143  \\ \hline 
  \end{tabular}
  \vspace*{.2cm}
  \caption{As in \sec \ref{sec:syn_non_one}, for an additional
    evaluation setting ($d=1.1$) and algorithm parameters.}
\label{tab:devs_nonst_other_paparams}
\end{table}

Table \ref{tab:devs_non_1000ticks} repeats the uniform-setting
experiments of Table \ref{tab:devs_non1}, but with the extra
constraint that each stable period be at least 1000 time points.

\begin{table}[t]  \center
  \begin{tabular}{ |c?c|c?c|c| }     \hline
    deviation threshold $\rightarrow$ &  1.5 & 2 & 1.5 & 2   \\ \medhline   
    & \multicolumn{2}{c?}{\qun, 5} &  \multicolumn{2}{c|}{\qun, 10}  \\ \hline
    $\mathcal{U}(0.01,1.0)$, 10 & 0.226 $\pm$ 0.036  & 0.064 $\pm$ 0.023  & 0.103 $\pm$ 0.035  & 0.021 $\pm$ 0.020  \\ \hline
  $\mathcal{U}(0.01,1.0)$, 50 & 0.233 $\pm$ 0.049  & 0.052 $\pm$ 0.022  & 0.102 $\pm$ 0.029  & 0.023 $\pm$ 0.022  \\ \medhline 
    & \multicolumn{2}{c?}{ static EMA, 0.01 } &  \multicolumn{2}{c|}{static EMA, 0.001 } \\ \hline
 $\mathcal{U}(0.01,1.0)$, 10 & 0.093 $\pm$ 0.045  & 0.041 $\pm$ 0.023  & 0.487 $\pm$ 0.114  & 0.271 $\pm$ 0.117  \\ \hline
    $\mathcal{U}(0.01,1.0)$, 50 & 0.099 $\pm$ 0.051  & 0.043 $\pm$ 0.028  & 0.462 $\pm$ 0.149  & 0.311 $\pm$ 0.164  \\ \medhline
    & \multicolumn{2}{c?}{ Harmonic EMA, 0.01 } &  \multicolumn{2}{c|}{Harmonic EMA, 0.001 } \\ \hline
$\mathcal{U}(0.01,1.0)$, 10 & 0.094 $\pm$ 0.052  & 0.038 $\pm$ 0.018  & 0.413 $\pm$ 0.125  & 0.237 $\pm$ 0.105  \\ \hline
$\mathcal{U}(0.01,1.0)$, 50 & 0.125 $\pm$ 0.084  & 0.050 $\pm$ 0.028  & 0.364 $\pm$ 0.161  & 0.241 $\pm$ 0.144  \\ \medhline 
  & \multicolumn{2}{c?}{ \qdn, 0.01 } &  \multicolumn{2}{c|}{\qdn, 0.001 } \\ \hline
$\mathcal{U}(0.01,1.0)$, 10 & 0.048 $\pm$ 0.023  & 0.025 $\pm$ 0.024  & 0.100 $\pm$ 0.039  & 0.039 $\pm$ 0.034  \\ \hline
$\mathcal{U}(0.01,1.0)$, 50 & 0.074 $\pm$ 0.060  & 0.022 $\pm$ 0.031  & 0.084 $\pm$ 0.044  & 0.036 $\pm$ 0.019  \\ \hline 
  \end{tabular}
  \vspace*{.2cm}
  \caption{Repeating the uniform setting of Table \ref{tab:devs_non1}
    but with the extra constraint that each stable period be at least
    1000 time points.  The deviation rates are averages over 20
    sequences.  Deviation rates substantially improve (compared to
    Table \ref{tab:devs_non1}), because the stable periods are
    uniformly longer. }
\label{tab:devs_non_1000ticks}
\end{table}

\co{

\begin{table}[t] 
  \begin{tabular}{ |c|c|c|c|c|c|c|c|c|c| }     \hline
    deviation & \multicolumn{3}{c|}{\qun,  3 }
    & \multicolumn{3}{c|}{\qdn, 0.1} &  \multicolumn{3}{c|}{\qdn, 0.0001}  \\
    \cline{2-10}
    threshold $\rightarrow$ & 1.1& 1.5& 2 & 1.1 & 1.5 & 2 & 1.1 & 1.5 & 2   \\ \hline  
  [0.025, 0.25], 10 & 0.906  & 0.536  & 0.264 & 0.898  & 0.616  & 0.471  & 0.703  & 0.331  & 0.204  \\ \hline

  [0.025, 0.25], 50 & 0.902  & 0.551  & 0.288 & 0.928  & 0.723  & 0.607 & 0.622  & 0.180  & 0.123  \\ \hline
  Uniform, 10 & 0.810 & 0.369 & 0.152 & 0.624 & 0.198 & 0.105 & 0.632 & 0.263 & 0.100 \\ \hline
  Uniform, 50 & 0.816 & 0.372 & 0.162 & 0.634 & 0.226 & 0.131 & 0.611 & 0.259 & 0.098 \\ \hline
  & \multicolumn{3}{c|}{ static EMA, 0.1 } & \multicolumn{3}{c|}{ static EMA, 0.01 } &
  \multicolumn{3}{c|}{\qdn, 0.01 } \\ \hline
  [0.025, 0.25], 10 & 0.893  & 0.597  & 0.443 & 0.708  & 0.313  & 0.184  & 0.703  & 0.296  & 0.171  \\ \hline
  [0.025, 0.25], 50 & 0.921  & 0.697  & 0.565 & 0.758  & 0.340  & 0.177  & 0.752  & 0.329  & 0.169  \\ \hline
  Uniform, 10 & 0.620 & 0.199 & 0.102 & 0.364  & 0.092  & 0.049   & 0.338  & 0.057 & 0.023 \\ \hline
  Uniform, 50 & 0.643 & 0.225 & 0.127 & 0.394  & 0.110  & 0.058   & 0.355  & 0.074 & 0.032 \\ \hline
  \end{tabular}
  \vspace*{.2cm}
  \caption{Like in paper section \ref{}, 0.025 $\leftrightarrow$
    0.25, for several other params: queue with capacity 3 (since we
    use it, as default, in qdyal), qdyal with a few more learning
    rates, and also static, all with 1.1 deviation too, and also, not
    showing stds.}
\label{tab:devs_non3}
\end{table}

}

\subsection{Multiple Items, $\maxprob=0.1$}
\label{app:syn_multi}



\begin{figure}[t]
\begin{center}
  \centering
  \subfloat[loss vs binomial tail in \qd]{{\includegraphics[height=5cm,width=6.5cm]{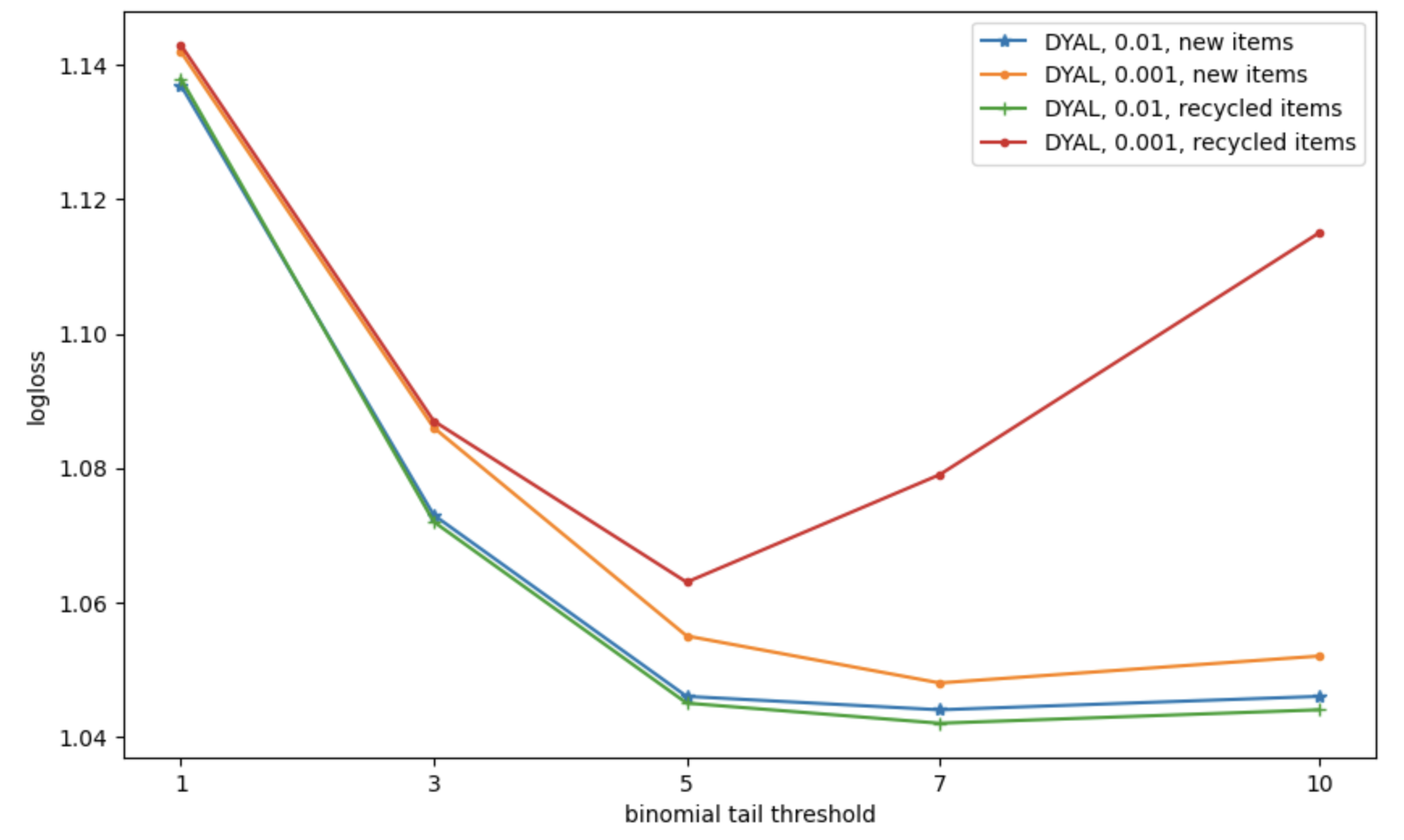}
  }}
  \subfloat[loss vs queue \qcap in \qd]{{\includegraphics[height=5cm,width=6.5cm]{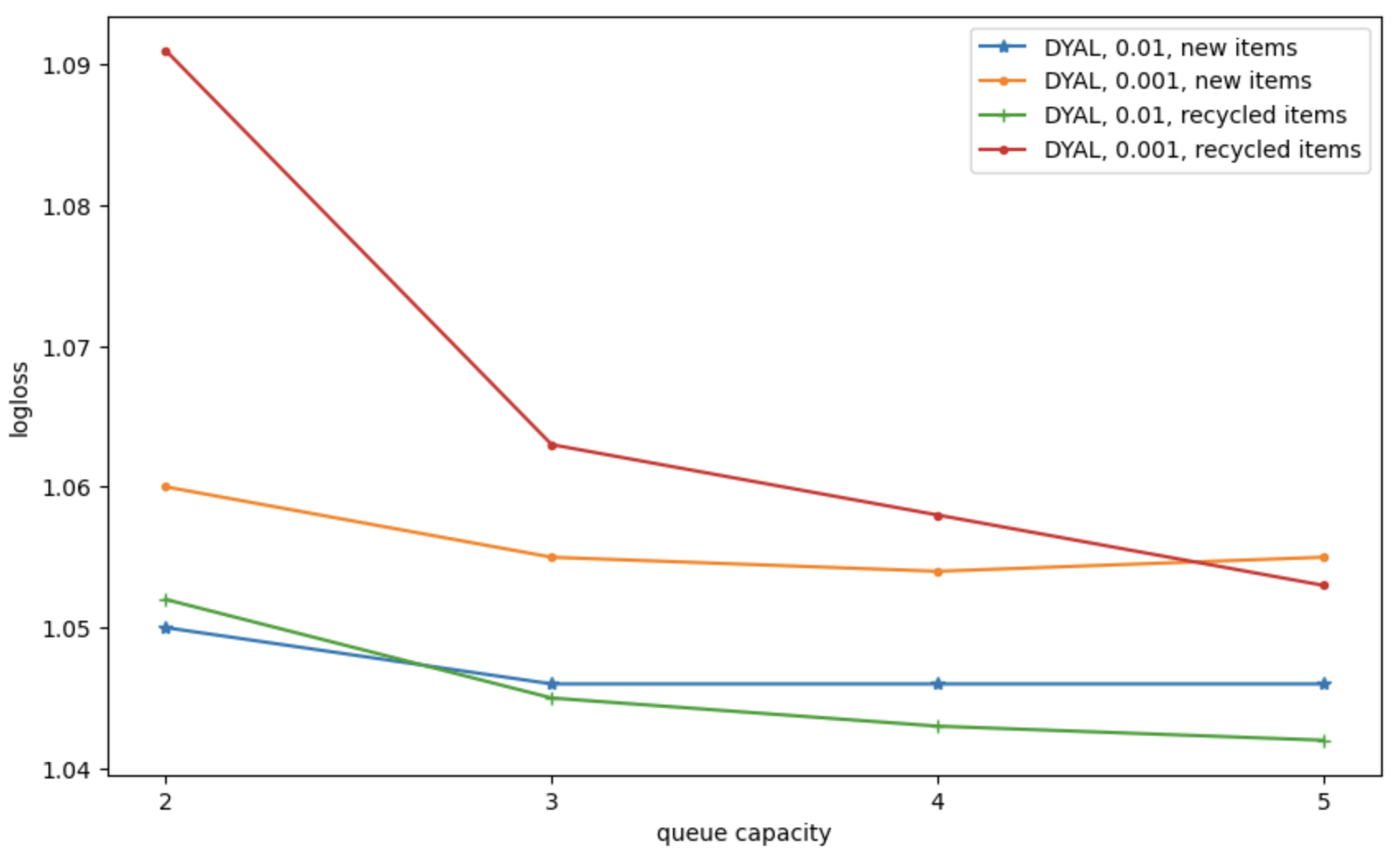}
  }}
\end{center}
\vspace{.2cm}
\caption{Sensitivity of \qd to parameters, in synthetic multi-item
  setting: (a) binomial tail threshold (controlling when to switch to
  queues) (b) the queue capacity \qcap. These are same 50 sequences as
  described in Table \ref{tab:devs_multi_non_v3}, under both new-item
  and recycle generation settings, and with $\minobs=50$). }
\label{fig:syn_binom_thrsh}
\end{figure}


Table \ref{tab:devs_multi_non2} presents multi-item synthetic experiments with
$\maxprob=0.1$ when generating the underlying true $\P$ (GenSD()).
Average number of different \sdn s, generating a sequence, in these
experiments is around 3 with $\minobs = 50$, and around 13 with
$\minobs = 10$.  We observe similar patterns to Table
\ref{tab:devs_multi_non_v3} in that \qd is less sensitive to the
choice of $\minlr$, and is competitive with the best of other EMA
variants. Compared to Table \ref{tab:devs_multi_non_v3}, with overall
lower \prs for the salient items and in a narrower range (0.01 to
0.1), the \logloss performances are worse, and static EMA with
$\lr=0.01$ is more competitive with \qdn.  Paired experiments and
counting the number of wins show that \qd with $\lrmin=0.01$ beats the
best of others convincingly, in Table \ref{tab:nwins_multi_non1},
while \qd with $\lrmin=0.001$ struggles compared to other EMA variants
with $\lr=0.01$.


\begin{table}[t]  \center
  \begin{tabular}{ |c?c|c|c?c|c|c?c| }     \hline
     & 1.5any & 1.5obs & logloss  & 1.5any & 1.5obs & logloss & opt. loss     \\ \thickhline   
new items $\downarrow$  & \multicolumn{3}{c?}{\qun, 5 } &  \multicolumn{3}{c?}{\qun, 10} &   \\ \hline
10, [0.01, 0.1]  & 1.00  & 0.46  & 3.02 & 1.00  & 0.40  & 3.05 & 2.84 $\pm$0.02  \\ \hline
50, [0.01, 0.1]  & 1.00  & 0.41  & 2.94 & 0.99  & 0.24  & 2.91 & 2.83 $\pm$0.04  \\ \medhline
  & \multicolumn{3}{c?}{static EMA, 0.01 } &  \multicolumn{3}{c?}{static EMA, 0.001} &   \\ \hline
10, [0.01, 0.1] & 1.00  & 0.28  & 3.02 & 1.00  & 0.75  & 3.68 & 2.84   \\ \hline
50, [0.01, 0.1] & 0.99  & 0.19  & 2.89 & 0.73  & 0.22  & 3.07 & 2.83  \\ \medhline
  & \multicolumn{3}{c?}{harmonic EMA, 0.01 } &  \multicolumn{3}{c?}{harmonic EMA, 0.001} &   \\ \hline
10, [0.01, 0.1]  & 1.00  & 0.28  & 3.02 & 1.00  & 0.75  & 3.70 & 2.84   \\ \hline
50, [0.01, 0.1]  & 0.99  & 0.19  & 2.89 & 0.71  & 0.22  & 3.07 & 2.83  \\ \medhline
   & \multicolumn{3}{c?}{\qdn, 0.01 } &  \multicolumn{3}{c?}{\qdn, 0.001} &   \\ \hline
10, [0.01, 0.1] & 1.00  & 0.28  & 3.01 & 0.98  & 0.27  & 3.11 & 2.84  \\ \hline
50, [0.01, 0.1]  & 1.00  & 0.20  & 2.88 & 0.73  & 0.07  & 2.88 & 2.83 \\ \thickhline  
recycle items $\downarrow$ & \multicolumn{3}{c?}{\qun, 5} &  \multicolumn{3}{c?}{\qun, 10} &   \\ \hline
10, [0.01, 0.1]  & 1.00  & 0.41  & 2.95 & 0.99  & 0.26  & 2.93 & 2.84 $\pm$0.02  \\ \hline
50, [0.01, 0.1]  & 1.00  & 0.40  & 2.94 & 0.99  & 0.21  & 2.89 & 2.84 $\pm$0.04  \\ \medhline 
  & \multicolumn{3}{c?}{static EMA, 0.01 } &  \multicolumn{3}{c?}{static EMA, 0.001} &   \\ \hline
10, [0.01, 0.1]   & 1.00  & 0.21  & 2.90 & 1.00  & 0.33  & 2.98 & 2.84  \\ \hline
50, [0.01, 0.1]   & 1.00  & 0.18  & 2.88 & 0.70  & 0.11  & 2.88 & 2.84  \\ \medhline 
  & \multicolumn{3}{c?}{harmonic EMA, 0.01 } &  \multicolumn{3}{c?}{harmonic EMA, 0.001} &   \\ \hline
10, [0.01, 0.1]   & 1.00  & 0.21  & 2.90 & 1.00  & 0.32  & 2.98 & 2.84   \\ \hline
50, [0.01, 0.1]   & 1.00  & 0.18  & 2.88 & 0.67  & 0.10  & 2.89 & 2.84  \\ \medhline
    & \multicolumn{3}{c?}{\qd, 0.01 } &  \multicolumn{3}{c?}{\qd, 0.001} &   \\ \hline
10, [0.01, 0.1]  & 1.00  & 0.22  & 2.90 & 1.00  & 0.30  & 2.95 & 2.84   \\ \hline
50, [0.01, 0.1]   & 1.00  & 0.18  & 2.88 & 0.65  & 0.08  & 2.87 & 2.84  \\ \hline
   \end{tabular}
  \vspace*{.2cm}
  \caption{Here, $\maxprob=0.1$ for GenSD() (instead of 1.0), and
    otherwise, as in Table \ref{tab:devs_multi_non_v3}: averages over
    50 sequences, about $\sim$10k length each, $\minobs$ of 10 and 50 (1st column
    of the table), uniform sampling of $\tp$ in [0.01, 0.1], and
    changing the \sd whenever {\em all items} are observed $\minobs$
    times. }
\label{tab:devs_multi_non2}
\end{table}


\begin{table}[t]  \center
  \begin{tabular}{ |c?c|c|c?c|c|c| }     \hline
\qd \vs $\ra$   & \qun, 10 & static, 0.01 & harmonic, 0.01 & \qun, 10 & static, 0.01 & harmonic, 0.01
    \\ \thickhline
 \qd, 0.01 $\downarrow$  & \multicolumn{3}{c?}{new items } & \multicolumn{3}{c|}{recycle items }   \\ \hline
 10, 0.01, [0.01, 0.10] & 0, 50 & 0, 50  & 0, 50 & 0, 50  & 0, 50 & 0, 50  \\ \hline
 50, 0.01, [0.01, 0.10] & 0, 50 & 0, 50 & 0, 50 & 0, 50  & 0, 50 & 0, 50 \\ \hline
 \thickhline 
 \qd, 0.001 $\downarrow$  & \multicolumn{3}{c?}{new items } & \multicolumn{3}{c|}{recycle items }   \\ \hline
 10, 0.01, [0.01, 0.10] & 50, 0  & 50, 0  & 50, 0 & 50, 0 & 50, 0 & 50, 0  \\ \hline
 50, 0.01, [0.01, 0.10] & 0, 50  & 43, 7  & 43, 7 & 0, 50  & 10, 40 & 7, 43 \\
 \thickhline 
   \end{tabular}
  \vspace*{.2cm}
  \caption{Number of losses and wins in 'paired' experiments for the setting of
    Table \ref{tab:devs_multi_non2}: the 2nd number in each pair is
    the number of wins of \qd (when \logloss lower), on the same set
    of 50 sequences (each $\sim$10k). For instance, in top left cell, \qu
    with \qcap=10 loses to \qd when $\minobs=10$, on all 50 sequences.
    \qd with $\minlr=0.01$ beats all others, but the results are mixed
    with $\minlr=0.001$. }
\label{tab:nwins_multi_non1}
\end{table}

\co{
  Methods: queues is with capacity 10, static is static EMA
  with fixed learning rate of 0.01, and harmonic is harmonic EMA with
  harmonic-decay with $\minlr$ of 0.01.
}


\co{ 

  \begin{table}[t] 
  \begin{tabular}{ |c|c|c|c| }     \hline
 \qd, 0.001   vs $\rightarrow$ & queues & static EMA, 0.01 & harmonic EMA, 0.01 \\ \hline
 10, 0.01, [0.01, 0.10] & 3, 50  & 50, 0  & 50, 0   \\ \hline
 50, 0.01, [0.01, 0.10] & 0, 50  & 0, 50  & 0, 50   \\ \hline
 \qd, 0.01   vs $\rightarrow$ & queues & static EMA, 0.01 & harmonic EMA, 0.01 \\ \hline
 10, 0.01, [0.01, 0.10] & 0, 50 & 32, 18  & 27, 23   \\ \hline
 50, 0.01, [0.01, 0.10] & 0, 50 & 19, 31 & 12, 38 \\ \hline
   \end{tabular}
  \vspace*{.2cm}
  \caption{Number of wins in 'paired' experiments.  Same set of 50
    sequences, each ~10k. Compare on average logloss on each of the 50 sequences,
    report number of wins (items are not constrained to be new).}
\label{tab:nwins_multi_non1}
\end{table}

}



\section{Further Material on Real-World Experiments}
\label{app:real}

Descriptions of the methods used for concept composition and
interpretation within the simplified \exped system.

\subsection{Simplified Interpretation}


Each interpretation episode begins at the character level. We do a
``bottom-up'' randomized search, where we repeatedly pick a primitive
concept and invoke and match its holonyms (bigrams that contain it) to
the input line, near where the invoking concept has matched, with the
constraint that the invoking concept should be matched (covered) too,
and repeat until there are no available holonyms or no holonyms
match. We only do exact match here, \ie no approximate matching ("ab"
matches if both "a" and "b" are present, with "b" immediately
following "a"). We keep the top two concepts along a bottom-up path,
the top-matching concepts. We repeat such bottom to top searches and
matches until all characters (positions) in the input line have been
covered, \ie have one or more matching concepts. From the set of
matching concepts, the process generates candidate interpretations and
selects a final interpretation: a candidate interpretation is a
sequence of top-matching concepts (one of the two that are end of each
path), covering all the positions (characters) in the input line, and
without any overlap, and where two new non-primitive concepts, \ie
below a frequency (observation-count) threshold, cannot be
adjacent. Often there are several candidate interpretations, and we
pick one with minimum number of concepts in it, breaking ties
randomly.

\subsection{Simplified Composition}
\label{app:simpc}

The simplified concept generation process is as follows: in every
episode, once the interpretation process is done, one obtains a final
interpretation, that is a sequences of concepts (a mix of unigrams,
bigrams, etc) covering all the characters in the line input to the
episode. Note that in the beginning of the entire learning process,
when no new concepts are generated, the final interpretation is simply
the sequences of primitive concepts corresponding to the
characters. Interpretation is easy at the character level (does not
involve search and matching).

Adjacent concepts pairs in an interpretation are processed in the
following manner to possibly generate new concepts (compositions or
holonyms): Skip the pair with some probability $p_g$ (0.9 in our
experiments). Otherwise, again skip the pair if either does not meet
the minimum observed-count requirement, where we experimented with
100, 500, and 2000. When a pair meets the count threshold, generate
the composition if it does not already exist (a hashmap look up).
Therefore, if "a" and "b" are two concepts, with "b" often occurring
after "a" in the input (lines or interpretations), with high
probability at some point the composition concept "ab" is generated.
Note that the generation probability $p_g$ as well as the threshold on
frequency affect how fast new concepts (items) are generated, or the
degree of internal non-stationarity: the higher the count threshold
the lower the generation rate of new concepts, and the more time the
system has for, in effect, evaluating and integrating a new concept.


\subsection{Discussion: Separation of Prediction from Concept Use}

Although we could use the prediction weights to influence the
generation and selection of candidate interpretations, we use the
above simple interpretation and concept generation methods to
completely divorce prediction weight learning from the trajectory of
concept generation and use, and to simulate internal
non-stationarity. We note that the bigrams and the higher n-grams that
are generated and used have good quality.  For instance, when we
define a split location, the boundary between two adjacent concepts as
{\em good} (good split) if it coincides with a blank space (or other
punctuation), and bad otherwise (splits within a word), we observe
that the ratio of bad to good splits substantially
improves over time (\eg from around 75\% to below 30\%), as larger
concepts are generated and used via the above simple processes,
suggesting that the larger n-grams learned increasingly correspond to
words, phrases, etc. And manual examination of the learned ngrams
supports this finding.

Learning and using the predictions should improve the interpretation
process, for instance approximate matching and use of predictions are
important when there is substantial noise, such as typos in the input
text, requiring inference based on larger context to fix.  However, in
the experiments of this paper, we have separated the processes to more
easily compare different prediction (weight learning) techniques.

\subsection{Further Evidence of Stationarity \vs Non-Stationarity}
\label{sec:evid}

\begin{figure}[thb]
\begin{center}
  \centering
\hspace*{-0.5cm}  \subfloat[\expedn, remain at character level.]{{
      \includegraphics[height=5cm,width=4.5cm]{
         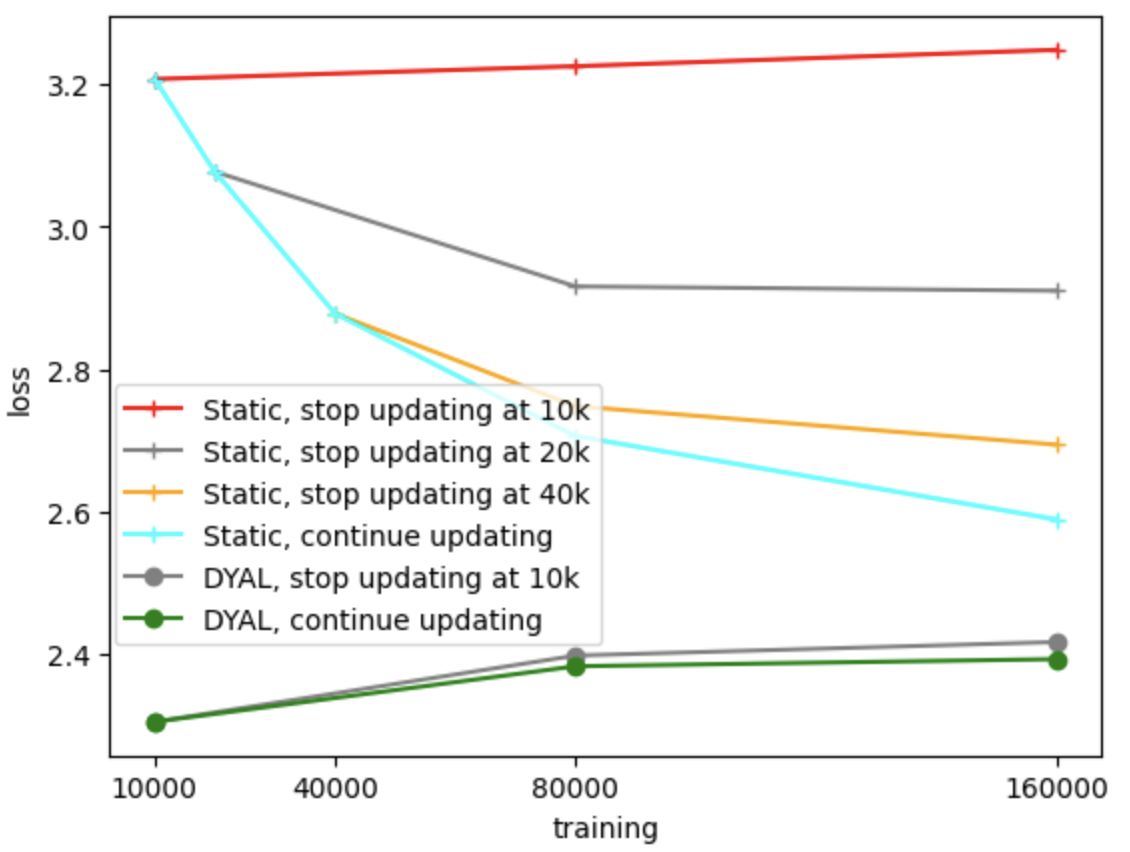}
  }}
\hspace*{.1cm}  \subfloat[52-scientists (Unix sequences).]{{
      \includegraphics[height=5cm,width=4.5cm]{
        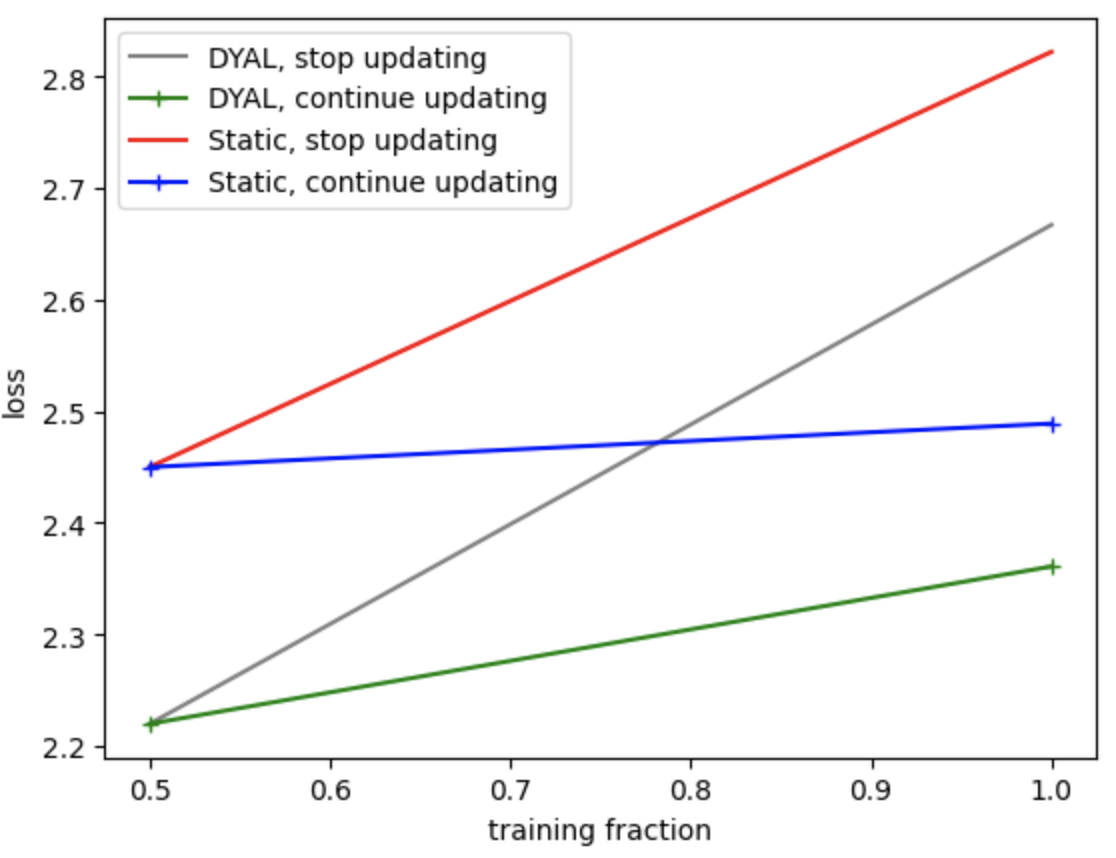}
  }}
  \subfloat[Masquerade (Unix  sequences).]{{
      \includegraphics[height=5cm,width=4.5cm]{
        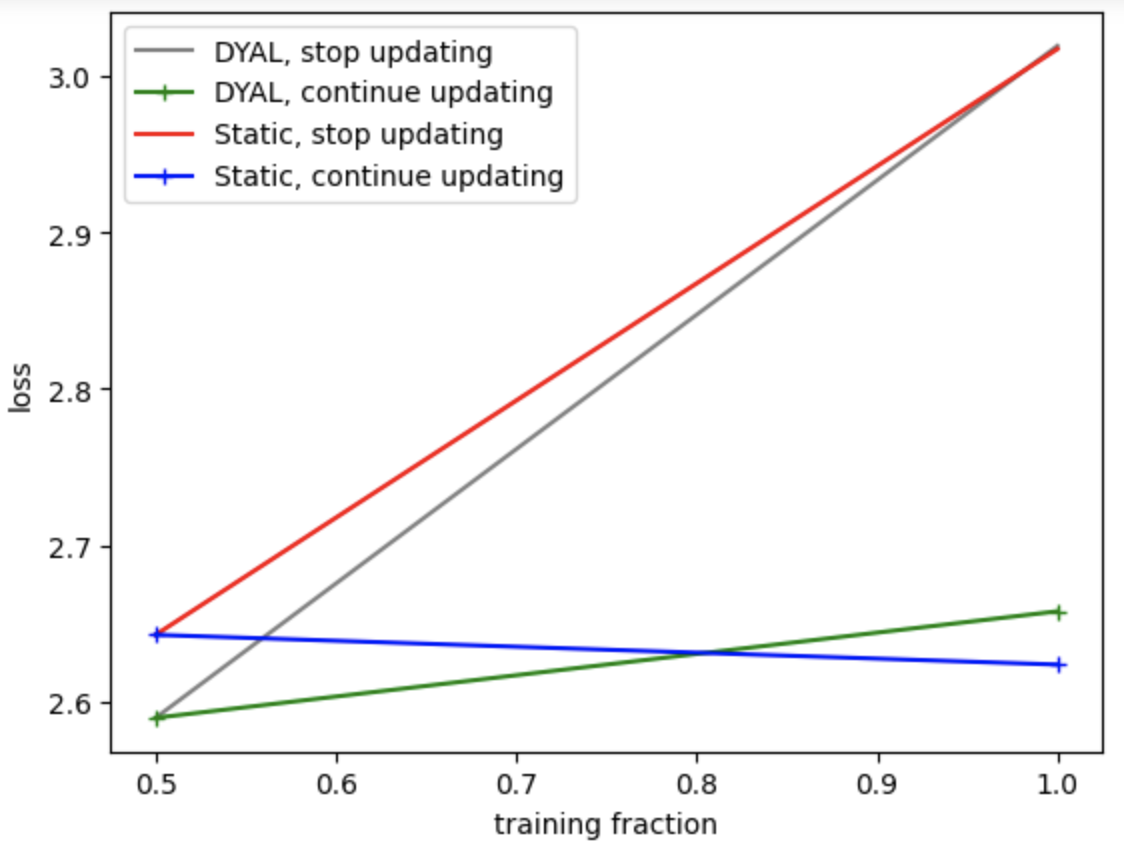}
  }}
\end{center}
\vspace{.2cm}
\caption{Performances (\loglossn) when we stop the learning at certain
  time $t$ \vs continued updating (continuing learning). (a) On the stationary
  character-level prediction. (b) and (c) on Unix
  sequences: Continuing the updates \vs stopping the updates at half
  the sequence length, and measuring performance at that point and
  once the sequence finishes (averaged \logloss over all the sequences
  in each data source). When we stop the updates, the performance
  takes a substantial hit compared to continued updating, whether using \qd or static EMA, on Unix
  sources (indicating non-stationarity), but not on the character-level \exped (indicating stationarity).  }
\label{fig:nonstat_evid}
\end{figure}

\fig \ref{fig:nonstat_evid} shows \logloss performances at two or more
time points $t$, comparing performance of continuing to update
(continued learning), \vs stopping of the updates (freezing the
current weights for prediction on the rest of the sequence).

\fig \ref{fig:nonstat_evid}(a) is on \exped at character level, the
same set up of \sec \ref{sec:expd_chars}, except that here we may stop
the learning at an early point. For static (EMA) and \qdn, using
$\lr=0.001$, we stop the updates at time 10k and we compare this to
the regime of continuing the updates. We report the loss at 10k and at
80k and 160k.  For static, we also included plots when stopping the
updates at 20k and 40k.  We observe the performance when stopping the
updates at say 10k, at 80k or 160k is close (often better) than what
it was at 10k, indicating the input distribution has not changed. The
plots should be contrasted against \fig \ref{fig:nonstat_evid}(b-c)
described next. We suspect the reason the performance of static at
times improves, even though we have stopped the updates, is that due
to its low fixed learning rate, and that we report the average over
all of the times till that point, the effect of latest learning is not
fully reflected, and as we allow more time (report \logloss at higher
$t$), the result is more reflective of more recent and better
performance. For \qdn, its performance converges quickly (by $t$=10k)
on this stationary data and the performance only slightly degrades
with higher $t$ as some new items are observed (similar to \fig
\ref{fig:loss_vs_time_expedition}).

For \fig \ref{fig:nonstat_evid}(b) and \ref{fig:nonstat_evid}(c), we
are comparing continued learning \vs stopping at half the sequence
length, on 52-scientists and Masquerade, using static EMA and \qd with
$\lr$ of $0.05$ (similar setting and results to \sec \ref{sec:unix3}). For
Masquerade, all sequences are 5k long, thus we report (average)
performance at $t=$2.5k (the average log-loss performance up to time
$t=$2.5k), and then either stop or continue the learning till the end,
\ie $t=$5k, and report performance at 5k as well for both
situations. As in previous results, the reported losses (at different
time snapshots) are the result of averaging over all the 50 sequences
(users) of Masquerade.  For 52-scientists, sequences have different
lengths, so the times at which we report performances (use for
averaging) are sequence dependent. For instance, for a sequence of
total length 200, we use the performance at $t=100$, then also use the
performances at $t=200$ for the two cases of continued updates and
stopping the updates at $t=100$ (and average each over the 52
sequences). We observe that in all cases, whether for \qd or static
EMA, stopping the learning leads to substantial increase in loss
compared to continuing the updates (learning), and worse than the
average performance at half the length, the latter observation
implying presence of (external) non-stationarity. We have tried
several learning rates (not just the best performing rate of 0.05),
and obtained similar results.

Another way to see evidence of non-stationarity is to look at the
max-rate, the maximum of the learning-rate, in the $\lrmap$ of \qdn,
as a function of time.  Whenever the max-rate goes up, it is evidence
that an item is changing \prn, for instance a new item is being
added. Furthermore, here for each sequence, we concatenate it with
itself $k$ times, $k=10$ or $k=3$ in \fig \ref{fig:concats}. For
example, if the original sequence has the three items [A, C] (\ie
$\ott{o}{2}=$\itm{C}), then the 3x replication (self-concatenation) is
[A,C,A,C,A,C]. If the start of a sequence has a distribution
substantially different from the ending then, when concatenated, we
should repeatedly see max-rate jump up. We looked at 10s of such
plots. The max-rate repeatedly go up in plots of all sequences from
the 52-scientists (whether with default parameters or with $\lrmin$ of
0.05), and all except 1 of the 50 Masquerade sequences, and similar
pattern on all except the shortest handful of sequences from
104-\exped. These short ones are around a 100 time points long, and in
the first pass, the predictor learns the distribution and does not
need to change it substantially. \fig \ref{fig:concats} shows max-rate
on 6 self-concatenated sequences, one example where the max-rate does
not change much after the first pass, the other 5 where we see
max-rate jumping up repeatedly. These sequences do not necessarily
have many (near 100) unique salient items (occurring a few times),
implying that in some, there exist salient items with high proportion
that are replaced with other such over time. For instance, there are
13 sequences in the 52-scientists with number of items seen 3 or more
times below 30. Replicating the longer sequences leads to too many ups
in the plots (spikes) to be discernable, such as the right example
plot from 52-scientists.\footnote{Note also that the number of times
max-rate goes up is a subset of the times that an item changes
substantially (when its \pr and $\lr$ are set according to the queue),
\ie an item's $\lr$ may go up but that event may not change the
maximum.}

As one concrete example, in one sequence from the 52-scientists, with
length 1490, the number of increases in max-rate was 10 on the
original sequence, and then 7 increases in all subsequent
replications. The times of increase, the command (item) observed, and
its rate (changing the maximum rate, and the EMA \pr in the first pass
were: [(15, 'quit', '0.11', '0.12'), (28, 'more', '0.20', '0.25'),
  (43, '/bin/pdplk', '0.14', '0.17'), (59, '/bin/pdp60', '0.07',
  '0.08'), (105, 'pdp60', '0.12', '0.14'), (170, 'mv', '0.20',
  '0.25'), (230, 'rm', '0.07', '0.07'), (607, 'users', '0.33',
  '0.40'), (1408, 'rlogin', '0.06', '0.12'), (1409, 'quit', '0.08',
  '0.17')], thus at time t=50, the item 'quit' gets the learning rate
of 0.11 (increasing the max-rate) and obtains \pr of 0.12. The maximum rate
before the increase is often 0.05 (the minimum rate in these
experiments).  In the 2nd pass, and all subsequent passes, the 7
increase points are: [(1492, 'mail', '0.33', '0.33'), (1519, 'more',
  '0.17', '0.37'), (1601, 'pdp60', '0.07', '0.15'), (1663, 'mv',
  '0.12', '0.29'), (2097, 'users', '0.33', '0.40'), (2898, 'rlogin',
  '0.06', '0.12'), (2899, 'quit', '0.08', '0.17')].








\begin{figure}[htb]
 \vspace*{-1cm}
\begin{center}
\hspace*{-0.75cm}  \subfloat[\expedn, sequence 5, original length of 127.]{{
      \includegraphics[height=5cm,width=7cm]{
        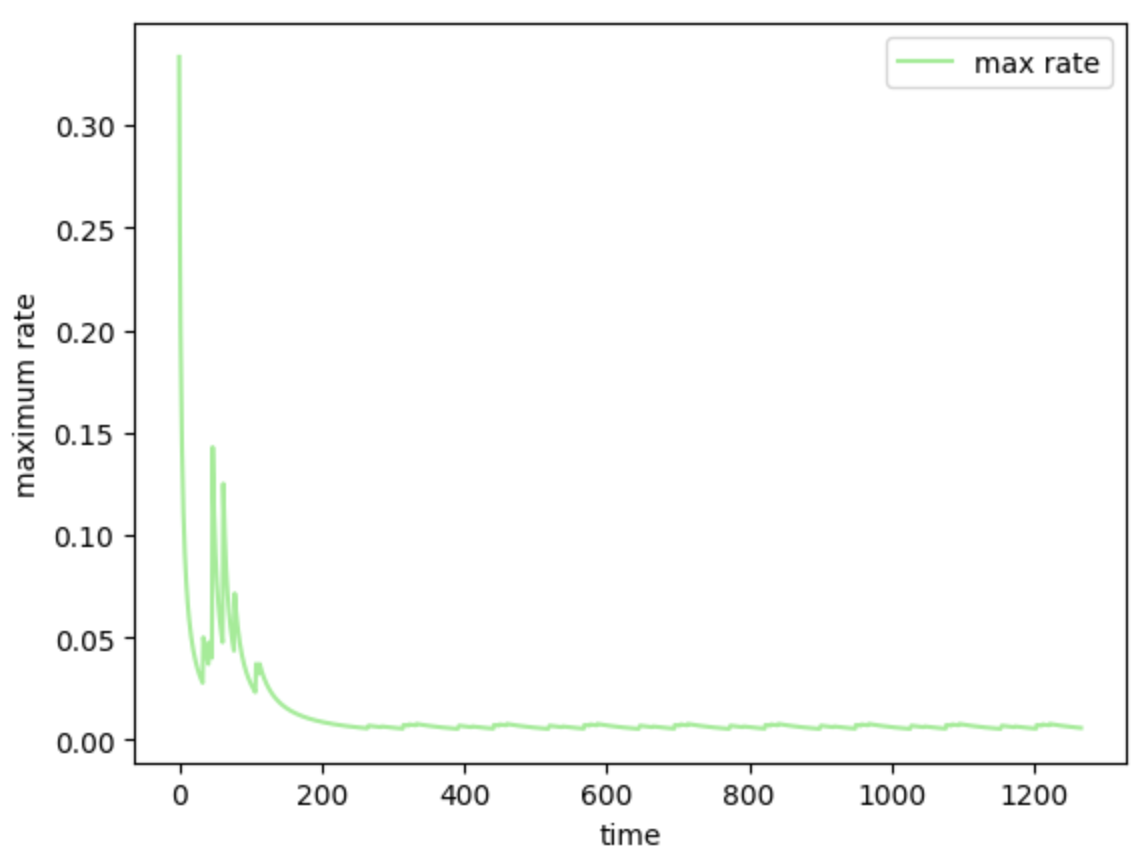}
  }}
  \subfloat[\expedn, sequence 36, original length 684. ]{{
      \includegraphics[height=5cm,width=7cm]{
        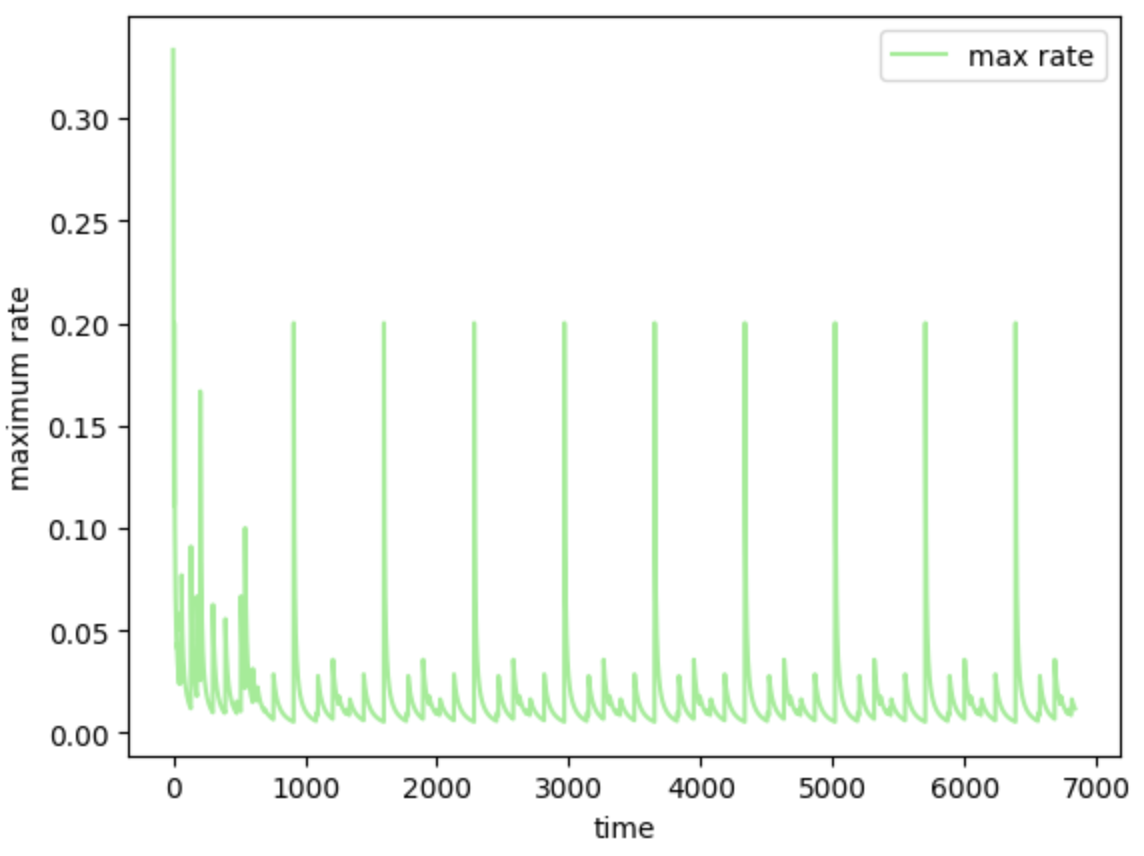}
  }} \\
\hspace*{-0.75cm}  \subfloat[52-scientists, shortest sequence.]{{
      \includegraphics[height=5cm,width=7cm]{
        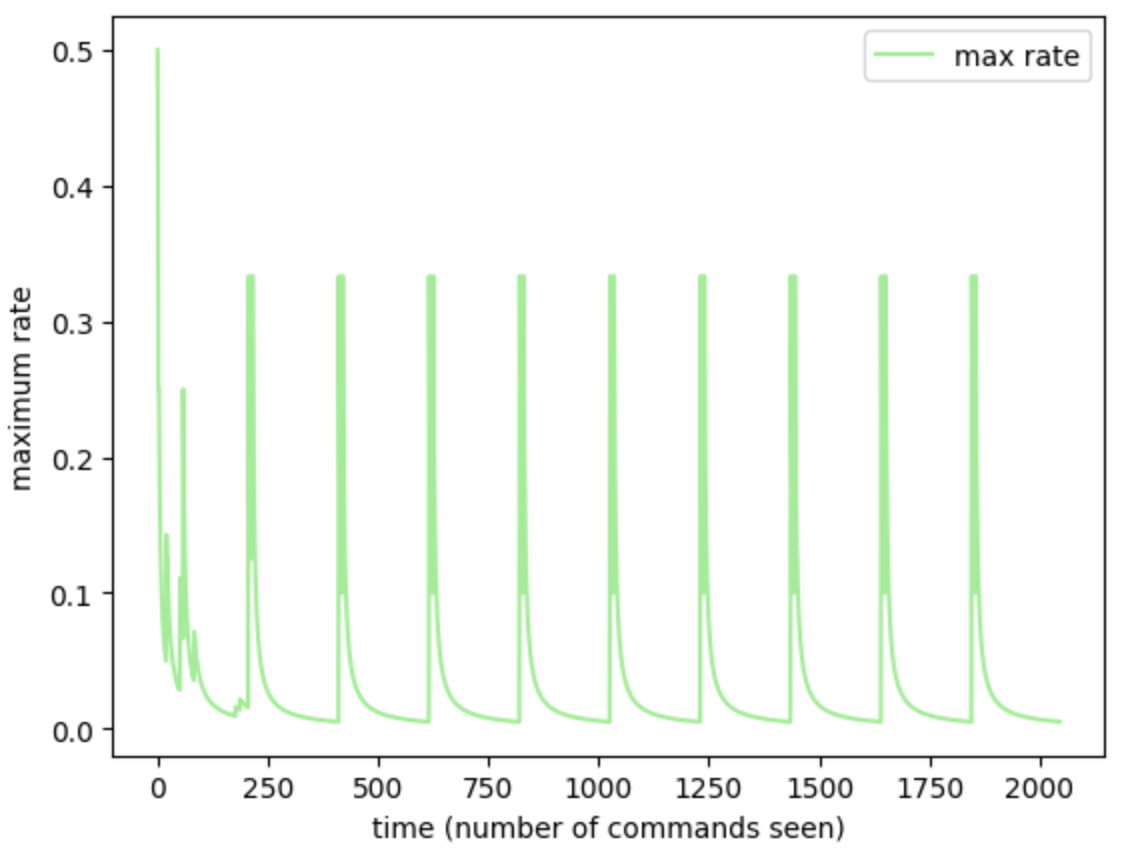}
  }}
  \subfloat[52-scientists, sequence 20. ]{{
      \includegraphics[height=5cm,width=7cm]{
        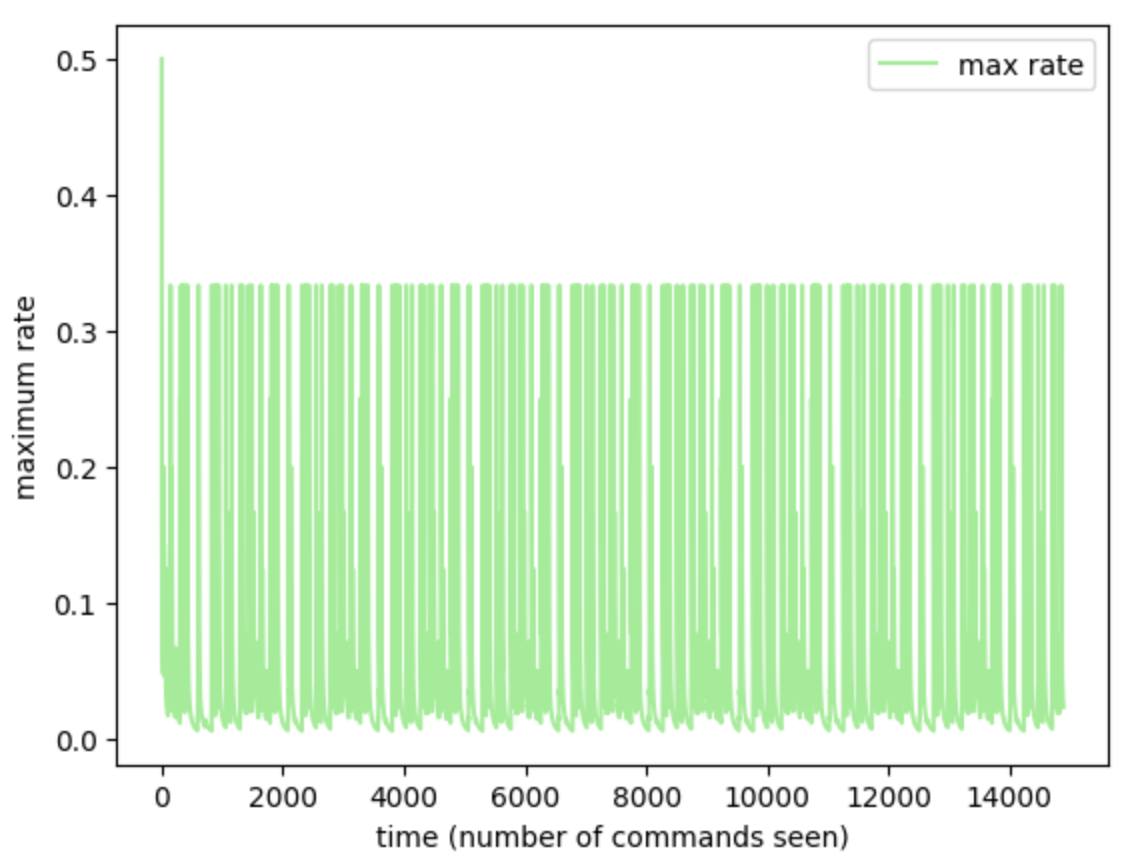}
  }} \\
\hspace*{-0.75cm}  \subfloat[Masquerade, 3x.]{{
      \includegraphics[height=5cm,width=7cm]{
        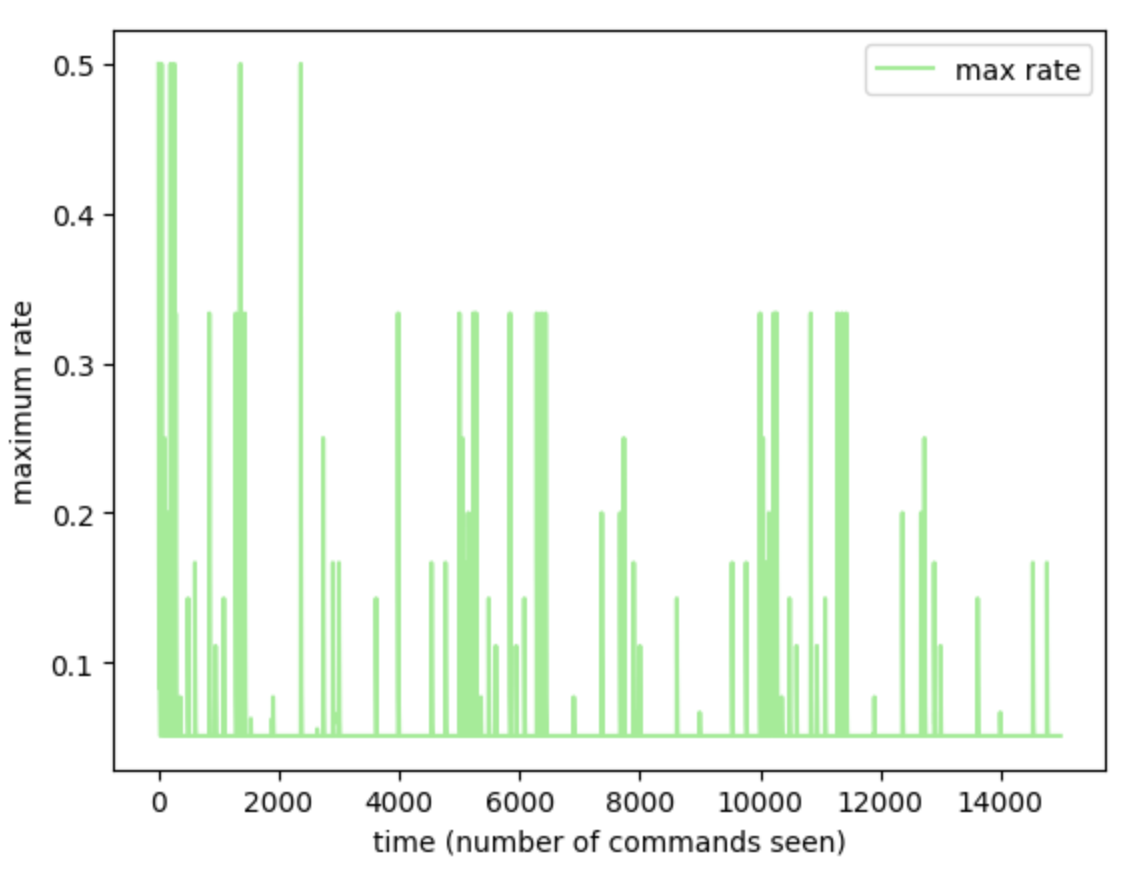}
  }}
  \subfloat[Masquerade, 3x. ]{{
      \includegraphics[height=5cm,width=7cm]{
        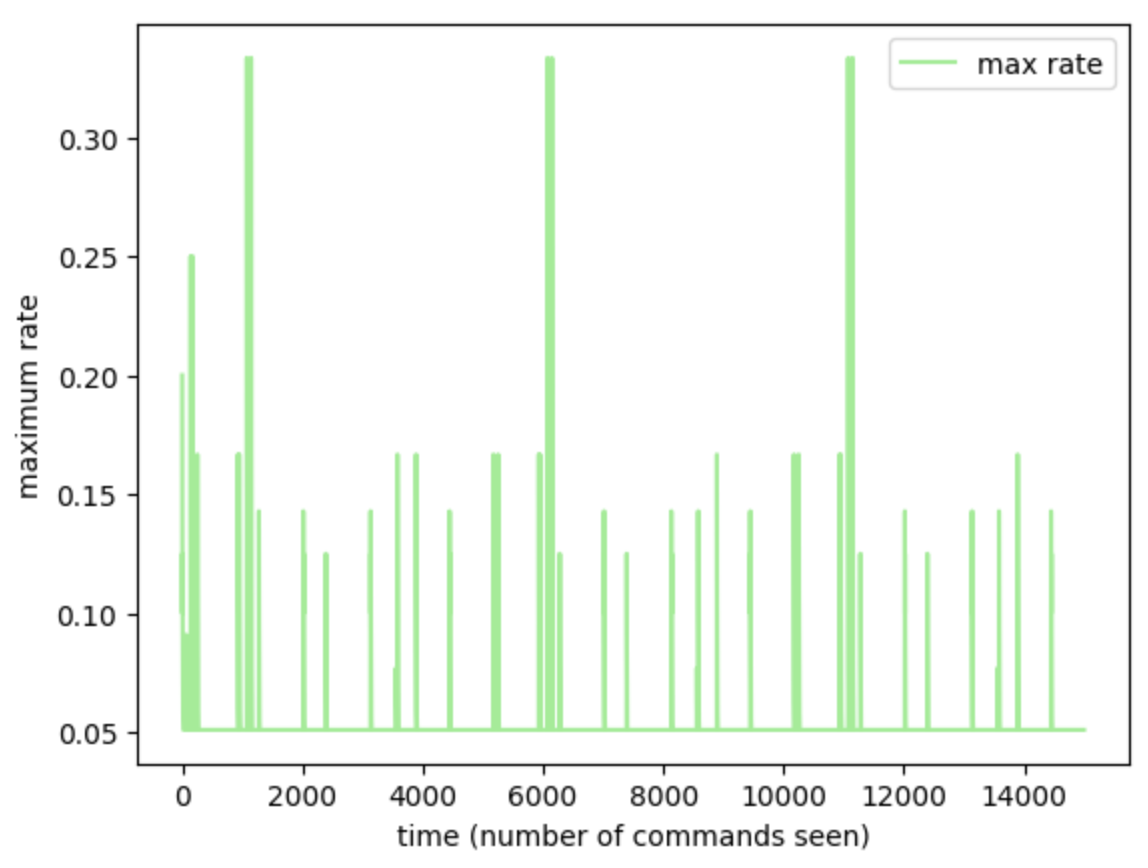}
  }}
\end{center}
\vspace{.0cm}
\caption{The maximum learning rate within \qd (max over the entries in
  $\lrmap$) as a function of time, on two sequences each from 
  104-\expedn, 52-scientists, and Masquerade, where each sequence is
  concatenated with itself 10 times, 3x for Masquerade, thus the length
  of the original sequence is 127 in top left, but after 10x
  concatenation, it is 1270 ('sequence 5' means the 5th shortest
  sequence). We observe that the max-rate repeatedly jumps up,
  exhibiting pulsing or spiking patterns, indicating change, except
  for a handful of very short sequences of the 104-\exped sequences,
  and one sequence with only a few unique items from Masquerade.}
\label{fig:concats}
\end{figure}


\end{document}